\documentclass[10pt]{article}
\usepackage[preprint]{tmlr}

\usepackage{amsmath,amsfonts,bm}

\def\eqref#1{equation~\ref{#1}}

\def\1{\bm{1}}

\def\va{{\bm{a}}}
\def\vb{{\bm{b}}}

\def\vp{{\bm{p}}}
\def\vq{{\bm{q}}}

\def\mC{{\bm{C}}}

\def\mK{{\bm{K}}}

\def\mT{{\bm{T}}}

\def\mX{{\bm{X}}}

\DeclareMathAlphabet{\mathsfit}{\encodingdefault}{\sfdefault}{m}{sl}
\SetMathAlphabet{\mathsfit}{bold}{\encodingdefault}{\sfdefault}{bx}{n}

\newcommand{\E}{\mathbb{E}}
\newcommand{\Ls}{\mathcal{L}}
\newcommand{\R}{\mathbb{R}}

\newcommand{\KL}{D_{\mathrm{KL}}}

\newcommand{\Ent}{\mathrm{Ent}}

\usepackage{hyperref}
\hypersetup{hidelinks,
  pdftitle={Training Non-Differentiable Networks via Optimal Transport},
  pdfauthor={An T. Le}}
\usepackage{xurl}
\usepackage{microtype}
\usepackage{graphicx}
\usepackage[section]{placeins}
\usepackage{booktabs}
\usepackage{siunitx}
\usepackage{threeparttable}
\usepackage{amsmath,amssymb,amsthm}
\usepackage{subcaption}
\usepackage{xcolor}
\usepackage{multirow}
\usepackage{enumitem}
\usepackage[ruled,vlined]{algorithm2e}

\newtheorem{theorem}{Theorem}[section]
\newtheorem{proposition}[theorem]{Proposition}
\newtheorem{corollary}[theorem]{Corollary}
\newtheorem{lemma}[theorem]{Lemma}
\theoremstyle{definition}
\newtheorem{definition}[theorem]{Definition}
\theoremstyle{remark}
\newtheorem{remark}[theorem]{Remark}
\theoremstyle{definition}
\newtheorem{assumption}{Assumption}

\graphicspath{{figures/}}

\newcommand{\sysname}{PolyStep}

\newcommand{\snnBestAcc}{$93.0 \pm 0.2$\%}
\newcommand{\snnBestBaseline}{$79.6 \pm 5.2$\%}
\newcommand{\snnAdamAcc}{$86.3 \pm 10.6$\%}
\newcommand{\mnistBestAcc}{$96.8 \pm 0.1$\%}

\newcommand{\mnistStepRatio}{$100\times$}
\newcommand{\timeseriesBestMse}{$0.253 \pm 0.023$}
\newcommand{\maxsatHundredKAcc}{$98.1 \pm 0.0$\%}

\title{Training Non-Differentiable Networks via Optimal Transport}
\def\month{09}
\def\year{2026}
\def\openreview{\url{https://openreview.net/forum?id=8mlcqTTMuU}}

\author{\name An T.~Le \email an@robot-learning.de \\
      \addr Center for AI Research, VinUniversity, Vietnam \\
      Intelligent Autonomous Systems, TU Darmstadt, Germany
}

\begin{document}
\maketitle

\begin{abstract}
Hard thresholds, quantization, and discrete routing can produce training losses
with flat regions and jumps, where ordinary gradients vanish or are undefined.
We introduce \sysname{}, a forward-only optimizer that evaluates rotated
polytope probes and moves parameter blocks along weighted averages of
the probe directions. We derive the weights from one-sided entropic transport
and use its uncoupled softmax solution in our primary experiments.
Our analysis explains when variation among probe costs produces motion
and when that motion decreases the loss. On a regular simplex, nonconstant
costs always give a nonzero direction. For monotone ridge losses, the
softmax update cannot increase the loss at any positive temperature;
a perturbation bound gives sufficient conditions for descent near curved jumps.
For bounded measurable losses, we randomize the probe radii and identify
an exact smoothing whose gradient equals the expected linear cost-weighted
direction up to scale. This identity yields a stationarity bound for an
idealized fixed-temperature variant: under regularity and sampling
assumptions stronger than those met by our trained configurations, the bound
has an $O(T^{-1/2})$ term and a persistent bias floor.
We evaluate the practical method on networks with hard operations, discrete
optimization, and policy search. On MNIST with hard-threshold spiking neurons,
\sysname{} reaches \snnBestAcc{}, compared with \snnBestBaseline{} for the
best tuned gradient-free baseline at matched evaluations. These gains come
with a query cost proportional to the search dimension per fresh step,
which limits the number of updates available at a fixed budget.
Code: \url{https://github.com/anindex/polystep}.
\end{abstract}

\section{Introduction}\label{sec:introduction}

Hard spike thresholds, quantized weights, and discrete routing allow models
to compute with discrete events or low-precision arithmetic. They can also
make the training loss locally constant or discontinuous as parameters
change. Ordinary derivatives then provide little guidance: they vanish on
plateaus and do not exist at jumps. Surrogate gradients
\citep{neftci2019surrogate} and the straight-through estimator (STE)
\citep{bengio2013ste} address this difficulty by approximating the backward
pass while retaining hard operations in the forward pass. When an evaluator
supplies only scalar losses, we must instead choose updates from the values
of the loss at perturbed parameters.

Evolution strategies and zeroth-order methods provide ways to use this
information \citep{hansen2016cma,salimans2017evolution,nesterov2017random}.
Random perturbations can estimate derivatives of a local average of the
loss even when the loss itself is discontinuous. Applications include
deep-network training and large language model (LLM) fine-tuning
\citep{chen2024deepzero,malladi2023mezo,qiu2026esatscale}, integer recurrent
models \citep{sarkar2025eggroll}, and forward-only quantization-aware
training \citep{tan2025zeroqat}. We study a question within this setting:
\emph{when does variation among a structured set of forward probes produce a
useful update?}

The probe scale is one part of the answer. In our diagnostic with 8-bit
integer (INT8) quantization, a finite-difference radius of $\mu=10^{-5}$
leaves $99.5\%$ of probe pairs at identical loss, compared with $1.6\%$
on a smooth control. Increasing the radius to $10^{-3}$ reduces the flat-pair
fraction to $0.04\%$ (Appendix~\ref{app:flat-pair}). Crossing a threshold
can therefore reveal variation that smaller probes miss. But observing
unequal costs does not ensure that their weighted directions avoid
cancellation, or that the resulting update improves the loss.

We introduce \sysname{}, extending the transport-weighted polytope search
of Sinkhorn Step \citep{le2023sinkhorn} to the parameters of a single model.
We divide the search coordinates into small blocks, rotate a polytope in
each block, and evaluate the loss along its vertex directions. Each block
then moves along a weighted average of those directions. Fixed per-layer
subspaces reduce the search dimension, and batched evaluation processes
candidates in parallel. The weights solve a one-sided entropic transport
problem from blocks to vertex labels. Its uncoupled endpoint gives the
closed-form softmax rule used in all our primary experiments; a column
penalty connects this rule to balanced transport.

Our analysis follows the information from probe costs to motion and then
to descent. We make three contributions:
\begin{itemize}[nosep,leftmargin=1.2em]
  \item We characterize how geometry and cost ordering determine the
  nonlinear update. On a regular simplex, nonconstant costs always produce
  a nonzero direction, with an exact formula for its length. For losses
  that depend monotonically on one linear coordinate, every centered
  frame gives a nonascending step at any positive temperature. A stability
  bound extends this argument to isolated curved jumps under explicit
  conditions on cost errors and step length
  (Sections~\ref{sec:nonlinear-jump}--\ref{sec:nonlinear-descent}).
  \item We identify the smoothing behind the mean linear cost-weighted
  direction for bounded measurable losses. It is a size-biased mixture of
  uniform balls, determined by the probe radii. Softmax weighting and the
  interaction between blocks introduce separate errors relative to a joint
  smoothed gradient. Controlling these errors gives an $O(T^{-1/2})$
  stationarity bound with a bias floor for an idealized fixed-temperature
  variant (Sections~\ref{sec:thm-piecewise-smooth}--\ref{sec:convergence}).
  \item We evaluate PolyStep on networks with hard operations, discrete
  optimization, and policy search. At matched evaluations, hard leaky
  integrate-and-fire (LIF) training reaches \snnBestAcc{}, compared with
  \snnBestBaseline{} for the best tuned gradient-free baseline. OpenAI-ES
  leads on argmax attention and 5,000-variable MAX-SAT. A ten-seed
  comparison with the same representation and candidate distribution
  finds no statistically established advantage for softmax over linear,
  rank, or greedy weights (Section~\ref{sec:experiments}).
\end{itemize}

The method trades evaluations for structured exploration: a fresh step
uses a number of candidates proportional to the search dimension.
This cost limits scale, and the $4.2$M-parameter from-scratch experiment
remains near chance. The theory describes mechanisms and a conditional
stationarity guarantee; its regularity and sampling assumptions do not
all hold for the practical configurations. We develop that distinction
with the theorem and discuss its empirical implications in
Section~\ref{sec:limitations-disc}.
\section{Problem Setting}\label{sec:problem}

We minimize a scalar loss $\Ls(\theta)$ over parameters $\theta\in\R^d$
using forward evaluations. The evaluator may contain hard thresholds,
quantizers, or discrete routing, and supplies no derivatives. We search
either the full parameter space or a fixed linear subspace. In stochastic
settings, an evaluation returns a minibatch loss or an empirical rollout
cost; the assumptions relating these estimates to $\Ls$ are stated with
the results that require them.

A change in a hard operation need not change the loss. We call a boundary
an \emph{attained jump} when the limiting losses on its two sides differ.
Near an isolated smooth boundary, these limits are the \emph{one-sided
traces}, and their difference is the jump size. Our local descent results
use this geometry to control the effect of crossing a boundary. The general
smoothing identity requires only a bounded measurable loss and does not
require a description of its discontinuities.
For the geometric results, we write $\mathcal D$ for a closed set outside
which the loss is continuously differentiable; this set may include continuous kinks as well as
jumps. Definition~\ref{def:pwc-loss} states the regularity and isolation
conditions used in those results.

For network losses, we work on a compact neighborhood containing the
iterates, probes, and smoothing supports. If the loss is bounded on this
neighborhood, we may extend it to a bounded measurable function outside
without changing the local convolutions. This lets us analyze losses such
as cross-entropy without assuming global boundedness of the original
network loss. Appendix~\ref{app:discontinuity-sets} describes the threshold
geometry of our architectures, and Appendix~\ref{app:measured-jump} measures
which crossings change the loss.

Local Lipschitz continuity means that, within a neighborhood, the absolute
difference between two loss values is bounded by a constant times the
distance between their parameters. At an attained jump, arbitrarily close
points can have losses separated by a fixed positive amount. Local
Lipschitz continuity therefore fails, so the locally
Lipschitz Clarke subdifferential and conservative-gradient framework
\citep{boltepauwels2020conservative} do not apply. We use the gradient of a
fixed local average as our general stationarity target because we can
relate it to the mean update. A small gradient means that this average
changes slowly in the search directions; it does not certify a local
minimum of the original loss. Section~\ref{sec:bg-gfo} discusses other
optimality notions for discontinuous optimization.

We use Euclidean vector norms and induced operator norms for linear and
multilinear maps. The $L^1$ norm of a tensor-valued function integrates its
pointwise norm, and $\|f\|_\infty$ denotes the supremum of that norm over
the stated domain. For a finite vector, $\|z\|_\infty=\max_j|z_j|$.
We write $\mathbf{1}_A$ for an indicator, $C^{k,1}$ for
functions with Lipschitz $k$th derivatives, and $W^{k,1}$ for functions whose
weak derivatives through order $k$, including the function itself, are integrable.

Table~\ref{tab:notation} collects the notation for the method described in
Section~\ref{sec:algorithm}. We define the smoothing quantities and proof
constants alongside the results that use them.

\begin{table}[!ht]
\centering
\caption{Core notation. The indices $i$, $v$, $k$, and $t$ refer to blocks,
vertices, probe radii, and iterations, respectively. A superscript $+$
denotes the output of one update.}
\label{tab:notation}
\small
\begin{tabular}{@{}p{0.23\linewidth}p{0.73\linewidth}@{}}
\toprule
Symbol & Meaning \\
\midrule
$\theta\in\R^d$, $\Ls(\theta)$ & Model parameters and scalar loss. \\
$d_{\mathrm{sub}}$, $d_p$, $P$ & Search dimension, coordinates per block, and number of blocks,
  $P=\lceil d_{\mathrm{sub}}/d_p\rceil$. \\
$x_i$, $x$, $\mX$ & Block vector $x_i\in\R^{d_p}$, stacked vector
  $x\in\R^{Pd_p}$, and matrix with rows $x_i^\top$. \\
$Q$, $z$ & Padding-removal map and unpadded coordinates $z=Qx\in\R^{d_{\mathrm{sub}}}$. \\
$\Pi$, $\mathcal S$, $\theta_{\mathrm{ref}}$ & Fixed linear search map, its range $\mathcal S$,
  and reference parameters; $\theta(x)=\theta_{\mathrm{ref}}+\Pi Qx$. \\
$\varphi$, $\widehat\varphi_t$ & Loss in search coordinates, $\varphi(x)=\Ls(\theta(x))$,
  and its stochastic estimate at iteration $t$. \\
$v_v$, $R_i$, $V$ & Unit template vertex, rotation for block $i$, and number of vertices. \\
$K$, $\xi_k$ & Probes per vertex and radius fractions $\xi_k=k/(K+1)$. \\
$\varepsilon$, $r_p$, $r_s$ & Weight temperature, probe multiplier, and step multiplier. \\
$\rho_k$, $\ell_i$ & Probe distance and block step scale. Without jitter,
  $\rho_k=r_p\varepsilon\xi_k$ and $\ell_i=r_s\varepsilon$. \\
$\eta_t$, $\eta_i$ & Shared probe jitter and independent step jitter for block $i$. \\
$C_{iv}$, $p_{iv}$ & Averaged probe cost and normalized weight for vertex $v$ in block $i$. \\
$D_i$ & Weighted direction $D_i=\sum_vp_{iv}R_iv_v$;
  the block update is $x_i^+=x_i+\ell_iD_i$. \\
$\mT^*$, $\va$, $\vb$ & Optimal transport plan and prescribed source and target marginals;
  $p_{iv}=T^*_{iv}/a_i$. \\
$\lambda$ & Column-marginal penalty strength; $\lambda=0$ gives independent softmax rows. \\
\bottomrule
\end{tabular}
\end{table}
\section{Algorithm}\label{sec:algorithm}

PolyStep evaluates structured perturbations of one model and uses their
losses to update its search coordinates. Each probe changes one block while
holding the others fixed. After evaluating all probes, we update the blocks
simultaneously. We first describe this core softmax update, then derive its
transport interpretation and explain how subspaces reduce its cost.

\subsection{Search Coordinates and Update}\label{sec:formulation}

Let $d_{\mathrm{sub}}$ be the number of search coordinates and $d_p$ the
block dimension. We use $P=\lceil d_{\mathrm{sub}}/d_p\rceil$ blocks,
called \emph{particles}, which together represent a single model.
Stacking their coordinates gives $x=(x_1,\ldots,x_P)\in\R^{Pd_p}$;
equivalently, we store $x_i^\top$ as row $i$ of the particle matrix
$\mX\in\R^{P\times d_p}$. We reconstruct parameters through
\[
 \theta(x)=\theta_{\mathrm{ref}}+\Pi Qx,
 \qquad \varphi(x)=\Ls(\theta(x)).
\]
Here $Q:\R^{Pd_p}\to\R^{d_{\mathrm{sub}}}$ discards any padding in the last block and
$\Pi:\R^{d_{\mathrm{sub}}}\to\R^d$ is a fixed linear search map.
We write $z=Qx$ for the unpadded coordinates and
$\mathcal S=\operatorname{range}(\Pi)$ for the search subspace.
In subspace mode, $\theta_{\mathrm{ref}}$ is the initial model and $x=0$
initially. In full-space mode, $d_{\mathrm{sub}}=d$,
$\theta_{\mathrm{ref}}=0$, and $\Pi=I$.
We write $\iota_i$ for the map that inserts a vector into block $i$ and
puts zeros in all other blocks. We suppress the iteration index $t$ when
describing one step; a superscript $+$ denotes its output.
Padding coordinates have no effect on the reconstructed parameters.
The stationarity theorem assumes complete blocks, for which $Q=I$.

Choose unit vertices $v_1,\ldots,v_V\in\R^{d_p}$ with
$\sum_v v_v=0$. In each block, draw a rotation $R_i$ and evaluate $K$
positive radii $\rho_k$ along each rotated vertex direction:
\begin{equation}\label{eq:cost-matrix}
 C_{iv}=\frac1K\sum_{k=1}^K
       \varphi(x+\rho_k\iota_i R_i v_v).
\end{equation}
The cost matrix is $\mC=(C_{iv})\in\R^{P\times V}$, with row
$C_{i:}=(C_{i1},\ldots,C_{iV})$. For each block, we give more weight to
directions with lower probe costs and move along their weighted average:
\begin{equation}\label{eq:compact-update}
 p_{iv}=\frac{\exp(-C_{iv}/\varepsilon)}{\sum_w\exp(-C_{iw}/\varepsilon)},
 \qquad D_i=\sum_vp_{iv}R_iv_v,
 \qquad x_i^+=x_i+\ell_iD_i.
\end{equation}
The temperature $\varepsilon>0$ controls weight concentration, and
$\ell_i\geq0$ sets the displacement scale. Since $D_i$ is a convex
combination of unit directions, $\|D_i\|\leq1$ and the actual step length
is $\ell_i\|D_i\|$. We do not rescale $D_i$ to unit length. Equal costs give uniform
weights and zero displacement because the vertices are centered.

Following \citet{le2023sinkhorn}, we parameterize the radii by
$\rho_k=r_p\varepsilon\xi_k$ and $\ell_i=r_s\varepsilon$, where
$\xi_k=k/(K+1)$. Thus $r_p$ and $r_s$ independently control probing and
movement at a fixed temperature. They are multipliers, while $\rho_k$ and
$\ell_i$ are lengths in search coordinates. At fixed costs, decreasing
$\varepsilon$ concentrates the weights. At fixed multipliers, it also
shrinks the probe and step lengths, so the costs themselves change.
The implementation also accepts schedules that specify the lengths directly,
without another temperature factor (Remark~\ref{rem:radius-parameterization}).
Radius jitter, when enabled, multiplies
all probe radii by one shared draw $1+\eta_t$ and each block's step scale
by an independent draw $1+\eta_i$. These draws are independent of the
rotations and lie in $[1-\eta_{\max},1+\eta_{\max}]$, with
$\eta_{\max}<1$. The core algorithm uses independent Haar-uniform
rotations, the rotation-invariant law on $\mathrm{SO}(d_p)$.
Appendix~\ref{app:method-implementation} gives the samplers and practical
schedule conventions.

For stochastic evaluations, we replace $\varphi$ by a scalar loss estimate
$\widehat\varphi_t(x)=\Ls_t(\theta(x))$. Supervised candidates share a
minibatch; RL uses the rollout protocol in Appendix~\ref{app:rl-policy-search}.
Sharing a minibatch compares candidates on the same data. The stationarity
analysis additionally requires conditional unbiasedness relative to the
sampling history, which shared data alone does not ensure.

\subsection{Core Optimization Step}\label{sec:core-step}

Algorithm~\ref{alg:polystep-step} summarizes one fresh softmax step, and
Figure~\ref{fig:method-overview} illustrates its geometry. Momentum settings
are given in Appendix~\ref{app:hyperparameters}; direction reuse and biased
rotations are described in Appendix~\ref{app:turbo-mode}. We use a regular
simplex with $V=d_p+1$ vertices in the primary architecture experiments.
The orthoplex, with vertices $\{\pm e_j\}_{j=1}^{d_p}$ and $V=2d_p$,
is used in theory mode and the controlled weighting comparison. Both are
centered unit frames and have the same coefficient $c$ in the mean linear direction
(Theorem~\ref{thm:exact-step}); only the simplex rules out cancellation
for every nonconstant softmax cost row
(Proposition~\ref{prop:simplex-detectability}).

\begin{algorithm}[t]
\caption{Core \sysname{} softmax step ($\lambda=0$). We hold $x$ fixed while evaluating every probe and store the updated blocks separately in $x^+$.}
\label{alg:polystep-step}
\KwIn{Coordinates $x\in\R^{Pd_p}$; reconstruction $\theta(x)=\theta_{\mathrm{ref}}+\Pi Qx$;
scalar loss evaluator $\widehat\varphi_t$; centered unit vertices $\{v_v\}_{v=1}^V$;
temperature $\varepsilon>0$; radius multipliers $r_p>0$, $r_s\geq0$; probe count $K\geq1$;
jitter law $q$ supported in $[-\eta_{\max},\eta_{\max}]$, $0\leq\eta_{\max}<1$;
separate on/off choices for probe and step jitter}
\KwOut{Updated coordinates $x^+$, parameters $\theta^+$, and mean probe cost $\bar C$}
Draw $\eta_t\sim q$ if probe jitter is enabled; otherwise set $\eta_t=0$\;
$\rho_k\leftarrow r_p(1+\eta_t)\varepsilon k/(K+1)$ for $k=1,\ldots,K$\;
\For{each block $i=1,\ldots,P$}{
  Draw $R_i\sim\mathrm{Uniform}(\mathrm{SO}(d_p))$ independently of the other draws\;
  $C_{iv}\leftarrow K^{-1}\sum_{k=1}^K\widehat\varphi_t(x+\rho_k\iota_iR_iv_v)$ for $v=1,\ldots,V$\;
  $p_{i:}\leftarrow\mathrm{softmax}(-C_{i:}/\varepsilon)$\;
  $D_i\leftarrow\sum_{v=1}^V p_{iv}R_iv_v$\;
  Draw $\eta_i\sim q$ independently if step jitter is enabled; otherwise set $\eta_i=0$\;
  $x_i^+\leftarrow x_i+(1+\eta_i)r_s\varepsilon D_i$\;
}
$\theta^+\leftarrow\theta_{\mathrm{ref}}+\Pi Qx^+$\;
$\bar C\leftarrow(PV)^{-1}\sum_{i=1}^P\sum_{v=1}^V C_{iv}$\;
\Return $x^+$, $\theta^+$, $\bar C$
\end{algorithm}

The returned mean $\bar C$ averages all $PVK$ probe evaluations; evaluating
the updated model's loss requires a separate call. In code, we subtract
each cost row's minimum before division by $\varepsilon$ to evaluate the
softmax stably; this leaves the weights unchanged.
For example, with a simplex in $d_p=8$ dimensions, $K=1$, and no jitter,
we compare nine candidates per block, each at distance $r_p\varepsilon/2$. The resulting block moves
by at most $r_s\varepsilon$. Because probe and step scales differ and all
blocks move together, the updated model generally differs from every
evaluated candidate. The primary image, ETTh1, and MAX-SAT configurations disable both
jitters. RL uses probe jitter only, and theory mode uses both.

A fresh step costs $PVK$ candidate evaluations: $P(d_p+1)K$ for a simplex
and $2Pd_pK$ for an orthoplex. At $K=1$, the simplex uses $9/16$ as many
evaluations as the orthoplex when $d_p=8$. For fixed $K$ and $d_p$, this
cost is $\Theta(d_{\mathrm{sub}})$. Candidates can be batched or processed
in memory-limited chunks, but each remains a distinct model evaluation.
Direction reuse reduces the number of fresh probes at the cost of using
older directions (Appendix~\ref{app:turbo-mode}).

\begin{figure*}[t]
  \centering
  \includegraphics[width=\linewidth]{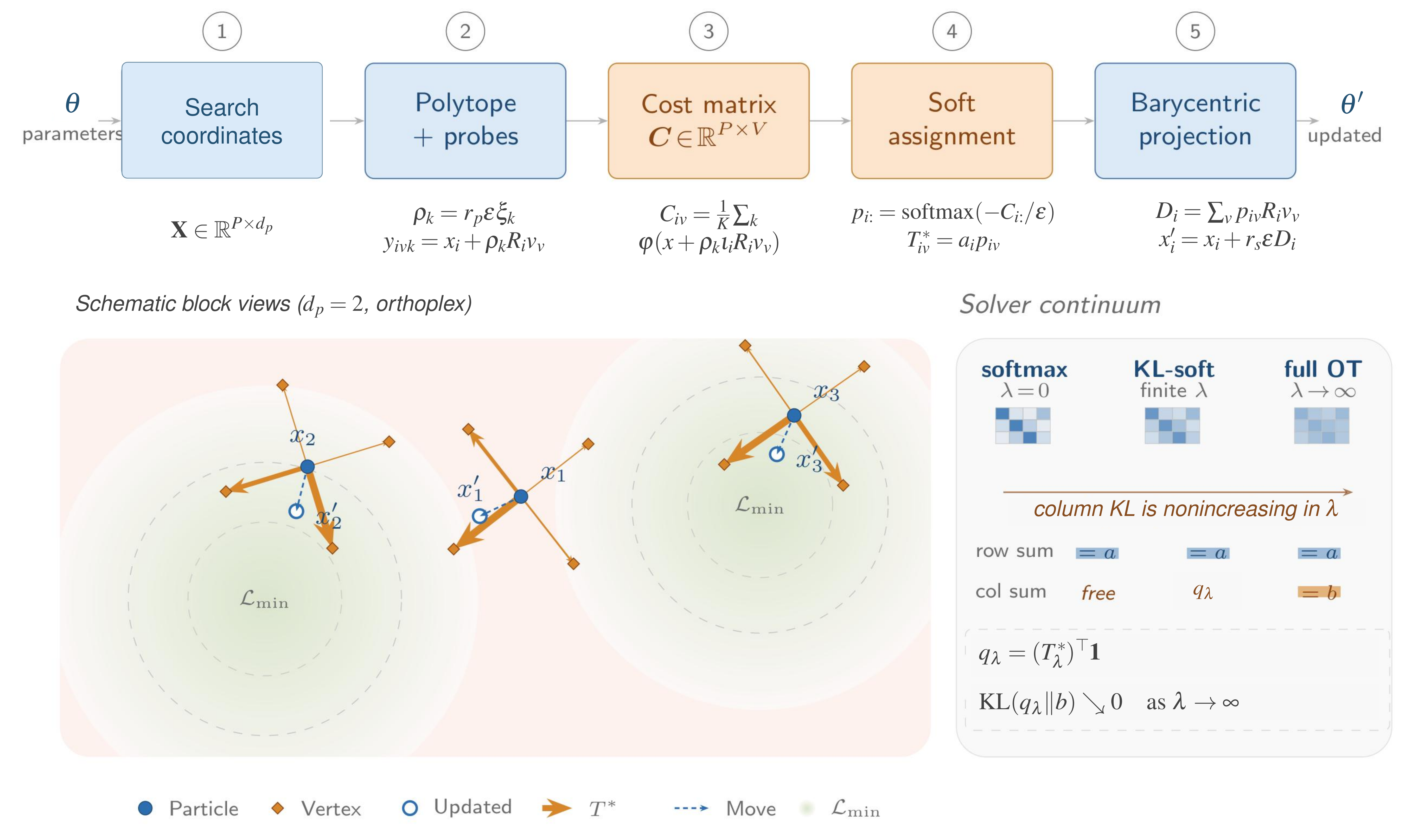}
  \caption{One \sysname{} step without jitter. Top: subspace coordinates
  are divided into blocks, probed along rotated vertex directions, and
  updated by weighted barycenters. The displayed assignment is row softmax
  ($\lambda=0$), with probe fractions $\xi_k=k/(K+1)$.
  Primes denote the updated variables.
  Bottom left: three block views superimposed schematically; diamonds are
  step vertices and open circles are barycentric updates. Shading and
  arrow widths are illustrative.
  Bottom right: a schematic column-penalty continuum. At fixed costs and
  temperature, $\mathrm{KL}(q_\lambda\Vert b)$ is nonincreasing and tends
  to zero (Theorem~\ref{thm:kl-softmax-endpoints}).}
  \label{fig:method-overview}
\end{figure*}

\subsection{Transport Interpretation}\label{sec:bg-ot}

A transport plan $T_{iv}\geq0$ assigns mass from $P$ sources to $V$
targets. Its row and column sums, called marginals, specify the total mass
sent by each source and received by each target. In PolyStep the sources
are parameter blocks, the targets are vertex labels, and $C_{iv}$ is the
probe loss. Each block has its own rotation, so a target label denotes a
local direction rather than a common point in parameter space. The costs
need not be geometric distances.

Write $\Delta_n=\{w\in\R^n:w\geq0,\ \sum_jw_j=1\}$ for the probability
simplex, and let $\va\in\Delta_P$ and $\vb\in\Delta_V$ have positive entries.
The implementation uses $a_i=1/P$ and $b_v=1/V$. The
Kantorovich problem minimizes $\langle\mC,\mT\rangle=\sum_{iv}C_{iv}T_{iv}$ subject to both
marginals. Entropic regularization adds
$\varepsilon\Ent(\mT)$, where
$\Ent(\mT)=\sum_{iv}T_{iv}(\log T_{iv}-1)$, favoring higher entropy
at a fixed transport cost. We retain the source marginal and
relax the target marginal through a KL penalty
\citep{chizat2018scaling}:
\begin{equation}\label{eq:intro-kl}
 \min_{\mT\geq0}\ \langle\mC,\mT\rangle+\varepsilon\Ent(\mT)
       +\lambda\KL(\mT^\top\bm1\|\vb)
 \quad\text{s.t.}\quad\mT\bm1=\va.
\end{equation}
Here $\lambda\geq0$ is the penalty strength and $\bm1$ is an all-ones vector
of the required dimension. The generalized Kullback--Leibler divergence is
$\KL(\vp\|\vq)=\sum_j[p_j\log(p_j/q_j)-p_j+q_j]$, with $0\log0=0$.
For finite costs and $\varepsilon>0$, strict convexity gives a unique plan.
At $\lambda=0$, the rows decouple. Minimizing each row under its mass
constraint gives
\begin{equation}\label{eq:softmax-rule}
 T^*_{iv}=a_i\,\mathrm{softmax}(-C_{i:}/\varepsilon)_v.
\end{equation}
Dividing by the row mass gives the update weights $p_{iv}=T^*_{iv}/a_i$
in Eq.~\ref{eq:compact-update}. We use this exact, closed-form solution in all
primary experiments; it requires no Sinkhorn iterations. At this endpoint
the target marginal $\vb$ drops out of the objective, and each block assigns
its mass independently. As $\lambda\to\infty$, the plans approach the
balanced entropic optimum of Eq.~\ref{eq:ot-objective}.
Appendix~\ref{sec:thm-kl-softmax} gives the dual potentials, proves both
endpoints, and establishes monotonicity of the column-marginal violation.

The optimal-transport interpretation concerns this assignment problem and
the barycentric update it defines. It connects PolyStep to the
transport-weighted polytope search of the Sinkhorn Step
\citep{le2023sinkhorn}. PolyStep applies the
assignment to blocks of one model, with per-layer subspaces and batched
forward evaluation. Column coupling can move a constant-cost block,
but this movement need not decrease the loss
(Appendix~\ref{app:proof-ghost-drift}). It also depends on the correspondence
between vertex labels across blocks. The solver comparisons in
Section~\ref{sec:ot-binding} do not establish a training benefit from this
coupling; a small marginal residual alone does not imply equal plans.

\subsection{Subspace Compression}\label{sec:subspace}

We reduce the number of blocks by searching a fixed subspace of the model
parameters. The primary image, ETTh1, and RL experiments use
\emph{HybridSubspace}, which assigns a separate projection to each layer.
For weight layer $l$ with $d_{\mathrm{out}}$ rows, $d_{\mathrm{in}}$
columns, and $d_l=d_{\mathrm{out}}d_{\mathrm{in}}$ parameters, we allocate
\[
 n_l=\min\{d_{\mathrm{out}}d_{\mathrm{in}},
                  (d_{\mathrm{out}}+d_{\mathrm{in}})r_l\},
 \qquad r_l=\min\{r,d_{\mathrm{out}},d_{\mathrm{in}}\},
\]
coordinates before any optional dimension cap, where $r$ is the chosen
compression parameter. This parameter sets the coordinate budget; it does
not restrict the reconstructed weight update to matrix rank $r$.
Let $z_l\in\R^{n_l}$ be the layer's slice of $z=Qx$.
A fixed map
$\Pi_l\in\R^{d_l\times n_l}$ reconstructs the layer displacement as
$\Delta\theta^{(l)}=\Pi_l z_l$. Biases and layers that retain all their
coordinates use the identity map. Combining the layer maps into a block-diagonal
matrix gives $\Pi$; the optimizer blocks partition their concatenated coordinates.
MAX-SAT uses full-space search.

Dense projections have orthonormal columns obtained by Gaussian QR.
For large layers, the implementation uses a sparse signed projection to
reduce storage. This fallback is active in the main image configurations
and is not orthonormal. It changes the relation between coordinate and
parameter-space lengths; our Euclidean subspace-gradient theorem assumes
orthonormal columns. The controlled weighting study uses dense QR bases.
Appendix~\ref{app:search-implementation} specifies the storage threshold,
sparse construction, dimension caps, and alternative projection modes.

The subspace controls both the directions available to the optimizer and
the number of evaluations per update. It therefore forms part of the
method being compared. Related approaches use layer-wise low-rank
perturbations \citep{subzero2024} or in-place full-space perturbations
\citep{malladi2023mezo}; these choices have different representation and
storage costs.
\section{Theoretical Analysis}\label{sec:theory}

We follow the update from observed costs to motion and then to descent.
The first two results concern the actual nonlinear softmax rule: simplex
geometry prevents cancellation, and ordered costs determine the sign of
motion on a ridge loss. For general losses, we identify the smoothing
behind the mean linear direction and quantify the errors needed to obtain
a stationarity bound. This separates a geometric property of each update
from a guarantee about its behavior over many iterations.

\subsection{From Cost Variation to Motion}\label{sec:nonlinear-jump}

A nonconstant cost row reveals variation at the probe scale. Whether its
weighted directions cancel depends on the frame. For the regular simplex
used in the primary experiments, the following identity gives an exact
relation between weight concentration and displacement.

\begin{proposition}[Simplex weights and displacement]\label{prop:simplex-detectability}
Let $v_1,\ldots,v_V$ be the vertices of a regular unit simplex in dimension
$d_p\geq1$, with $V=d_p+1$, so $\langle v_v,v_w\rangle=-1/(V-1)$ for $v\ne w$. For any probability
vector $p$ and any orthogonal rotation $R$, set $D=\sum_v p_vRv_v$. Then
\begin{equation}\label{eq:simplex-detectability}
 \|D\|^2=\frac{V}{V-1}\sum_v\left(p_v-\frac1V\right)^2.
\end{equation}
In particular, for $p=\mathrm{softmax}(-C/\varepsilon)$ with finite costs
and $\varepsilon>0$, $D=0$ if and only if $C$ is constant.
\end{proposition}

\begin{proof}
Rotation preserves inner products. Expanding the squared norm gives
$\|D\|^2=\sum_vp_v^2-(1-\sum_vp_v^2)/(V-1)$, which equals the right-hand
side because $\sum_vp_v=1$. The sum of squares vanishes exactly for uniform
weights, and finite-temperature softmax is uniform exactly for constant costs.
\end{proof}

The identity uses only the simplex Gram matrix, so it holds for every
rotation and every probability vector. On an orthoplex, unequal costs can
cancel: costs $(0,0,1,1)$ on $(e_1,-e_1,e_2,-e_2)$ give nonuniform
softmax weights but zero displacement. Thus the simplex converts every
observed cost difference into motion in search coordinates. Whether that
motion improves the loss is a separate question.

\subsection{Nonlinear Descent from Ordered Costs}\label{sec:nonlinear-descent}

A monotone ridge loss depends on the parameters through one linear
coordinate. A single threshold step is one example. The ordering of that
coordinate orders the probe costs, and softmax favors directions with
smaller costs. We work in search coordinates $x=(x_1,\ldots,x_P)$.
Write $L(x)$ for the loss expressed in these coordinates.
A fixed linear reconstruction transfers
the loss comparison to parameter space without an isometry assumption.

\begin{theorem}[Monotone ridge losses]\label{thm:ridge-descent}
Let $L(x)=h(b^\top x)$, where $h:\R\to\R$ is nondecreasing and
$b=(b_1,\ldots,b_P)$ is partitioned into the same blocks as $x$. In block $i$,
let $u_{iv}$ be unit vectors with $\sum_vu_{iv}=0$, and let the positive probe
radii $\rho_{ik}$ be shared across vertices. Define
\[
 C_{iv}=K^{-1}\sum_k h(b^\top x+\rho_{ik}b_i^\top u_{iv}),\qquad
 p_{iv}=\frac{e^{-C_{iv}/\varepsilon}}{\sum_w e^{-C_{iw}/\varepsilon}},\qquad
 D_i=\sum_vp_{iv}u_{iv},\quad \varepsilon>0.
\]
Then $b_i^\top D_i\leq0$. The inequality is strict if two vertices in that block have
different costs. For any nonnegative block step scales $\ell_i$,
$L(x+(\ell_iD_i)_i)\leq L(x)$. Strict motion in the ridge coordinate need
not strictly decrease a loss with plateaus.
\end{theorem}

\begin{proof}
Fix a block and write $s_v=b_i^\top u_{iv}$. Centeredness gives
$\sum_vs_v=0$. The common positive radii and monotonicity of $h$ imply
that $C_v$ is nondecreasing in $s_v$, so $p_v$ is nonincreasing in $s_v$.
The finite-sum identity
\[
 V\sum_v s_vp_v-(\sum_vs_v)(\sum_vp_v)
       =\sum_{v<w}(s_v-s_w)(p_v-p_w)
\]
has a nonpositive right-hand side. Distinct costs give a strictly negative
summand: equal $s_v$ would have equal costs, and the exponential is strictly
decreasing. Thus $b_i^\top D_i\leq0$, strictly in this case. Summing
$\ell_i b_i^\top D_i$ and applying $h$ proves the loss comparison.
\end{proof}

For a unit-normal planar jump $L(x)=J\mathbf{1}\{n^\top x\geq0\}$,
$J>0$, at $n^\top x=0$, an orthoplex in one $d_p$-dimensional block gives,
whenever no vertex lies on the plane,
\begin{equation}\label{eq:jump-softmax-drift}
 n^\top D=-\frac{\tanh(J/(2\varepsilon))}{d_p}
              \sum_{j=1}^{d_p}|n^\top R e_j|<0.
\end{equation}
The step points toward the lower side for every positive temperature,
including when $J/\varepsilon$ is too large for a small-cost expansion.
This example shows why a direct analysis of the nonlinear rule can apply
where a linear approximation does not.

Branch variation and noisy evaluations perturb these ideal planar costs.
The next bound shows how much normal drift survives such errors.

\begin{proposition}[Stability under cost errors]\label{prop:softmax-cost-stability}
Fix $\varepsilon>0$ and unit directions $u_v$. Write
$p(C)=\mathrm{softmax}(-C/\varepsilon)$ and $D(C)=\sum_vp(C)_vu_v$.
If $\|\widehat C-C-b\bm1\|_\infty\leq\beta$ for some scalar $b$, then
\begin{equation}\label{eq:softmax-cost-stability}
 \|D(\widehat C)-D(C)\|\leq\min\{2,\beta/\varepsilon\}.
\end{equation}
In particular, a normal drift $n^\top D(C)=-m_0$ for unit $n$ remains
negative if $\beta<\varepsilon m_0$.
\end{proposition}

Adding the same constant to all costs leaves softmax unchanged. The scalar
$b$ in Proposition~\ref{prop:softmax-cost-stability} removes that common offset,
so $\beta$ measures only the cost errors that can change the weights.

We can now transfer the planar argument to a curved boundary. The next
result separates the error in the probe costs from the curvature allowance
for the update itself.

\begin{corollary}[Descent across a curved jump]\label{cor:curved-jump}
Work in the coordinates of a single active block. Suppose $L=g_+$ on $\{\psi\geq0\}$ and $L=g_-$ on $\{\psi<0\}$
in a neighborhood of $x$, where $g_\pm$ extend to $G$-Lipschitz functions
and $\psi$ is twice continuously differentiable. Assume
\[
 \psi(x)=0,\qquad \nabla\psi(x)=n,\qquad \|n\|=1,\qquad
 J=g_+(x)-g_-(x)>0,\qquad \|D^2\psi\|\leq\kappa.
\]
Use the rotated orthoplex directions $u_v=Rv_v$, with positive radii
$\rho_k$ shared across vertices and maximum radius $\rho_{\max}$.
Define $C_v=K^{-1}\sum_k L(x+\rho_k u_v)$ and suppose
$\min_v|n^\top u_v|>\kappa\rho_{\max}/2$ and
$\|\widehat C-C\|_\infty\leq b_{\mathrm{noise}}$. Let $m_0$ be the
magnitude of the planar normal drift in Eq.~\ref{eq:jump-softmax-drift} and set
\[
 m=m_0-\frac{G\rho_{\max}+b_{\mathrm{noise}}}{\varepsilon}.
\]
If $m>0$, $\ell>0$, $\kappa\ell<2m$, and $G\ell<J$, then the step
$\ell D(\widehat C)$ crosses to $\{\psi<0\}$ and strictly decreases $L$,
provided all probe and update segments remain in the neighborhood.
\end{corollary}

The probe condition preserves the tangent plane's classification of the
sampled points. Proposition~\ref{prop:softmax-cost-stability} then gives
$n^\top D(\widehat C)\leq-m$. Taylor's theorem bounds the boundary's
curvature along the update by $\kappa\ell^2/2$, while the lower branch
can increase by at most $G\ell$. These are precisely the two allowances
in the statement. Appendix~\ref{app:nonlinear-softmax} supplies the full
proof and a probability bound for minibatch cost errors.

Cost ordering is the essential property in the ridge argument. Greedy and
rank-based weights share it when tied costs receive equal weights.
Our controlled comparison in
Section~\ref{sec:controlled-comparison} tests these rules with their natural
updates, so both direction and displacement length can affect performance.

\subsection{The Smoothing Behind the Mean Direction}\label{sec:thm-piecewise-smooth}

For a general loss, the expected update need not decrease the original
objective. We instead relate it to a local average of the loss. First, replace
softmax weights by their linear approximation in the costs. For each block,
the expected linear direction differentiates an average of the loss over
balls. To sample this average, choose a probe radius with probability
weighted by its length, then sample uniformly inside the corresponding
ball. Thus larger probe radii contribute more to the average. This
\emph{size-biased mixture of balls} has the tail-kernel density derived
below; it differs from the law that places probes on the balls' surfaces.

There are then two errors to control. Softmax differs from the linear cost
weights. Also, each block probes with the other blocks fixed, whereas the
convergence target averages over perturbations to all blocks together.
This second difference is the \emph{block-smoothing discrepancy}. A field is
called \emph{conservative} if it is the gradient of a scalar function; neither
the combined block field nor the actual softmax mean has this property in general.

\subsubsection{The Exact Smoothed-Gradient Identity}

For the smoothing analysis, use complete blocks and a fixed isometric
search map $\Pi$, so $d_{\mathrm{sub}}=Pd_p$ and $\Pi^\top\Pi=I$.
Let $E_i\cong\R^{d_p}$ be coordinate block $i$ and $B_i=\Pi\iota_i$ its embedding
into parameter space. In blockwise formulas, $\theta+u$ with $u\in E_i$
means $\theta+B_i u$, and $\nabla_{E_i}$ returns the $d_p$ coordinate
derivatives in that block. We write $m_i$ for a copy of the radial kernel
$m$ on block $i$. Thus convolution over $E_i$ means
$(\Ls*m_i)(\theta)=\int_{E_i}\Ls(\theta-B_i u)m(u)\,du$.

The next theorem identifies the smoothed loss behind the linear mean.
To compare this mean with the actual softmax update, we also need a
remainder bound. Write $\Delta_i=\max_v C_{iv}-\min_v C_{iv}$ for the cost spread of
block $i$. The constant
$\kappa_0=3\exp(1)/2$ satisfies
$|\mathrm{softmax}(z)_v - V^{-1}(1 + \tilde z_v)| \leq \kappa_0 V^{-1}\|\tilde z\|_\infty^2$
for $\|\tilde z\|_\infty \leq 1$, where
$\tilde z=z-\bar z\mathbf1$ and $\bar z=V^{-1}\sum_v z_v$.

\begin{theorem}[Tail-kernel identity and softmax error]\label{thm:exact-step}
Fix particle $i$ and its coordinate block $E_i$, with $d_p\geq2$.
Let the displacement attached to vertex $v$ be $u_v = \rho_v\,R v_v$,
where $R \sim \mathrm{Uniform}(\mathrm{SO}(d_p))$ is shared across vertices.
Assume $\|v_v\| = 1$ and $\sum_v v_v = 0$; centeredness is used in the
remainder estimate. Each radius $\rho_v$ is independent of $R$ and has
marginal density
\begin{equation}\label{eq:mixture-radius}
q_{\mathrm{mix}}(\rho) \;=\; \frac{1}{K}\sum_{k=1}^{K}
\frac{1}{r_p\xi_k\varepsilon}\, q\!\Big(\frac{\rho}{r_p\xi_k\varepsilon} - 1\Big),
\qquad \xi_k = \frac{k}{K+1},
\end{equation}
where $K\geq1$, $r_p,\varepsilon>0$, and $q$ is a probability density
supported in $[-\eta_{\max},\eta_{\max}]$ for $0<\eta_{\max}<1$.
The radii need not be independent of one another: the expectation below
is linear, so only their marginals enter. Write $p_{\mathrm{probe}}(u)=\tilde p(\|u\|)$ for the induced density
on $E_i$, where
$\tilde p(\rho)=q_{\mathrm{mix}}(\rho)/(\sigma_{d_p-1}\rho^{d_p-1})$
and $\sigma_{d_p-1}$ is the unit sphere's surface area. Define the
\emph{tail kernel} and the block-smoothed loss
\begin{equation}\label{eq:surrogate-def}
m(u) := \int_{\|u\|}^{\infty}\!\!\tilde p(s)\,\mathrm{d}s,
\qquad
Z := \int m \;=\; \frac{\E\|u\|}{d_p} \;=\; \frac{\bar r_p\,\varepsilon}{d_p},
\qquad
\Ls_{\delta,i} := \frac{\Ls * m}{Z},
\end{equation}
with $\bar r_p := r_p\bar\xi\,\E[1+\eta]$, $\bar\xi := K^{-1}\sum_k \xi_k = 1/2$
and $\delta := \bar r_p \varepsilon$ its scale. For every bounded measurable
$\Ls$, expectation over the probe law gives
\begin{equation}\label{eq:stein-softmax-body}
\E\Big[\sum_v \Ls(\theta + \rho_v R v_v)\, R v_v\Big]
\;=\; V\,Z\,\nabla_{E_i} \Ls_{\delta,i}(\theta) ,
\end{equation}
and, with $p_{iv}=\mathrm{softmax}(-C_{i:}/\varepsilon)_v$ as in
Eq.~\ref{eq:compact-update}, the expected barycentric direction satisfies
\begin{equation}\label{eq:step-error}
\E\Big[\sum_v p_{iv}\, R v_v\Big]
\;=\; -\,c\,\nabla_{E_i}\Ls_{\delta,i}(\theta) \;+\; e_i,
\qquad
c \;=\; \frac{\bar r_p}{d_p} \;=\; \frac{r_p}{2 d_p}\ \text{ for }\E[\eta]=0 .
\end{equation}
The first identity also holds with the radius-averaged costs $C_{iv}$
of Eq.~\ref{eq:cost-matrix}, by linearity. The softmax in the second identity
is applied to that averaged row. Its error comes entirely from the softmax remainder:

\[
\|e_i\| \leq \kappa_0\,\E[(\Delta_i/\varepsilon)^2\,\mathbf{1}[\Delta_i \leq \varepsilon]]
+ \Pr[\Delta_i > \varepsilon] + \E[(\Delta_i/\varepsilon)\,\mathbf{1}[\Delta_i > \varepsilon]] ;
\]
A second bound, valid at every base point, is
$\|e_i\| \leq 1 + c\|\nabla_{E_i}\Ls_{\delta,i}(\theta)\|$.
\end{theorem}

\begin{proof}
The tail satisfies
$\nabla m(u)=-p_{\mathrm{probe}}(u)u/\|u\|$ almost everywhere. The positive lower bound on
the radii gives $m\in W^{1,1}$, so convolution with a bounded loss is
differentiable along the block. Symmetry and differentiation in $L^1$ give
\[
 \int \Ls(\theta+B_i u)p_{\mathrm{probe}}(u)\frac{u}{\|u\|}\,\mathrm du
 =\nabla_{E_i}(\Ls*m_i)(\theta).
\]
Polar integration gives $Z=\E\rho/d_p$. Each Haar-rotated unit vertex
has the same spherical marginal, so summing vertices and averaging radii
proves the first identity. Centeredness cancels the uniform term in the
softmax expansion, leaving $c=Z/\varepsilon$. On $\{\Delta_i\leq\varepsilon\}$,
the remainder is at most $\kappa_0(\Delta_i/\varepsilon)^2$.
On its complement, the direction has norm at most one and its
linearization has norm at most $\Delta_i/\varepsilon$. Splitting the
expectation proves the error bound. Appendix~\ref{app:proof-piecewise-smooth}
provides the kernel and remainder estimates in detail.
\end{proof}

The mixture of balls has density
\begin{equation}\label{eq:size-biased-balls}
 \frac{m(u)}{Z}=\int_0^\infty
 \frac{\rho q_{\mathrm{mix}}(\rho)}{\E\rho}
 \frac{\mathbf{1}\{\|u\|\leq\rho\}}{\operatorname{vol}(B_\rho)}\,\mathrm{d}\rho.
\end{equation}
The radius density $\rho q_{\mathrm{mix}}(\rho)/\E\rho$ is the
size-biased law described above. With no jitter, the $K$ ball components
have weights proportional to $k$, rather than equal weights.

In the bound on $e_i$, the first term controls rows with spread at most
$\varepsilon$, called \emph{covered} rows. The other two terms control the
probability and cost spread of saturated rows. Appendix~\ref{app:weight-spread}
shows how to recover the spread from the weights and relates nearly uniform
weights to small cost variation relative to temperature.
The constant $c$ is the same for simplex and orthoplex frames: each
unit vertex has the same spherical marginal, and the number of vertices
cancels against the normalization of the weights. Smooth neighborhoods
permit a sharper remainder estimate, given in
Appendix~\ref{app:proof-piecewise-smooth}.

The tail kernel averages ball interiors, whereas the probe law averages
their surfaces. These averages can have gradients pointing in opposite
directions, even for a smooth loss (Corollary~\ref{cor:two-kernels}).
With $K=1$ and no jitter, the linear identity reduces to the classical
sphere-to-ball identity \citep[Lem.~2.1]{flaxman2005online}. For bounded
measurable losses that identity holds almost everywhere in the base point;
for continuous losses it holds everywhere. A jitter density makes the
tail smoothing differentiable at every base point, which is the version
used here.

The mean probe radius $\delta=\bar r_p\varepsilon$ sets the smoothing scale;
its outer support radius is
$\delta_{\mathrm{out}}=r_p\xi_K(1+\eta_{\max})\varepsilon$.
Their ratio is fixed by the probe fractions and jitter law.

\subsubsection{From Block Derivatives to a Joint Objective}

A gradient interpretation for each block does not ensure that the combined
update follows the gradient of a single objective. We define the joint smoothing by averaging
independently over the tail kernel in every block. Assume $\Pi^\top\Pi=I$, as in Assumption~\ref{asm:main}(iii).
The probe planes decompose the search coordinates as
$\R^{d_{\mathrm{sub}}}=\bigoplus_iE_i$. The product kernel has a density in
these coordinates, and the joint surrogate and its subspace gradient are
\begin{equation}\label{eq:joint-surrogate}
\begin{aligned}
\Ls_\delta(\theta)
&:= \int_{\R^{d_{\mathrm{sub}}}} \Ls(\theta + \Pi u)\,
        \frac{m^{\otimes P}(u)}{Z^P}\,\mathrm{d}u, \\
\nabla_{\mathcal S}\Ls_\delta(\theta)
&:= \Pi\left.\nabla_z\big[\Ls_\delta(\theta + \Pi z)\big]\right|_{z=0}.
\end{aligned}
\end{equation}
If the search subspace is proper, the product law's ambient
pushforward $\Pi_\#(m^{\otimes P}/Z^P)$ is singular in $\R^d$; the proof uses
the density in search coordinates and does not require a density in ambient space.
In contrast, $\Ls_{\delta,i}=(\Ls*m_i)/Z$ averages only block $i$, while
$\Ls_\delta$ averages all blocks. We use this joint smoothing as the target of every
stationarity bound below.

\begin{remark}[Block smoothing and the actual mean update]\label{rem:block-coupling}
Algorithm~\ref{alg:polystep-step} perturbs one particle at a time but updates all
particles in the same step. Its concatenated mean direction is
$F=-c\Phi+e$, where
$\Phi(\theta) := (\nabla_{E_i}\Ls_{\delta,i}(\theta))_i$ smooths block $i$
and leaves the others fixed.
Let $g=\Pi^\top\nabla_{\mathcal S}\Ls_\delta$ be the coordinate representation
of this gradient, so $\|g\|=\|\nabla_{\mathcal S}\Ls_\delta\|$. Then
\[
 F=-cg+e-c(\Phi-g).
\]
Here $e$ is the softmax error and $\Phi-g$ is the block-smoothing discrepancy.
The linear field $\Phi$ need not be the gradient of a single scalar loss.
To bound the discrepancy, let $U_{-i}\in\R^{d_{\mathrm{sub}}}$ collect the
independent tail-kernel perturbations in all blocks except $i$, with zeros
in block $i$. Then
$g_i(\theta)=\E[\Phi_i(\theta+\Pi U_{-i})]$.
Since $\E U_{-i}=0$, Taylor's theorem under assumption~(vii-L) below gives
\[
 \|\Phi_i-g_i\|
 \leq \frac{\Lambda_3}{2}\E\|U_{-i}\|^2
 \leq \frac{P-1}{2}\,\Lambda_3\delta_{\mathrm{out}}^2.
\]
Thus $\|\Phi-g\|=O(P^{3/2}\delta^2\Lambda_3)$.
Its square gives the third term of Eq.~\ref{eq:bias-floor}; a cubic example
in Appendix~\ref{app:coupling-witnesses} attains the $P^3$ factor.
The actual nonlinear field $F$ can also be nonconservative;
Appendix~\ref{app:nonlinear-field} gives a two-block example with unequal
cross-partial derivatives. Updating one block at a time does not by itself
remove the discrepancy between smoothing targets.
\end{remark}

The next result removes the block-smoothing discrepancy by changing the
probe locations. It samples the other blocks from their tail kernels
before probing the active block.

\begin{proposition}[Joint background perturbations]\label{prop:joint-background}
Consider a modified cost row that, before probing block $i$, independently
perturbs every other block by $U_j\sim m_j/Z$. Use the same active-block
probe law as in Theorem~\ref{thm:exact-step}. Then
\[
 \E\Big[\sum_v\widetilde C_{iv}R_i v_v\Big]
       =VZ\nabla_{E_i}\Ls_\delta,
\]
where $\Ls_\delta$ is the joint surrogate of Eq.~\ref{eq:joint-surrogate}.
Consequently $\Phi-g=0$ for this modified linear field. The convergence
bounds apply without the $\Lambda_3$ assumption or the block-coupling term,
provided their remaining assumptions hold
for the modified costs and iterates.
\end{proposition}

\begin{proof}
Condition on $U_{-i}$. Theorem~\ref{thm:exact-step} gives the active-block
linear mean $VZ\Phi_i(\theta+\Pi U_{-i})$. Averaging the background gives
$VZg_i(\theta)$ by the product-kernel definition and Fubini's theorem.
Hence the block discrepancy vanishes, leaving the softmax and saturation
terms in the convergence proof.
\end{proof}

Appendix~\ref{app:joint-background} describes how to sample the background.
This variant perturbs the other blocks before evaluating each cost row. It adds randomness and changes the set of parameter vectors
that can be probed. The proposition concerns its mean linear field; it does not
establish an accuracy or variance advantage for the nonlinear update.

\subsection{Conditional Stationarity at Fixed Resolution}\label{sec:convergence}

\paragraph{The target.}
For a bounded measurable loss, the target is a small gradient of the
smoothed loss within the search subspace, $\|\nabla_{\mathcal S}\Ls_\delta\|$.
This means that the local average changes slowly along the directions the
optimizer can search. It does not imply that the iterate is near a local
minimum of the original discontinuous loss.

When the original loss is locally Lipschitz on the required neighborhood,
a small smoothed gradient also gives Goldstein stationarity
\citep{zhang2020complexity,lin2022gradientfree,chen2023faster,kornowski2024algorithm}.
This criterion asks whether a convex combination of nearby generalized
gradients can approach zero. At radius $\delta_{\mathrm{joint}}$ and
tolerance $\nu$, the subspace criterion is
$\mathrm{dist}(0,\Pi^\top\partial_{\delta_{\mathrm{joint}}}\Ls(\theta))\leq\nu$.
Corollary~\ref{cor:goldstein} transfers our expected squared-gradient bound
to the expected square of this distance; it does not guarantee a prescribed
tolerance in every run.
A neighborhood containing an attained jump fails the local Lipschitz
hypothesis, so only the smoothed-gradient target applies there.
The tolerance $\nu$ is distinct from the entropic temperature $\varepsilon$.

\paragraph{Fixed resolution.}\label{sec:smoothing-limits}
Averaging an isolated jump of size $J$ across radius $h$ forces slopes of
order at least $J/h$ and a second-difference seminorm of order at least
$J/h^2$, for sufficiently small $h$ under the geometric conditions of
Theorem~\ref{thm:smoothing-blowup}. For a fixed scaled smooth kernel, the
second-order descent allowance is $O(s_h^2/h^2)$ for steps of length at most
$s_h$. Shrinking probes and steps in a fixed ratio need not reduce this
allowance. We therefore fix $\varepsilon$ and $r_p$ and reduce $r_s$:
the target stays fixed while the steps shrink. This motivates the analysis;
it does not establish nonconvergence of a vanishing-radius schedule.

Write $A_i(t) := \{\Delta_i(t) > \varepsilon\}$ for the saturated draws. Let
$\mathcal F_t$ contain the information available before step $t$: the iterates
and all rotations, jitters, and minibatches drawn through step $t-1$. Conditioning
on $\mathcal F_t$ thus fixes the history while leaving the current draws random. The
joint kernel radius is
\begin{equation}\label{eq:joint-radius}
\delta_{\mathrm{joint}} \;:=\; \sqrt{P}\,\delta_{\mathrm{out}},
\qquad
\delta_{\mathrm{out}} := r_p\xi_K(1+\eta_{\max})\varepsilon ,
\end{equation}
because the squared norms of the $P$ block displacements add, with each
block reaching at most $\delta_{\mathrm{out}}$. Write $\bar q$ for a deterministic
bound satisfying
$\E[(\Delta_i/\varepsilon)^2\,\mathbf{1}[A_i^c] \mid \mathcal{F}_t] \leq \bar q$
almost surely for all $i$ and $t$. This conditional second moment controls
the linearization term. Observed cost spreads provide diagnostics, but do
not establish a uniform conditional bound.

A stationarity bound also needs control of smoothness, sampling, and
approximation errors. We state these conditions next.

\begin{assumption}[Conditions for the stationarity bound]\label{asm:main}
Run the row-wise softmax update ($\lambda=0$) of
Algorithm~\ref{alg:polystep-step}, with $d_{\mathrm{sub}}=Pd_p$ and no discarded
coordinates, on a horizon $0\leq t<T$, using:
\begin{enumerate}[label=(\roman*),leftmargin=2.6em]
\item a \emph{centered frame of unit vertices}, $\sum_v v_v = 0$ and $\|v_v\| = 1$, at
particle dimension $d_p \geq 2$;
\item rotations $R_i \sim \mathrm{Uniform}(\mathrm{SO}(d_p))$ drawn independently each
step;
\item a projection with orthonormal columns, $\Pi^\top\Pi = I$;
\item probe- and step-radius jitter with a density $q \in W^{2,1}$ symmetric about
zero and supported in $[-\eta_{\max}, \eta_{\max}]$ for some $\eta_{\max} \in (0,1)$,
so that $\E[\eta] = 0$, the outer probe reach $\delta_{\mathrm{out}}$ is finite, and the
tail kernel of Eq.~\ref{eq:surrogate-def} has $m \in W^{3,1}$;
\item fixed temperature $\varepsilon>0$, fixed probe multiplier $r_p>0$,
and deterministic positive step multipliers $r_{s,t}$;
\item iterates confined to a compact $\mathcal{K} \subset \R^d$ almost surely;
\item two derivative bounds. For the joint smoothing, (vii-K) bounds
by $\Lambda_\times$ the operator norm of the off-block-diagonal part of
$D_z^2[\Ls_\delta(\theta+\Pi z)]|_{z=0}$, uniformly in $\theta$.
This bound is finite under (iv) for a bounded loss. For the block field,
(vii-L) requires each $\Phi_i=\nabla_{E_i}(\Ls*m_i)/Z$ to be twice
continuously differentiable in the other blocks, with
$\|D_\perp^2\Phi_i\|\leq\Lambda_3$ on
$\mathcal K_{\delta_{\mathrm{joint}}}$, the set of points within distance
$\delta_{\mathrm{joint}}$ of $\mathcal K$. Here $D_\perp$ differentiates
in all blocks except $i$;
\item \emph{sampling and saturation control}: the current minibatch,
rotations, shared probe jitter, and per-block step jitters are mutually
independent conditional on $\mathcal F_t$; the
minibatch loss is jointly measurable in the parameter and the batch and bounded on
$\mathcal{K}_{\delta_{\mathrm{joint}}}$ by a deterministic constant, uniformly in $t$,
and is conditionally unbiased,
$\E[\Ls_t(\cdot) \mid \mathcal{F}_t] = \Ls(\cdot)$.
For some $0\leq\bar\pi<1$, the \emph{saturated mass} satisfies, almost surely
for every $i,t$,
\begin{equation}\label{eq:crossing-fraction}
\Big\| \E\Big[\textstyle\sum_v \mC_{iv} R_i v_v\,\mathbf{1}_{A_i(t)} \,\Big|\, \mathcal{F}_t\Big] \Big\|
\;\leq\;
\bar\pi \, \Big\| \E\Big[\textstyle\sum_v \mC_{iv} R_i v_v \,\Big|\, \mathcal{F}_t\Big] \Big\| .
\end{equation}
The \emph{saturation probability}
$\sigma(t) := P^{-1}\sum_i \Pr[A_i(t)\mid\mathcal{F}_t]$ enters additively and is
controlled through its $r_{s,t}$-weighted average: $\bar\sigma$ is any constant
with $\E\sum_{t<T} r_{s,t}\,\sigma(t) \leq \bar\sigma\, S_T$,
$S_T := \sum_{t<T} r_{s,t}$, for the stated horizon $T$.
\end{enumerate}
\end{assumption}

Conditions (i)--(iii) give the exact linear identity and identify coordinate
gradients with Euclidean gradients in the search subspace. Condition (iv)
makes the joint smoothing regular even when the loss jumps. In particular,
let $\Lambda_\delta$ bound the Lipschitz constant of each
$\nabla_{E_i}\Ls_{\delta,i}$ as block $i$ varies. Then
$\Lambda_\delta=O(\|\Ls\|_\infty\delta^{-2})$, and the joint gradient
is Lipschitz with constant $\Lambda_\delta+\Lambda_\times$
(Lemma~\ref{lem:smooth-kernel}). A smooth compactly supported jitter density
satisfies (iv); uniform jitter and zero jitter do not.

Condition (vii-L) controls a different derivative: how a gradient smoothed
in one block varies as the other blocks move. It is an assumption on the
loss, not a consequence of smoothing that block. Condition (viii) likewise
has two distinct roles. Conditional unbiasedness relates the sampled costs
to $\Ls$. The constants $\bar\pi$ and $\bar\sigma$ control, respectively,
the directional contribution and frequency of saturated rows. A small
frequency alone does not bound their directional contribution. For the
planar jump of Eq.~\ref{eq:jump-softmax-drift} with $J>\varepsilon$, every
row is saturated and $\bar\pi<1$ is impossible at the boundary, although
the nonlinear step decreases the loss. The stationarity theorem and the
nonlinear descent result therefore apply under different conditions.

We state the bound for an iterate that can be selected without evaluating
the smoothed gradient. Define the error contribution
\begin{equation}\label{eq:bias-floor}
\begin{aligned}
 B(r_p,\varepsilon,\bar\sigma)
 &:=\frac{3\kappa_0^2P\bar q^{\,2}}{c^2}
       +\frac{3P\bar\sigma}{c^2}
       +3P(P-1)^2\delta_{\mathrm{out}}^4\Lambda_3^2\\
 &=\underbrace{O\!\Big(\frac{P d_p^2\bar q^{\,2}}{r_p^2}\Big)}_{\text{softmax linearization}}
   +\underbrace{O\!\Big(\frac{P d_p^2\bar\sigma}{r_p^2}\Big)}_{\text{saturated rows}}
   +\underbrace{O(P^3\delta^4\Lambda_3^2)}_{\text{block coupling}}.
\end{aligned}
\end{equation}
Here $c=r_p/(2d_p)$, and $\delta_{\mathrm{out}}/\delta$ is fixed by the
probe fractions and jitter law. The notation suppresses dependence on
$P,d_p,\bar q,\Lambda_3$ and that law. The first two errors are divided by
$c^2$. The block discrepancy already carries a factor $c$ in the mean
displacement, which cancels before squaring. This bias floor need not
shrink with additional iterations at the same settings.

\begin{theorem}[Finite-horizon stationarity with a sampled output]\label{cor:finite-horizon}
Let $\Ls$ be bounded and measurable, fix $T\geq1$, and suppose
Assumption~\ref{asm:main} holds.
Let $S_T=\sum_{t<T}r_{s,t}$ and choose $\tau$ independently of the run with
$\Pr(\tau=t)=r_{s,t}/S_T$. Write
\[
 A_0=\E\Ls_\delta(\theta_0)-\inf\Ls,\qquad
 H_0=\tfrac12(\Lambda_\delta+\Lambda_\times)
       (1+\eta_{\max})^2P\varepsilon^2.
\]
Then
\begin{equation}\label{eq:finite-horizon}
 \E\|\nabla_{\mathcal S}\Ls_\delta(\theta_\tau)\|^2
 \leq\frac{2(A_0+H_0\sum_{t<T}r_{s,t}^2)}{c(1-\bar\pi)\varepsilon S_T}
       +\frac{B(r_p,\varepsilon,\bar\sigma)}{(1-\bar\pi)^2}.
\end{equation}
In particular, $r_{s,t}=a/\sqrt T$, $a>0$, gives $S_T=a\sqrt T$,
$\sum_{t<T}r_{s,t}^2=a^2$, and an $O(T^{-1/2})$ term plus the same floor,
provided the other constants are uniform in $T$.
\end{theorem}

\begin{proof}
Fix $t$ and condition on $\mathcal F_t$. Let $D_i$ be the realized
softmax direction from Eq.~\ref{eq:compact-update}, $h_i=-(V\varepsilon)^{-1}\sum_v C_{iv}R_iv_v$ its
linearization, $g_i=\nabla_{E_i}\Ls_\delta(\theta_t)$, and
$w_i=\Phi_i(\theta_t)-g_i$. By Theorem~\ref{thm:exact-step} and conditional
unbiasedness, $\E[h_i\mid\mathcal F_t]=-c(g_i+w_i)$. Splitting saturated
and covered draws gives
\[
 \E[D_i\mid\mathcal F_t]=-cg_i-cw_i-b_i+\zeta_i,
 \quad
 b_i=\E[h_i\mathbf1_{A_i}\mid\mathcal F_t],
\]
where $\zeta_i=\E[(D_i-h_i)\mathbf1_{A_i^c}+D_i\mathbf1_{A_i}\mid\mathcal F_t]$.
Write $q_i=\E[(\Delta_i/\varepsilon)^2\mathbf1_{A_i^c}\mid\mathcal F_t]$
and $p_i=\Pr(A_i\mid\mathcal F_t)$. The remainder bound and
Eq.~\ref{eq:crossing-fraction} give
\[
 \|\zeta_i\|\leq\kappa_0q_i+p_i,\qquad
 \|b_i\|\leq\bar\pi c(\|g_i\|+\|w_i\|).
\]
Set $\alpha=c(1-\bar\pi)>0$ and
$d_i=c(1+\bar\pi)\|w_i\|+\kappa_0q_i+p_i$. Completing the square yields
\[
 \langle g_i,\E[D_i\mid\mathcal F_t]\rangle
 \leq-\alpha\|g_i\|^2+d_i\|g_i\|
 \leq-\frac\alpha2\|g_i\|^2+\frac{d_i^2}{2\alpha}.
\]
Let $\mathcal E_t=c^{-2}\sum_i d_i^2$. The block-discrepancy estimate in
Remark~\ref{rem:block-coupling}, $q_i\leq\bar q$, $p_i^2\leq p_i$, and
$(x+y+z)^2\leq3(x^2+y^2+z^2)$ imply
\[
 \mathcal E_t\leq
 3P(P-1)^2\delta_{\mathrm{out}}^4\Lambda_3^2
 +\frac{3\kappa_0^2P\bar q^{\,2}}{c^2}
 +\frac{3P\sigma(t)}{c^2}.
\]
Thus $\E\sum_{t<T}r_{s,t}\mathcal E_t\leq BS_T$. Independent mean-zero
step jitter gives $\E[\Delta x_i\mid\mathcal F_t]=r_{s,t}\varepsilon\E[D_i\mid\mathcal F_t]$,
and the squared displacement is at most
$(1+\eta_{\max})^2P r_{s,t}^2\varepsilon^2$.
Apply the descent lemma with constant $\Lambda_\delta+\Lambda_\times$,
take expectations, and telescope. Since $\Ls_\delta\geq\inf\Ls$,
\[
 \frac\alpha2\varepsilon\sum_{t<T}r_{s,t}\E\|g(\theta_t)\|^2
 \leq A_0+H_0\sum_{t<T}r_{s,t}^2
       +\frac{c\varepsilon}{2(1-\bar\pi)}BS_T.
\]
Finally, isometry and the independent choice of $\tau$ give
$\E\|\nabla_{\mathcal S}\Ls_\delta(\theta_\tau)\|^2
=S_T^{-1}\sum_{t<T}r_{s,t}\E\|g(\theta_t)\|^2$.
Dividing by $\alpha\varepsilon S_T/2$ proves the result.
\end{proof}

The index $\tau$ can be sampled before training and its iterate saved.
For $r_{s,t}=a/\sqrt T$, it is uniform over the pre-update iterates.
This output rule requires no evaluation of the smoothed gradient. We bound
the squared gradient norm in expectation, so the theorem does not guarantee
that it is small in every run.

Infinite schedules with summable squared steps give a corresponding
bound on the smallest gradient encountered; we state it in
Appendix~\ref{app:infinite-schedule}.

\paragraph{Interpretation and scope.}
The bound separates iteration error from approximation error. More steps
reduce the first term, while the floor records softmax error, saturation,
and the discrepancy between blockwise and joint smoothing. Saturated mass
also weakens the descent coefficient through $1-\bar\pi$. The assumptions
needed to control these terms are stronger than the geometry of an isolated
threshold; Appendix~\ref{app:theory-remarks} gives examples and examines
their dependence on the probe radius.

This is a guarantee for an idealized fixed-temperature algorithm. The
trained configurations use different jitter, projections, sampling, or
schedules, and we have not established the required transverse regularity
or confinement for their losses. Appendix~\ref{sec:hypothesis-coverage}
gives the precise comparison. Moreover, the worst-case bias bound is too
large to predict their performance: $P=1{,}068$ in the MNIST configuration
gives $P^3\approx10^9$ in the coupling term. The value of the bound is to
identify which errors must be controlled to obtain a common descent
interpretation, rather than to certify the reported accuracies.
\section{Experiments}\label{sec:experiments}

We evaluate how PolyStep uses a limited budget of forward evaluations
and test the mechanisms identified by the analysis. The benchmarks cover
hard spike thresholds, INT8 quantization, argmax attention, staircase
activations, and hard mixture-of-experts (MoE) routing. We also study
MAX-SAT at up to $10^6$ variables, continuous MNIST and ETTh1 controls,
and quantized policy search with noisy rollout costs. We then examine
explicit jump losses and compare weighting rules with the same candidate
distribution.

The comparison budget matters because methods use different numbers of
candidates per update. We distinguish candidate evaluations, sequential
updates, and wall-clock time in Section~\ref{sec:matching-axis}. Our main
accuracy result is the matched-evaluation hard-LIF comparison; the argmax
and MAX-SAT experiments show cases where OpenAI-ES leads. The controlled
studies address the update mechanism without assuming that the convergence
theorem describes the trained networks.
Primary comparisons use five seeds unless stated otherwise. Evaluation-matched
argmax and hard-MoE comparisons use three; theory-mode MAX-SAT uses two, and the
theory-mode ladder uses one. Ablations state their own sample counts and
selection protocol. Error bars denote one sample standard deviation.

\subsection{Experimental Setup}\label{sec:setup}

\paragraph{Datasets.}
We use MNIST \citep{lecun1998gradient} for the spiking, quantized,
staircase, and continuous MLP models. Argmax attention uses Fashion-MNIST,
and hard MoE combines both datasets into 20 classes by shifting the
Fashion-MNIST labels by 10. ETTh1 \citep{zhou2021informer} supplies hourly
oil-temperature readings for the time-series experiment.
For each run seed, we reserve a random 10\% of the official
training data for validation: 54K/6K/10K train/validation/test images for each
single dataset and 108K/12K/20K for hard MoE. Methods share the split within
each seed. Inputs are divided by 255, then normalized using $(0.1307,0.3081)$
for MNIST and $(0.2860,0.3530)$ for Fashion-MNIST (mean, standard deviation).
We use no image augmentation. The controlled weighting comparison uses the fixed
split specified in Section~\ref{sec:controlled-comparison}. Dataset details are
in Appendix~\ref{app:datasets}.

\paragraph{Models.}
The continuous MNIST model is a two-layer multilayer perceptron (MLP),
MNISTNet, with 102K trainable parameters. The hard-LIF accuracy benchmark
uses SpikingMNISTNet with $T_{\mathrm{seq}}=15$ timesteps. For memory
scaling, we use the larger VGG-11-style SNNVGG11Small with approximately
2.4M parameters and vary the sequence length. Appendix~\ref{app:architectures}
gives the architectures.

The accuracy network maps $784\to128\to10$ through two hard LIF layers.
At each of 15 steps, $\widetilde v_t=0.95v_{t-1}+I_t$,
$z_t=\mathbf{1}\{\widetilde v_t\geq1\}$, and
$v_t=\widetilde v_t(1-z_t)$, starting at zero. Its first-layer input current
is constant over time and the output counts how many spikes each output unit emits. These counts
are used directly as cross-entropy logits, without division by 15. INT8 uses widths
$784\to128\to128\to10$ with ReLU and quantizes the middle layer's
weights and biases to $0.01\operatorname{clip}(\operatorname{round}(w/0.01),-128,127)$.
Staircase uses the same widths with two activations $s(x)=\lfloor5\operatorname{sigmoid}(x)\rfloor/5$.
Argmax attention uses a 128-dimensional representation and eight slots.
Hard MoE uses a 128-dimensional representation, four experts with hard
argmax gating, and a 20-class output. The supervised classification objective
is mean cross-entropy on the current minibatch. Rounding uses nearest
even integers at ties; argmax selects the first maximizing index.

\paragraph{Baselines.}
We compare six gradient-free methods: CMA-ES \citep{hansen2016cma},
OpenAI-ES \citep{salimans2017evolution}, EGGROLL
\citep{sarkar2025eggroll}, MeZO \citep{malladi2023mezo}, SPSA
\citep{spall1992multivariate}, and subspace random search. These cover
covariance adaptation, population perturbations, finite-difference updates,
and random improvement steps. Subspace random search uses PolyStep's
representation but different candidates and weights. The controlled
comparison in Section~\ref{sec:controlled-comparison} holds the candidate
distribution fixed as well.
Adam \citep{kingma2015adam} is the gradient-based reference.
On the four hard image architectures with an Adam row, it trains a relaxed
forward model: sigmoid spikes for SNN, unquantized middle-layer weights for
INT8, softmax attention for argmax, and sigmoid activations for staircase.
We also evaluate the relaxed model on test. These references change the
forward computation; they do not evaluate surrogate-gradient or STE training
with hard operations retained in the forward pass. The blocked-gradient GPT-2 and RL controls retain their
specified hard operations and use no surrogate derivative or STE.

\paragraph{Protocol.}
For the primary supervised results, we train each method on the training split,
select a checkpoint by validation accuracy, and evaluate that checkpoint
once on test. PolyStep considers epoch-end checkpoints, retaining
the earlier one on an accuracy tie. The population baselines consider their
mean iterate and incumbent candidate at their scheduled validation checks;
we record these schedules with each run. We label separately the ablations
that instead report maximum test accuracy over epochs.
Accuracy differences are reported in percentage points (pp).
Reproducibility settings and limits across devices are described in
Appendix~\ref{app:reproducibility}.

\paragraph{Hyperparameters.}
Every primary PolyStep architecture comparison uses the softmax solver (Eq.~\ref{eq:softmax-rule}), with
$K{=}1$ probe per vertex on the simplex, at $d_p{=}8$ in subspace mode and $d_p{=}2$ in
full-space mode. Section~\ref{sec:ot-binding} reports separate comparisons with balanced
transport; their results do not establish a benefit from column coupling.
We select hyperparameters for each gradient-free method using a two-stage
validation grid, with nine configurations per stage. We match the number of
configurations; training budgets for each configuration follow the method's sweep protocol
(Appendix~\ref{app:hyperparameters}).
The second stage recenters each axis on the first stage's selection. If the
first-stage selection lies at a boundary, the second stage extends the search
range. Baselines start from reference defaults.

The baseline grids use a per-coordinate scale derived from the reference
PolyStep configuration and the baseline's search dimension. This conversion
sets a starting scale; it does not match the distributions of blockwise and
full-subspace perturbations. Each method then tunes its own scale.
Baseline scales are fixed, while several PolyStep configurations anneal
their radii or temperature.
Appendix~\ref{app:hyperparameters} gives the grids, selected configurations,
and scale-conversion rule.

Within an update, all candidate losses use the same minibatch. The image
experiments use batches of 512, including the smaller final batch of an epoch.
A candidate evaluation is one parameter vector scored on that batch, regardless
of whether several candidates share a GPU call. Candidate counts therefore
differ from processed-example counts. Validation and final test evaluation are
separate from the reported training-candidate budgets. The selected settings,
epoch counts, and reuse intervals are in Table~\ref{tab:hyperparameters};
the test tables use the validation-selected checkpoint.

\paragraph{Compute and statistical testing.}
The supervised runs use one NVIDIA RTX 5090 (32\,GB), with software versions recorded
in each result file. Unless a comparison states otherwise, the seeds are
$\{42,123,456,789,1337\}$. Pairwise comparisons use an exact two-sided
sign-flip permutation test over all $2^5$ sign assignments to the paired
seed differences, assuming independent differences symmetric about zero
under the null. Its smallest possible $p$-value is $2/2^5=0.0625$,
so five seeds cannot establish significance at the 5\% level with this test.
The tests are not adjusted for multiple comparisons across benchmarks
(Appendix~\ref{app:stats}).

\subsection{Comparison Budgets}\label{sec:matching-axis}

The method ranking depends on whether the budget counts optimizer updates,
candidate evaluations, or elapsed time. Each table states its convention.
One \emph{optimizer step} is one sequential update, or one generation for a
population method. Batching can parallelize candidates within that step but
cannot remove the dependency on the previous update.
For the \emph{step-based budget}, let $N$ be PolyStep's update count:
$3{,}180$ for the image benchmarks and $6{,}330$ for hard MoE.
Population baselines receive $32N$ candidate evaluations, using a reference
population of $32$. MeZO and SPSA receive $2N$ evaluations; random search
receives $N$, including its initial candidate. A population method with
selected population size $M$ therefore takes approximately $32N/M$
generations. This budget matches $N$ updates only when $M=32$.
Table~\ref{tab:step-sweep} applies this convention across all $36$
architecture-baseline pairings.

The per-architecture tables use the following budgets. Population methods
are matched by evaluations on SNN, MNIST, INT8, and staircase, and use the
step-based budget on argmax and hard MoE. Finite-difference methods receive
$N$ updates throughout.
At two evaluations per sequential step, spending the full MNIST evaluation
budget would require $5{,}094{,}360$ such updates. Each row states which
budget it uses.

Matching \emph{candidate evaluations} permits different numbers of updates.
At MNIST's matched-evaluation budget, OpenAI-ES takes \mnistStepRatio{} as
many updates as PolyStep. For PolyStep, the available step count decreases
as $1/d_{\mathrm{sub}}$ at fixed evaluation budget.
Section~\ref{sec:other-axes} gives the evaluation-matched comparisons,
including the argmax and $5{,}000$-variable MAX-SAT reversals.

\emph{Wall-clock time} measures the elapsed training time and depends on
batching and hardware. In the reported profiling, population $32$ has about
$90\times$ lower throughput than a tuned population. The timings and budget
definitions are in Appendix~\ref{app:wallclock-matched}.
Within a table, gradient-free methods share the subspace representation,
minibatch stream, and tuning-grid shape, except that EGGROLL uses its own
factored subspace to preserve layer shapes. We test EGGROLL in selected
settings, and cannot separate the effects of its optimizer and representation
in these comparisons. In the weighting control of
Section~\ref{sec:controlled-comparison}, we use the same representation and
candidate distribution for all four rules.

Table~\ref{tab:step-sweep} reports \emph{validation} accuracy, since a
trajectory truncated at the step-based budget has no test score. The
per-architecture tables report \emph{test} accuracy at the validation-selected
checkpoint. Where an evaluation-matched comparison uses three seeds, the
\sysname{} summary uses the same three seeds.

\subsection{The Benchmark Set}\label{sec:gallery-overview}

The five hard architectures place discontinuities in different operators
(Table~\ref{tab:gallery-overview}). A hard operator can have zero derivative
between thresholds while the complete network still varies through other paths.
Surrogate-gradient and quantization-aware training provide alternative backward
rules; they need not change the hard forward model. MAX-SAT supplies a fully
discrete objective through its clause count.

A threshold crossing need not change the batch loss. The measured nonzero jumps
in Table~\ref{tab:gallery-overview} establish examples, rather than a uniform
lower bound across all walls. In the INT8 diagnostic, the batch jump lies
$13$--$300\times$ above the simple $J=j/n_B$ prediction over a $64\times$ range
of batch size. That prediction assumes that only one example changes, with
individual loss change $j$ in a batch of $n_B$ examples. Appendix~\ref{app:discontinuity-sets} derives the local threshold
geometry, and Appendix~\ref{app:measured-jump} gives the measurement protocol.

\begin{table}[ht]
\centering
\caption{Local threshold geometry and measured loss changes. The zero fraction counts resolved crossings with no loss change; median $J$ is computed over nonzero changes (Appendix~\ref{app:measured-jump}). These measurements do not certify a uniform jump bound on every threshold sheet.}
\label{tab:gallery-overview}
\small
\begin{tabular}{lllccc}
\toprule
Benchmark & Thresholds / loss structure & Geometry & zero frac. & median $J$ & Section \\
\midrule
SNN (hard LIF spikes)            & spike-threshold surfaces & App.~\ref{app:discontinuity-sets}(a) & $0.69$ & $4.1{\times}10^{-7}$ & \ref{sec:nondiff} \\
INT8 quantization                & quantization-grid walls     & App.~\ref{app:discontinuity-sets}(b) & $0.11$ & $1.6{\times}10^{-5}$ & \ref{sec:nondiff} \\
Hard MoE (argmax routing)        & gate-score ties               & App.~\ref{app:discontinuity-sets}(c) & $0.00$ & $1.2{\times}10^{-3}$ & \ref{sec:nondiff} \\
Staircase activations            & sigmoid level sets          & App.~\ref{app:discontinuity-sets}(d) & $0.00$ & $2.9{\times}10^{-6}$ & \ref{sec:nondiff} \\
Argmax attention                 & score-tie surfaces       & App.~\ref{app:discontinuity-sets}(e) & $0.00$ & $1.5{\times}10^{-5}$ & \ref{sec:nondiff} \\
MAX-SAT 1M variables             & piecewise-constant          & Cor.~\ref{cor:piecewise-constant} & -- & -- & \ref{sec:maxsat} \\
\bottomrule
\end{tabular}
\end{table}

\subsection{Non-Differentiable Model Training}\label{sec:nondiff}

\begin{figure}[t]
\centering
\includegraphics[width=\linewidth,keepaspectratio]{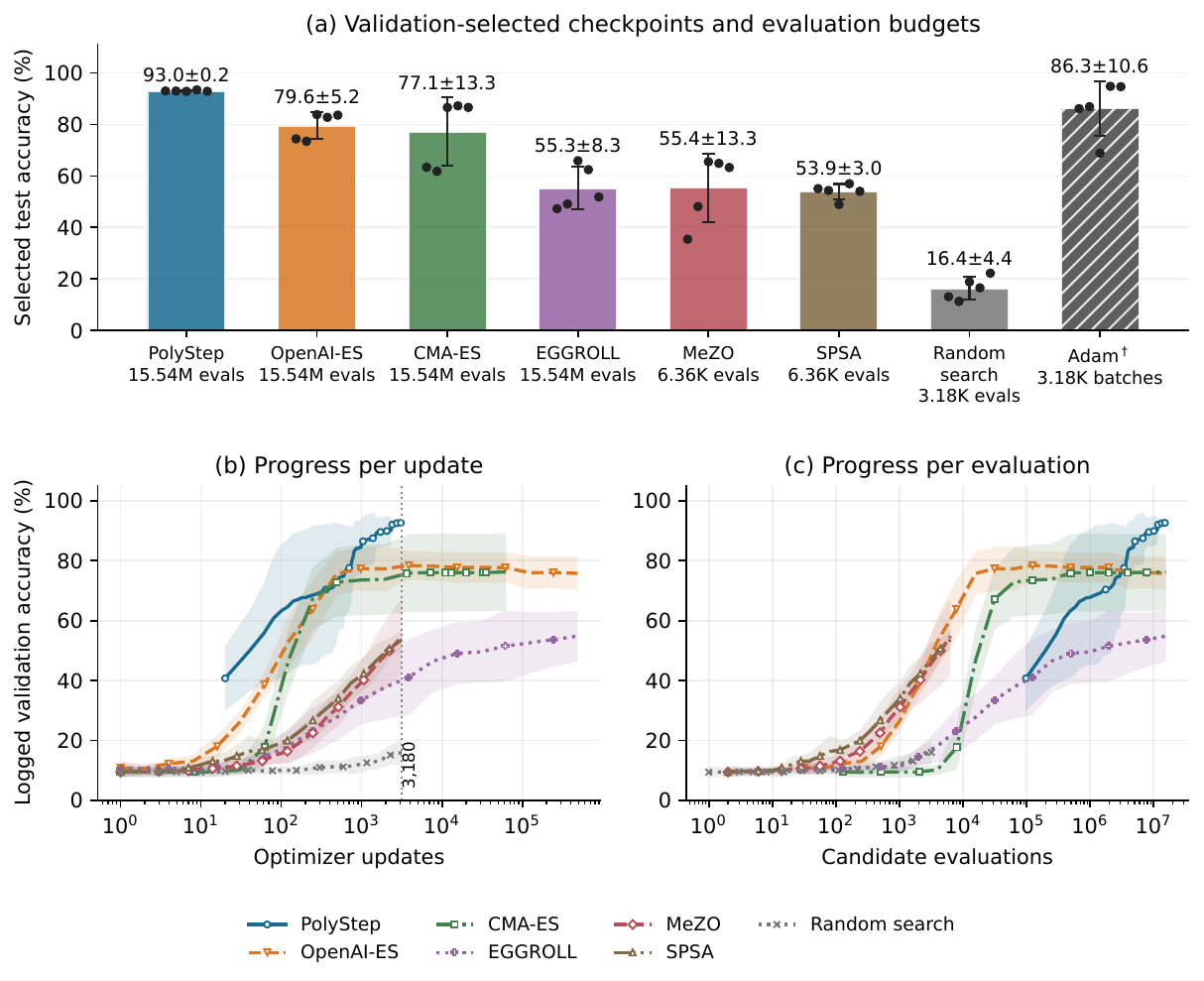}
\caption{Hard-LIF SNN on MNIST (five seeds, mean $\pm$ sample std).
(a)~Test accuracy at the validation-selected checkpoint; dots show individual
seeds and labels give evaluation budgets. Adam$^\dagger$ trains a smooth
forward relaxation. (b,c)~Logged validation accuracy for all seven gradient-free
methods against optimizer updates and candidate evaluations. Curves linearly
interpolate logs within the five seeds' shared support. The dotted line marks
$3{,}180$ updates. Population methods receive $15.54$M evaluations;
MeZO, SPSA, and random search use the smaller budgets shown in (a).}
\label{fig:primary-summary}
\end{figure}

Hard spike thresholds make SNNs a direct test of optimization from scalar
losses. A forward-only method can also train through a hardware interface
that supplies losses but no derivatives. We use SpikingMNISTNet with hard
LIF thresholds and $T_{\mathrm{seq}}=15$ timesteps on MNIST.

\begin{table}[t]
\caption{SNN test accuracy (\%, MNIST, hard LIF, 5 seeds, mean $\pm$ std at the validation-selected checkpoint); best gradient-free result in bold. Population baselines are matched on \emph{evaluations}, so their step counts follow their populations; the finite-difference baselines receive $3{,}180$ updates. The ordering is unchanged under the step-based budget (Table~\ref{tab:axes}). Adam$^\dagger$ trains a smooth forward relaxation of the same architecture.}
\label{tab:snn-accuracy}
\centering
\small
\begin{threeparttable}
\begin{tabular}{lrrS[table-format=2.1(3)]}
\toprule
{Method} & {Steps} & {Evaluations} & {Test accuracy (\%)} \\
\midrule
\sysname{} & 3{,}180 & 15.54M & \bfseries 93.0 \pm 0.2 \\
OpenAI-ES & 485{,}645 & 15.54M$^{e}$ & 79.6 \pm 5.2 \\
CMA-ES & 64{,}684 & 15.54M$^{e}$ & 77.1 \pm 13.3 \\
EGGROLL & 485{,}645 & 15.54M$^{e}$ & 55.3 \pm 8.3 \\
MeZO & 3{,}180 & 6.36K$^{s}$ & 55.4 \pm 13.3 \\
SPSA & 3{,}180 & 6.36K$^{s}$ & 53.9 \pm 3.0 \\
Random search & 3{,}179 & 3.18K$^{s}$ & 16.4 \pm 4.4 \\
Adam$^\dagger$ & 3{,}180 & 3.18K & 86.3 \pm 10.6 \\
\bottomrule
\end{tabular}
\begin{tablenotes}[para,flushleft]
\scriptsize\raggedright
$^{e}$Evaluation-matched. $^{s}$Step-based allowance: $2N$ evaluations for MeZO/SPSA and $N$ for random search, with $N=3{,}180$; random search counts its initial candidate. $^\dagger$Adam trains and tests a relaxed forward model; its count is forward--backward passes.
\end{tablenotes}
\end{threeparttable}
\end{table}

\sysname{} leads the best tuned gradient-free baseline on every seed
(Table~\ref{tab:snn-accuracy}, Figure~\ref{fig:primary-summary}), giving
$p=0.0625$, the smallest value possible with five paired seeds.
The five-seed per-architecture tables also show a lead on each seed for
MNIST, INT8, staircase, argmax attention, and hard MoE under their stated
budget conventions.

Figure~\ref{fig:primary-summary} shows the selected test scores alongside
progress per update and per evaluation. \sysname{} reaches its reported accuracy in $3{,}180$ steps,
while the evaluation-matched population baselines take many more updates.
OpenAI-ES and CMA-ES make faster early progress per evaluation. The final
accuracy advantage therefore does not imply uniform superiority along the
learning curve or lower evaluation cost. MeZO and SPSA improve under their
much smaller evaluation budgets, but finish with lower selected test accuracy.

\paragraph{The smooth-spike reference.}
Adam trains by BPTT after replacing each hard spike with
$z_t=\operatorname{sigmoid}(10(\widetilde v_t-1))$ in the forward pass,
retaining the reset $v_t=\widetilde v_t(1-z_t)$. Its mean test accuracy is
\snnAdamAcc{}, compared with PolyStep's \snnBestAcc{} on the hard-LIF model.
A separate slope sweep spans
$20$~pp across $\alpha\in[0.5,10]$, but it does not isolate the cause of that
variation. SLTT~\citep{meng2023towards} and OTTT~\citep{xiao2022online}
provide memory-efficient surrogate-gradient training on other architectures
and datasets. Our experiment evaluates the stated architecture and tuning
protocol, so it does not establish a general advantage over these methods.
The selected \sysname{}
configuration uses fixed $\varepsilon=0.5$, $r_s=2.0$, and rotations biased
toward the previous update direction. These settings differ from the independent rotations and decreasing
step sizes used in the infinite-schedule analysis.

\paragraph{The four remaining architectures.}
The other architectures place hard operations in quantized weights,
staircase activations, expert routing, and attention selection. They test
the method on different sources of threshold crossings.
PolyStep trains all four (Appendix~\ref{app:gallery-tables}). On INT8,
argmax attention, and staircase, the Adam$^\dagger$ reference on the
relaxed forward model exceeds PolyStep by $0.8$--$3.2$~pp. The SNN comparison favors
PolyStep. Binary and ternary weight networks are not measured under this protocol.

\paragraph{Quantized head, pretrained backbone.}
To study quantization with fixed pretrained features, we freeze the
124M-parameter GPT-2 backbone and train a $1{,}538$-parameter classification
head on the SST-2 sentiment task. The head is either quantized or left smooth, with every gradient-free method tuned by
the same nine-configuration validation sweep at a matched $500$K-evaluation budget
(Appendix~\ref{app:gpt2-headquant}). The ordering reverses with the quantizer
(Table~\ref{tab:gpt2-headquant}).

The head receives cached final-token features from the frozen backbone.
Of 5000 SST-2 training examples, a seed-0 permutation assigns 4500 to
updates and 500 to validation; the official 872-example validation set is
used as the held-out test set because the official test labels are unavailable.
The nine-configuration sweep uses 100K candidate evaluations per setting;
the selected configuration then receives 500K evaluations on each of five
seeds. Validation selection considers both the mean iterate and the best
sampled candidate. INT8 quantizes weight and bias at scale 0.01, binary
quantizes only weight by its sign, and the smooth control quantizes neither.
The selected head is tested once. The separate all-parameter and
head-only GPT-2 experiments in the appendix use a different selection rule
and are not pooled with this comparison.

On the smooth control, MeZO and EGGROLL lead PolyStep. A MeZO update costs
$2$ evaluations against PolyStep's approximately $2{,}300$, permitting many more
updates at the same evaluation budget. Quantization changes the ordering, but this
comparison does not isolate estimator quality from update count or tuning.
MeZO's INT8 sweep selects a radius spanning about one hundred quantization bins;
such finite-radius differences can still estimate a smoothed gradient.
Adam's gradient through the INT8 head is zero. In the binary head, the two bias
parameters remain differentiable. The flat-pair diagnostic of
Appendix~\ref{app:flat-pair} uses this head.

\begin{table}[t]
\centering
\caption{GPT-2 SST-2 head fine-tuning through a quantizer (frozen 124M backbone, 1,538-parameter head; test accuracy \%, 5 seeds, mean $\pm$ std at the validation-selected checkpoint). Gradient-free methods are matched at $500$K evaluations and tuned by the same nine-configuration validation sweep; best gradient-free result per column in bold. Adam's gradient through the INT8 head is exactly zero; the binary bias (2 of 1,538 parameters) stays differentiable, a partial control.}
\label{tab:gpt2-headquant}
\small
\begin{tabular}{lccc}
\toprule
Method & INT8 & Binary & Smooth control \\
\midrule
\sysname{} & $67.1 \pm 6.9$ & $\mathbf{72.7 \pm 3.3}$ & $51.2 \pm 1.4$ \\
MeZO & $\mathbf{67.3 \pm 10.9}$ & $52.0 \pm 3.5$ & $62.6 \pm 2.5$ \\
EGGROLL & $65.8 \pm 2.5$ & $71.7 \pm 4.0$ & $\mathbf{63.7 \pm 1.0}$ \\
\midrule
Adam & $49.8 \pm 1.1$ & $50.1 \pm 1.1$ & $81.6 \pm 0.6$ \\
\bottomrule
\end{tabular}
\end{table}

\begin{table}[t]
\caption{Validation accuracy (\%, five-seed means) under the step-based budgets of Section~\ref{sec:matching-axis}; highest mean per column in bold. For PolyStep's $N$ updates ($3{,}180$; $6{,}330$ on hard MoE), population baselines receive $32N$ evaluations, MeZO and SPSA receive $2N$, and random search receives $N$. Population sizes vary, so generation counts need not match. Per-architecture tables report test means and standard deviations at their stated budgets. On MNIST, PolyStep evaluates fresh probes every third update and reuses a direction between them (Section~\ref{sec:ablation}).}
\label{tab:step-sweep}
\centering
\small
\begin{tabular}{lcccccc}
\toprule
Method & SNN & MNIST & INT8 & Stair & Argmax & MoE \\
\midrule
\sysname{} & $\mathbf{92.9}$ & $\mathbf{96.7}$ & $\mathbf{97.0}$ & $\mathbf{94.3}$ & $\mathbf{87.8}$ & $\mathbf{91.1}$ \\
OpenAI-ES & $77.6$ & $86.2$ & $87.6$ & $86.2$ & $80.7$ & $82.6$ \\
CMA-ES & $72.8$ & $89.7$ & $88.5$ & $74.5$ & $80.2$ & $77.4$ \\
EGGROLL & $37.8$ & $73.2$ & $72.5$ & $41.6$ & $71.4$ & $51.7$ \\
MeZO & $54.8$ & $82.5$ & $79.9$ & $60.9$ & $72.4$ & $54.9$ \\
SPSA & $53.6$ & $21.9$ & $81.6$ & $44.4$ & $67.4$ & $25.6$ \\
Random search & $16.1$ & $9.2$ & $29.0$ & $17.9$ & $43.0$ & $13.5$ \\
\bottomrule
\end{tabular}
\end{table}

\paragraph{Step-based budgets.}
Table~\ref{tab:step-sweep} places all six architectures and all six gradient-free
baselines under the step-based budget. A run with that budget keeps its recorded
selection; a longer run contributes its best logged validation point within
the allowance. PolyStep's $N$ updates can use many more evaluations than
the baselines' allowances, so these entries do not compare evaluation
efficiency. Sparse-expert systems
\citep{fedus2022switch,lepikhin2021gshard,zhou2022expertchoice} use different
routing and training recipes. The present benchmark does not establish a
state-of-the-art comparison with those systems.

\subsubsection{The Evaluation Axis}\label{sec:other-axes}

Evaluation matching preserves PolyStep's lead over the compared population
baselines on SNN, MNIST, INT8, and staircase. OpenAI-ES leads on argmax
attention by $0.54$ percentage points; the hard-MoE means differ by $0.03$ points.

On SNN, MNIST, INT8, and staircase, the tested baselines attain their best
validation scores between $0.2\%$ and $25\%$ of their evaluation budgets.
Against the closest baseline in each comparison, PolyStep's validation
margins are $2.5$ to $14.0$~pp at matched evaluations and $7.0$ to
$15.3$~pp under step-based budgets (Table~\ref{tab:axes}). These trajectories describe
the tested settings; they do not exclude gains from different population
sizes or schedules.

\emph{OpenAI-ES leads on argmax attention.} Given \sysname{}'s $42.01$M evaluations
and population $2{,}048$ with a validation-selected learning rate,
OpenAI-ES takes $20{,}514$ steps and reaches $86.95\%$ test over seeds $42/123/456$
against \sysname{}'s $86.41\%$ on the same three seeds, leading on each
paired seed. The recorded runtime is about one fifth as long, with other jobs sharing
the device. The matched-step margins of $6.6$~pp on test
and $7.1$~pp on validation therefore do not imply an evaluation-matched
advantage.

\emph{Hard MoE differs by $0.03$ percentage points.} At the full $100.55$M
evaluations, OpenAI-ES reaches $90.52\pm0.11\%$ against \sysname{}'s
$90.49\pm0.13\%$ on the same three seeds. The step-matched comparison gives PolyStep $496\times$ as many
evaluations. The evaluation-matched difference is smaller than either
method's sample standard deviation; this does not establish equivalence.

\emph{Reducing the argmax search rank does not improve these results.} Subspace rank controls PolyStep's tradeoff between search dimension and
step count at a fixed budget. Accuracy increases with rank in the three
tested settings: halving the subspace rank permits $1.9\times$ as many
steps but lowers accuracy by $1.0$~pp; quartering it permits $3.7\times$ as
many steps but lowers accuracy by $2.4$~pp
(Appendix~\ref{app:wallclock-matched}). The reported configuration is the best of the
three. For these settings, the additional updates do not compensate for the
smaller search space.

\paragraph{Memory scaling.}\label{sec:snn-memory}
Forward-only evaluation avoids retaining the activation graph used by
ordinary BPTT. On SNNVGG11Small at batch size $4$, the measured forward-only
peak rises from $31.8$ to $51.6$\,MiB as $T_{\mathrm{seq}}$ increases from
$25$ to $400$; forward plus BPTT rises from $132.3$ to $1{,}538.2$\,MiB
(Figure~\ref{fig:scalability-plate}b). The ratio is $29.8$ at the longest
horizon. These measurements isolate evaluation memory. They exclude a
complete \sysname{} step's candidate and assignment buffers and do not
establish an asymptotic memory law.

\subsection{Discrete Optimization at Scale: MAX-SAT}\label{sec:maxsat}

MAX-SAT seeks a Boolean assignment that satisfies as many logical clauses
as possible. We use clauses with three literals and test the optimizer as
the number of variables increases. We generate one instance per size with
seed 42, three distinct variables per clause, independently sampled literal
signs, and clause-to-variable ratio approximately $4.27$. The 100- and
500-variable instances contain 426 and 2,135 clauses. The optimized loss
is the fraction of unsatisfied clauses under the assignment
$z_j=\operatorname{round}(\operatorname{sigmoid}(x_j))$; no continuous
penalty is added. At $10^6$ variables the instance has about $4.27$M clauses,
our largest non-differentiable run.
PolyStep reports its final iterate; CMA-ES and OpenAI-ES return the
best candidate observed within their evaluation budgets.
Appendix~\ref{app:maxsat-configurations} gives the configurations and evaluation counts.

\begin{table}[t]
\centering
\caption{Random MAX-3-SAT, $\alpha\approx4.27$, final clauses satisfied
(\%). One fixed instance per size; general-purpose methods use five
optimizer seeds (mean $\pm$ sample std), and domain-specialized
references use one run.
\sysname{} and OpenAI-ES use $1{,}000$ updates below $10^6$ variables
and $500$ at $10^6$. CMA-ES receives OpenAI-ES's evaluation budget
($50{,}000$ below $10^6$), so its update count varies; it uses diagonal
covariance from $1{,}000$ variables. Matching \sysname{}'s evaluations
at $5{,}000$ variables reverses the ordering
(Figure~\ref{fig:maxsat-scaling-full}d). Bold marks the highest mean in
each row within each method group.}
\label{tab:maxsat-scaling}
\small
\begin{tabular}{lccc@{\hspace{1.2em}}ccc}
\toprule
 & \multicolumn{3}{c}{\textbf{General-purpose GFO}} & \multicolumn{3}{c}{\textbf{Domain-specialized}} \\
\cmidrule(lr){2-4} \cmidrule(lr){5-7}
Variables & \sysname{} & CMA-ES & OpenAI-ES & probSAT$^{\ddagger}$ & SLS & RC2$^\dagger$ \\
\midrule
100         & $98.50 \pm 0.36$ & $99.01 \pm 0.35$ & $\mathbf{99.39 \pm 0.13}$ & $\mathbf{99.8}$ & $\mathbf{99.8}$ & $\mathbf{99.8}$ \\
500         & $\mathbf{98.74 \pm 0.18}$ & $98.43 \pm 0.25$ & $98.15 \pm 0.15$ & $\mathbf{100.0}$ & $99.9$ & timeout \\
1{,}000     & $\mathbf{98.1 \pm 0.2}$ & $97.3 \pm 0.1$ & $96.9 \pm 0.2$ & $\mathbf{100.0}$ & $99.9$ & timeout \\
5{,}000     & $\mathbf{98.2 \pm 0.1}$ & $95.2 \pm 0.2$ & $92.9 \pm 0.2$ & $\mathbf{99.9}$ & $99.6$ & timeout \\
20{,}000    & $\mathbf{98.2 \pm 0.0}$ & $92.9 \pm 0.2$ & $90.5 \pm 0.1$ & $\mathbf{99.8}$ & $99.1$ & timeout \\
100{,}000   & $\mathbf{98.1 \pm 0.01}$ & $90.7 \pm 0.1$ & $88.9 \pm 0.0$ & $\mathbf{99.6}$ & $97.4$ & timeout \\
$10^6$$^{\S}$ & $\mathbf{92.6 \pm 0.02}$ & -- & $87.8 \pm 0.0$ & $\mathbf{98.9}$ & $89.4$ & timeout \\
\bottomrule
\end{tabular}
\vspace{0.3em}

{\raggedright\scriptsize $^\dagger$RC2 reaches its time limit for
$n\ge500$. $^{\ddagger}$probSAT uses ten restarts, up to $10^7$ flips
($5\times10^6$ at $10^6$ variables), with a wall-time allowance tied to
\sysname{}; it finishes in $13$--$89$ seconds.
$^{\S}$The million-variable results use a separate configuration, with incremental evaluation for \sysname{}. CMA-ES was not
run at that size.\par}
\end{table}

This experiment tests how forward-only optimization scales with dimension.
The specialized solver probSAT~\citep{balint2013probsat} leads at every
scale. A uniformly random assignment satisfies $7/8=87.5\%$ of clauses in
expectation \citep{johnson1974approximation}. This gives a reference for
measuring progress as the number of variables grows and individual updates
change a smaller fraction of clauses.

At 100 variables, OpenAI-ES reaches $99.39\%$ against PolyStep's
$98.50\%$. At 500 variables, PolyStep reaches $98.74\%$ against
CMA-ES's $98.43\%$, under the stated unequal evaluation and
configuration-selection budgets. From $1{,}000$ to $100{,}000$ variables,
PolyStep remains at $98.1$--$98.2\%$, while CMA-ES decreases from
$97.3\%$ to $90.7\%$ and OpenAI-ES from $96.9\%$ to $88.9\%$
(Table~\ref{tab:maxsat-scaling}).
The separate million-variable configuration reaches $92.6\%$.
The budget matters: at $5{,}000$ variables, OpenAI-ES reaches $98.96\%$
against \sysname{}'s $98.15\%$ when given \sysname{}'s evaluation
budget (Figure~\ref{fig:maxsat-scaling-full}d). CMA-ES uses diagonal covariance from $1{,}000$ variables and
is not evaluated at $10^6$; limited-memory variants
\citep{loshchilov2017lmcma,akimoto2020ddcma} are outside this comparison.
The spread measures optimizer variation on a fixed instance, not
variation across new instances; it does not identify the mechanism
behind the observed accuracy.

The million-variable solve averages $1{,}216$ seconds on one RTX~5090
with $2{,}310.5$\,MiB peak GPU memory. Its incremental evaluator and
$500$-update budget differ from the $100{,}000$-variable runs, which
average $9{,}236$ seconds over $1{,}000$ updates with the full evaluator
(Figure~\ref{fig:maxsat-scaling-full}, Appendix~\ref{app:scalability}).
These timings describe the specified implementations and budgets.

\subsection{Theory Mode}\label{sec:theory-mode}

We next compare the selected configurations with a variant based on the
infinite-schedule analysis of Theorem~\ref{thm:piecewise-smooth}. We call
this variant \emph{theory mode}. It uses fixed $\varepsilon$, a step
multiplier $r_{s,t} = r_0 (t+1)^{-0.6}$ so that
$\sum r_{s,t} = \infty$ and $\sum r_{s,t}^2 < \infty$, independently sampled
rotations, jitter targeting a mollifier law, the orthoplex
frame, HybridSubspace, and none of the acceleration heuristics of
Appendix~\ref{app:turbo-mode}. We retain all other selected settings, so the
comparison measures the combined effect of these changes. It does not show
that satisfying the theorem necessarily reduces performance. Even theory
mode retains sparse bases that are not orthonormal and epoch-shuffled
minibatches that are not conditionally unbiased given the sampling history
(Appendix~\ref{sec:hypothesis-coverage}).
The ideal jitter density meets condition~(iv), but the numerical samplers
approximate this density without satisfying its regularity requirement.
We distinguish the target density from the implemented sampler in
Table~\ref{tab:hypothesis-coverage}.
The change to the orthoplex is optional under condition~(i), which also
admits the simplex we use in the main experiments. It increases the number of
evaluations per step (Section~\ref{sec:ablation}).

\begin{table}[t]
\centering
\caption{Theory mode versus the selected experimental configuration. Theory mode uses the prescribed frame, rotations, target jitter distribution, and schedule, but does not satisfy all the theorem's assumptions (Appendix~\ref{sec:hypothesis-coverage}). SNN and MNIST use 5 seeds and report test accuracy at the validation-selected checkpoint. Theory-mode MAX-SAT uses 2 seeds (42 and 123) and reports final clause satisfaction; its comparison uses 5 seeds. The accuracy gaps reflect combined configuration changes, not the effect of each assumption separately.}
\label{tab:theory-mode}
\small
\begin{tabular}{@{}lccc@{}}
\toprule
Benchmark & Theory mode & Tuned & Gap \\
\midrule
SNN (hard LIF) & $54.7 \pm 16.0$\% & \snnBestAcc{} & $38.3$~pp \\
MNIST          & $91.7 \pm 0.3$\% & \mnistBestAcc{} & $5.1$~pp \\
MAX-SAT (100K) & $96.9 \pm 0.1$\% & \maxsatHundredKAcc{} & $1.2$~pp \\
\bottomrule
\end{tabular}
\end{table}

\paragraph{Individual configuration changes.}
Theory mode changes five settings at once, so Table~\ref{tab:theory-ladder} switches the
four listed settings one at a time in both directions. In this single-seed comparison, the largest change is associated with
the decreasing step multiplier in Theorem~\ref{thm:piecewise-smooth}. The other changes are smaller. The fifth change, disabling the acceleration heuristics, has no separate row because probe jitter disables direction reuse
(Section~\ref{sec:ablation}). Its effect is included in the jitter rows
and the remaining difference.
The effects depend on the starting configuration and do not sum to the
full difference. The ladder's seed-42 theory-mode result lies within the
five-seed spread of Table~\ref{tab:theory-mode}.

\begin{table}[ht]
\centering
\caption{Effects of individual theory-mode configuration changes (SNN, seed 42, test accuracy at the validation-selected checkpoint). Requirements refer to Assumption~\ref{asm:main} or the indicated theorem. Upper block starts from the tuned configuration and changes one setting; lower block starts from theory mode and reverses one change. Each $\Delta$ is against the first row of its own block, not against the row above.}
\label{tab:theory-ladder}
\small
\begin{tabular}{llcc}
\toprule
Configuration & Requirement & Test acc. (\%) & $\Delta$ from reference (pp) \\
\midrule
tuned configuration &  & $92.83$ &  \\
$+$ radius jitter & (iv), target law & $93.28$ & $+0.45$ \\
$+$ orthoplex frame & (i) & $93.26$ & $+0.43$ \\
$+$ independent rotations & (ii) & $92.97$ & $+0.14$ \\
$+$ decreasing step multiplier & Thm.~\ref{thm:piecewise-smooth} & $67.96$ & $-24.87$ \\
\midrule
theory mode &  & $65.40$ &  \\
$-$ radius jitter & (iv), target law & $67.28$ & $+1.88$ \\
$-$ orthoplex frame & (i) & $63.26$ & $-2.14$ \\
$-$ independent rotations & (ii) & $64.83$ & $-0.57$ \\
$-$ decreasing step multiplier & Thm.~\ref{thm:piecewise-smooth} & $93.61$ & $+28.21$ \\
\bottomrule
\end{tabular}
\vspace{0.3em}

{\raggedright\scriptsize Conditions~(iii), (vi), (vii), and~(viii) are not varied
in this table. Both configurations retain the sparse map that fails (iii).
Compactness and conditional saturation control remain trajectory assumptions;
(vii)'s transverse derivative bound is an assumption on the loss
(Table~\ref{tab:hypothesis-coverage}).\par}
\end{table}

\paragraph{Finite-horizon behavior.}
Over $3{,}180$ steps, the schedule $r_{s,t}=r_0(t+1)^{-0.6}$ reduces the
step multiplier by a factor of $126$. The displacement upper bound shrinks
with it; the lower envelope in Proposition~\ref{prop:fragility} additionally
requires concentration. In a seed-42 check, multiplying $r_0$ by
$2$, $4$, $10$, and $100$ gives $56.2\%$, $47.6\%$, $9.8\%$, and $20.8\%$,
respectively, compared with $65.4\%$ at the ladder's original $r_0$.
These settings do not recover the tuned result. They do not rule out other
schedule choices or isolate decay as the sole cause of the gap.
The horizon-dependent constant step multiplier of
Theorem~\ref{cor:finite-horizon} is a distinct theoretical schedule.

\subsection{MNIST}\label{sec:vision}

The MNIST ReLU MLP is a continuous, piecewise-smooth control without hard
thresholds or quantization. Its gradients are available almost everywhere.

\begin{table}[t]
\caption{MNIST test accuracy (\%, 5 seeds, mean $\pm$ std at the validation-selected checkpoint), with both cost axes and the axis each row was budgeted on; best gradient-free result in bold.}
\label{tab:mnist-accuracy}
\centering
\small
\begin{threeparttable}
\begin{tabular}{lrrS[table-format=2.1(2)]}
\toprule
{Method} & {Steps} & {Evaluations} & {Test accuracy (\%)} \\
\midrule
\sysname{} & 3{,}180 & 10.19M & \bfseries 96.8 \pm 0.1 \\
OpenAI-ES & 318{,}397 & 10.19M$^{e}$ & 88.0 \pm 0.4 \\
CMA-ES & 128{,}416 & 10.19M$^{e}$ & 94.4 \pm 0.2 \\
EGGROLL & 221{,}493 & 10.19M$^{e}$ & 78.0 \pm 0.8 \\
MeZO & 3{,}180 & 6.36K$^{s}$ & 83.3 \pm 0.9 \\
SPSA & 3{,}180 & 6.36K$^{s}$ & 21.3 \pm 3.5 \\
Random search & 3{,}179 & 3.18K$^{s}$ & 9.1 \pm 2.2 \\
Adam$^\dagger$ & 3{,}180 & 3.18K & 97.7 \pm 0.1 \\
\bottomrule
\end{tabular}
\begin{tablenotes}[para,flushleft]
\scriptsize\raggedright
$^{e}$Evaluation-matched. $^{s}$Step-based allowance: $2N$ evaluations for MeZO/SPSA and $N$ for random search, with $N=3{,}180$; random search counts its initial candidate. $^\dagger$Adam uses exact gradients of the same model; its count is forward--backward passes.
\end{tablenotes}
\end{threeparttable}
\end{table}

\sysname{} trails Adam by $0.9$~pp (Table~\ref{tab:mnist-accuracy}). It uses
a cosine temperature schedule, subspace rank 8, and no momentum. It computes
a fresh direction every three steps and reuses the stored displacement
between solves (Appendix~\ref{app:turbo-mode}). Adding momentum reduces accuracy by $0.9$~pp in the
reported ablation. All gradient-free methods finish below Adam in this
comparison; the experiment does not isolate the cause of that gap.
The primary result uses the $30$-epoch protocol above. Other MNIST
experiments change configurations or selection rules and are reported separately.

\subsection{Reinforcement-Learning Policy Search}\label{sec:rl-policy-search}

Forward-only policy search uses scalar episodic returns and can train
policies with hard quantization. On CartPole-v1 and Acrobot-v1,
\sysname{} reaches $0.995$--$1.000$ normalized final return across Float32,
INT8, and binary hidden activations. It leads OpenAI-ES under a shared
$200$-update budget, with unequal rollout costs. The normalized metric
clips at one, so it can conceal return differences near the reference.
Proximal policy optimization (PPO) and deep Q-learning (DQN) provide differentiable controls and additional diagnostics
with hidden and output quantization, which blocks their gradients without
a straight-through estimator. The different quantization placement limits
comparisons between these controls and the forward-only methods.
Appendix~\ref{app:rl-policy-search} gives the curves, costs, and protocol;
Proposition~\ref{prop:rl-policy-search} controls the effect of rollout
noise on assignment weights.

Both forward-only policies have one hidden layer of width 16. Float32 uses
$\tanh$; INT8 replaces it by rounding the hidden preactivation to a grid
with scale $s=\max\{\max|x|,10^{-8}\}/127$, computed per candidate across
its rollout batch, and binary uses the preactivation's sign. The final policy is evaluated
without validation-based checkpoint selection. PPO/DQN instead quantize all
hidden activations and their outputs without STE, use additional
algorithm-specific networks, and receive 1M training transitions on CartPole
and 0.5M on Acrobot. Curves retain raw returns; the table clips normalized
scores using $(R-22)/478$ and $(R+500)/420$, respectively. The reported
rollout counters differ: PolyStep counts realized transitions, ES includes
allocated but unused horizon slots, and PPO/DQN include final evaluation
allowances. They are not a common sample-efficiency measure.

\subsection{LSTM Time-Series Forecasting}\label{sec:lstm-timeseries}

ETTh1 provides a smooth regression control with exact gradients. We train a
$23$K-parameter long short-term memory (LSTM) network with a $96$-step
lookback and forecast horizon
(Table~\ref{tab:lstm-timeseries}, Appendix~\ref{app:lstm-details}).
PolyStep and Adam both have mean validation MSE $0.109$. Their test MSEs
are \timeseriesBestMse{} and $0.247\pm0.023$, respectively, compared with
EGGROLL's $0.240\pm0.007$ and persistence's $0.069$.
Every learned method trails persistence, which predicts the last observed
value at each future step. A level shift
between train and test may contribute, but this experiment does not isolate
its cause. PolyStep uses a hand-selected configuration and the baselines use
reference defaults; we do not apply the two-stage architecture sweep to this task.
The chronological split uses observations 0--8639 for training, 8640--11519
for validation, and 11520--14399 for testing. Normalization uses training
statistics only. The final 3020 observations are outside this experiment.

\subsection{Observed Training Instability}\label{sec:instability}

In the MNIST schedule ablation, cosine temperature decay lowers mean
last-three-epoch accuracy by $2.24$ points at 102K parameters and $2.46$
points at 1M parameters relative to fixed temperature. The concentration
bound in Proposition~\ref{prop:fragility} controls a single displacement;
the quadratic drift calculation has additional hypotheses. Neither result
predicts an accuracy drop or proves instability for the trained networks.
Likewise, smoothing blow-up is a local seminorm bound, not a lower bound
on every realized update near a threshold. The fixed-temperature,
decaying-step schedule is the variant analyzed by
Theorem~\ref{thm:piecewise-smooth}. The empirical schedule comparisons in
Appendix~\ref{app:schedule-fragility} describe the tested configurations.

\subsection{Ablation Studies}\label{sec:ablation}

A sweep of $253$ runs over $66$ configurations on MNIST isolates subspace mode, the two
radii and the particle dimension; the temperature schedule, blockwise strategy,
schedule fragility and sensitivity plots are in Appendix~\ref{app:ablation}.

\paragraph{Ablation protocol.}
This sweep reports maximum test accuracy over epochs rather than using the
validation-selection protocol of
Section~\ref{sec:setup}. The radius and dimension comparisons use a fixed epoch budget. Most
settings have five seeds; the heatmaps mark cells with only one seed. Maximum-test selection can still bias settings differently,
so even the ordering should be interpreted cautiously. These ablations are
not compared numerically with the primary validation-selected results.
Each ablation changes one setting in the experimental configuration
(Appendix~\ref{app:turbo-mode}). The controlled weighting study in
Section~\ref{sec:controlled-comparison} disables acceleration in a separate
fixed-subspace setup; differences between the studies cannot be attributed
to any single acceleration feature.

\paragraph{What the update rule contributes.}
With the same cost evaluator, entropic OT and row-wise softmax reach
$96.4\%$, while min-cost greedy and top-$k$ averaging reach about $12\%$
under this ablation's settings (Table~\ref{tab:ablation-summary}).
The plateau theorem applies to deterministic rules that treat tied vertices
symmetrically. The greedy and top-$k$ implementations in this ablation can
break ties by vertex label and do not generally satisfy its equivariance
assumption. The theorem therefore cannot explain the accuracy differences. The greedy
solver returns uniform weights on entirely constant rows, whereas top-$k$
can select a proper subset and move even on a constant row.
This ablation shows no measurable accuracy effect from column coupling
(Section~\ref{sec:ot-binding}). Direction reuse reduces the number of fresh solves, but can reapply a displacement
computed before crossing a loss jump. Its performance depends on the reuse interval; the five non-differentiable architecture benchmarks
therefore solve afresh at every step, as does theory mode everywhere.

\paragraph{Subspace and radius choices.}
Per-layer projection at rank $4$ has the highest mean among the four tested
subspace modes (Table~\ref{tab:ablation-summary}). The global projection
also has much larger variation across seeds. These comparisons do not
isolate the geometric cause of the differences. The radius sweeps in
Figure~\ref{fig:ablation-plate} favor larger configured multipliers over
the tested ranges. Their maximum-test metric does not establish an optimal
radius for validation-selected performance.

\begin{table}[t]
\centering
\caption{Ablation summary (MNIST, maximum test accuracy over epochs under the protocol above, so comparable within the table only). Subspace modes are read at fixed steps; particle dimension is read in the sweep's orthoplex configurations, and the simplex configurations used in the primary architecture comparisons return identical per-seed accuracies at both dimensions swept.}
\label{tab:ablation-summary}
\small
\begin{tabular}{@{}llc@{}}
\toprule
Choice & Setting & Test acc.\ (\%) \\
\midrule
\multirow{4}{*}{Subspace mode}
 & per-layer, rank 4 & $\mathbf{93.7 \pm 0.2}$ \\
 & single global projection & $86.1 \pm 17.2$ \\
 & full space & $84.2 \pm 0.2$ \\
 & adaptive & $71.5 \pm 1.3$ \\
\midrule
\multirow{3}{*}{Particle dimension (orthoplex)}
 & $d_p = 2$ & $\mathbf{95.7}$ \\
 & $d_p = 4$ & $94.9$ \\
 & $d_p = 8$ & $93.7$ \\
\midrule
\multirow{4}{*}{Update rule}
 & entropic OT & $\mathbf{96.4 \pm 0.2}$ \\
 & per-row softmax & $\mathbf{96.4 \pm 0.2}$ \\
 & min-cost greedy & ${\sim}12$ \\
 & top-$k$ averaging & ${\sim}12$ \\
\bottomrule
\end{tabular}
\end{table}

\paragraph{Particle dimension and schedules.}
The orthoplex ablation favors smaller $d_p$ at fixed steps. This comparison
changes both the number of particles and the number of vertices per particle.
It does not establish the same ordering for the simplex used in the main experiments. At $d_p=8$,
the simplex uses $9$ vertices instead of $16$, reducing evaluations per
fresh step by a factor of $1.78$; both frames satisfy the geometric
condition in Theorem~\ref{thm:piecewise-smooth}.

The tested cosine schedules reduce late-run accuracy by $2.24$ pp at
$102$K parameters and $2.46$ pp at $1$M parameters, with up to $4.5$ times
the reported standard deviation. These are observations under the ablation
protocol. The displacement envelope does not predict an accuracy difference.

\subsubsection{Binding of the Transport Constraint}\label{sec:ot-binding}

In the tested subspace diagnostic, the column-marginal residual is about
$3\times10^{-7}$ across the reported penalty sweep. The reference subspace MNIST comparison gives both solvers
$96.4\pm0.2\%$ maximum test accuracy. Neither this agreement nor a small
marginal residual implies identical plans or barycentric updates. Full-space
MAX-SAT and RL are separate configurations.

The data do not establish a training benefit when the column constraint
changes the assignment.
Appendix~\ref{app:ot-highparticle} gives both cases: a full-space MNIST run where
entropic OT leads softmax by $23$~pp, scored by maximum test accuracy over epochs with sample standard
deviations of $11.0$ and $13.9$~pp; and full-space
MAX-SAT at four instance sizes under the reported protocol, where the two differ by
between $-0.084$ and $+0.028$~pp with alternating sign. The theoretical distinction is precise: coupling can move a block whose
row-wise update is zero (Theorem~\ref{prop:plateau-freeze},
Proposition~\ref{prop:ghost-drift}). A corresponding accuracy gain is not established.

\subsection{Nonlinear Jump Experiments}\label{sec:jump-experiments}

The planar and curved-jump experiments in Figure~\ref{fig:nonlinear-jumps}
illustrate the finite-temperature predictions of Section~\ref{sec:nonlinear-descent}.
Across 4,096 Haar frames per setting, the planar normal drift remains
negative even when $J>\varepsilon$, where the stationarity theorem's
saturation assumption fails. On curved jumps, all 173,274 trials satisfying
the sufficient conditions cross to the lower branch and decrease the loss;
327,343 of 327,680 trials decrease it overall. The sufficient conditions
are therefore conservative. These explicit-loss experiments describe local
motion, not convergence of a trained network. Appendix~\ref{app:nonlinear-protocol}
gives the losses, sampling laws, and parameter settings.

\begin{figure}[t]
\centering
\includegraphics[width=\linewidth]{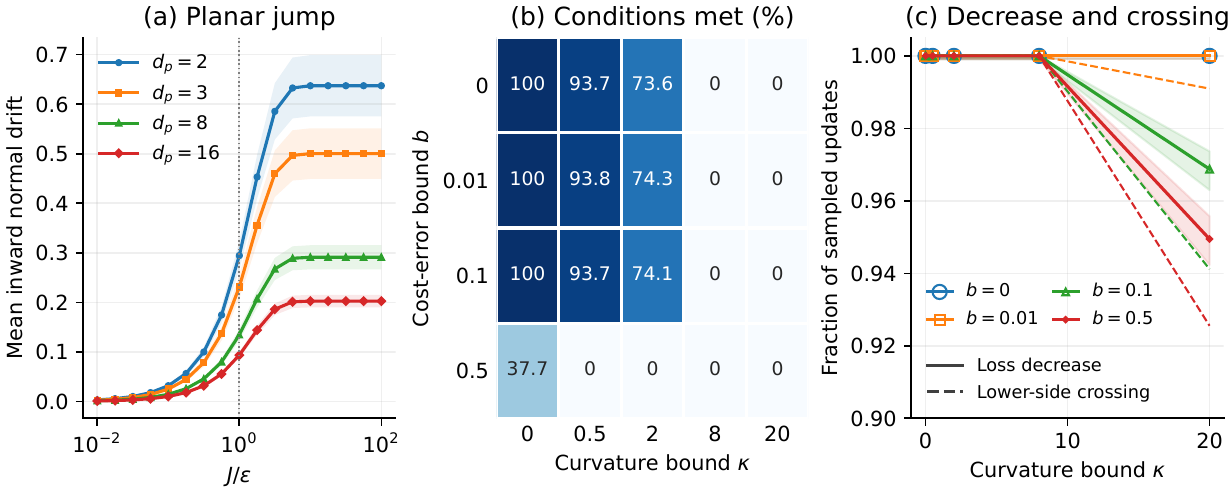}
\caption{Nonlinear softmax motion on explicit jumps, 4,096 frames per setting.
(a)~Mean inward normal drift with one sample standard deviation; the dotted
line marks $J=\varepsilon$. (b,c)~Curved-jump results at $\rho=0.2$,
$J=2$, $\varepsilon=1$, and $\ell=0.1$.
(b)~Percentage of sampled frames meeting the sufficient conditions, with one
cell per cost-error bound $b$ and curvature.
(c)~Observed loss decrease (solid, pointwise 95\% Wilson bands) and
lower-side crossing (dashed), colored by $b$; nested markers show coincident
decrease rates. The vertical axis is enlarged. Bands describe Monte Carlo sampling.}
\label{fig:nonlinear-jumps}
\end{figure}

\subsection{Controlled Weighting Comparison}\label{sec:controlled-comparison}

We compare softmax, linear, exponential-rank, and greedy weighting on the
hard-LIF MNIST model using the same candidate distribution and independently
tuned step scales. We select each rule's settings on validation data and
evaluate them over ten final seeds.
Table~\ref{tab:controlled-weighting} summarizes all ten final seeds for each rule;
no run had a numerical failure.
We multiply each direction by its selected step scale without normalizing
the direction first. The weighting rule therefore affects both the direction
and length of an update. This comparison cannot establish which rule would
perform best if all updates had the same length.

\begin{table}[t]
\centering
\caption{Controlled hard-LIF MNIST comparison, ten paired final seeds with
$10^7$ candidate evaluations per run. Accuracy is mean $\pm$ sample standard
deviation at the validation-selected checkpoint. Differences are softmax
minus the named rule, in percentage points, with 95\% paired percentile-bootstrap
intervals. Two-sided sign-flip $p$-values are unadjusted.}
\label{tab:controlled-weighting}
\small
\begin{tabular}{lccc}
\toprule
Weighting & Test accuracy (\%) & Difference [95\% interval] & $p$ \\
\midrule
Softmax & $74.49\pm9.73$ & $\mathrm{n/a}$ & $\mathrm{n/a}$ \\
Linear & $71.96\pm10.94$ & $2.53\ [0.20,4.83]$ & $0.082$ \\
Exponential rank & $78.42\pm10.10$ & $-3.93\ [-7.62,-0.18]$ & $0.088$ \\
Greedy & $75.89\pm9.27$ & $-1.40\ [-5.82,2.62]$ & $0.568$ \\
\bottomrule
\end{tabular}
\end{table}

Rank weighting had the highest mean accuracy. None of the three sign-flip tests
rejected its null hypothesis at the 5\% level. The bootstrap intervals and sign-flip tests use
different constructions; the intervals are not inversions of these tests.
We find no statistically established advantage for softmax, and the
nonsignificant tests do not establish equivalence between the rules.
Because this control fixes the search representation and candidate distribution,
it does not tell us how much the primary implementation gains from its
acceleration or representation choices.

We specify the data split, candidate distributions, tuning grids, checkpoint
rules, and uncertainty calculations in Appendix~\ref{app:controlled-protocol}.
\section{Related Work}\label{sec:background}

\subsection{Gradient-Free Optimization}\label{sec:bg-gfo}

Zeroth-order optimization uses function evaluations to choose parameter updates
\citep{nesterov2017random}. CMA-ES adapts a search covariance
\citep{hansen2016cma}, with limited-memory and diagonal variants reducing its
quadratic storage cost \citep{loshchilov2017lmcma,akimoto2020ddcma}. Natural
Evolution Strategies estimate natural gradients of expected fitness
\citep{wierstra2014natural}; OpenAI-ES uses isotropic perturbations and shared
random seeds \citep{salimans2017evolution}. These methods can evaluate
nondifferentiable objectives. A score-function identity differentiates the
perturbation distribution rather than the evaluated loss. With a suitable
smooth distribution, it can therefore estimate a smoothed gradient even
when the evaluated network is deterministic and discontinuous.

EGGROLL uses low-rank matrix perturbations to reduce per-candidate storage and
arithmetic, supporting LLM post-training and recurrent language models computed
in integer arithmetic \citep{sarkar2025eggroll}. We include it as a baseline.
Its finite-rank mean update can differ from the gradient of the objective
smoothed by its perturbations. \citet{kaya2026eggrollunrolled}
characterize this discrepancy and introduce a leave-one-out estimator that
reduces evaluation cost while preserving the finite-rank expected update.
PolyStep instead probes rotated polytope directions within parameter blocks.
Its linear weights differentiate a radial-tail smoothing in each block;
softmax and simultaneous block updates introduce separate discrepancies
from a joint smoothed gradient. Both analyses distinguish the objective averaged by a perturbation law
from the field followed by the corresponding update. PolyStep additionally
uses frame geometry to characterize cancellation and cost ordering to
establish nonlinear descent.

SPSA uses two simultaneous perturbation evaluations per update
\citep{spall1992multivariate}. MeZO applies in-place zeroth-order updates at
inference-level memory \citep{malladi2023mezo}, while DeepZero uses coordinate-wise
estimation for deep-network training \citep{chen2024deepzero}. On a locally
constant loss, probes confined to one cell provide no cost difference. Larger
random probes can instead estimate variation of a smoothed objective. Our
flat-pair experiment measures this dependence on scale for a quantized head.

Consensus-based optimization also uses softmax-weighted averages
\citep{pinnau2017cbo,dedeyn2025cbo}, but attracts candidate particles toward a
shared consensus point. PolyStep assigns weights to local vertex directions for
each parameter block. The Sinkhorn Step \citep{le2023sinkhorn} is our direct
predecessor: it uses transport-weighted polytope updates for smooth objectives
in motion planning. The row-softmax endpoint and KL interpolation follow the
one-sided entropic-transport formulation
\citep{cuturi2013sinkhorn,chizat2018scaling,peyre2019computational}; the present
analysis addresses discontinuous losses and simultaneous parameter blocks.

\paragraph{Stationarity and smoothing.}
For locally Lipschitz objectives, Goldstein stationarity measures the distance
from zero to the convex hull of nearby Clarke subgradients
\citep{zhang2020complexity}. A small gradient of an appropriate smoothing is a
sufficient certificate for this criterion.
Gradient-free algorithms attain this target through randomized smoothing
\citep{lin2022gradientfree,chen2023faster,kornowski2024algorithm}, and comparison
oracles can also support such guarantees \citep{elbakkali2026comparison}.
The classical sphere-to-ball identity \citep{flaxman2005online} is the
single-radius special case of our cost-weighted identity. The extension
identifies the size-biased ball mixture induced by the probe radii and
keeps the softmax approximation error explicit. Our stationarity bound
concerns a restricted search subspace and a fixed smoothing, with a bias
floor. These differences preclude a direct oracle-complexity comparison
with full-space guarantees to arbitrary accuracy.

The locally Lipschitz Clarke and conservative-gradient frameworks
\citep{boltepauwels2020conservative} do not apply at an attained jump.
Discontinuous optimization nevertheless has other stationarity and convergence
notions. Directional direct search uses generalized directional derivatives
\citep{vicente2012discontinuous}; \citet{audet2024counterexample} give a
counterexample to one convergence result and repair it with a revealing poll.
\citet{glasmachers2020global} establishes existence of a critical limit
point for the $(1+1)$-ES under assumptions on measurable objectives. Mollification provides another route
\citep{gupal1977discontinuous,ermoliev1995mollifier}, and
\citet{scheinberg2026comparison} optimize preference relations that can represent
discontinuous functions. These approaches use different regularity or exploration conditions.
Our fixed-smoothing target describes PolyStep's mean update; the blow-up
result explains how the local smoothness constants depend on resolution
near an attained jump.

\paragraph{Quantized and black-box training.}
Quantized evolution strategies \citep{xu2026qes} and ZeroQAT
\citep{tan2025zeroqat} study forward-only optimization with quantized computation.
Randomized rounding can smooth quantized losses \citep{kwun2025lotion}, while
low-precision zeroth-order estimators can use stochastic rounding
\citep{zhou2025quzo}. These approaches share the forward-only oracle model with PolyStep and
differ in their perturbation distributions and trainable variables.
Hybrid methods can also combine zeroth-order updates for physical parameters
with learned differentiable surrogates \citep{chertkov2025blackbox}.

\subsection{Spiking Neural Networks and Non-Differentiable Models}\label{sec:bg-snn}

Surrogate-gradient methods keep hard spikes in the forward pass and replace their
derivatives in the backward pass \citep{neftci2019surrogate,eshraghian2023snntorch}.
They remain practical alternatives to forward-only training even when
the forward objective jumps. Standard BPTT stores temporal activations;
memory-efficient alternatives such as Spatial Learning Through Time reduce that
cost \citep{meng2023towards}. We compare evaluation memory with ordinary BPTT; memory-efficient
alternatives require separate measurements.

Local zeroth-order estimates can recover surrogate derivatives in expectation
while retaining backpropagation through the network \citep{mukhoty2023localzo}.
Forward-Forward variants use local training objectives \citep{xu2025ffgafsnn},
and pseudo-zeroth-order training uses direct feedback in addition to forward
noise \citep{xiao2024opzo}. Fully gradient-free hard-LIF training on N-MNIST with EGGROLL
perturbations \citep{patankar2026gfsnn} and forward-only trust-region training of
binary recurrent spiking networks \citep{li2026satr} are closer in oracle model.
Our contribution to this setting is the structured polytope assignment
and its analysis.

Node perturbation provides another route to learning without backpropagation.
\citet{dalm2023activity} relate activity perturbations to directional
derivatives and improve training through layerwise input decorrelation.
Their activity-based rule also accommodates noise whose realization is
inaccessible. LOCO combines node perturbation with orthogonal weight
modifications based on clustered neural activities \citep{ma2026loco}.
Both methods use internal activity information as well as loss changes;
PolyStep requires only parameter probes and scalar losses.

Hardware costs introduce a further distinction. \citet{oripov2025scaling}
study weight and node perturbations through multiplexed gradient descent,
distinguishing gradient-estimation cost from time to a target accuracy.
For spiking transformers, \citet{lei2026ipzo} implement weight perturbations
in the accumulation circuitry to avoid repeated writes of perturbed weights.
These studies distinguish evaluation count from the hardware cost of
each evaluation. Our experiments measure GPU implementations.

Binary networks commonly use STE \citep{courbariaux2016binarynet,bengio2013ste}.
Its coarse gradient has been analyzed under specific network and input
assumptions \citep{yin2019understanding,jeong2025beyond}; learned sensitivity
surrogates provide another backward approximation \citep{yi2026jacquant}.
Gumbel-Softmax and Concrete relaxations replace discrete samples with continuous
ones during training \citep{jang2017categorical,maddison2017concrete}.
PolyStep uses scalar losses from the hard forward model to choose
parameter updates. Its oracle therefore differs from both backward
approximations and continuous forward relaxations.
\section{Discussion}\label{sec:discussion}

\subsection{What the Analysis Explains}\label{sec:positioning}

The analysis separates the information in the costs from the geometry of
the update. A simplex turns every nonuniform weight vector into motion,
and its displacement identity tells us exactly how far the block moves.
Descent needs information about the loss as well. On a monotone ridge,
ordering the costs orders the directions, so the nonlinear rule remains
nonascending even when the weights are too concentrated for a useful
linear approximation.

For general losses, we identify the local average differentiated by the
mean linear direction. Each block has its own smoothing, and combining
these derivatives can produce a field that is not the gradient of any
single objective. This discrepancy is separate from the softmax error.
The stationarity bound makes both explicit; background perturbations
remove the linear block discrepancy by giving every block the same
smoothing target.

The empirical results also depend on more than the presence of hard
operations. At matched evaluations, PolyStep exceeds OpenAI-ES by $13.4$
percentage points on hard-LIF training, while OpenAI-ES leads by $0.54$
points on argmax attention. The controlled weighting study finds no
statistically established softmax advantage. These findings motivate
studying the probe distribution, representation, and displacement together.

\subsection{Limitations}\label{sec:limitations-disc}

\paragraph{Scale and query cost.}
A fresh simplex step costs
$\lceil d_{\mathrm{sub}}/d_p\rceil(d_p+1)K$ evaluations. At
$d_{\mathrm{sub}}=4096$, $d_p=8$, and $K=1$, a $500$K-evaluation budget
permits roughly $100$ steps. Population ES sets its query count per update
through the population size \citep{sarkar2025eggroll}. PolyStep reaches
$49.4\%$ on the $4.2$M-parameter from-scratch SST-2 transformer, and the
all-parameter GPT-2 experiment with a $128$-dimensional projection also
remains near chance. We have not isolated the causes of these failures.
The frozen GPT-2 head experiment demonstrates training on pretrained
features, but leaves the backbone unchanged.

\paragraph{Comparative evidence.}
Our image study concentrates on MNIST: four of the five hard architectures
use it alone or in the combined MoE dataset. OpenAI-ES leads on argmax
attention and 5,000-variable MAX-SAT at matched evaluations; the hard-MoE
test means differ by $0.03$ points. Adam leads on the continuous MNIST
control, and every learned ETTh1 model trails persistence. These
experiments compare the stated implementations and protocols, rather than
establishing state-of-the-art spiking or quantized training.
In the EGGROLL comparisons, its factored subspace differs from PolyStep's
HybridSubspace, so the optimizer and representation effects are confounded.
The weighting control holds the representation and candidate distribution
fixed. Its nonsignificant differences neither establish equivalence nor
show a softmax advantage.

\paragraph{Compute, memory, and information access.}
Batching makes large candidate sets practical but retains their evaluation
cost. The MNIST run uses $10.19$M candidate evaluations for $3{,}180$
updates with direction reuse, compared with $3{,}180$ forward-backward
passes for Adam. A practical speed comparison would require isolated
wall-clock measurements with batching tuned separately for each method.
The recurrent-memory measurements concern the specified forward and BPTT
evaluators; projections, parameters, and optimizer buffers add memory.
In RL, probes consume rollouts, while policy-gradient methods can use
trajectory information unavailable through a scalar-return interface.
The quantized PPO and DQN controls diagnose blocked gradients in their
implementations; they do not evaluate quantization-aware policy training.

\paragraph{Theory and trained configurations.}
The direct nonlinear results apply to monotone ridge losses and isolated
curved boundaries with controlled cost errors. The simplex identity
concerns motion in search coordinates, which can disappear during
reconstruction if the map has a null space. For general losses, the
stationarity theorem assumes fixed smoothing, suitable sampling,
transverse regularity, and compact iterates. The selected configurations
do not meet all these conditions (Appendix~\ref{sec:hypothesis-coverage}),
and the worst-case constants are too large to predict their performance.
A small smoothed gradient can also occur on a plateau with poor loss.
When all cost rows are constant, fresh softmax updates vanish; momentum
or direction reuse can still carry an earlier displacement forward.

\subsection{Questions Suggested by the Results}\label{sec:future}

The simplex identity offers a way to separate direction from step length.
Comparing softmax, rank, and greedy rules at equal nonzero displacement
lengths would test their directional choices; zero directions would need
separate accounting. Recording cost spread and post-update loss alongside
displacement would then connect observed variation to actual improvement.

The planar-jump result suggests a complementary theoretical direction.
It gives descent even when every probe row is saturated. Extending this
argument to interacting thresholds requires control of ordering errors
from curvature, multiple crossings, and sampling noise. The central issue
is whether directions that help individual blocks still help when applied
simultaneously.

Background perturbations provide a common linear smoothing target, but
add sampling noise. A comparison at matched evaluations could measure
whether the reduced block discrepancy compensates for that variance.
For coupled transport, the corresponding question is how to match local
vertex labels across blocks: the example in
Appendix~\ref{sec:thm-kl-softmax} shows that relabeling the same physical
probes can change a balanced update. Both questions follow directly from
the difference between local cost information and a joint direction.

The distinction between probe radius, temperature, and step scale also
suggests adaptive rules. The probe radius determines which loss variation
is visible, temperature determines how strongly that variation affects the
weights, and the step scale controls displacement along the weighted
direction. When probe
costs are flat, increasing the probe radius may reveal a boundary; when
the costs distinguish directions but the joint update fails to improve the
loss, reducing the step scale may help. Such rules would need to account
for noise and interactions between blocks. Comparing them at matched
evaluations, including any extra probes or post-update checks, would test
whether adapting these scales separately improves training.
\section{Conclusion}\label{sec:conclusion}

We introduced \sysname{}, a forward-only optimizer that combines rotated
polytope probes with weighted updates of parameter blocks. Its practical
form uses the uncoupled softmax solution of a one-sided entropic transport
problem. The analysis explains two distinct sources of useful motion:
simplex geometry prevents cancellation of nonuniform weights, and ordered
costs give nonlinear descent on monotone ridge losses. For more general
losses, an exact smoothing identity reveals the errors that separate the
mean update from a joint smoothed gradient and yields a conditional
stationarity bound with a bias floor.

The experiments show substantial gains on hard-LIF training at matched
evaluations, alongside settings where OpenAI-ES performs better. They also
show the cost of searching many blocks: the number of evaluations per
fresh step grows with the search dimension. We therefore see a concrete
direction for future work: design probes that reveal useful cost variation
with fewer evaluations, while controlling the interactions between block
updates. Discontinuity alone
neither predicts which optimizer will work best nor determines whether its
probes provide a useful direction.
Code and experiment records are available at
\url{https://github.com/anindex/polystep}.

\section*{Broader Impact Statement}
This work studies training from scalar loss evaluations when a model contains
hard operations. It may support low-precision or event-driven models and
hardware-in-the-loop optimization. Forward-only evaluation reduces activation
storage in the measured SNN setting, but its additional evaluations can increase
computation; training energy was not measured. The applications share the usual
risks of learned decision systems. We identify no application-specific ethical
issue introduced by the optimization rule.

\section*{Reproducibility Statement}
We release the source code, experiment configurations, result records,
and trajectory checkpoints at \url{https://github.com/anindex/polystep}.
Primary supervised comparisons use validation-selected configurations and
checkpoints, each evaluated once on test; maximum-test ablations are labeled
separately. Appendix~\ref{app:reproducibility} gives the seeds, software
versions, rank-tie conventions, and limits of reproducibility across devices.
Appendix~\ref{app:controlled-protocol} specifies the controlled weighting study.

\bibliography{references}
\bibliographystyle{tmlr}

\newpage
\appendix
\section{Implementation of the Core Update}\label{app:method-implementation}

This section specifies the rotations, radii, reconstruction, and subspace
maps used by Algorithm~\ref{alg:polystep-step}. Experimental schedules and
acceleration settings are given in Appendix~\ref{app:hyperparameters}.

\subsection{Sampling and Reconstruction}

\paragraph{Rotations and polytope templates.}
For each particle $i$ the target law is Haar-uniform rotation,
$R_i\sim\mathrm{Uniform}(\mathrm{SO}(d_p))$. The implementation uses
Gaussian QR with sign and determinant corrections for small batches,
recursive Householder reflections for larger batches, and a uniform angle
for $d_p=2$ \citep{mezzadri2007generate}.
In two dimensions, the simplex
is an equilateral triangle and the orthoplex consists of two perpendicular
antipodal pairs. Three templates define the vertex set $\mathcal{V}$:
\begin{itemize}[nosep]
  \item \textbf{Simplex}: $d_p + 1$ vertices. Used in the primary architecture runs; theory
  mode runs the orthoplex (Section~\ref{sec:theory-mode}).
  \item \textbf{Orthoplex} (cross-polytope): $2 d_p$ vertices $\{\pm e_j\}_{j=1}^{d_p}$.
  \item \textbf{Cube} (hypercube): $2^{d_p}$ vertices, used only in the template ablation.
\end{itemize}
The simplex and orthoplex are centered frames of unit vertices: each vertex has unit
length and the vertices sum to zero. This is the class used in the analysis
(condition~(i) of Assumption~\ref{asm:main}), and both have the same
step constant (Theorem~\ref{thm:exact-step}). Proposition~\ref{prop:simplex-detectability} gives a further distinction:
nonconstant softmax costs always produce motion on the simplex, whereas
opposite pairs can cancel on the orthoplex. They also differ in cost: the simplex uses $(d_p+1)/(2d_p)$ as many evaluations per fresh step,
which is $9/16$ at $d_p=8$. The controlled weighting study uses the orthoplex.
For comparison with the transport formulation, the step vertices without jitter are
\begin{equation}\label{eq:step-vertices}
y^{\mathrm{step}}_{iv} = x_i + r_s \varepsilon \, R_i \, v_v, \quad v_v \in \mathcal{V},~v = 1,\ldots,V.
\end{equation}

\paragraph{Probe radii.}
For each particle--vertex pair $(i,v)$, we place $K$ probes along the rotated
direction $R_i v_v$. When radius jitter is enabled, one draw $\eta_t\sim q$
rescales all probes in that step. It is independent of the rotations, and
$q$ is supported in $(-\eta_{\max},\eta_{\max})$ with $\eta_{\max}\in(0,1)$:
\begin{equation}\label{eq:probe-points}
y_{ivk} = x_i + r_p (1+\eta_t) \varepsilon \, R_i \, v_v \, \xi_k, \quad \xi_k = \frac{k}{K+1},~k = 1, \ldots, K.
\end{equation}
Probe distances are fractions of $r_p(1+\eta_t)\varepsilon$. Their directions
agree with the step vertices in Eq.~\ref{eq:step-vertices}, but their distances
are controlled by $r_p$, independently of the step multiplier $r_s$.
When $K = 1$, each probe sits at half the (jittered) probe radius:
$y_{iv1} = x_i + \tfrac{1}{2} r_p (1+\eta_t) \varepsilon \, R_i \, v_v$.
The probe jitter $\eta_t$ is a single draw per step, shared across particles and
probes; only its marginal law enters Theorem~\ref{thm:exact-step}. The step jitter
$\eta_i$ of Eq.~\ref{eq:step-jitter} is drawn per particle.
For bounded measurable losses, a jitter density makes the smoothing identity
valid at every base point; without jitter it holds almost everywhere
(Section~\ref{sec:thm-piecewise-smooth}). The convergence
analysis uses the stronger condition $q\in W^{2,1}$; the numerical sampler
approximates that law (Appendix~\ref{sec:hypothesis-coverage}).

\paragraph{Step-radius jitter.}
Step-radius jitter independently rescales each particle's \emph{step}:
\begin{equation}\label{eq:step-jitter}
x_i^+ = x_i + (1 + \eta_i)\, r_s \varepsilon \,
\frac{1}{a_i}\sum_v T^*_{iv} R_i v_v ,
\qquad \eta_i \overset{\text{iid}}{\sim} q ,
\end{equation}
where $q$ is the same jitter density and $1/a_i$ is the row renormalization of
Eq.~\ref{eq:bary-proj} (the implementation divides by the realized row sum).
Probe jitter changes where costs are measured; step jitter changes how far
the iterate moves after the costs have been measured. It rescales a nonzero
displacement but cannot move a block whose softmax displacement is zero
(Corollary~\ref{cor:piecewise-constant}). Both jitters are off in the primary
image, ETTh1, and MAX-SAT configurations ($\eta_{\max}=0$) and on in theory mode ($\eta_{\max}=0.05$;
Section~\ref{sec:theory-mode}). Jitter disables
the reuse heuristics of Appendix~\ref{app:turbo-mode}, whose progress
checks can react to the added noise. The RL runs use probe jitter $0.05$
and no step jitter (Appendix~\ref{app:rl-policy-search}).

\paragraph{Transport barycenters.}
Each particle moves to the barycenter of its vertices under the soft assignment (softmax weights at $\lambda=0$ and a column-penalized transport plan at $\lambda > 0$):
\begin{equation}\label{eq:bary-proj}
x_i^+ = \frac{1}{a_i}\sum_{v=1}^{V} T^*_{iv} \cdot y^{\mathrm{step}}_{iv}.
\end{equation}
With step-radius jitter this is Eq.~\ref{eq:step-jitter}; Eq.~\ref{eq:bary-proj} is
the case with step jitter disabled.
The updated coordinates reconstruct the model as
$\theta^+=\theta_{\mathrm{ref}}+\Pi Qx^+$, followed by reshaping
the parameter vector into the model's parameter arrays.

\subsection{Layerwise Search Maps}\label{app:search-implementation}

The primary experiments use full-space search for MAX-SAT and fixed
per-layer maps for the image, ETTh1, and RL benchmarks. In RL the map
retains all 114 CartPole parameters and reduces Acrobot from 163 to 68
coordinates. A single global projection and a displacement-biased adaptive
map are used only in the projection ablation.

\paragraph{HybridSubspace.}
Each weight layer $l$ with shape $(d_{\mathrm{out}},d_{\mathrm{in}})$ and
$d_l=d_{\mathrm{out}}d_{\mathrm{in}}$ parameters initially receives
$n_l=\min\{d_{\mathrm{out}}d_{\mathrm{in}},(d_{\mathrm{out}}+d_{\mathrm{in}})r_l\}$
coordinates, where $r_l=\min(r,d_{\mathrm{out}},d_{\mathrm{in}})$, before
any optional dimension cap.
The dense implementation takes $\Pi_l\in\R^{d_l\times n_l}$ from a thin QR
decomposition of an independent Gaussian matrix. Its columns are orthonormal
and its range is rotation invariant. When dense float32 storage would exceed
$10^9$ bytes, the default implementation instead samples a sparse signed
matrix with $\lfloor\sqrt{d_l}\rfloor$ entries per column, with replacement,
and scales entries by the reciprocal square root of that count. This sparse
basis is not orthonormal. The first layer triggers this fallback in the
reported SNN, INT8, argmax, staircase, hard-MoE, and MNIST configurations.
We use dense QR bases in the controlled weighting comparisons.
Layers for which $n_l=d_l$ use the identity map. Biases and other
one-dimensional parameter arrays also stay at full dimension.
Reconstruction maps $z_l \in \R^{n_l}$ back to the full layer: $\Delta\theta^{(l)} = \Pi_l\, z_l$.
Without an additional dimension cap, the total is
$d_{\mathrm{sub}} = \sum_l n_l + \sum_b d_b$, where $b$ indexes the remaining
one-dimensional arrays and $d_b$ is their length (e.g., 8{,}538 total
coordinates for MNISTNet at rank~8);
Algorithm~\ref{alg:polystep-step} then operates on $\mX \in \R^{P \times d_p}$.
The optional cap used in RL reduces only projected layers and retains
identity layers and biases, so the requested cap may be unattainable.
The subspace basis is fixed after initialization; basis rotation is disabled
in these configurations (Section~\ref{sec:ablation}).
The MNIST ablation reports $93.7\pm0.2\%$ maximum test accuracy for
per-layer projection and $86.1\pm17.2\%$ for a single global projection
(Table~\ref{tab:ablation-summary}); we cannot attribute this difference to a
single geometric property of the projections.

Appendix~\ref{app:system} describes blockwise OT, whose per-layer Sinkhorn
solves use $O(\max_l P_lV)$ plan storage when processed sequentially,
where $P_l$ counts the particles assigned to layer $l$, and
sparse projection storage. The primary runs do not use blockwise OT;
sparse layer projections follow the storage threshold above.

\subsection{KL-Softmax Interpolation}\label{sec:thm-kl-softmax}\label{app:proof-kl-softmax}

For the transport results below, $\mT \in \R_{\geq 0}^{P \times V}$
is a transport plan, $\mC \in \R^{P \times V}$ a cost matrix, $\va \in \Delta_P$
the source marginal (uniform), $\vb \in \Delta_V$ the target marginal (uniform),
$\varepsilon > 0$ the entropic temperature, and $\lambda \in [0, +\infty]$
the KL-penalty strength. We write
$\Ent(\mT) = \sum_{i,j} T_{ij}\,(\log T_{ij} - 1)$ for the \emph{negentropy}
and, as in Eq.~\ref{eq:intro-kl},
$\KL(\vp \| \vq) = \sum_j[p_j \log(p_j/q_j) - p_j + q_j]$ for the generalized
Kullback-Leibler divergence (with $0 \log 0 = 0$), which on probability vectors
agrees with $\sum_j p_j\log(p_j/q_j)$ and otherwise includes an additional term linear in $\vp$. This term removes
the additive constant from the column potential; omitting it would shift
$g$ by $-\lambda$ without changing the plan. The regularizer is strictly convex, with
$\partial\Ent/\partial T_{ij}=\log T_{ij}$.

The balanced endpoint minimizes
\begin{equation}\label{eq:ot-objective}
 \min_{\mT\geq0}\ \langle\mC,\mT\rangle
       +\varepsilon\KL(\mT\|\va\otimes\vb)
 \quad\text{s.t.}\quad\mT\bm1=\va,\quad\mT^\top\bm1=\vb.
\end{equation}
Under fixed marginals, this KL regularizer differs from $\Ent(\mT)$ only by
a constant. Sinkhorn iterations enforce the two marginals by alternating
row and column scalings \citep{cuturi2013sinkhorn,peyre2019computational}.

The KL penalty connects two assignment rules. At $\lambda=0$, each block
uses its own softmax weights. As $\lambda$ increases, the aggregate weights
must more closely match the target marginal; the limit enforces that
marginal exactly. The following theorem establishes continuity and
monotonicity for the program in Eq.~\ref{eq:intro-kl}, with
$\mC\in\R^{P\times V}$, $\va\in\Delta_P$, $\vb\in\Delta_V$,
$\varepsilon>0$, and $\lambda\in[0,\infty)$
\citep{chizat2018scaling}. Section~\ref{sec:ot-binding} reports the empirical
solver comparisons.

\begin{theorem}[Endpoints and marginal monotonicity of the KL-penalized program]\label{thm:kl-softmax-endpoints}
Let $\mT^\star_\lambda$ be the unique solution of
Eq.~\ref{eq:intro-kl} for given $\varepsilon > 0$, $\lambda$, $\mC$,
with $\va, \vb$ of full support, let
$\vq_\lambda := (\mT^\star_\lambda)^\top \bm{1}$, and let $\mT^\star_{\mathrm{Sink}}$
be the unique minimizer of $\langle \mC, \mT\rangle + \varepsilon\,\Ent(\mT)$
over plans with \emph{both} marginals $\va$ and $\vb$ enforced. Then
$\lambda \mapsto \mT^\star_\lambda$ is continuous on $[0,\infty)$ with
\[
\lim_{\lambda \downarrow 0} \mT^\star_\lambda
  = \mathrm{diag}(\va)\, \mathrm{softmax}(-\mC/\varepsilon),
\qquad
\lim_{\lambda \to \infty} \mT^\star_\lambda = \mT^\star_{\mathrm{Sink}},
\]
and the column-marginal violation is non-increasing in $\lambda$:
$\KL(\vq_{\lambda_2} \| \vb) \leq \KL(\vq_{\lambda_1} \| \vb)$ for
$0 \leq \lambda_1 \leq \lambda_2 < \infty$, extending to the endpoint
$\mT^\star_\infty := \mT^\star_{\mathrm{Sink}}$ by the displayed limit, with
$\KL(\vq_\lambda \| \vb) \to 0$ as $\lambda \to \infty$.
\end{theorem}

These properties describe the transport solutions as the penalty varies.
They do not identify a penalty that improves training accuracy.
Proposition~\ref{prop:ghost-drift} describes the motion coupling can induce
on a constant-cost block.

\paragraph{Column labels affect coupled updates.}\label{rem:column-labels}
Independent relabeling of each row's cost--direction pairs leaves the
row-wise softmax displacement unchanged. A shared column marginal need
not preserve this invariance. For example, take two rows with directions
$(u_i,-u_i)$ and costs $(0,J)$, $J>0$, and uniform row and column masses.
In balanced transport, the cost is fixed by the column marginal, so entropy
selects the uniform plan and both displacements vanish. Swap the two
cost--direction pairs in the second row. The same physical probes now have
cost matrix $\left(\begin{smallmatrix}0&J\\J&0\end{smallmatrix}\right)$.
Its row-wise softmax plan already has the required column marginal and is
therefore the balanced minimizer. Both physical directions are now
$\tanh(J/(2\varepsilon))u_i\ne0$. Thus a coupled assignment depends on how
local directions are matched across rows, even when the evaluated points
are unchanged. Independent rotations give no common ambient direction to
a vertex label (Section~\ref{sec:bg-ot}).

\paragraph{(1) Existence, uniqueness, continuity.}
Write $F_0(\mT)=\langle\mC,\mT\rangle+\varepsilon\Ent(\mT)$ and
$H(\mT)=\KL(\mT^\top\bm1\|\vb)$ on the row-feasible set
$\mathcal T=\{\mT\geq0:\mT\bm1=\va\}$.
This set is nonempty and compact. Full support of $\vb$ makes $H$
continuous, and $F_0+\lambda H$ is continuous and strictly convex, so
it has a unique minimizer.

To prove continuity, take $\lambda_n\to\lambda<\infty$ and any convergent
subsequence of the minimizers, with limit $\overline{\mT}$.
For every $\mT\in\mathcal T$, optimality gives
$F_0(\mT^\star_{\lambda_n})+\lambda_n H(\mT^\star_{\lambda_n})
\leq F_0(\mT)+\lambda_n H(\mT)$.
Passing to the limit shows that $\overline{\mT}$ minimizes
$F_0+\lambda H$, hence equals $\mT^\star_\lambda$.
Compactness then gives convergence of the entire sequence, including
when $\lambda=0$.

\paragraph{(2) KKT and dual potentials.}
For finite $\lambda$, the minimizer has positive entries: if an entry
were zero, mixing the plan with the positive feasible plan $\va\vb^\top$
would give a right derivative of $-\infty$ in the objective, contradicting
optimality. Thus the logarithms below are well defined.
With a row multiplier $f_i$, stationarity reads
$C_{ij}+\varepsilon\log T^\star_{ij}
+\lambda\log(q^\star_j/b_j)-f_i=0$, where
$q^\star_j=\sum_iT^\star_{ij}$. Equivalently,
\begin{equation}\label{eq:kkt-form}
\log T^\star_{ij}=\frac{f_i+g^\star_j-C_{ij}}{\varepsilon},
\qquad g^\star_j=\lambda(\log b_j-\log q^\star_j).
\end{equation}
Setting $u_i=e^{f_i/\varepsilon}$, $v_j=e^{g^\star_j/\varepsilon}$,
and $\mK=e^{-\mC/\varepsilon}$ gives
$\mT^\star=\mathrm{diag}(u)\mK\mathrm{diag}(v)$.
Since $\vq^\star=v\odot(\mK^\top u)$, the column equation becomes
\[
 v=\left(\frac{\vb}{\mK^\top u}\right)^{\lambda/(\lambda+\varepsilon)},
\]
with entrywise products, division, and powers. This is the column scaling
update of \citet[Section~3]{chizat2018scaling}.

\paragraph{(3) Endpoint identification.}
At $\lambda=0$, the column potential vanishes. Normalizing
$T^\star_{ij}=e^{(f_i-C_{ij})/\varepsilon}$ to row mass $a_i$ gives
$T^\star_{ij}=a_i\,\mathrm{softmax}(-C_{i:}/\varepsilon)_j$,
as in Eq.~\ref{eq:softmax-rule}. Continuity gives the zero-penalty limit.

For $\lambda\to\infty$, compare with the balanced optimum,
which belongs to $\mathcal T$ and has $H(\mT^\star_{\mathrm{Sink}})=0$:
\[
 F_0(\mT^\star_\lambda)+\lambda H(\mT^\star_\lambda)
 \leq F_0(\mT^\star_{\mathrm{Sink}}).
\]
Since $F_0$ is bounded below on $\mathcal T$, this gives
$H(\mT^\star_\lambda)=O(1/\lambda)\to0$.
Every limit point is therefore balanced and, by the same inequality,
has $F_0$ no larger than the balanced optimum. Uniqueness identifies
every limit point with $\mT^\star_{\mathrm{Sink}}$, proving convergence.

\paragraph{(4) Marginal-violation monotonicity.}
Let $0\leq\lambda_1<\lambda_2<\infty$ and abbreviate
$\mT_k=\mT^\star_{\lambda_k}$. Optimality at the two penalties gives
\begin{equation}\label{eq:opt-ineq-1}
 F_0(\mT_2)+\lambda_2H(\mT_2)
 \leq F_0(\mT_1)+\lambda_2H(\mT_1),
\end{equation}
\begin{equation}\label{eq:opt-ineq-2}
 F_0(\mT_1)+\lambda_1H(\mT_1)
 \leq F_0(\mT_2)+\lambda_1H(\mT_2).
\end{equation}
Adding cancels $F_0$ and leaves
\[
 (\lambda_2-\lambda_1)\big[H(\mT_2)-H(\mT_1)\big]\leq0.
\]
Thus the marginal violation is nonincreasing. Equal penalties give the
same plan, and the balanced limit extends the conclusion to infinity. \qed

\section{Proofs for the Nonlinear Update}\label{app:proofs}

We give the cost-stability and curved-jump proofs, then examine
plateaus, coupling, and concentrated weights.

\subsection{Nonlinear Softmax Directions}\label{app:nonlinear-softmax}

\paragraph{Ordered weights and planar jumps.}
The proof of Theorem~\ref{thm:ridge-descent} uses
only centered projections and oppositely ordered weights; it also applies
to other nonnegative normalized weighting rules with that ordering.

For Eq.~\ref{eq:jump-softmax-drift}, orient each antipodal pair so that
$q_j=|n^\top Re_j|>0$. The positive and negative costs are $J$ and $0$.
Writing $a_J=e^{-J/\varepsilon}$, their normal contributions sum to
$(a_J-1)\sum_jq_j/[d_p(a_J+1)]$. This is the stated expression.
Since $\sum_jq_j^2=1$, its magnitude is at least
$\tanh(J/(2\varepsilon))/d_p$. If $J>\varepsilon$, the saturated event
has probability one, while the normal projection of
$\E\sum_vC_vRv_v$ equals $J\E\sum_jq_j>0$.
The relative saturated-mass inequality therefore cannot hold with
$\bar\pi<1$. The example satisfies the direct nonlinear result even
though it violates that convergence assumption.

\paragraph{Proof of Proposition~\ref{prop:softmax-cost-stability}.}
Softmax is invariant to a common cost offset, so let
$e=\widehat C-C-b\bm1$, $|e_v|\leq\beta$, and interpolate
$p(t)=\mathrm{softmax}(-(C+te)/\varepsilon)$, $0\leq t\leq1$.
For any unit $z$, differentiation of a finite sum gives
\[
 \frac{\mathrm d}{\mathrm dt}\sum_vp_v(t)z^\top u_v
       =-\varepsilon^{-1}\operatorname{Cov}_{p(t)}(z^\top u,e).
\]
The two variables lie in $[-1,1]$ and $[-\beta,\beta]$.
Their variances are at most $1$ and $\beta^2$, so Cauchy--Schwarz bounds
the derivative by $\beta/\varepsilon$ in absolute value. Integration and
the dual characterization of Euclidean norm prove the bound. Both directions
lie in the unit ball, giving the additional bound $2$. \qed

\paragraph{Proof of Corollary~\ref{cor:curved-jump}.}
Assume the probe and update segments lie in a neighborhood with the stated
regularity of $g_+$, $g_-$, and $\psi$. Taylor's theorem gives
\[
 |\psi(x+\rho u)-\rho n^\top u|\leq\kappa\rho^2/2.
\]
Hence $|n^\top u|>\kappa\rho_{\max}/2$ preserves the sign of every
positive-radius probe. The cost row then differs from the ideal planar-jump
row plus the common offset $g_-(x)$ by at most
$\beta=G\rho_{\max}+b_{\mathrm{noise}}$. The preceding proposition yields
$n^\top D(\widehat C)\leq-m$. For $\ell>0$, the conditions
$\kappa\ell<2m$ and $G\ell<J$, together with $\|D(\widehat C)\|\leq1$, give
\[
 \psi(x+\ell D(\widehat C))\leq-\ell m+\kappa\ell^2/2<0,
 \qquad L(x+\ell D(\widehat C))-L(x)\leq G\ell-J<0.
\]
The Lipschitz bound is applied to the branch selected by each probe, so the
jump height may vary along the boundary. Assigning the boundary to $g_+$
is needed for the last strict loss comparison. Under the other boundary convention, the direction still points toward
the lower side, but strict loss decrease at the boundary need not follow.

For a batch of $n_B$ independent samples, suppose each sample's radius-averaged
loss has mean $C_v$ and range of length at most $H$ at vertex $v$.
Assume the batch is independent of the probes. Hoeffding's inequality and a union bound
give, conditionally on the probes,
\[
 b_{\mathrm{noise}}=H\sqrt{\frac{\log(2V/\alpha)}{2n_B}}
\]
with probability at least $1-\alpha$, $0<\alpha<1$. The same batch may be
used at all vertices; independence across vertices is unnecessary. For a Haar
frame in dimension $d_p\geq2$, set $t=\kappa\rho_{\max}/2$. If $0\leq t<1$, the probability of
violating the geometric margin is at most
\[
 \min\{1,Vb_{d_p}(t)\},\qquad
 b_d(t)=\frac{\Gamma(d/2)}{\sqrt\pi\,\Gamma((d-1)/2)}
          \int_{-t}^t(1-s^2)^{(d-3)/2}\,\mathrm ds.
\]
Together with $m_0\geq\tanh(J/(2\varepsilon))/d_p$, this yields uniform
sufficient drift and step-size conditions with success probability at least
$\max\{0,1-\alpha-Vb_{d_p}(t)\}$. Unbounded cross-entropy requires a
separate range or concentration assumption; this bound is not automatic.
Interactions with additional thresholds are excluded by the local model.

\subsection{Plateau Freezing of Row-Wise Rules}\label{sec:rule-impossibility}\label{app:proof-plateau-freeze}

On a plateau wider than the probe reach, every probe in a block has the
same cost. A deterministic assignment that treats vertex labels symmetrically
must then assign equal weights. Since the vertices sum to zero, so does
the displacement. The formal symmetry condition is \emph{relabeling equivariance}:
reordering the input costs reorders the output weights in the same way.

\begin{theorem}[Plateau freezing for equivariant row-wise rules]\label{prop:plateau-freeze}
Let the polytope be centered, $\sum_v v_v = 0$ (orthoplex, cube and regular
simplex all qualify), and let the plan be \emph{any} rule $T_{i:} = F(\mC_{i:})$
reading particle $i$'s cost row alone and equivariant under relabeling the
vertices, $F(\mC_{i:}\circ\sigma) = F(\mC_{i:})\circ\sigma$, with
nonnegative weights summing to $a_i>0$.
If particle $i$'s cost row is
constant then equivariance forces uniform weights and its barycentric step is exactly zero,
\[
\sum_v T_{iv}\, R v_v \;=\; \frac{a_i}{V}\, R \sum_v v_v \;=\; 0 ,
\]
for every rotation $R$, every rotation law, and every temperature $\varepsilon$.
Consequently, the core \sysname{} update at $\lambda=0$, without momentum
or direction reuse, remains fixed whenever every block's cost row is
constant. This includes a plateau containing every block's probe support
out to $\delta_{\mathrm{out}}$.
\end{theorem}

Row-wise softmax satisfies this symmetry. So do greedy and top-$k$
averaging rules when they average tied vertices symmetrically. Changing
the rotation distribution cannot move a zero weighted sum. Momentum can
continue an existing motion for a finite distance, while label-based or
random tie-breaking can generate motion by departing from the symmetry or
determinism assumptions. Under Haar rotations, selecting a fixed vertex
label still gives zero mean direction. External perturbations are also
outside the stated assignment rule.
Column coupling provides a different source of motion within the transport family.

\begin{proof}
A constant row is fixed by every relabeling, $\mC_{i:}\circ\sigma = \mC_{i:}$, so
equivariance gives $F(\mC_{i:}) = F(\mC_{i:})\circ\sigma$ for every $\sigma$. Taking
$\sigma$ the transposition of $v$ and $v'$ gives $T_{iv} = T_{iv'}$, so the row is uniform
and $T_{iv} = a_i/V$. The display follows by centering. For softmax the first step is
direct: $\mathrm{softmax}(-\mC_{i:}/\varepsilon)$ of a constant row is the uniform
vector.
\end{proof}

\subsection{Coupled Motion on a Constant-Cost Block}\label{app:proof-ghost-drift}

The plateau result above concerns independent rows. We now ask whether
column coupling can move a constant-cost block and whether that motion
has a preferred direction.

\begin{proposition}[Coupled motion at $\lambda > 0$]\label{prop:ghost-drift}
Under the KL-penalized plan of Theorem~\ref{thm:kl-softmax-endpoints}
(Appendix~\ref{sec:thm-kl-softmax}) with
$\lambda > 0$, a constant cost row gives
\begin{equation}\label{eq:ghost-weights}
T_{iv} \;=\; a_i\,\mathrm{softmax}_v\!\big(g_v/\varepsilon\big) ,
\end{equation}
where $g$ are the column potentials. Take a positive step scale
$\ell=r_s\varepsilon$, without step jitter. Assume $d_p\geq2$ and that $R_i$ is
independent of the other current draws; condition on those draws below.
The potentials are determined jointly by all rows, row $i$
included, because row $i$ contributes $a_i\,\mathrm{softmax}(g/\varepsilon)$ to the
column marginal. They are independent of $R_i$ provided row $i$ is constant
for every rotation in the support, not only the realized rotation. The constant-cost particle can then move in a realized step, while under
a Haar rotation law $\E_R[\sum_v T_{iv} R v_v] = 0$. The mean step is nonzero exactly when $\E[R]\sum_v w_v v_v \neq 0$, with
$w_v := \mathrm{softmax}_v(g_v/\varepsilon)$ the ghost weights of
Eq.~\ref{eq:ghost-weights}, a condition on the rotation law's first moment and not on
its being non-uniform: on
$\mathrm{SO}(2)$ the density $(1 + a\cos 2\vartheta)/2\pi$ is non-uniform for
$0 < |a| < 1$ yet has $\E[R] = 0$.
\end{proposition}

With column coupling, costs in other blocks can make a constant-cost
block move. Under the proposition's independence and Haar assumptions,
this motion has mean zero. If all rows have constant costs and the target marginal is uniform, the
weights are uniform and every block remains fixed. Motion induced by other
rows also supplies no descent guarantee for the block on the plateau.

For a simplex, additional conditional lower and upper bounds on the spread
of the weights give matching expected exit-time bounds of
$\Theta(W^2/(r_s\varepsilon)^2)$ for a radius-$W$ ball contained in the
plateau, with $W\geq r_s\varepsilon$. The lower weight bound gives the
upper exit-time bound; the upper weight bound gives the lower one.
The derivation below states these weight-spread assumptions explicitly. The simplex condition matters:
on an orthoplex, nonuniform weights $(0.4,0.4,0.1,0.1)$ on
$(\pm e_1,\pm e_2)$ still give zero displacement.
Without a nonzero weight-spread lower bound, only the distinction between
row-wise freezing and possible coupled motion remains. Leaving an inner
ball also need not leave a larger plateau, so the plateau's inradius alone
does not give an upper bound on its exit time.

\begin{proof}
The KKT conditions of Theorem~\ref{thm:kl-softmax-endpoints} give
$T_{iv} = \exp\big((f_i + g_v - C_{iv})/\varepsilon\big)$. A constant row
$C_{iv} \equiv c_i$ makes that factor common to the row, so
$T_{iv} \propto \exp(g_v/\varepsilon)$, and the row constraint $\sum_v T_{iv} = a_i$
normalizes it to Eq.~\ref{eq:ghost-weights}. The potentials $g$ solve the column
condition; row $i$ enters it only through Eq.~\ref{eq:ghost-weights}, in which its
constant costs have canceled against the row constraint, so $g$ does not depend on
row $i$'s costs, and in particular not on $R_i$. Writing
$w_v := \mathrm{softmax}_v(g_v/\varepsilon)$ and normalizing by the row mass,
we obtain the block displacement
$\Delta x_i=\ell\sum_v(T_{iv}/a_i)R_iv_v=\ell R_i\sum_vw_vv_v$.
Since $\ell>0$ and $R_i$ is invertible, this displacement is nonzero exactly
when $\sum_v w_vv_v\neq0$. On a simplex, this is equivalent to nonuniform
weights because the vertices have no relation beyond $\sum_vv_v=0$.
On an orthoplex, it is equivalent to unequal weights in at least one
antipodal pair. Under Haar rotations,
$\E_R[\Delta x_i]=\ell\,\E[R]\sum_v w_v v_v = 0$
because $\E[R] = 0$ under Haar measure on $\mathrm{SO}(d_p)$: $\E[R] = Q\,\E[R]$
for every $Q \in \mathrm{SO}(d_p)$ forces $\E[R] = 0$ for $d_p \geq 2$.
For any other rotation law, the same expectation formula gives the claimed
criterion for nonzero mean motion.
\end{proof}

\paragraph{The conditional crossing rate.}
Assume throughout, as the proposition does, that row $i$ is constant for every rotation
in the support, so that $w$ does not depend on $R_i$. Write $m_w := \sum_v w_v v_v$ and
$x_t$ for the constant-cost particle's displacement from its entry point, $x_0 = 0$, so that
$x_{t+1} - x_t = r_s\varepsilon\, R_t m_{w_t}$ with $R_t$ Haar and independent across
steps; the $1/a_i$ of Eq.~\ref{eq:bary-proj} has already normalized the row. Isotropy of Haar measure gives
$\E[R\,a a^\top R^\top] = (\|a\|^2/d_p)\,I$, so the step has mean zero and total
variance $(r_s\varepsilon)^2\|m_{w_t}\|^2$.
Proposition~\ref{prop:simplex-detectability} gives
$\|m_w\|^2 = \tfrac{d_p+1}{d_p}\,\|w - \bar w\|^2$ with $\bar w$ the uniform
vector: the covariance is nondegenerate exactly when $w$ is non-uniform, as required by the weight-spread assumption. The identity is the simplex's, and the
rate below is stated for the simplex alone: an antipodal frame carries the extra
relations $v + (-v) = 0$, on which $w = (0.4,0.4,0.1,0.1)$ over
$(\pm e_1, \pm e_2)$ is non-uniform with $m_w = 0$, so no floor on
$\|w - \bar w\|$ bounds the step below.
The process
$\|x_t\|^2 - \sum_{u<t} (r_s\varepsilon)^2\|m_{w_u}\|^2$
is a martingale. First stop at $\tau\wedge n$, where $\tau$ is the exit time
from $B(0,W)$ and $n$ is deterministic. Bounded increments give
$\|x_{\tau\wedge n}\|\leq W+r_s\varepsilon$, and optional stopping at this
bounded stopping time yields
\[
 \E\|x_{\tau\wedge n}\|^2
 =\E\!\sum_{u<\tau\wedge n}(r_s\varepsilon)^2\|m_{w_u}\|^2.
\]

Under the floor
$\|w_t - \bar w_t\| \geq w_0$ the summands are at least
$\underline{s}^2 := \tfrac{d_p+1}{d_p} (r_s\varepsilon)^2 w_0^2$, so
$\underline{s}^2\,\E[\tau\wedge n] \leq (W + r_s\varepsilon)^2$.
Monotone convergence gives $\E\tau<\infty$ and hence almost-sure exit.

Dominated convergence on the stopped position and monotone convergence on the
variance sum give the unstopped identity; under a cap
$\|w_t - \bar w_t\| \leq w_1$ they are at most
$\bar s^2 := \tfrac{d_p+1}{d_p} (r_s\varepsilon)^2 w_1^2$, so $\bar s^2\,\E[\tau] \geq W^2$.
Together,
\[
\frac{W^2}{\bar s^2} \;\leq\; \E[\tau] \;\leq\; \frac{(W + r_s\varepsilon)^2}{\underline{s}^2}.
\]
The floor and cap give an order-$W^2/(r_s\varepsilon)^2$ exit time from
this ball when $W\geq r_s\varepsilon$. For a plateau containing the ball,
its exit time is at least the ball exit time, so only the lower bound transfers.
An upper bound for exit from the entire plateau needs a bound on its outer extent
and the variance floor throughout that region. The inradius alone is insufficient.

\subsection{Displacement Under Concentrated Weights}\label{sec:prop-fragility}\label{app:proof-fragility}

A small temperature concentrates a cost row on its cheapest vertex when the
margin is large enough. The following bound describes the resulting displacement;
it does not imply that lowering the temperature always destabilizes training.

\begin{proposition}[Displacement envelope under a concentrated softmax]\label{prop:fragility}
Consider the core softmax update with unit vertices and a realized cost
matrix $\mC$ (Eq.~\ref{eq:cost-matrix}). Assume each row has a unique minimum, with margin
$\Delta^{\mathrm{mar}} := \min_i (C_{i, v^{(2)}_i} - C_{i, v^{(1)}_i}) > 0$, with $v^{(1)}_i$ and
$v^{(2)}_i$ the indices of its cheapest and second-cheapest vertices.
This margin differs from the row spread $\Delta_i$ of
Section~\ref{sec:thm-piecewise-smooth}. Let $\varrho \in (0, 1/2)$ and
take $\eta_{\max} = 0$. With step jitter, both ends of the envelope below
scale by the realized $1+\eta_i$. If
$\varepsilon \leq \Delta^{\mathrm{mar}} / \log(V \varrho^{-1})$, the softmax assigns mass at least
$1 - \varrho$ to a single vertex per particle, and
\begin{equation}\label{eq:fragility-bound}
(1 - 2\varrho)\, r_s \varepsilon
\;\leq\;
\big\| x_i^{(t+1)} - x_i^{(t)} \big\|_2
\;\leq\;
r_s \varepsilon
\qquad\text{for every block }i.
\end{equation}
The displacement magnitude thus lies within a deterministic envelope while its
direction is resampled every step. For a fixed concentration tolerance, the displacement bound depends on
$r_s\varepsilon$. The temperature threshold also depends on the row margin
and on $V$, which can depend on the particle dimension. The step direction need not be uniform on the
$\mathrm{SO}(d_p)$-orbit of a unit vertex: the winning vertex index is itself a
function of the realized rotation.
\end{proposition}

The envelope bounds the length of one step. It does not keep the trajectory
near a minimum: steps of that size can form a random walk whose typical
distance grows as $\sqrt t\,r_s\varepsilon$. Proposition~\ref{prop:stable-schedules}
therefore studies \emph{drift}, the expected change toward or away from the
minimum, on a quadratic loss. Appendix~\ref{app:theory-numerics} compares this model with
a walk that has the same step envelope and no restoring drift.

\paragraph{Step 1: Single-vertex concentration of softmax.}
For row $i$ with $C_{i, v^{(1)}_i} \leq C_{i, v^{(2)}_i} \leq \ldots \leq C_{i, v^{(V)}_i}$
and row margin $\Delta^{\mathrm{mar}}_i := C_{i, v^{(2)}_i} - C_{i, v^{(1)}_i} \geq \Delta^{\mathrm{mar}}$
(distinct from the row \emph{spread} $\Delta_i$ of Theorem~\ref{thm:exact-step}, which is
the best-to-worst difference and is not used here),
the softmax weight on the best vertex is
\[
T^*_{i, v^{(1)}_i}/a_i = \frac{e^{-C_{i, v^{(1)}_i}/\varepsilon}}{\sum_v e^{-C_{i,v}/\varepsilon}}
\geq \frac{1}{1 + (V-1) e^{-\Delta^{\mathrm{mar}}_i/\varepsilon}}
\geq \frac{1}{1 + (V-1) e^{-\Delta^{\mathrm{mar}}/\varepsilon}}.
\]
Setting $(V-1) e^{-\Delta^{\mathrm{mar}}/\varepsilon} \leq \varrho/(1-\varrho)$ requires
$\varepsilon \leq \Delta^{\mathrm{mar}} / \log\big((V-1)(1-\varrho)/\varrho\big)$, and since
$(V-1)(1-\varrho)/\varrho \leq V\varrho^{-1}$ for $\varrho \leq 1/2$ the hypothesis
$\varepsilon \leq \Delta^{\mathrm{mar}} / \log(V \varrho^{-1})$ of the proposition is the stronger
of the two and therefore sufficient. Thus
$T^*_{i, v^{(1)}_i}/a_i \geq 1 - \varrho$.

\paragraph{Step 2: Lower bound on displacement.}
The barycentric step (Eq.~\ref{eq:bary-proj}) for particle $i$ is
$\Delta x_i = r_s \varepsilon \sum_v (T^*_{iv}/a_i)\, R_i v_v$.
Decomposing into the dominant term plus residual,
$\Delta x_i = r_s \varepsilon (1 - \varrho_i)\, R_i v_{v^{(1)}_i} + r_s \varepsilon\, \psi_i$
with $\varrho_i \leq \varrho$ and $\|\psi_i\|_2 \leq \varrho$ (since the
remaining mass spreads over polytope vertices of unit Euclidean norm).
By the reverse triangle inequality,
$\|\Delta x_i\|_2 \geq (1 - \varrho_i)\, r_s \varepsilon - \varrho r_s \varepsilon
= (1 - \varrho_i - \varrho) r_s \varepsilon \geq (1 - 2\varrho) r_s \varepsilon$,
which is Eq.~\ref{eq:fragility-bound}. For the upper bound, the barycentric step is
a convex combination of unit vectors scaled by $r_s\varepsilon$, so
$\|\Delta x_i\|_2 \leq r_s\varepsilon$ by the triangle inequality.

\paragraph{Scope of the envelope.}
The winning index depends on the rotation through the evaluated costs.
The displacement bound therefore does not imply a uniform direction law.
It also does not bound the accumulated excursion: independent isotropic
increments of length $s$ have root-mean-square displacement $s\sqrt{t}$.
A confinement argument requires a drift condition, such as the conditional
quadratic calculation in Appendix~\ref{app:proof-stable-schedules}.\qed

\subsection{Nonconservativity of the Nonlinear Mean}\label{app:nonlinear-field}

Use two two-dimensional blocks, with active coordinates $x$ and $y$, and let
$L=x^3y^2$. The other coordinates do not affect the loss. Each block uses an
independent Haar-rotated orthoplex, one fixed probe radius $\rho>0$, and
temperature $\varepsilon>0$. Let $F$ be the expected softmax
direction. For a fixed rotation of the first block, write $q_v$ for the first
coordinate of its vertex. At $x=0$ its costs are $y^2\rho^3q_v^3$, so
\[
 \partial_yF_x(0,y)
 =-\frac{2y\rho^3}{\varepsilon}
       \E_R\operatorname{Cov}_{p_R}(q,q^3)<0\quad(y>0).
\]
Every softmax weight is positive, and
\[
 \operatorname{Cov}_p(q,q^3)
 =\tfrac12\sum_{v,w}p_vp_w(q_v-q_w)(q_v^3-q_w^3)>0.
\]
Each summand is nonnegative because its polynomial factor is
$(q_v-q_w)^2(q_v^2+q_vq_w+q_w^2)$. Some summand is positive because
an orthoplex has distinct projections onto every nonzero axis. In the second
block the costs are $x^3(y+\rho q_v)^2$; their $x$-derivatives vanish at
$x=0$. Thus $\partial_xF_y(0,y)=0$. The cross derivatives differ, and the
smooth field $F$ cannot locally be a scalar gradient.
Differentiation under the rotation integral is justified by bounded derivatives
on compact parameter neighborhoods and the compact rotation group. A smooth
cutoff equal to one on all nearby probe points makes $L$ globally bounded
without changing the argument. Multiplying by a fixed positive physical step
scale preserves the nonzero curl. This proof concerns the actual nonlinear
mean, separately from the linear witness in Appendix~\ref{app:coupling-witnesses}.

\section{Smoothing and Stationarity}\label{app:smoothing-proofs}

We derive the smoothing identities and error bounds used in
Section~\ref{sec:theory}, followed by their stationarity consequences
and the relation to the experimental configurations.

\subsection{Regularity of the Tail Kernel}\label{app:proof-smooth-kernel}

The descent bound needs a Lipschitz gradient, even when the original loss
jumps. The next lemma obtains this regularity from the smoothing kernel.
Here $m$, $Z$ and $\Ls_{\delta,i} = (\Ls * m)/Z$ are as in Eq.~\ref{eq:surrogate-def}, and
$q_{\mathrm{mix}}$ is the mixture radius density of Eq.~\ref{eq:mixture-radius}.

\begin{lemma}[Regularity of the surrogate along the probe plane]\label{lem:smooth-kernel}
Fix a particle $i$ with probe plane $E_i$, and let the radius jitter $\eta$ have a density
$q \in W^{2,1}$ compactly supported in $(-\eta_{\max},\eta_{\max})$,
$\eta_{\max}\in(0,1)$. Then for every bounded measurable $\Ls$: $\Ls_{\delta,i}$ is $C^{1,1}$ as
a function on $E_i$, with $\|\nabla_{E_i}\Ls_{\delta,i}\|_\infty = O(\|\Ls\|_\infty/\delta)$ and
$\Lambda_\delta = \mathrm{Lip}(\nabla_{E_i}\Ls_{\delta,i}) = O(\|\Ls\|_\infty\delta^{-2})$; and
the induced probe density is radial with profile
$\tilde p(\rho) = q_{\mathrm{mix}}(\rho)/(\sigma_{d_p-1}\rho^{d_p-1})$.
\end{lemma}

Less regular jitter gives weaker smoothness guarantees. For $\eta \sim \mathrm{Uniform}$, whose $q'$ is a
pair of Dirac masses, the tail kernel retains $\nabla m$ of bounded variation, so
$\Ls_{\delta,i}$ remains $C^{1,1}$ with $\Lambda_\delta$ controlled by the total variation
$|D^2 m|$, but $m \notin W^{2,1}$ and the $W^{3,1}$ conclusion of condition~(iv) need not hold; with no jitter at all $m$ is a
step function, $\nabla m$ is a surface measure, and $\Ls_{\delta,i}$ is merely Lipschitz.

\paragraph{The kernel.}
For a fixed unit vertex $v$ and a Haar rotation
$R\sim\mathrm{Uniform}(\mathrm{SO}(d_p))$, the direction $Rv$ is uniform
on the unit sphere of the particle plane. It is independent of the
radius $\rho=r_p\xi_k(1+\eta)\varepsilon$ for each probe fraction $\xi_k$. The probe
displacement $u = \rho\,\hat u$ with $\hat u \sim \mathrm{Unif}(S^{d_p-1})$
therefore has density, with respect to Lebesgue measure on the $d_p$-plane,
\begin{equation}\label{eq:annulus-density}
p(u) \;=\; \frac{q_{\mathrm{mix}}(\rho)}{\sigma_{d_p-1}\, \rho^{\,d_p-1}},
\qquad
q_{\mathrm{mix}}(\rho) \;=\; \frac{1}{K}\sum_{k=1}^{K} \frac{1}{r_p\xi_k\varepsilon}\,
q\!\Big(\frac{\rho}{r_p\xi_k\varepsilon} - 1\Big),
\qquad \rho = \|u\| ,
\end{equation}
where $\sigma_{d_p-1}$ is the surface measure of $S^{d_p-1}$ and each Jacobian
factor $1/(r_p\xi_k\varepsilon)$ comes from the change of variables
$\eta \mapsto \rho$ at scale $k$. The radius law is a mixture and not a single
annulus, because the probes sit at the $K$ fractions $\xi_k = k/(K+1)$ of the
probe radius (Eq.~\ref{eq:probe-points}) and the common jitter scales each fraction.
The support is contained in the union of the closed annuli
$r_p\xi_k(1-\eta_{\max})\varepsilon \leq \rho \leq r_p\xi_k(1+\eta_{\max})\varepsilon$,
which excludes the origin because $\eta_{\max} < 1$ and $\xi_1 > 0$;
consequently $\rho^{-(d_p-1)}$ is smooth and bounded on the support. Every
statement below holds for each mixture component and therefore for the mixture,
since a finite convex combination of $W^{2,1}$ densities is $W^{2,1}$.

\paragraph{(1) Regularity of the tail smoothing.}
To study $\Ls_{\delta,i} = (\Ls * m)/Z$, first differentiate the radial
tail kernel. $m$ is radial, $m(u) = M(\|u\|)$ with
$M(\rho) = \int_\rho^\infty \tilde p(s)\,ds$ and
$\tilde p(\rho) = q_{\mathrm{mix}}(\rho)/(\sigma_{d_p-1}\rho^{d_p-1})$ the radial profile of
Eq.~\ref{eq:annulus-density}, so $\nabla m = -\tilde p(\rho)\,u/\rho$.
If $\tilde p\in W^{k,1}$, the derivatives of $m$ through order $k+1$
belong to $L^1$. On the support of $\tilde p$, $\rho$ is
bounded away from zero, so $\rho^{-(d_p-1)}$ is smooth there and $\tilde p$ inherits
$q_{\mathrm{mix}}$'s regularity: $q \in W^{2,1}$ gives $\tilde p \in W^{2,1}$ and hence
$m \in W^{3,1}$. For bounded measurable $\Ls$, every derivative
$\partial^\alpha \Ls_{\delta,i} = (\Ls * \partial^\alpha m)/Z$ with $|\alpha| \leq 3$
exists and is bounded by $\|\Ls\|_\infty \|\partial^\alpha m\|_{L^1}/Z$
(by differentiation of translations in $L^1$), giving
$\|\nabla_{E_i} \Ls_{\delta,i}\|_\infty = O(\|\Ls\|_\infty/\delta)$,
$\Lambda_\delta = O(\|\Ls\|_\infty\delta^{-2})$, and bounded third derivatives along the active block.
These kernel derivatives do not establish the transverse loss regularity
required by condition~(vii-L). No $C^\infty$ hypothesis is used: a mollifier satisfies
$q \in W^{2,1}$, and the score $\nabla\log p$ may diverge at the support boundary
without affecting the argument, since only $\|\partial^\alpha m\|_{L^1}$ enters.

\paragraph{(2) The gradient identity.}
Differentiating the \emph{probe-law} smoothing
$\Ls_p(\theta) := \int \Ls(\theta + u)\,p(u)\,du
= \int \Ls(y)\, p(y - \theta)\,dy$, a different object from the lemma's tail smoothing
$\Ls_{\delta,i}$, under the integral,
\[
\nabla \Ls_p(\theta) = -\int \Ls(y)\, (\nabla p)(y - \theta)\, dy
= -\,\E_{u \sim p}\big[\Ls(\theta + u)\, \nabla \log p(u)\big] ,
\]
and differentiating Eq.~\ref{eq:annulus-density} gives the radial score
\begin{equation}\label{eq:annulus-stein}
\nabla \log p(u)
= \Big(\frac{q_{\mathrm{mix}}'(\rho)}{q_{\mathrm{mix}}(\rho)} - \frac{d_p-1}{\rho}\Big)\frac{u}{\rho},
\qquad
q_{\mathrm{mix}}'(\rho) = \frac{1}{K}\sum_{k=1}^{K} \frac{1}{(r_p\xi_k\varepsilon)^2}\,
q'\!\Big(\frac{\rho}{r_p\xi_k\varepsilon} - 1\Big).
\end{equation}
Eq.~\ref{eq:annulus-stein} is the score-weighted estimator of the smoothing by the probe
law $p$. It is not the estimator the algorithm computes, which weights by the unit
direction and by Theorem~\ref{thm:exact-step} differentiates the tail kernel $m/Z$
instead. The single-scale expression
$q'(\eta)/(q(\eta) r_p\varepsilon)$ corresponds to $K = 1$, $\xi_1 = 1$,
which differs from our probe fractions. Since $\xi_k \leq K/(K+1) < 1$, a mixture
component lies wholly below the single annulus whenever
$\xi_k(1+\eta_{\max}) < 1-\eta_{\max}$, as it does for the inner fractions in all reported settings. Those probes
lie outside the single-annulus support, where its score is undefined. Even
on the shared support, the two scores differ by a constant-order factor (Appendix~\ref{app:theory-numerics}).

\paragraph{Uniform jitter and boundary terms.}
If $\eta \sim \mathrm{Uniform}[-\eta_{\max},\eta_{\max}]$ then
$q = (2\eta_{\max})^{-1}\mathbf{1}_{[-\eta_{\max},\eta_{\max}]}$, $q'$ is a
difference of Dirac masses at $\pm\eta_{\max}$, so an ordinary score function omits singular surface terms on the two
boundary spheres,
and the \emph{probe-law} smoothing $\Ls * \check p$ is Lipschitz but not $C^1$ in
general. The tail smoothing retains more regularity: $\nabla m = -\tilde p(\rho)\,u/\rho$ has
bounded variation on the annulus, so $\Ls_{\delta,i}$ stays $C^{1,1}$ through the
translation bound
$\|\nabla m(\cdot - x) - \nabla m(\cdot - y)\|_{L^1} \leq |x-y|\,|D^2 m|(E_i)$,
but $D^2 m$ is a measure rather than an $L^1$ function, $m \notin W^{2,1}$, and the
$W^{3,1}$ conclusion of condition~(iv) need not hold. Uniform jitter
therefore fails that sufficient assumption even though the second-order
kernel bound remains available.

The uniform-ball identity,
$\nabla \E_{u \sim \mathrm{Ball}}[\Ls(\theta + \varepsilon u)]
= (d/\varepsilon)\,\E_{u \sim \mathrm{Sphere}}[\Ls(\theta+\varepsilon u)\,u]$
\citep[Lem.~2.1]{flaxman2005online}, is the divergence-theorem boundary term for
the \emph{uniform ball} kernel. It follows from integration against the distributional derivative of
the ball indicator. The ordinary interior score is zero and cannot recover
the boundary term; Eq.~\ref{eq:annulus-stein} as a function-valued score
formula does not apply to that density. The annulus kernel has non-vanishing interior score and two boundary
spheres of opposite orientation, so the ball formula differs in shape and not
only by a constant factor.

Appendix~\ref{app:theory-numerics} evaluates Eq.~\ref{eq:annulus-stein} against finite
differences for the single annulus and for the mixture, and quantifies the gap to the
ball estimator. The difference reflects their distinct smoothing targets.

\subsection{The Linear Identity and Softmax Remainder}\label{app:proof-piecewise-smooth}

Throughout, $\varepsilon>0$ and $r_p$ are fixed, so the smoothing and its
derivative bounds remain fixed. We give the kernel and remainder estimates
used in Theorem~\ref{thm:exact-step}, then record the conditional descent
inequality underlying Theorem~\ref{cor:finite-horizon}.

\paragraph{Probe and tail kernels.}
The costs evaluate the original, possibly discontinuous loss $\Ls$. The proof
does not expand $\Ls$ in a Taylor series or require probes to avoid
$\mathcal{D}$. The cost matrix in Eq.~\ref{eq:cost-matrix} is therefore
well defined for any bounded measurable loss.
The auxiliary section lemmas have separate hypotheses and are not used here.

Write the probe displacement as $u = \rho\,\hat u$ with $\hat u$ uniform on
$S^{d_p-1}$, independent of $\rho \sim q_{\mathrm{mix}}$ of Eq.~\ref{eq:mixture-radius}. The induced
density on the probe plane is radial,
$p(u) = \tilde p(\|u\|)$ with $\tilde p(\rho) = q_{\mathrm{mix}}(\rho)/(\sigma_{d_p-1}\rho^{d_p-1})$ and
$\sigma_{d_p-1}$ the unit sphere's surface measure. Define $M(\rho) := \int_\rho^\infty \tilde p$,
$m(u) := M(\|u\|)$ and $\Ls_{\delta,i} := (\Ls * m)/Z$ as in Eq.~\ref{eq:surrogate-def}.
The proof uses three properties of this kernel.

\emph{(a) $m$ is compactly supported and $W^{1,1}$, and $W^{k+1,1}$ when $q \in W^{k,1}$.} $\tilde p$
vanishes beyond the outer probe reach $\delta_{\mathrm{out}}$, so $M$ does too; $M$ is
absolutely continuous with $M' = -\tilde p$, so $\nabla m(u) = -\tilde p(\|u\|)\,\hat u$
almost everywhere, and $\nabla m \in L^1$. Because the probe radii are bounded away from zero, the factors
$1/\|u\|$ in further derivatives are bounded on the annulus. Thus
$q\in W^{1,1}$ suffices for $D^2m\in L^1$; condition~(iv) imposes
$W^{2,1}$ to obtain one more integrable derivative.

\emph{(b) $Z = \E\|u\|/d_p$.} In polar coordinates and integrating by parts,
\[
\int m \;=\; \int_0^\infty M(\rho)\,\sigma_{d_p-1}\rho^{d_p-1}\,\mathrm{d}\rho
\;=\; \Big[M(\rho)\,\tfrac{\sigma_{d_p-1}\rho^{d_p}}{d_p}\Big]_0^\infty
+ \int_0^\infty \tilde p(\rho)\,\tfrac{\sigma_{d_p-1}\rho^{d_p}}{d_p}\,\mathrm{d}\rho
\;=\; \frac{1}{d_p}\int_0^\infty \rho\,q_{\mathrm{mix}}(\rho)\,\mathrm{d}\rho ,
\]
the boundary term vanishing at both ends. With $\E\|u\| = r_p\bar\xi\E[1+\eta]\varepsilon$
this is $Z = \bar r_p\varepsilon/d_p$.

\emph{(c) $\Ls_{\delta,i}$ inherits the regularity needed for descent from $m$.} For bounded measurable $\Ls$, $\Ls * m$ is differentiable with
$\nabla(\Ls * m) = \Ls * \nabla m$ because $\|m(\cdot - h) - m + \langle h,\nabla m\rangle\|_{L^1} = o(\|h\|)$
for $m \in W^{1,1}$, and the same argument at one more order gives
$\Lambda_\delta := \mathrm{Lip}(\nabla \Ls_{\delta,i}) \leq \|\Ls\|_\infty \|D^2 m\|_{1}/Z = O(\|\Ls\|_\infty\delta^{-2})$
and $\|\nabla\Ls_{\delta,i}\|_\infty = O(\|\Ls\|_\infty/\delta)$, consistent with
Theorem~\ref{thm:smoothing-blowup}. Both remain finite and fixed at fixed $\varepsilon$ (condition~(v)).

\paragraph{Proof of Theorem~\ref{thm:exact-step}.}
\emph{The exact identity.} Since $\|v_v\| = 1$ and $R$ is Haar on $\mathrm{SO}(d_p)$
(supplied by Assumption~\ref{asm:main}(i)--(ii)), each
$Rv_v$ is uniform on $S^{d_p-1}$ and, because each $\rho_v$ is independent of $R$,
independent of the radius, so
$\rho_v R v_v$ has law $p$ for every $v$ and
\[
\E\Big[\sum_v \Ls(\theta + \rho_v Rv_v)\,Rv_v\Big]
\;=\; V\,\E_{u\sim p}\big[\Ls(\theta + u)\,\hat u\big]
\;=\; V\!\int \Ls(\theta+u)\,\tilde p(\|u\|)\,\hat u\,\mathrm{d}u
\;=\; -V\!\int \Ls(\theta+u)\,\nabla m(u)\,\mathrm{d}u ,
\]
which by (c) is $V\,\nabla_{E_i} (\Ls * m_i)(\theta) = V Z\,\nabla_{E_i}\Ls_{\delta,i}(\theta)$.
This proves Eq.~\ref{eq:stein-softmax-body} for bounded measurable $\Ls$.
It uses neither a Taylor expansion of the loss nor a condition that the
probes avoid the discontinuity set.

With $K = 1$ and no jitter, $m/Z$ is the uniform-ball density and the display is
\citet[Lem.~2.1]{flaxman2005online}. Here the pointwise conclusion requires care: without a jitter density $m \notin W^{1,1}$ and
$\Ls_{\delta,i}$ is only Lipschitz, so for a merely bounded measurable $\Ls$ the display
holds almost everywhere in the base point and everywhere once $\Ls$ is
continuous. Differentiability of the ball smoothing at a particular point alone
is insufficient: changing the loss on the sphere centered there changes the
sphere expectation while leaving the ball convolution unchanged. For example,
$\Ls(x)=x_1\mathbf{1}[\|x\|=1]$ is bounded and zero almost everywhere,
so its ball smoothing vanishes identically, but at the origin
$\E[\Ls(U)U]=e_1/d_p$ for $U$ uniform on the unit sphere.
The theorem assumes a jitter density, which avoids this problem and extends
the identity to the radial mixture. The resulting smoothing uses the tail
kernel $m/Z$, not the probe density.

\emph{The softmax remainder.} First take centered logits $z$ with
$a:=\|z\|_\infty\leq1$. Taylor's theorem gives
$e^{z_v}=1+z_v+r_v$, where $0\leq r_v\leq\exp(1)a^2/2$.
Writing $\bar r=V^{-1}\sum_vr_v$ and using $\sum_vz_v=0$,
\[
 \left|\mathrm{softmax}(z)_v-\frac{1+z_v}{V}\right|
 =\frac{|r_v-(1+z_v)\bar r|}{V(1+\bar r)}
 \leq\frac{3\exp(1)}{2V}a^2.
\]
This proves the bound with $\kappa_0=3\exp(1)/2$.
Now write $\tilde C_{i,v} := C_{i,v} - V^{-1}\sum_{v'} C_{i,v'}$,
so $\mathrm{softmax}(-C_{i:}/\varepsilon) = \mathrm{softmax}(-\tilde C_{i:}/\varepsilon)$ by
shift invariance, and $\max_v|\tilde C_{i,v}| \leq \Delta_i$. On $\{\Delta_i \leq \varepsilon\}$
the elementary bound
$|\mathrm{softmax}(z)_v - V^{-1}(1 + \tilde z_v)| \leq \kappa_0 V^{-1}\|\tilde z\|_\infty^2$
at $\|\tilde z\|_\infty \leq 1$ gives, after multiplying by $Rv_v$, summing, and using
$\sum_v v_v = 0$ to cancel the constant term,
\[
\Big\|\sum_v p_{iv}\,Rv_v \;+\; \frac{1}{V\varepsilon}\sum_v C_{i,v}\,R v_v\Big\|
\;\leq\; \kappa_0 \Big(\frac{\Delta_i}{\varepsilon}\Big)^{2} .
\]
Taking expectations and substituting the identity turns the second term into
$(Z/\varepsilon)\,\nabla_{E_i}\Ls_{\delta,i} = c\,\nabla_{E_i}\Ls_{\delta,i}$ with
$c = \bar r_p/d_p = r_p\E[1+\eta]/(2d_p)$, since
$\bar\xi = K^{-1}\sum_k k/(K+1) = 1/2$ for every $K$.
For mean-zero jitter this reduces to $c=r_p/(2d_p)$.
The pointwise remainder bound requires $\Delta_i\leq\varepsilon$.
Splitting the expectation according to this condition gives
\[
\|e_i\| \;\leq\; \kappa_0\,\E\big[(\Delta_i/\varepsilon)^2\,\mathbf{1}[\Delta_i \leq \varepsilon]\big]
\;+\; \Pr[\Delta_i > \varepsilon]
\;+\; \E\big[(\Delta_i/\varepsilon)\,\mathbf{1}[\Delta_i > \varepsilon]\big] ,
\]
For draws with $\Delta_i>\varepsilon$, we use two pointwise bounds. The step is a convex
combination of unit vectors, so $\|\sum_v p_v Rv_v\| \leq 1$, and centering the
row gives $\|(V\varepsilon)^{-1}\sum_v C_{iv}\,Rv_v\| \leq \Delta_i/\varepsilon$.
The restricted expectation of the linear term is bounded by the second of these,
not by a multiple of $\|\nabla_{E_i}\Ls_{\delta,i}\|$: only the full, unrestricted
expectation equals $-c\,\nabla_{E_i}\Ls_{\delta,i}$, and the restriction to
$\{\Delta_i > \varepsilon\}$ can prevent the cancellation used in that equality.
Unconditionally $\|e_i\| \leq 1 + c\|\nabla_{E_i}\Ls_{\delta,i}(\theta)\|$.
This proves the stated error estimates. \qed

Only $\|v_v\| = 1$ and $\sum_v v_v = 0$
enter above; tightness is not needed for the common constant $c$. For unit
tight frames, as the orthoplex and the simplex admitted by condition~(i) both are, the
identity
$\sum_v (Rv_v)(Rv_v)^\top = R(\sum_v v_v v_v^\top)R^\top = (V/d_p)I$, which holds pointwise
in $R$ and not merely in expectation, gives the same $c$ through the second moment: the
$1/V$ of the barycentric step cancels the $V/d_p$, and the probe fractions halve the
result. This agrees with Eq.~\ref{eq:stein-softmax-body}.

\paragraph{Cancellation on a smooth neighborhood.}
Suppose the loss is $C^{1,1}$ throughout the probe neighborhood, with
uniform bounds on its gradient and its gradient's Lipschitz constant.
Keep $\varepsilon$ fixed and let $r_p\to0$. Set $u_v=Rv_v$,
$\bar\rho=K^{-1}\sum_k\rho_k$, and
$a_v=\langle\nabla_{E_i}\Ls(\theta),u_v\rangle$.
The common radii give $C_v=\Ls(\theta)+\bar\rho a_v+O(r_p^2)$ and
$\sum_va_v=0$. Expanding softmax through second order yields
\[
 \sum_vp_{iv}u_v+\frac{1}{V\varepsilon}\sum_vC_vu_v
 =\frac{\bar\rho^2}{2V\varepsilon^2}\sum_va_v^2u_v+O(r_p^3).
\]
Each $u_v$ has a symmetric spherical marginal, so $\E[a_v^2u_v]=0$.
The radii are independent of $R$, and averaging therefore gives
$e_i=O(r_p^3)$. Restricting the expectation to a cost-dependent event
can destroy this cancellation; the covered remainder in
Corollary~\ref{cor:optimal-rp} requires its stated additional hypothesis.

\paragraph{Comparison with smoothing by the probe law.}
\begin{corollary}[Exactness across the discontinuity set]\label{cor:two-kernels}
Eq.~\ref{eq:stein-softmax-body} identifies the expected cost-weighted direction with a scaled gradient of the
\emph{tail} smoothing $\Ls_{\delta,i} = (\Ls * m)/Z$, without restricting probes to avoid $\mathcal{D}$. It does not identify the smoothing by the probe law,
which is a different functional even where $\Ls$ is smooth: on a plane wave the two
gradients have different Bessel multipliers and can point in opposite directions
(as shown below). A Taylor comparison requires
regularity throughout the region between the base point and its probes;
piecewise smoothness alone does not supply it across a jump.
\end{corollary}

The tail smoothing and the probe-law smoothing are different even for smooth
losses. In $d_p=2$, a plane wave $\sin\langle k,\theta\rangle$ smoothed on a
sphere of radius $\rho$ has multiplier $J_0(\|k\|\rho)$, while its uniform-ball
smoothing has multiplier $2J_1(\|k\|\rho)/(\|k\|\rho)$. At
$\|k\|\rho=3$ these have opposite signs. This can be checked by hand from
the power series
$J_0(3)=\sum_{j\geq0}(-1)^j(9/4)^j/(j!)^2$ and
$2J_1(3)/3=\sum_{j\geq0}(-1)^j(9/4)^j/[j!(j+1)!]$.
The term magnitudes decrease from $j=1$ onward, so alternating-series bounds give
\[
 J_0(3)\leq 1-\frac94+\frac{81}{64}-\frac{81}{256}
                 +\frac{729}{16384}=-\frac{4199}{16384}<0,
 \qquad
 \frac{2J_1(3)}3\geq 1-\frac98+\frac{27}{64}-\frac{81}{1024}
                 =\frac{223}{1024}>0.
\]
A radial mixture averages the
corresponding kernels with the radius weights specified by $m/Z$.
On a $C^{2,1}$ neighborhood, symmetry cancels the first-order term in the
Taylor expansion of $\nabla\Ls$, so the two smoothed gradients differ by
$O(r_p^2)$ at fixed $\varepsilon$. This local comparison requires smoothness
throughout the neighborhood. The exact tail-kernel identity requires no
Taylor expansion of the raw loss.

\paragraph{From the batch to the population.}
The preceding argument proves the identity for a fixed bounded measurable loss. In the descent argument
the loss evaluated at step $t$ is the minibatch loss $\Ls_t$, so the identity is applied
conditionally on $\mathcal{F}_t$ \emph{and} on the batch, under which the rotations and
radii are still fresh (condition~(viii)). The conclusion transfers to the population loss
by conditional Fubini: condition~(viii) makes $(\theta,\text{batch}) \mapsto \Ls_t(\theta)$
jointly measurable and bounds it on $\mathcal{K}_{\delta_{\mathrm{joint}}}$ by a
deterministic constant, so the kernel integral and the batch expectation may be
exchanged, and $\E[\Ls_t \mid \mathcal{F}_t] = \Ls$ turns the batch surrogate into
$\Ls_{\delta,i}$. This is the step that relates the sampled costs to the population objective.

\paragraph{Conditional descent and telescoping.}
The proof of Theorem~\ref{cor:finite-horizon} combines the exact identity
with the block discrepancy and softmax bounds. For reference, write
$g_i=\nabla_{E_i}\Ls_\delta(\theta_t)$, $w_i=\Phi_i-g_i$,
$h_i=-(V\varepsilon)^{-1}\sum_vC_{iv}R_iv_v$, and
$D_i=\sum_vp_{iv}R_iv_v$. Conditional on $\mathcal F_t$, let
$b_i=\E[h_i\mathbf1_{A_i}\mid\mathcal F_t]$ and
$\zeta_i=\E[(D_i-h_i)\mathbf1_{A_i^c}+D_i\mathbf1_{A_i}\mid\mathcal F_t]$.
Then
\begin{equation}\label{eq:step-error-decomposition}
 \E[D_i\mid\mathcal F_t]=-cg_i-cw_i-b_i+\zeta_i,
 \qquad
 \E[\Delta x_i\mid\mathcal F_t]=r_{s,t}\varepsilon\E[D_i\mid\mathcal F_t].
\end{equation}
With $q_i(t)=\E[(\Delta_i/\varepsilon)^2\mathbf1_{A_i^c}\mid\mathcal F_t]$,
$p_i(t)=\Pr(A_i\mid\mathcal F_t)$, and
$d_i=c(1+\bar\pi)\|w_i\|+\kappa_0q_i(t)+p_i(t)$, set
$\mathcal E_t=c^{-2}\sum_i d_i^2$. The descent lemma gives
\begin{equation}\label{eq:clarke-descent}
\begin{aligned}
 \E[\Ls_\delta(\theta_{t+1})-\Ls_\delta(\theta_t)\mid\mathcal F_t]
 &\leq-\frac{c(1-\bar\pi)}2r_{s,t}\varepsilon
       \|\nabla_{\mathcal S}\Ls_\delta(\theta_t)\|^2
       +\frac{c}{2(1-\bar\pi)}r_{s,t}\varepsilon \mathcal E_t
       +H_0r_{s,t}^2.
\end{aligned}
\end{equation}
Here $\E\sum_{t<T}r_{s,t}\mathcal E_t\leq BS_T$, with $B$ from
Eq.~\ref{eq:bias-floor}. Writing
$N_t=\|\nabla_{\mathcal S}\Ls_\delta(\theta_t)\|^2$ and
$A_T=A_0+H_0\sum_{t<T}r_{s,t}^2$, telescoping yields
\begin{equation}\label{eq:telescoped}
 \frac{c(1-\bar\pi)}2\varepsilon\sum_{t<T}r_{s,t}\E N_t
 \leq A_T+\frac{c\varepsilon}{2(1-\bar\pi)}
                   \E\sum_{t<T}r_{s,t}\mathcal E_t.
\end{equation}
The sampled-output bound follows by dividing by
$c(1-\bar\pi)\varepsilon S_T/2$. Appendix~\ref{app:infinite-schedule}
applies the same inequality to an infinite schedule.

\subsection{A Common Smoothing Through Background Perturbations}\label{app:joint-background}

Let $\omega_{d_p}$ be the volume of the unit ball in $\R^{d_p}$. The radial density satisfies
$q_{\mathrm{mix}}(\rho)=d_p\omega_{d_p}\rho^{d_p-1}\widetilde p(\rho)$.
Substituting this relation and $Z=\E\rho/d_p$ into
Eq.~\ref{eq:size-biased-balls} proves that formula, also for a discrete
radius law with the integral replaced by a sum. In particular, the no-jitter
radii $r_p\varepsilon k/(K+1)$ have normalized weights
$2k/[K(K+1)]$.

To sample a normalized tail kernel, draw a radius from the size-biased law
$\rho q_{\mathrm{mix}}(\rho)/\E\rho$, then draw uniformly in its ball.
For a bounded radius law, rejecting an ordinary radius draw with probability
$1-\rho/\rho_{\max}$ implements the first step. A uniform unit direction
times $\rho W^{1/d_p}$, $W\sim\mathrm{Uniform}[0,1]$, implements the second.
Independent draws in each block give the required product distribution.

For block $i$, draw $U_{-i}$ from the other normalized tail kernels,
independently of its active probes and batch, and evaluate
\[
 \widetilde C_{iv}=K^{-1}\sum_k
 \Ls(\theta+\Pi(U_{-i}+\rho_k\iota_iR_iv_v)).
\]
The insertion map $\iota_i$ places the active probe in block $i$. Conditioning on
$U_{-i}$, Theorem~\ref{thm:exact-step} identifies the active-block linear
mean. Integrating the background differentiates the product smoothing,
proving Proposition~\ref{prop:joint-background}. Boundedness and the kernel
regularity justify Fubini and differentiation. One product draw may supply
$U_{-i}$ for all rows, because only the stated marginals and independence
from each row's active probes enter the expectation.

In the proof of Theorem~\ref{cor:finite-horizon}, replace $C$ by $\widetilde C$.
Then $\Phi_i=g_i$, hence $w_i=0$, while the deterministic displacement bound,
kernel-side smoothness, and the remaining softmax and saturation terms are
unchanged. No transverse derivative of the partially smoothed loss is used.
Compactness and the sampling and saturation assumptions must still hold for
this modified process. The finite-horizon output rule follows by the same
telescoping argument. These results do not assert that background perturbations
improve training accuracy or reduce total variance.

\subsection{Schedules with Summable Squared Steps}\label{app:infinite-schedule}
The same argument also bounds the smallest gradient encountered under an
infinite schedule with summable squared step multipliers.

\begin{theorem}[Convergence rate to subspace stationarity of the smoothed loss]\label{thm:piecewise-smooth}
Let $\Ls$ be bounded and measurable and suppose Assumption~\ref{asm:main}
holds at every finite horizon with common constants. Assume additionally
$\sum_{t\geq0}r_{s,t}=\infty$ and $\sum_{t\geq0}r_{s,t}^2<\infty$.
Then, with $\mathcal S=\mathrm{range}(\Pi)$,
\begin{equation}\label{eq:pws-rate}
 \E\big[\min\nolimits_{t<T}
       \|\nabla_{\mathcal S}\Ls_\delta(\theta_t)\|^2\big]
 \leq\frac{C}{(1-\bar\pi)S_T}
       +\frac{B(r_p,\varepsilon,\bar\sigma)}{(1-\bar\pi)^2},
\end{equation}
where $B$ is defined in Eq.~\ref{eq:bias-floor},
$C=2C'/(c\varepsilon)$, and
$C':=A_0+H_0\sum_{t\geq0}r_{s,t}^2<\infty$, with $A_0,H_0$ as in
Theorem~\ref{cor:finite-horizon}.
\end{theorem}

\begin{proof}
Theorem~\ref{cor:finite-horizon} applies at every finite horizon. Its
numerator $A_0+H_0\sum_{t<T}r_{s,t}^2$ is at most $C'$, and the minimum
is at most the weighted average, pathwise. Taking expectations proves
the result. The deterministic weights also give the same bound for
$\min_{t<T}\E\|\nabla_{\mathcal S}\Ls_\delta(\theta_t)\|^2$.
\end{proof}

For $r_{s,t}=r_0(t+1)^{-(1/2+\gamma)}$, $\gamma\in(0,1/2)$, the decreasing
term is $O(T^{-(1/2-\gamma)})$. The boundary $\gamma=0$ has a nonsummable
squared schedule; applying the finite-horizon bound instead gives
$O(T^{-1/2}\log T)$. Neither result guarantees a decrease at every step,
since the approximation errors can exceed the descent term.

\begin{remark}[Effect of a nonzero bias floor]\label{rem:no-as-claim}
When $B>0$ and $\sum_t r_{s,t}=\infty$, the additive error allowance
proportional to $r_{s,t}B$ is not summable. The usual Robbins--Siegmund
argument \citep{robbins1971convergence} therefore gives no almost-sure
convergence conclusion from Eq.~\ref{eq:clarke-descent}.
To see the obstruction, take $r_t=t^{-3/4}$ for $t\geq1$ and
$Y_{t+1}=Y_t-r_t(a_t-1)$, where $a_t\in\{0,2\}$ switches to $0$
when $Y_t\leq1.6$, switches to $2$ when $Y_t\geq2.4$, and otherwise
retains its value. Starting with $Y_1=2$ and $a_1=0$ keeps
$Y_t\in[0.6,3.4]$. Since $r_t\to0$ and $\sum_t r_t=\infty$,
the sequence crosses the band in both directions infinitely often.
Thus a nonnegative sequence can satisfy an inequality of this form
without converging. Our stationarity conclusions are bounds in expectation.
\end{remark}

\subsection{Hypotheses and Constants}\label{app:theory-remarks}

\begin{remark}[Fr\'echet and limiting subdifferentials at a jump]\label{rem:jump-subdifferentials}
For the specific unit step $f(x)=\mathbf{1}[x\geq0]$ at the origin,
$\hat\partial f(0)=\varnothing$ whereas $\partial f(0)=\{0\}$ for the
$f$-attentive limiting subdifferential. Thus these two stationarity tests give
different answers at the same jump. A direct calculation is given below.
\end{remark}

The following remarks explain the hypotheses, constants, and limitations of
Section~\ref{sec:theory}.

\subsubsection{Saturation and Block Interactions}

\paragraph{Reading cost spread from the weights.}\label{app:weight-spread}
For an uncoupled softmax row with finite costs, the weights determine the
cost spread relative to temperature:
\[
 \frac{\Delta_i}{\varepsilon}
 =\log\frac{\max_v p_{iv}}{\min_v p_{iv}}.
\]
Thus nearly uniform weights indicate that the probes reveal little loss
variation at the chosen temperature. More precisely, if
$\omega_i:=\max_v|Vp_{iv}-1|<1$, then
\[
 \frac{\Delta_i}{\varepsilon}
 \leq\log\frac{1+\omega_i}{1-\omega_i}.
\]
For a coupled transport plan, the column potentials $g_v$ also enter:
the range of $g_v-\varepsilon\log p_{iv}$ over $v$ equals $\Delta_i$.
These identities hold in exact arithmetic; computing from the costs or
log-weights avoids numerical underflow. They describe a realized cost row.
They do not establish the conditional saturation and remainder bounds
required by Assumption~\ref{asm:main}.

\paragraph{The mass part of condition~(viii).}
Write $\pi(t)$ for the smallest nonnegative coefficient satisfying
Eq.~\ref{eq:crossing-fraction} at step $t$ for all blocks. A zero denominator
contributes zero when its numerator also vanishes, and $+\infty$ otherwise.
Condition~(viii) requires the resulting coefficient to be at most $\bar\pi<1$.

It is an inequality between norms rather than a ratio
because the ratio is $0/0$ where the step vanishes (the plateau of
Theorem~\ref{prop:plateau-freeze}), and it is imposed at each $t$
because $\pi(t)$ multiplies $\|\nabla_{\mathcal{S}}\Ls_\delta(\theta_t)\|^2$ in
the descent inequality and an $r_{s,t}$-weighted average of $\pi$ does not bound the
$r_{s,t}\|\nabla\|^2$-weighted one.
The saturation probability and directional coefficient are
$\mathcal{F}_t$-measurable. A run records realized cost
spreads and saturation indicators, which do not establish the required
conditional bounds (Appendix~\ref{sec:hypothesis-coverage}).

\paragraph{The loss-side constant of condition~(vii).}
If the loss is $C^3$ on the enlarged neighborhood with bounded third
derivatives, that bound controls $\Lambda_3$ independently of the smoothing radius. Without
that regularity, boundedness alone gives the weaker estimate
\[
 \|\Phi-g\|
 \leq 2\sqrt P\,\|\Ls\|_\infty\frac{\|\nabla m\|_1}{Z}
 =O(\sqrt P\,\|\Ls\|_\infty/\delta).
\]
Indeed, each block field and its transverse convolution have norm at most
$\|\Ls\|_\infty\|\nabla m\|_1/Z$. The loss bound must hold on the
entire joint probe neighborhood. An oscillation measured only within one
block's variance radius does not control values sampled by the other
blocks. This estimate remains finite for a bounded measurable loss, but
its squared contribution grows as the smoothing radius shrinks.

\paragraph{The linearization term and the measured $\bar q$.}
We require the bound on the covered softmax remainder at every step, just
as we do the saturated-mass bound: the descent sum averages
\emph{squares} of per-step remainders. Where the
entire probe neighborhood is $C^{1,1}$ and rows remain unsaturated,
rotational cancellation gives an $O(r_p^3)$ mean remainder at fixed
$\varepsilon$. Cancellation after a cost-dependent event restriction is
an additional assumption in Corollary~\ref{cor:optimal-rp} and the term
reduces to $O(P d_p^2 r_p^4)$; this uses a bound on $e_i$, beyond
$\bar q = \Theta(r_p^2)$. The reduction is not automatic: a probe crossing a wall of jump
$a < \varepsilon$ leaves the row covered while $\|e_i\| =
\Theta(a/(2\varepsilon) - \tanh(a/(2\varepsilon)))$ has no dependence on $r_p$
($\Ls = a\mathbf{1}[x_1 > 0]$; Appendix~\ref{app:tanh-remainder}). The term is
therefore stated in $\bar q$;
Appendix~\ref{sec:hypothesis-coverage} reports the measured
$\max_{i,t}\Delta_i/\varepsilon$, whose square bounds its realized counterpart.

\paragraph{The saturation term and $\bar\sigma$.}
Under the additional hypothesis
$\bar\sigma = O(\sqrt{P} r_p\varepsilon/L_{\mathrm{exc}})$, with $L_{\mathrm{exc}}$
the excursion length between visits to the wall tube (Lemma~\ref{lem:occupancy}),
the term becomes $O(P^{3/2} d_p^2\varepsilon/(r_p L_{\mathrm{exc}}))$, linear
in $\varepsilon$ and growing as the probe radius shrinks.
Corollary~\ref{cor:optimal-rp} balances this term against the other errors.
A uniform-base-point model
can give this scaling if saturation is confined to the wall tube and its
volume fraction has the stated order. Those are additional hypotheses.
Lemma~\ref{lem:occupancy} alone supplies neither this conditional
probability bound nor a conversion from path-length to step-radius weights.

\paragraph{Sharpness of the coupling term.}
The $P^3$ dependence is attained by the examples below, which support
Remark~\ref{rem:block-coupling}. Summing block errors also gives the factor
$P$ in the linearization term. $\Lambda_3$ is the only point where regularity of
$\Ls$ itself enters; where it is unavailable, only the weaker bounded-loss estimate above
is available. Corollary~\ref{cor:optimal-rp} studies a vanishing upper bound
under additional uniform assumptions.

\paragraph{Examples for Remark~\ref{rem:block-coupling}.}\label{app:coupling-witnesses}
Each example uses independent mean-zero block noise with per-coordinate
variance $\delta_{\mathrm{blk}}^2$. We write scalar blocks for clarity;
the same examples embed in vector blocks by using one active coordinate.
\emph{Non-conservativity.} For $\Ls(x,y) = x^3y^2$,
$\Phi_x = \E[3(x+u)^2]\,y^2 = 3(x^2+\delta_{\mathrm{blk}}^2)y^2$ and $\Phi_y = 2x^3y$, so
$\partial_y \Phi_x - \partial_x \Phi_y = 6\delta_{\mathrm{blk}}^2 y \neq 0$.
\emph{Why the bound uses derivatives of the block field.} For $\Ls(x,y) = \sin(\ell x_1)\sin(k y_1)$ with
kernel transform $\hat\kappa(\zeta) := \E[\cos(\zeta u)]$,
$\Phi_x = \ell\cos(\ell x_1)\,\hat\kappa(\ell)\sin(k y_1)$, and its transverse smoothing
multiplies by $\hat\kappa(k)$, so the discrepancy has amplitude
$\ell\,|\hat\kappa(\ell)|\,(1 - \hat\kappa(k)) \to \ell\,|\hat\kappa(\ell)| \neq 0$ as
$k \to \infty$, while every third derivative of the jointly smoothed loss carries a
factor $k^{b}\hat\kappa(k)$ with $b \leq 3$ and vanishes in the same limit, the
$W^{3,1}$ kernel of condition~(iv) having $\hat\kappa(k) = o(k^{-3})$.
\emph{The factor $P$ within a block.} For $\Ls(x) = \tfrac12 x_1\sum_{j\geq2}x_j^2$:
smoothing block $1$ alone leaves $\partial_{x_1}\Ls = \tfrac12\sum_{j\geq2}x_j^2 = 0$
at the origin, while joint smoothing gives
$\partial_{x_1}\E\Ls(x+u) = \tfrac12\sum_{j\geq2}(x_j^2 + \delta_{\mathrm{blk}}^2)
= (P-1)\delta_{\mathrm{blk}}^2/2$ there.
\emph{The factor $P^3$ in the floor.} For
$\Ls(x) = \tfrac12\sum_i x_i\sum_{j\neq i}x_j^2$, every point supported in a single
block has $\Ls = 0$, so every cost row at the origin is constant and the step is
exactly zero by Theorem~\ref{prop:plateau-freeze}; joint smoothing gives
$\E\Ls(x+u) = \Ls(x) + \tfrac{P-1}{2}\delta_{\mathrm{blk}}^2\sum_i x_i$, whose
gradient at the origin has every entry $(P-1)\delta_{\mathrm{blk}}^2/2$ and norm
$\sqrt{P}(P-1)\delta_{\mathrm{blk}}^2/2 = \Theta(P^{3/2}\delta^2)$ at
$\Lambda_3 = O(1)$, ruling out a $P^2\delta^4\Lambda_3^2$ floor under these assumptions.
(The witnesses are polynomials; a smooth cutoff equal to one on the joint-kernel
neighborhood makes them admissible for Definition~\ref{def:pwc-loss} without
changing any quantity used.)

\paragraph{The sub-temperature wall remainder.}\label{app:tanh-remainder}
For $\Ls = a\,\mathbf{1}[x_1 > 0]$ and a probe pair straddling the wall, the cost
difference within the antipodal pair is $a$, the softmax puts weights
$e^{\mp a/2\varepsilon}/(e^{a/2\varepsilon} + e^{-a/2\varepsilon})$ on the two signs,
and the realized step along the pair axis is $-\tanh(a/2\varepsilon)$ against the
linearization's $-a/2\varepsilon$; the remainder is
$\Theta(a/(2\varepsilon) - \tanh(a/(2\varepsilon)))$, a function of $a/\varepsilon$
alone, carrying no power of $r_p$. A jump $a \leq \varepsilon$ leaves the row
covered ($\Delta_i \leq \varepsilon$) while the remainder stops shrinking with the
probe radius, which is why the first floor term is stated in the measured $\bar q$
rather than reduced to $O(r_p^4)$ unconditionally.

\paragraph{Single-block and joint clearance.}
Take a tail kernel positive throughout $B(0,\delta_{\mathrm{out}})$.
Let $\mathcal{D} = \{n^\top x = 0\}$ with $n = (e_1 + e_{d_p+1})/\sqrt2$, so the normal
has equal mass in blocks $1$ and $2$. A single-block displacement of norm at most
$\delta_{\mathrm{out}}$ changes $n^\top x$ by at most $\delta_{\mathrm{out}}/\sqrt2$,
while two simultaneous block displacements can change it by $\sqrt2\,\delta_{\mathrm{out}}$.
For $\delta_{\mathrm{out}}/\sqrt2 < n^\top\theta < \sqrt2\,\delta_{\mathrm{out}}$,
every individual block's probe ball avoids $\mathcal D$, while the joint
product kernel still crosses it. On $\Ls = \mathbf{1}[n^\top x \geq 0]$,
the branches are constant and every blockwise cost row is constant, yet
the joint smoothing has a nonzero gradient. A bound that sets the coupling
error to zero whenever all individual probe balls avoid $\mathcal D$
would therefore miss this discrepancy. This is why the coupling term of Eq.~\ref{eq:bias-floor} takes its
constant at the joint radius $\delta_{\mathrm{joint}} = \sqrt{P}\,\delta_{\mathrm{out}}$,
also used in Corollary~\ref{cor:goldstein}.

\subsubsection{Geometry at a Jump}

\paragraph{Folded and disconnected sheets.}
Connectedness and reach exclude two geometric obstructions. If $S$
is allowed to be disconnected within $B(\theta^\dagger, r_{\mathrm{iso}})$, then
$\Ls = \mathbf{1}_{(0,a)}$ with $S = \{0, a\}$ satisfies the isolation identity at
arbitrarily large $r_{\mathrm{iso}}$, since the complement is still a disjoint union of
two open sets, and the bound fails for $a \ll h$ by the cancellation described below. If $S$ is
connected but has only a second-fundamental-form bound $\kappa$ rather than a positive
reach, a sheet may fold back to within $h$ of itself while remaining locally flat
to order $\kappa$, and the normal segment the coupling argument uses then crosses
$\mathcal{D}$ more than once. Requiring $S$ and both sides to be connected, with reach
at least $\tau$, excludes both obstructions, and $\tau$ replaces $1/\kappa$ throughout
Theorem~\ref{thm:smoothing-blowup}.

\paragraph{The isolation radius and the two-sided trace.}
The admissible smoothing radius depends on $r_{\mathrm{iso}}$: only inside
$B(\theta^\dagger, r_{\mathrm{iso}})$ does $S$ contain every point of $\mathcal{D}$. A second
sheet closer than the straddle can cancel the first exactly. In one dimension
$\Ls = \mathbf{1}_{(0,a)}$ has an attained jump $J = 1$ at the origin with $G = 0$,
$\kappa = 0$ and $r_{\mathrm{iso}} = a$, yet for a scaled kernel $\mu_h$,
$\|(\Ls * \mu_h)'\|_\infty = \|\mu_h - \mu_h(\cdot - a)\|_\infty
\leq a\|\mu'\|_\infty h^{-2}$, which is \emph{below} $J/(4h)$ once
$a \ll h$: the two jumps cancel under smoothing. Thus $r_{\mathrm{iso}}$ enters the
admissible range of $h$ in Theorem~\ref{thm:smoothing-blowup} and cannot be
dropped from it.

Pointwise oscillation alone is also insufficient. The function
$\Ls = \mathbf{1}[\theta_1 = 0]$ has oscillation $1$ at every point of
$\{\theta_1 = 0\}$, yet its smoothing by any absolutely continuous kernel is
identically zero, so no blow-up occurs. The theorem requires a surface patch with a positive two-sided trace gap. The architectural threshold
sets in Appendix~\ref{app:discontinuity-sets} require this additional loss-level
condition before the theorem applies.

\paragraph{Subdifferentials at an attained jump.}
To verify Remark~\ref{rem:jump-subdifferentials}, consider
$f = \mathbf{1}[x \geq 0]$ at $0$. The difference quotient
$(f(x) - f(0) - vx)/|x| \to -\infty$ as $x \uparrow 0$ for every $v$, so the
Fr\'echet subdifferential is empty, $\hat\partial f(0) = \varnothing$.
The limiting subdifferential gives
$\partial f(0) = \{0\}$, from locally constant points on the right with
$f(x)\to f(0)$. It therefore declares this jump point stationary.

\paragraph{$\Lambda_h$ as a seminorm.}
$\Lambda_h$ is the second-difference seminorm of
$D(r)=\Ls_h(\theta^\dagger+rn)$, equal to $\mathrm{Lip}(D')$ when $D\in C^{1,1}$.
For the uniform kernel in one dimension, $D$ is piecewise affine: its
classical second derivative is zero wherever it exists, but its
second-difference seminorm is infinite. The descent inequality instead
requires a Lipschitz bound for the full gradient in the searched coordinates.
That bound dominates the normal-segment seminorm when the segment lies in
those coordinates. The scaling hypothesis is needed only for the upper bound: a
kernel supported in $B(0,h)$ yet concentrated at scale $h^2$ obeys both lower
bounds while having smoothness constant $\Theta(J/h^4)$. The constants in
Eq.~\ref{eq:blowup} are absolute and dimension-free.

\subsubsection{Error Bounds and Their Interpretation}

\paragraph{The $\varepsilon$-independent error term.}
The radius conventions in Remark~\ref{rem:radius-parameterization} are
equivalent up to a constant only at fixed $\varepsilon$. When $\varepsilon$
decays, both still give the same conclusion for the curvature allowance: with multiplier radii the two factors of $\varepsilon$ in
$\Lambda_h r_s^2\varepsilon^2$ cancel against $h = \Theta(r_p\varepsilon)$,
while with directly scheduled lengths the allowance is
$O(\ell_t^2/(\rho_t^{\mathrm{base}})^2)$, where
$\rho_t^{\mathrm{base}}$ is the probe scale before the fractions $\xi_k$.
Neither expression has an additional temperature factor at fixed length ratio.

\paragraph{Weighting of condition~(viii).}
The infinite-schedule result requires the bound at every finite horizon.
A $\limsup$ condition alone would leave an uncontrolled initial segment. The weighting by $r_{s,t}$ is the quantity the descent sum
accumulates. Both quantities the condition bounds are conditional on $\mathcal{F}_t$ and so
are averages rather than indicators, the saturation event
$A_i(t) = \{\Delta_i(t) > \varepsilon\}$ not being $\mathcal{F}_t$-measurable. A geometric event in
its place would be $\mathcal{F}_t$-measurable and $\{0,1\}$-valued, so a pointwise bound strictly below one would exclude the event entirely.
The conditional saturation probability can instead lie between zero and one;
its weighted average is the quantity needed in the descent sum.

Lemma~\ref{lem:occupancy} controls path length spent near a discontinuity.
Saturation and straddling are different events: a steep smooth loss can
have spread greater than $\varepsilon$ without a crossing, and a jump
smaller than $\varepsilon$ can be crossed without saturation. An implication
from saturation to tube membership needs a separate bound on within-region
cost variation and sampling error. Even then, the path-length weights
must be compared with $r_{s,t}$, and event frequency does not control the
vector cancellation in Eq.~\ref{eq:crossing-fraction}.

\paragraph{Selecting an iterate.}
For deterministic positive weights, the pathwise minimum of
$N_t=\|\nabla_{\mathcal S}\Ls_\delta(\theta_t)\|^2$ is bounded by its
weighted average. The same average also bounds $\min_t\E N_t$, and
$\E[\min_tN_t]\leq\min_t\E N_t$. Theorem~\ref{thm:piecewise-smooth}
therefore gives both forms. The independent sampled index in
Theorem~\ref{cor:finite-horizon} provides an output rule that requires
no evaluation of the smoothed gradient.

\paragraph{Error sources and the step constant.}
The coupling bound requires derivatives of the field smoothed in one block.
Derivatives of the jointly smoothed loss can tend to zero while the block
discrepancy remains nonzero (Remark~\ref{rem:block-coupling}). Transverse
derivatives of the block field act on the loss itself. If the required
regularity is unavailable at a wall, the weaker bounded-loss estimate gives
a coupling contribution of $O(P\,\|\Ls\|_\infty^2/\delta^2)$, which grows as
the probe radius and temperature fall. If the full smoothing neighborhood
lies in one smooth region with the derivative bounds specified below,
the coupling bound can use constants independent of $\delta$.
The error is squared because it enters the descent inequality as
$\langle g,e\rangle$. This inner product can
offset the leading descent term when the gradient norm is at most $\|e\|/c$.

The three errors scale differently when divided by $c$. The linearization remainder and the
saturation probability are absolute displacement errors, so each acquires the
factor $(d_p/(r_p\bar\xi))^2$. The coupling term does not: $\Phi-g$
is a discrepancy between the block field and the joint gradient
(Remark~\ref{rem:block-coupling}), so it enters the step already multiplied by $c$,
the $c$ cancels, and the term contains neither $d_p$ nor any power of $r_p$ beyond
those inside $\delta$.
The $O(r_p^3)$ vector bound requires the conditional cancellation assumptions
stated in Corollary~\ref{cor:optimal-rp}.

\paragraph{Excursion length, not sheet separation.}
The crossing count needs the excursion length $L_{\mathrm{exc}}$: sheet separation
bounds the path between visits to \emph{different} sheets and says
nothing about re-crossing the same sheet. A trajectory that oscillates across a
single sheet between $\pm(\delta + a)$ satisfies the crossing hypothesis at every
visit, yet has occupancy $\delta/(\delta+a)$, which tends to $1$ as $a \to 0$
however large $L_{\mathcal{D}}$ is. With one sheet the separation hypothesis is
vacuous. The bound therefore assumes a lower bound on the actual excursion
length. Reach does not supply it: two parallel planes each have infinite
reach at any separation.
A trajectory that enters the tube and slides \emph{along} a sheet, or that steps out
and back across a short outside path, has $L_{\mathrm{exc}} \to 0$ and no
$\delta$-dependent bound.
Both hypotheses are needed and both are testable on a realized trajectory; Appendix~\ref{app:theory-numerics}
reports the bound holding when they are met and the counterexample above appearing when
either is dropped.

The saturation contribution before the final relaxation is proportional to
\[
 \frac{1}{S_T}\E\sum_{t<T}r_{s,t}\sum_i p_i(t)^2,
 \qquad p_i(t)=\Pr(A_i(t)\mid\mathcal F_t).
\]
It is bounded by $P\bar\sigma$ because $p_i(t)^2\leq p_i(t)$.
A bound on the weighted average of $p_i(t)$ does not imply the corresponding
upper bound by $P\bar\sigma^2$. If these probabilities are of order one
near a wall and zero elsewhere, their first and second moments have the
same order. Uniformly small conditional probabilities can instead give a
tighter estimate. The interior optimum in Corollary~\ref{cor:optimal-rp}
is a property of the stated upper bound, under its additional assumptions.

The radius sweep favors larger tested multipliers, while the saturation
diagnostic reports no saturated rows in its measured non-spiking runs.
These observations do not establish an interior accuracy optimum.
Path-weighted occupancy also does not bound time spent near a wall: a
trajectory can stop in the tube without adding to either path-length sum.

\paragraph{Sharpness of the $\sqrt{P}$ Goldstein radius.}
The factor $\sqrt{P}$ is sharp under support bounds alone. Fix
$\rho < \sqrt{P}\,\delta_{\mathrm{out}}$, let $u^\star$ have every block at full reach
and aligned with a unit $\hat n$, so $\|u^\star\| = \sqrt{P}\,\delta_{\mathrm{out}}$, and take
$\mu = \tfrac12(\delta_{u^\star} + \delta_{-u^\star})$, admissible since
$\mathrm{supp}\,\mu \subset \bar B(0, \sqrt{P}\delta_{\mathrm{out}})$. For
$\varphi(z) = \max(0, \langle \hat n,z\rangle - t_0)$ with
$\rho < t_0 < \langle \hat n, u^\star\rangle$, the loss vanishes identically on $B(0,\rho)$, so
$\partial_\rho\varphi(0) = \{0\}$, while
$\nabla\varphi_\delta(0) = \tfrac12 \hat n \neq 0$. Letting $t_0 \uparrow \sqrt{P}\delta_{\mathrm{out}}$
rules out every smaller radius, so no constant below $\sqrt{P}$ is available.

\paragraph{Failure of the ambient form.}
A gradient computed only within $\mathcal S$ need not belong to the full
ambient Goldstein set. Thus one cannot replace the projected conclusion by
$\nabla_{\mathcal S}\Ls_\delta(\theta)\in\partial_{\delta_{\mathrm{joint}}}\Ls(\theta)$.
For example, take the bounded smooth loss $\Ls(x,y)=\tanh y$ and
$\mathcal S=\mathrm{span}(e_1)$. Smoothing along $x$ leaves the loss
unchanged, and its subspace gradient is zero everywhere. At the origin,
however, the full gradient is $(0,1)$. Every gradient in a finite-radius
ball has a strictly positive second component, uniformly bounded away
from zero, so the ambient Goldstein set excludes zero at every finite
radius. Both the Goldstein
corollary and Eq.~\ref{eq:pws-rate} therefore certify only directions
within $\mathcal S$.

\paragraph{Derivative bounds in condition~(vii).}
Condition (vii-K) bounds derivatives of the jointly smoothed loss. It is
finite for a bounded loss under condition (iv). Condition (vii-L) instead
controls transverse derivatives of the block field and requires additional
regularity of the loss.

To make $\Lambda_3$ independent of the smoothing scale, assume $\Ls$ is $C^3$ throughout the enlarged neighborhood with
second- and third-derivative bounds
$\Lambda_2,\Lambda_3$ on
$\mathcal{K}_{\delta_{\mathrm{joint}}} := \{\theta : \mathrm{dist}(\theta,\mathcal{K}) \leq \delta_{\mathrm{joint}}\}$
rather than on $\mathcal{K}$, since the joint displacement leaves $\mathcal{K}$ by up
to $\sqrt{P}\,\delta_{\mathrm{out}}$. The two error estimates need different
regularity: $C^{1,1}$ suffices for the $O(r_p^3)$ mean softmax remainder
by the rotational cancellation proved in Appendix~\ref{app:proof-piecewise-smooth},
but does not supply the transverse second derivatives of the block field
required for the $O(\delta^2\Lambda_3)$ discrepancy bound.

\paragraph{Detectability as a subset-sum constraint.}
Unequal costs do not always give a nonzero step. Eq.~\ref{eq:pwc-detect}
requires the successful vertices to have a nonzero vector sum. Indeed, with weight $\alpha$ on $\mathcal{Z}$ and $\alpha e^{-1/\varepsilon}$ off it,
$\sum_v v_v = 0$ collapses the step to
$\sum_v w_v R v_v = \alpha(1 - e^{-1/\varepsilon})\,R\!\sum_{v \in \mathcal{Z}} v_v$.
On the orthoplex this says $\mathcal{Z}$ separates an antipodal pair: costs $(0,0,1,1)$ on
$(\pm e_1, \pm e_2)$ give $\mathcal{Z} = \{\pm e_1\}$, sum zero, and a step of exactly zero
however unequal the entries. On the simplex it holds for every proper
nonempty $\mathcal{Z}$, since all linear relations among the $d_p+1$
vertices are multiples of $\sum_v v_v = 0$. The condition therefore
reduces to nonconstancy of the row. The primary
architecture comparisons use the simplex. Either way $\|\sum_{v\in \mathcal{Z}} v_v\| \geq 1$ when
Eq.~\ref{eq:pwc-detect} holds, since the simplex gives
$\|\sum_{v\in \mathcal{Z}}v_v\|^2 = |\mathcal{Z}|(d_p{+}1{-}|\mathcal{Z}|)/d_p \geq 1$ and the orthoplex leaves an
unpaired unit vector.

\paragraph{Adjacent guarantees.}
\citet{glasmachers2020global} proves existence of a critical limit point
for the $(1+1)$-ES under stated assumptions on measurable objectives.
The result extends criticality beyond locally Lipschitz functions and
concerns isotropic single-offspring search; it does not establish convergence
of the entire sequence to one point. \citet{vicente2012discontinuous} analyzes directional direct
search on discontinuous objectives through Rockafellar upper subderivatives; that
analysis needs the objective to be directionally Lipschitz with respect to the
limit direction, a condition that is not automatic near a jump. The theorem
in its final section, and a corollary of that theorem, have since been refuted by an
explicit counterexample \citep{audet2024counterexample}, which also supplies a
modified poll step that restores them; the rest of the analysis stands. Runtime analyses also bound hitting times on pseudo-Boolean plateaus
\citep{antipov2021plateau}. Discrete Fourier analysis gives exact expected
runtimes and optimal mutation rates for sequential plateaus
\citep{doerr2024plateau}. These results concern bit flips and Hamming
geometry, whereas our probes explore finite-radius neighborhoods in $\R^d$.

The result of \citet{glasmachers2020global} applies independently of the
initial state under its assumptions. Our bound requires entry into a
neighborhood with a uniform hitting probability. That qualification is necessary here because
Theorem~\ref{prop:plateau-freeze} makes $\tau_{\mathcal{A}} = \infty$ almost
surely from a plateau beyond probe reach. Neither result is unconditional in the
stronger sense: \citet{glasmachers2020global} also shows that a non-optimal critical
point induces premature convergence with positive probability. That paper
characterizes when global convergence holds under its assumptions.
Corollary~\ref{cor:piecewise-constant} gives a geometric tail bound under
an assumed per-step success probability. The accompanying examples identify
configurations where that assumption fails. It is a bound for an exploration process, not a
descent guarantee, and
Theorem~\ref{prop:plateau-freeze} gives the limiting case: on a plateau
wider than the probe reach, the row-wise update is zero.

\subsubsection{Radii, Schedules, and Query Cost}

\paragraph{Two parameterizations of the radii.}
We distinguish the multipliers $r_{p,t},r_{s,t}$ from the lengths
$\rho_t^{\mathrm{base}}=r_{p,t}\varepsilon_t$ and
$\ell_t=r_{s,t}\varepsilon_t$, both before jitter. Without jitter, the probe distances
are $\xi_k\rho_t^{\mathrm{base}}$ and the displacement is $\ell_tD_i$.

\begin{remark}[Radius parameterizations]\label{rem:radius-parameterization}
Eq.~\ref{eq:surrogate-def} is written in the parameterization in which $r_p$ and $r_s$ are
\emph{multipliers on} $\varepsilon$. The implementation admits a second, in which a radius
is supplied as a schedule and is the \emph{physical} radius itself, with $\varepsilon$ not
applied. At fixed $\varepsilon$, as required by condition~(v), the two conventions
differ only by a constant factor. When $\varepsilon$ varies, the physical
radii must be tracked explicitly. The note to
Table~\ref{tab:hypothesis-coverage} describes the reported schedules.
\end{remark}

One translation matters for tuning: the interior optimum
$r_p^\star = \Theta(\varepsilon^{1/5})$ of Eq.~\ref{eq:optimal-rp} is a statement about the
multiplier, so the physical optimum is $\Theta(\varepsilon^{6/5})$.

\paragraph{Changing temperature and radius.}
With $\delta=\Theta(r_p\varepsilon)$, the second-order allowance
$\Lambda_\delta r_{s,t}^2\varepsilon^2$ is $O(r_{s,t}^2/r_p^2)$.
At fixed $r_p$ this upper bound is independent of temperature, while reducing
$r_p$ can enlarge it. The optimized conditional bias bound of
Corollary~\ref{cor:optimal-rp} decreases as $\varepsilon\to0$, but the transient
constant changes too: $c=r_p/(2d_p)$ and $C'$ includes the smoothing-dependent
curvature allowance. The fixed-parameter rate cannot therefore be converted
into an evaluation complexity for the new parameters without tracking these constants and the
uniformity of the saturation and regularity hypotheses.

\paragraph{Rate in context.} Eq.~\ref{eq:pws-rate} is a per-iteration rate, and
each iteration issues $P \times V \times K$ queries, so the query complexity
includes a factor proportional to the search dimension $d_{\mathrm{sub}}$.
Under their stochastic Lipschitz-oracle assumptions,
\citet{lin2022gradientfree} obtain $(\delta,\nu)$-Goldstein stationarity
in expectation using $O(d^{3/2}\delta^{-1}\nu^{-4})$ function evaluations,
improved to $O(d\,\delta^{-1}\nu^{-3})$ by
\citet{kornowski2024algorithm}. These expressions suppress the Lipschitz
and initial objective-gap parameters. For the general class of Lipschitz
nonsmooth nonconvex functions, \citet{jordan2023deterministic} prove an
$\Omega(d)$ lower bound for deterministic black-box first-order algorithms;
their result rules out dimension-free deterministic rates in that oracle
model. Our subspace restriction changes the search dimension and the
stationarity target. It does not transfer these full-space bounds to
PolyStep (Section~\ref{sec:limitations-disc} discusses the restriction).
The guarantees also have different targets. The cited bounds assume
Lipschitz objectives and certify Goldstein stationarity in the full space
to arbitrary accuracy. An attained jump violates that local Lipschitz
assumption. Eq.~\ref{eq:pws-rate} instead bounds a smoothed gradient within
the search subspace, up to its stated bias floor. Replacing $d$ by
$d_{\mathrm{sub}}$ therefore changes the stationarity target as well as the
per-iteration cost.

\paragraph{The compression ratio against the Johnson--Lindenstrauss threshold.}
Section~\ref{sec:limitations-disc} reports the all-parameter GPT-2 fine-tuning failure with a rank-$128$ projection of $d = 1.24\times10^{8}$ parameters, a
compression ratio of $1.0\times10^{-6}$. The comparison scale is a sufficient rank from a standard Johnson--Lindenstrauss bound: a random
projection to $k$ dimensions preserves all pairwise distances among $N$ points to
distortion $\epsilon_{\mathrm{JL}}$ once
$k \geq 4\ln N/(\epsilon_{\mathrm{JL}}^2/2 - \epsilon_{\mathrm{JL}}^3/3)$
\citep{dasgupta2003elementary}. For an illustrative fixed set of $N=10^6$ points, distortion
$\epsilon_{\mathrm{JL}} = 1/2$ is ensured by $k \geq 664$, five times the rank used in this experiment, and
$\epsilon_{\mathrm{JL}} = 1/10$ is ensured by $k \geq 11{,}842$, ninety times it; the
corresponding ratios are $5.3\times10^{-6}$ and $9.5\times10^{-5}$ against the
experimental ratio $1.0\times10^{-6}$. These are sufficient thresholds for a fixed finite set under the specified
random projection law. They are not necessary ranks, do not guarantee optimization
success, and do not establish why an adaptive training trajectory fails.

\subsection{Optimizing the Bias-Floor Bound}\label{app:proof-optimal-rp}

Under additional assumptions, larger probes increase the smooth
approximation errors while smaller probes increase the saturation term.
The next corollary balances these terms in the \emph{bound}; it does not
identify the radius that maximizes network accuracy.

\begin{corollary}[Interior minimizer of the bias-floor bound]\label{cor:optimal-rp}
Assume the hypotheses of Theorem~\ref{thm:piecewise-smooth}. Take the linearization
term in the form $O(Pd_p^2r_p^4)$ by assuming the conditional mean
covered softmax remainder is $O(r_p^3)$ uniformly in the regime considered.
This requires cancellation after the covered-event restriction, not just
piecewise smoothness. Also assume $\Lambda_3$ is independent of $\delta$,
and take the quantitative form
$\bar\sigma = O(\sqrt{P}\,r_p\varepsilon/L_{\mathrm{exc}})$ discussed in Appendix~\ref{app:theory-remarks}, with $L_{\mathrm{exc}}$ bounded below uniformly in $r_p$. Using the
covered-remainder bound in the descent proof gives the alternative bias bound
$\hat B(r_p,\varepsilon) = \beta_1 r_p^4 + \beta_2\varepsilon/r_p + \beta_3 r_p^4\varepsilon^4$,
with $\beta_1, \beta_2, \beta_3 > 0$ uniform in $r_p$ and $\varepsilon$ ($\beta_1$ absorbing the factor
$P d_p^2$ and $\beta_2$ the factor $P^{3/2} d_p^2/L_{\mathrm{exc}}$).
Then $\hat B$ is minimized at an interior probe radius, and as
$\varepsilon \downarrow 0$
\begin{equation}\label{eq:optimal-rp}
r_p^\star \;=\; \Theta\!\big(\varepsilon^{1/5}\big),
\qquad
\hat B(r_p^\star, \varepsilon) \;=\; \Theta\!\big(\varepsilon^{4/5}\big) .
\end{equation}
Probing more finely than $r_p^\star$ increases the bound; the displayed exponents
are small-$\varepsilon$ asymptotics, with the exact stationary point and the
large-$\varepsilon$ regime given in the proof below.
\end{corollary}

We optimize the upper bound $\hat B$.
The assumptions bound each error term from above and do not establish
a matching lower bound on the residual error. An interior optimum also exists under the $\bar q$-based form of the first
term at the cost of one further nondegeneracy hypothesis, $\bar q = \Theta(r_p^2)$,
assumed and not derived; only the exponents change, to
$r_p^\star = \Theta(\varepsilon^{1/3})$ and $\hat B^\star = \Theta(\varepsilon^{2/3})$.
We have not verified these stronger assumptions for the trained losses.
In the INT8 diagnostic, the joint smoothing support intersects $\mathcal D$
throughout the sampled runs, so we cannot justify these assumptions by
placing the full smoothing neighborhood inside one smooth region. We have
not checked the remaining configurations for this property. On a synthetic loss with the saturated term
active by construction the measured floor grows as $r_p^{-0.92}$, consistent with the middle term's inverse-radius dependence. Section~\ref{sec:ablation}'s accuracy sweep favors larger
tested multipliers and does not establish an interior optimum.
These are asymptotics of a conditional bias bound. An iteration-complexity claim
when parameters vary jointly would also need to track $c$, $C'$, the schedule, and the
uniformity of all hypotheses (Appendix~\ref{app:theory-remarks}).

\begin{proof}
Use the descent bound in Theorem~\ref{cor:finite-horizon}, bounding the
conditional mean covered remainder directly by $O(r_p^3)$ in place of
$\kappa_0\bar q$. Its squared contribution divided by $c^2$ is then
$O(Pd_p^2r_p^4)$. The saturation term
$O(P d_p^2\bar\sigma/r_p^2)$ becomes $\beta_2\varepsilon/r_p$ under the assumed
$\bar\sigma = O(\sqrt{P}\,r_p\varepsilon/L_{\mathrm{exc}})$; and a $\Lambda_3$ independent of $\delta$ makes the coupling term $O(P^3 r_p^4\varepsilon^4\Lambda_3^2)$ at
$\delta = \Theta(r_p\varepsilon)$. Choosing uniform positive upper constants
gives $\hat B(r_p,\varepsilon)=\beta_1r_p^4+\beta_2\varepsilon/r_p+\beta_3r_p^4\varepsilon^4$.
The first and third terms increase in $r_p$ and the
second decreases, with $\hat B \to \infty$ at both ends of $(0,\infty)$, so the
minimum is interior. Setting $\partial_{r_p}\hat B = 0$ gives
$4(\beta_1 + \beta_3\varepsilon^4)\,r_p^3 = \beta_2\varepsilon\,r_p^{-2}$, so
$r_p^\star = [\beta_2\varepsilon/(4(\beta_1 + \beta_3\varepsilon^4))]^{1/5}$ and
$\hat B(r_p^\star,\varepsilon) = \Theta\big((\beta_2\varepsilon)^{4/5}(\beta_1 + \beta_3\varepsilon^4)^{1/5}\big)$.
For $\varepsilon^4 \leq \beta_1/\beta_3$ this is $r_p^\star = \Theta(\varepsilon^{1/5})$,
$\hat B^\star = \Theta(\varepsilon^{4/5})$, the displayed asymptotics; for
$\varepsilon^4 \gg \beta_1/\beta_3$ it is $r_p^\star = \Theta(\varepsilon^{-3/5})$,
$\hat B^\star = \Theta(\varepsilon^{8/5})$.
\end{proof}

Under the $\bar q$-based form of term one, add the hypothesis
$\bar q(r_p) = Q\,r_p^2(1 + o(1))$ with $Q > 0$ uniform in the asymptotic regime, so
that the term reads $\Theta(r_p^2)$; the identical optimization then gives
$r_p^\star = \Theta(\varepsilon^{1/3})$ and
$\hat B^\star = \Theta(\varepsilon^{2/3})$, so both forms of the linearization bound give an interior optimum, with
different exponents. The
hypothesis is not automatic: at a stationary point of a quadratic the spread is
$\Delta_i/\varepsilon = O(r_p^2)$, giving $\bar q = O(r_p^4)$ and different exponents,
and on a constant loss $\bar q = 0$ and the balance disappears.

\subsection{Goldstein Stationarity in the Search Subspace}\label{sec:goldstein}\label{app:proof-goldstein}

For a base point $\theta$ and subspace coordinates $z \in \R^{d_{\mathrm{sub}}}$, write
$\varphi(z) := \Ls(\theta + \Pi z)$ for the restriction of the loss to the search
subspace and $\varphi_\delta := \varphi * \mu$ for its convolution with any probability
measure $\mu$ on $\R^{d_{\mathrm{sub}}}$ supported in
$\bar B(0,\delta_{\mathrm{joint}})$.

Under local Lipschitz continuity, the following inclusion gives a Goldstein
certificate. Its support radius is sharp when no further restrictions are
placed on the perturbation law.
Here $\partial\varphi$ is the Clarke subdifferential, the closed convex hull
of limiting gradients at nearby differentiable points. The Goldstein set
then combines these subgradients over an entire radius-$\rho$ neighborhood.
The convolution gradient belongs to this larger set because smoothing
averages local variation within the kernel's support.

\begin{corollary}[Goldstein stationarity within the subspace for a locally Lipschitz loss, at the sharp radius $\sqrt{P}\,\delta_{\mathrm{out}}$]\label{cor:goldstein}
Let $\Ls$ be locally Lipschitz on an \emph{open} $U \supseteq
\mathcal{K}_{\delta_{\mathrm{joint}}}$. The kernel reaches this enlarged
neighborhood. Openness allows the mean-value argument to use slightly larger balls.
Let $\Pi$ have orthonormal columns,
$\Pi^\top\Pi = I$ (condition~(iii)), and fix a base point $\theta \in \mathcal{K}$,
with $\delta_{\mathrm{joint}}$ the joint kernel radius of Eq.~\ref{eq:joint-radius}.
The tail product kernel $m^{\otimes P}/Z^P$ of Eq.~\ref{eq:joint-surrogate}, whose
gradient Eq.~\ref{eq:pws-rate} bounds, is one admissible $\mu$, and the inclusion
below uses only the support, so it holds for the probe law as well.
For the first inclusion, take any $z$ at which $\varphi_\delta$ is
differentiable and whose closed $\delta_{\mathrm{joint}}$-ball maps into $U$:
$\theta + \Pi\,\bar B(z,\delta_{\mathrm{joint}}) \subseteq U$.
The second inclusion also requires the \emph{ambient} ball to lie inside $U$:
$\bar B(\theta + \Pi z,\delta_{\mathrm{joint}}) \subseteq U$,
which holds at $z = 0$ with $\theta \in \mathcal{K}$.
This is the case used for the stationarity bound:
\begin{equation}\label{eq:goldstein-inclusion}
\nabla \varphi_\delta(z) \;\in\; \partial_{\delta_{\mathrm{joint}}} \varphi(z)
\;\subseteq\; \Pi^\top \partial_{\delta_{\mathrm{joint}}} \Ls(\theta + \Pi z) ,
\qquad
\partial_{\rho} \varphi(z)
:= \overline{\mathrm{conv}}\,\big\{\partial \varphi(y) : y \in \bar B(z,\rho)\big\},
\end{equation}
the Goldstein subdifferential of $\varphi$ at radius $\delta_{\mathrm{joint}}$, the
second inclusion by Clarke's chain rule and $\|\Pi z\| = \|z\|$.

For the sampled output $\theta_\tau$ of Theorem~\ref{cor:finite-horizon},
the inclusions and the isometry of $\Pi$ give
\[
\E\!\left[\mathrm{dist}\big(0,\Pi^\top
\partial_{\delta_{\mathrm{joint}}}\Ls(\theta_\tau)\big)^2\right]
\leq\E\|\nabla_{\mathcal S}\Ls_\delta(\theta_\tau)\|^2.
\]
Thus the finite-horizon bound applies to the squared Goldstein distance
within the search subspace. For the infinite-schedule result, if the iterates also satisfy
the hypotheses of Theorem~\ref{thm:piecewise-smooth}, let $t_\star$ be the smallest index minimizing
$\|\nabla_{\mathcal S}\Ls_\delta(\theta_t)\|^2$ over $t<T$.
This is a measurable random index, and Eq.~\ref{eq:pws-rate} gives
\[
\E\!\left[\mathrm{dist}\big(0,\; \Pi\Pi^\top
\partial_{\delta_{\mathrm{joint}}}\Ls(\theta_{t_\star})\big)^2\right]
\;\leq\; \frac{C}{(1-\bar\pi)S_T} + \frac{B}{(1-\bar\pi)^2} ,
\]
an expected squared stationarity bound with
$\nu_T^2=C/((1-\bar\pi)S_T)+B/(1-\bar\pi)^2$.
This does not certify a deterministic tolerance for each run or a stationary
limit point. The saturation factor is the one Eq.~\ref{eq:pws-rate} carries; the Goldstein
inclusion does not remove it.
\end{corollary}

The conclusion requires both
local Lipschitz continuity of the original loss and an isometric subspace
map. It concerns the projected subgradients. For example, if $\Ls(x,y)=x+y$
and only $x$ is searched, the restricted gradient is $1$, but embedding it
as $(1,0)$ does not recover the full gradient $(1,1)$.

The radius $\sqrt P\,\delta_{\mathrm{out}}$ follows from combining $P$ block
displacements of length at most $\delta_{\mathrm{out}}$. A hinge-loss example
with a two-point perturbation distribution shows that no smaller radius
works using only these support bounds. Additional structure in a specified
perturbation law may allow a stronger result. The radius bound is therefore
a consequence of block geometry, rather than a property unique to PolyStep.
Each iteration still requires $PVK$ evaluations. Restricting the search to
$d_{\mathrm{sub}}$ coordinates changes both the cost and the stationarity
target; Appendix~\ref{app:theory-remarks} compares this guarantee with
full-space Lipschitz results.

An almost-sure convergence statement remains open
(Remark~\ref{rem:no-as-claim}). For a quadratic model, a separate finite-horizon
bound controls the probability of leaving a specified compact region
(Proposition~\ref{prop:confinement-ville}). Its conditional drift hypothesis
is not verified for the trained losses.

\begin{proof}
Work entirely in the subspace coordinates. $\varphi(z) = \Ls(\theta + \Pi z)$ is
locally Lipschitz on the open set $\{z : \theta + \Pi z \in U\}$
because $\|\Pi z\| = \|z\|$; openness is what lets the mean-value step below run on
balls of radius $\delta_{\mathrm{joint}} + t$ for small $t > 0$. The argument uses nothing of $\mu$ beyond
$\mathrm{supp}\,\mu \subseteq \bar B(0,\delta_{\mathrm{joint}})$; the joint tail kernel
$m^{\otimes P}/Z^P$ of Eq.~\ref{eq:joint-surrogate} and the joint probe law both
qualify, each supported in that ball. Fix $z$ at which
$\varphi_\delta = \varphi * \mu$ is differentiable and a unit
$w \in \R^{d_{\mathrm{sub}}}$. For each $u \in \mathrm{supp}\,\mu$, Lebourg's mean
value theorem \citep{clarke1990optimization} applied to $\varphi$ on the segment
$[z - u,\, z - u + tw]$ gives
$\varphi(z - u + tw) - \varphi(z - u) = t\,\langle g_u, w\rangle$ for some
$g_u \in \partial\varphi(y_u)$ with
$y_u \in \bar B(z, \delta_{\mathrm{joint}} + t)$. Averaging over $\mu$ and
dividing by $t$, the difference quotient of $\varphi_\delta$ at $z$ along $w$ is at
most the support function of the convex compact
$\partial_{\delta_{\mathrm{joint}}+t}\varphi(z)$; letting $t \downarrow 0$ and using
$\partial_{\delta_{\mathrm{joint}}}\varphi(z)
= \bigcap_{t>0}\partial_{\delta_{\mathrm{joint}}+t}\varphi(z)$, with the support
functions of the nested compacts converging monotonically,
$\langle\nabla\varphi_\delta(z), w\rangle$ is dominated by the support function of
$\partial_{\delta_{\mathrm{joint}}}\varphi(z)$ in every direction $w$, hence
$\nabla\varphi_\delta(z)$ lies in that convex compact set: the first inclusion. The
second is Clarke's chain rule for the linear map $z \mapsto \theta + \Pi z$ with
$\Pi^\top\Pi = I$: $\partial\varphi(y) \subseteq \Pi^\top\partial\Ls(\theta + \Pi y)$
for every $y$, and $\|\Pi(y-z)\| = \|y-z\|$ carries the radius.
\end{proof}

\subsection{Assumptions and Experimental Configurations}\label{sec:hypothesis-coverage}

In Table~\ref{tab:hypothesis-coverage}, we compare the experimental
configurations with Assumption~\ref{asm:main} and the infinite-schedule
conditions of Theorem~\ref{thm:piecewise-smooth}.
None of the configurations selected for the main experiments satisfies all
of these requirements. The primary architecture configurations do not meet the jitter condition
(iv) and the full schedule comparison. In theory mode, we use the prescribed
frame, rotations, and schedule, and a numerical approximation to the
prescribed jitter distribution. We have not established compactness,
transverse regularity, or conditional saturation control for these runs.
Configurations with a padded final block also require an extension of the
theorem, which assumes complete equal-dimensional blocks. The projection column
refers to the active columns before padding. The sparse fallback in
Appendix~\ref{app:search-implementation} violates condition~(iii), including in the
theory-mode runs that retain that basis. ETTh1 uses dense QR bases. MAX-SAT
and CartPole use the identity map; Acrobot combines a dense QR projection
with unprojected parameters. Both RL configurations have padded final blocks.

\begin{table}[t]
\centering
\caption{Assumption~\ref{asm:main} compared with experimental configurations. Column~(v) also requires the infinite-schedule conditions of Theorem~\ref{thm:piecewise-smooth}. Column~(viii) reports an empirical estimate of $\bar\sigma$: the $r_{s,t}$-weighted fraction of rows whose cost spread exceeds the temperature, measured over $150$ steps at seed $42$. Parentheses give the largest realized spread $\max_{i,t}\Delta_i/\varepsilon$; ``?'' means not established or not measured; asm.\ abbreviates assumed. These measurements do not establish the saturated-mass condition $\bar\pi < 1$, which remains an assumption.}
\label{tab:hypothesis-coverage}
\footnotesize
\setlength{\tabcolsep}{2pt}
\begin{tabular}{@{}lcccccccl@{}}
\toprule
Configuration & (i) frame & (ii) rot. & (iii) proj. & (iv) jitter & (v) sched. & (vi) compact & (vii) $\Lambda_3$ & (viii) $\bar\sigma$ \\
\midrule
SNN (LIF)      & yes & no  & no & no & no & asm. & ? &$2{\times}10^{-5}$ ($2.00$) \\
INT8           & yes & yes & no & no & no & asm. & ? &$0.00$ ($0.04$) \\
Argmax         & yes & yes & no & no & no & asm. & ? &$0.00$ ($0.10$) \\
Staircase      & yes & yes & no & no & no & asm. & ? &$0.00$ ($0.11$) \\
Hard MoE       & yes & no  & no & no & no & asm. & ? &? \\
MNIST          & yes & yes & no & no & no & asm. & ? &? \\
ETTh1          & yes & yes & yes & no & no & asm. & ? &? \\
MAX-SAT        & yes & yes & yes & no & no & asm. & ? &? \\
RL             & yes & yes & yes & no & no & asm. & ? &? \\
\midrule
Theory mode (SNN)        & yes & yes & no & ideal & yes & asm. & asm. &$1.8{\times}10^{-4}$ ($1.65$) \\
Theory mode (INT8)       & yes & yes & no & ideal & yes & asm. & asm. &$0.00$ ($0.26$) \\
Theory mode (Argmax)     & yes & yes & no & ideal & yes & asm. & asm. &$0.00$ ($0.12$) \\
Theory mode (Staircase)  & yes & yes & no & ideal & yes & asm. & asm. &$0.00$ ($0.14$) \\
Theory mode (MNIST) & yes & yes & no & ideal & yes & asm. & asm. &? \\
Theory mode (MAX-SAT) & yes & yes & yes & ideal & yes & asm. & asm. &? \\
\bottomrule
\end{tabular}

\medskip
{\raggedright\small
The schedule comparison uses the \emph{realized} radii of
Remark~\ref{rem:radius-parameterization} rather than the nominal settings. The SNN and
hard-MoE rows fail (ii) by biased rotation; every tuned row fails the schedule comparison, through a changing
probe scale or temperature, a nonsummable squared-step schedule, or both. The
MAX-SAT row fails both halves: in the $100{,}000$-variable configuration,
the radius is \emph{nominally} annealed $3000 \to 600$
and realized as $3000 \to 2999.9$, the cosine period having been inferred from a
default decay rate rather than from the step budget, so the realized radius is nearly constant.
The theory-mode settings match (i), (ii), and (v). In (iv), ``ideal'' refers
to the target mollifier density; numerical sampling does not establish
the required regularity, as detailed below.\par}
\end{table}

\paragraph{Ideal jitter and numerical sampling.}
The supplied implementation linearly interpolates an inverse-CDF table
with $4{,}097$ knots. Even with continuous uniform random inputs, that map
produces a piecewise-constant density, generally with jumps, rather than a
$W^{2,1}$ density. Rejection sampling capped at $128$ proposals, with a zero
fallback when all proposals are rejected, instead introduces an atom. These are approximations to the smooth target
law. Neither satisfies condition~(iv) exactly, and a small distributional
error alone does not imply the required derivative bounds. In the table,
we therefore distinguish the smooth target density from the density of
the implemented sampler.
Uncapped rejection sampling realizes the stated law in the ideal
continuous-randomness model. The convergence theorem concerns that idealized algorithm.

For a fixed general linear map, the coordinate identity still applies to
$\varphi(z)=\Ls(\theta_0+\Pi z)$. However, $\Pi\Pi^\top$ need not be an
orthogonal projector, and a coordinate displacement $h$ only satisfies
$\|\Pi h\|\leq\|\Pi\|_{\mathrm{op}}\|h\|$. The ambient radii and projected-gradient
interpretations that use condition~(iii) therefore do not transfer unchanged
to the sparse runs.

Condition~(viii) concerns conditional expectations. The logs give realized
spreads and saturation indicators, which do not establish its uniform constants.
Shuffling the data once per epoch does not make each minibatch loss
unbiased conditional on the sampling history. Independent uniform
sampling with replacement from a fixed training pool is one way to meet
that condition for the empirical loss.
The absence of saturation among $N=1{,}048{,}500$ logged non-spiking rows is
inconsistent at the $99\%$ level with a model assigning every sequentially
conditioned row a saturation probability above $4.4\times10^{-6}$; it gives no
pointwise upper bound (Appendix~\ref{app:numerics-limits}). Compactness is also
assumed: these updates do not project iterates onto a compact set.
The quadratic drift model below illustrates a possible stabilizing mechanism
without proving confinement for the trained losses. Finally, failure of the
sufficient jitter condition does not prove failure of the actual derivative
bounds in (vii). We have not established those bounds for the configurations
selected for our main experiments.

Theory mode is stricter than the theorem in one place: condition~(i) admits the simplex,
but the antipodal half of Proposition~\ref{prop:stable-schedules} does not, so theory
mode uses $2d_p/(d_p+1)$ times as many evaluations per step as the simplex.
Section~\ref{sec:theory-mode} measures the effects of these configuration changes.

\subsection{Technical Remarks on the Convergence Analysis}\label{app:convergence-remarks}

These remarks specify the regularity and sampling conditions used in
Section~\ref{sec:theory}.

\begin{remark}[Per-particle radius jitter]\label{rem:per-particle-jitter}
The boundary-avoidance argument uses the density of a scalar radius.
It does not require an absolutely continuous joint step law:
on a plateau every displacement vanishes
(Theorem~\ref{prop:plateau-freeze}) and the one-step law is a point mass.
We nevertheless draw the jitter independently per particle rather than once per
step.
\end{remark}

On a plateau, the iterate does not move, so a point off $\mathcal{D}$ stays off it. With a shared radius the argument needs the \emph{aggregate}
displacement to be non-orthogonal to the wall normal; with independent radii it needs
only \emph{some} block to be, without requiring a nonzero aggregate normal displacement.

\begin{remark}[Restriction of the smoothing to the search subspace]\label{rem:plane-smoothing}
The regularity result of Lemma~\ref{lem:smooth-kernel} holds on $E_i$
and cannot be strengthened to
$\Ls_\delta \in C^\infty(\R^d)$. The probe law is supported on a
$d_p$-dimensional plane, so as a measure on $\R^d$ it is singular, and
convolution with it averages only along the plane: any jump of $\Ls$ across a
direction transverse to $E_i$ remains. Aggregating over the $P$
particles improves this to smoothness along $\mathcal{S}$, but no
further. The stationarity bounds therefore concern
$\nabla_{\mathcal{S}} \Ls_\delta$, the gradient within $\mathcal{S}$.
\end{remark}

The restriction limits what stationarity can certify. The failed
$4.2$M-parameter experiment in Section~\ref{sec:limitations-disc} does not
by itself identify insufficient subspace coverage as the cause.

\begin{remark}[Local Lipschitz scope]\label{rem:definable}
Corollary~\ref{cor:goldstein} applies where the required neighborhood is
locally Lipschitz. Continuous clipped or piecewise-affine networks and hinge
losses satisfy this condition locally. Hard-LIF, quantized, staircase, and
argmax models can have attained jumps; neighborhoods containing those jumps
do not satisfy it. The conclusion there remains the fixed-resolution
smoothed-gradient bound. This is a choice of target for the present analysis,
not an impossibility statement for other forward-only methods.
\end{remark}

\begin{remark}[Decaying schedules]\label{rem:pwc-schedule}
With a changing temperature, a constant lower hitting probability $p_0$
need not persist. Let $p_t=\Pr(\theta_{t+1}\in\mathcal A\mid\mathcal F_t)$.
The conditional Borel--Cantelli lemma implies infinitely many visits on the
event $\sum_t p_t=\infty$. Remaining within geometric step reach alone does
not establish this divergence: the probes may miss the target, the
barycenter may vanish, or the reachable overlap may shrink. No
schedule-dependent lower bound on $p_t$ is proved here.
\end{remark}

\section{Jump Geometry and Auxiliary Results}\label{app:auxiliary-theory}

These results describe the geometry of attained jumps, avoidance of
threshold sets, and sufficient conditions for exploration and confinement.
They are independent of the exact smoothing identity.

\subsection{Geometry and Smoothing Near an Attained Jump}\label{app:proof-smoothing-blowup}

The \emph{reach} of a closed set is the supremum of distances $r$ such
that every point at distance less than $r$ from the set has a unique nearest
point on it. For a smooth jump surface, a lower bound on reach controls
both curvature and how closely distinct folds can approach one another.
The isolation radius has a separate role: it excludes other discontinuity
surfaces from the neighborhood where we compare the two branches.

\begin{definition}[Piecewise-smooth loss with an attained jump]\label{def:pwc-loss}
A bounded measurable $\Ls : \R^d \to \R$ is \emph{piecewise smooth} if there is a
closed set $\mathcal{D} \subset \R^d$ such that
(a) $\Ls$ is continuously differentiable on $\R^d \setminus \mathcal{D}$, with a
uniform Lipschitz constant on compact subsets of each connected component;
(b) $\mathcal{D}$ is definable in an o-minimal structure with
$\dim \mathcal{D} \leq d - 1$.
We do \emph{not} require $\Ls$ to be continuous across $\mathcal{D}$.
The set $\mathcal D$ may include points where the loss is continuous but
not differentiable; it contains all discontinuities under (a).
Say $\Ls$ has an \emph{attained jump} $J>0$ at $\theta^\dagger$ with
isolation radius $r_{\mathrm{iso}}>0$ and reach $\tau>0$ if there is a closed
$C^1$ hypersurface $\Sigma$ of reach at least $\tau$ containing $\theta^\dagger$,
such that $S:=\Sigma\cap B(\theta^\dagger,r_{\mathrm{iso}})$ is connected and
\begin{equation}\label{eq:isolation}
B(\theta^\dagger,r_{\mathrm{iso}})\cap\mathcal D=S.
\end{equation}
Assume this ball minus $S$ consists of two connected open sets $U^\pm$.
The one-sided traces $\Ls^\pm$ exist on $S$ and satisfy
$\Ls^+-\Ls^-\geq J$, and $\|\nabla\Ls\|\leq G<\infty$ on
$U^+\cup U^-$. The extended sheet $\Sigma$ supplies the reach bound;
isolation and the jump bound are required only in the specified ball.
\end{definition}

Definability in an o-minimal structure rules out infinitely oscillating
threshold geometry on a compact domain and permits finite decompositions
into smooth pieces. It is used for the geometric statements; the smoothing
identity does not require it.

The lower bounds compare smoothed values along the normal line through
the jump. Two points on opposite sides give a slope bound; three points
give a curvature bound. The argument uses the density's support and the
actual jump surface. The upper bound additionally uses kernel derivatives.

\begin{theorem}[Smoothing near an attained jump]\label{thm:smoothing-blowup}
Let $\Ls$ satisfy Definition~\ref{def:pwc-loss}, with an attained jump
$J>0$ at $\theta^\dagger$, unit normal $n$, branch gradient bound $G$,
isolation radius $r_{\mathrm{iso}}$, and reach at least $\tau$.
For $h>0$, let $\mu_h$ be any probability density supported in $\bar B(0,h)$.
Write $D(r)=(\Ls*\mu_h)(\theta^\dagger+rn)$ and let
$\Lambda_h$ be its second-difference seminorm on $[-2h,2h]$:
\[
 \Lambda_h=\sup_{r,\ell:\,\ell>0,\ r\pm\ell\in[-2h,2h]}
 \frac{|D(r+\ell)-2D(r)+D(r-\ell)|}{\ell^2}.
\]
This equals $\mathrm{Lip}(D')$ when $D$ is $C^{1,1}$ and may be infinite.
Then
\begin{equation}\label{eq:blowup}
 \mathrm{Lip}_{[-h,h]}D\geq\frac{J}{2h}-G\geq\frac{J}{4h},
 \qquad h\leq\min\{J/(4G),\tau/2,r_{\mathrm{iso}}/2\},
\end{equation}
and
\begin{equation}\label{eq:blowup-hess}
 \frac{J}{8h^2}\leq\Lambda_h
 \leq\|\Ls\|_\infty\|D^2\mu\|_1h^{-2},
 \qquad h\leq\min\{J/(16G),\tau/2,r_{\mathrm{iso}}/3\}.
\end{equation}
The upper bound additionally requires
$\mu_h(u)=h^{-d}\mu(u/h)$ with $\mu\in W^{2,1}$; the lower bounds
require only the stated support and density. Read the ratios involving
$G$ as $+\infty$ when $G=0$. Consequently:
\begin{enumerate}[label=(\arabic*),nosep,leftmargin=1.6em]
\item the smoothing seminorms cannot remain uniformly bounded near the
attained jump as $h\to0$;
\item these are suprema over nearby points, not divergence at every point
or along every vanishing-radius sequence;
\item for a fixed scaled $W^{2,1}$ kernel and steps of length at most $s_h$,
the descent-lemma allowance is $O(s_h^2/h^2)$. This bound need not vanish
at fixed $s_h/h$, whereas it vanishes when steps shrink at fixed resolution.
\end{enumerate}
\end{theorem}

\paragraph{Setup.}
Let $\theta^\dagger \in \mathcal{D}$ be a point of an attained jump. By
Definition~\ref{def:pwc-loss}(b), $\mathcal{D}$ is definable of dimension at most
$d-1$, and a definable set admits a $C^p$ cell decomposition for every finite $p$
\citep{coste2000omin}, but this alone does not imply an attained jump or a uniform positive reach.
Those properties are hypotheses of the theorem.
The hypothesis supplies a point $\theta^\dagger$, an extended sheet
$\Sigma$ of reach at least $\tau$, and its local patch $S$ on which the jump
is attained. The reach estimates below concern $\Sigma$. Take
$U := B(\theta^\dagger, r_{\mathrm{iso}})$, the isolating ball of
Eq.~\ref{eq:isolation}, inside which $\mathcal{D}$ coincides with $S$ and $S$
therefore separates $U$ into $U^+$ and $U^-$, on each of
which $\Ls$ is $C^1$ with gradient bounded by $G$. Finiteness of $G$ is part of
Definition~\ref{def:pwc-loss}: the gradient bound must hold up to $S$.
Piecewise $C^1$ regularity alone does not ensure this, and
$\Ls = 1 - \sqrt{x}$ on $(0,\tfrac14)$ has an attained unit jump at $0$ with
$|\Ls'| = (2\sqrt x)^{-1}$ unbounded there. Orient the unit normal $n$ toward $U^+$, so
that the traces satisfy $\Ls^+ - \Ls^- \geq J$ on $S \cap U$. Choose coordinates
with $\theta^\dagger=0$ and $n=e_d$.

We need a graph only where the probe segments can meet the sheet.
The quantitative local graph theorem for positive reach
\citep[Theorem~1]{lieutier2026positive} states that, writing
$R=\operatorname{reach}(\Sigma)$, the set
$\Sigma\cap B(0,\sqrt{2}R)\cap\{y:\|y_\perp\|<R\}$ is a graph over
$\{y_\perp:\|y_\perp\|<R\}$, with $g(0)=0$ and $\nabla g(0)=0$.
Here and below $B$ denotes an open ball. For $h\leq\tau/2$, every point
of a segment $rn+u$, $|r|\leq h$, $\|u\|<h$, satisfies
$\|rn+u\|<2h\leq\tau\leq R$ and $\|u_\perp\|<R$.
Thus that segment meets the sheet at most once. If $R=\infty$, the closed
hypersurface is an affine hyperplane and the same conclusion holds directly.
No global graph assumption on $S$ is needed.

The rolling-ball property \citep{federer1959curvature} gives the other
geometric fact: the open balls $B(\theta^\dagger \pm \tau n, \tau)$
meet $S$ nowhere, so for $h \leq \tau/2$ every ball
$B(\theta^\dagger + rn, h)$ with $|r| \in [h, 2h]$ is contained in one
of them and lies wholly on one side of $S$. This explains the restriction $h \leq \tau/2$.

Every comparison below must stay inside $U$: outside it, a segment may
cross another discontinuity sheet and the $C^1$ bound $G$ no longer controls it.
Steps~1 and~2 reach distance $2h$ from $\theta^\dagger$ and Step~3 reaches
$3h$, giving the restrictions $r_{\mathrm{iso}}/2$ and $r_{\mathrm{iso}}/3$.

For a curved sheet, subtracting a planar jump leaves a discontinuous
remainder between the sheet and its tangent plane. The proof instead
compares values across the actual sheet.

\paragraph{Regularity.}
For any kernel $\mu_h\in L^1$ and bounded $\Ls$, the smoothing
$\Ls_h=\Ls*\mu_h$ is continuous. Write $\check\mu_h(u)=\mu_h(-u)$,
so $\Ls_h(\theta)=\int\Ls(\theta+u)\check\mu_h(u)\,\mathrm du$.
Reflection preserves the support and all kernel norms used below. Write
$D(r) := \Ls_h(\theta^\dagger + rn)$, a continuous function on $[-2h, 2h]$.
Eq.~\ref{eq:blowup} bounds $\mathrm{Lip}(D)$, while Eq.~\ref{eq:blowup-hess}
bounds the second-difference seminorm
$\sup_{r,\ell}|D(r+\ell) - 2D(r) + D(r-\ell)|/\ell^2$.
Both are defined, possibly as infinity, for every continuous $D$. When
$D$ is $C^{1,1}$, they equal $\sup|D'|$ and $\mathrm{Lip}(D')$, respectively.
An arbitrary $L^1$ density need not make the convolution differentiable: a
step loss against a density proportional to $|u|^{-1/2}$ has a smoothing
with an unbounded derivative. The difference-quotient arguments below require
only continuity. The second-difference convention also covers the uniform
one-dimensional kernel, whose step smoothing is piecewise affine and has
zero classical second derivative wherever that derivative exists.

\paragraph{Step 1: the straddle offset.}
Let $\ell$ be the infimum of offsets for which both open kernel balls
$B(\theta^\dagger\pm\ell n,h)$ lie on their respective sides of the sheet
within $U$. For $h\leq\tau/2$ and $2h\leq r_{\mathrm{iso}}$, the rolling-ball
property gives
\[
 B(\theta^\dagger\pm hn,h)
 \subseteq B(\theta^\dagger\pm\tau n,\tau)\cap U.
\]
Thus $\ell\leq h$. Conversely, an offset less than $h$ puts
$\theta^\dagger$ inside both kernel balls, so neither ball can lie strictly
on one side. Therefore
\begin{equation}\label{eq:straddle-offset}
 \ell=h.
\end{equation}
The kernel has a density, so its boundary sphere has zero probability.

\paragraph{Step 2: gradient lower bound.}
Fix $u$ with $\|u\| < h$ and consider the segment
$r \mapsto rn + u$ for $r \in [-h, h]$. The whole segment lies within
$2h \leq r_{\mathrm{iso}}$ of $\theta^\dagger$, so it meets $\mathcal{D}$ only
in $S$. The local graph property makes any intersection unique, and
Step~1 puts its endpoints in opposite components $U^-$ and $U^+$, so an
intersection exists. Write it as $p=r^\times n+u$.
By isolation the two
sub-segments meet no discontinuity at all. Splitting at $p$ and using the
$C^1$ bound on each side,
\[
\Ls(h n + u) - \Ls(-h n + u)
\;=\; \underbrace{\Ls(h n + u) - \Ls^+(p)}_{\geq\, -G\,|h - r^\times|}
\;+\; \underbrace{\Ls^+(p) - \Ls^-(p)}_{\geq\, J}
\;+\; \underbrace{\Ls^-(p) - \Ls(-h n + u)}_{\geq\, -G\,|r^\times + h|} ,
\]
and the two path lengths sum to $2h$. Taking expectations over
$u \sim \check\mu_h$,
\begin{equation}\label{eq:two-point}
D(h) - D(-h) \;\geq\; J - 2h G .
\end{equation}
Dividing Eq.~\ref{eq:two-point} by $2h$ directly gives
\begin{equation}\label{eq:grad-lb}
\mathrm{Lip}_{[-h,h]}D \;\geq\; \frac{J}{2h}-G
\;\geq\;\frac{J}{4h}
\quad\text{for }h\leq\min\{J/(4G),\tau/2,r_{\mathrm{iso}}/2\}.
\end{equation}
If $D$ is absolutely continuous, the same lower bound holds for
$\operatorname*{ess\,sup}_{|r|<h}|D'(r)|$. This proves the seminorm statement
without asserting differentiability for an arbitrary $L^1$ kernel. For a smooth
kernel, the ordinary mean value theorem also gives a point where the derivative
satisfies the same lower bound.

\paragraph{Step 3: Hessian lower bound.}
For this step, require $3h\leq r_{\mathrm{iso}}$. When
$|r|\in[h,2h]$, the ball $B(\theta^\dagger+rn,h)$ reaches at most $3h$
from the jump point and remains inside $U$. The rolling-ball property
places it entirely on one side of $S$, with no other discontinuity surface
inside it. For $r_1,r_2$ in the same outer interval, either $[h,2h]$ or
$[-2h,-h]$, and $\|u\|<h$, the segment joining $r_1n+u$ and $r_2n+u$
stays in that component of $U$. The branch gradient bound therefore gives
$|\Ls(r_2n+u)-\Ls(r_1n+u)|\leq G|r_2-r_1|$. Averaging yields
\[
 |D(r_2)-D(r_1)|\leq G|r_2-r_1|.
\]
This includes the endpoints without requiring differentiability there.
We use second differences because $D'$ need not exist throughout the
central interval. By Step~2, $D(h) - D(-h) \geq J - 2h G$, so at least one half-interval accounts for half this difference: either $D(h) - D(0) \geq (J - 2h G)/2$
or $D(0) - D(-h) \geq (J - 2h G)/2$. In the first case take the second
difference at $(r, \ell) = (h, h)$; since $D(2h) - D(h) \leq Gh$
by the outer-interval bound above,
\begin{equation}\label{eq:hess-lb}
\frac{\big|D(2h) - 2D(h) + D(0)\big|}{h^2}
\;\geq\; \frac{(J - 2h G)/2 \;-\; Gh}{h^{2}}
\;=\; \frac{J}{2h^2} - \frac{2G}{h}
\;\geq\; \frac{J}{8h^{2}}
\quad\text{for } h \leq \min\{\tfrac{J}{16G},\tfrac{\tau}{2},\tfrac{r_{\mathrm{iso}}}{3}\} ,
\end{equation}
and in the second case, symmetrically, $(r, \ell) = (-h, h)$ with
$D(-h) - D(-2h) \leq Gh$ from the same outer-interval bound. The middle expression stays above
$J/(8h^2)$ up to $h \leq 3J/(16G)$; the display keeps the stricter
$J/(16G)$ for a simpler statement of the theorem.

This is the lower half of Eq.~\ref{eq:blowup-hess}, and it too holds for every
kernel with a density supported in $B(0,h)$: the scaling hypothesis
$\mu_h(x) = h^{-d}\mu(x/h)$ with $\mu \in W^{2,1}$ is needed only
for the matching \emph{upper} bound, which is standard:
$\|\nabla^2(\Ls * \mu_h)\|_\infty \leq \|\Ls\|_\infty\,\|\nabla^2\mu_h\|_{L^1}
= O(\|\Ls\|_\infty/h^{2})$, since
$\mu_h(x) = h^{-d}\mu(x/h)$ gives
$\|\nabla^2\mu_h\|_{L^1} = h^{-2}\|\nabla^2\mu\|_{L^1}$.
Hence both the normal-segment seminorm and the full gradient Lipschitz
constant on this neighborhood have order $h^{-2}$. Only the lower bound of
Eq.~\ref{eq:blowup-hess} carries $J$; the upper one carries $\|\Ls\|_\infty\|D^2\mu\|_1$,
so the two-sided statement is in $h$ alone.

\paragraph{Step 4: consequences.}
The bounds imply that the seminorms on a neighborhood containing the jump
cannot stay bounded as $h\to0$. They do not imply divergence along every
sequence. A one-dimensional step smoothed at radius $h$ has zero derivative at
$x=2h$, although $x\to0$. A kernel supported away from its center can even
have zero derivative at the jump itself. Thus no conclusion about all pointwise
limits follows from the supremum bound.

For claim~(3), distinguish the probe radius from the step radius. The
algorithm induces the tail kernel of Eq.~\ref{eq:surrogate-def}, of support radius
$h = \delta_{\mathrm{out}} = r_p\xi_K(1+\eta_{\max})\varepsilon$ and scale
$\delta = \bar r_p\varepsilon$, both $\Theta(r_p\varepsilon)$; the step length is bounded above by $r_s\varepsilon$ per particle
without step jitter; it can be zero. For a scaled smooth kernel whose searched
coordinates contain this jump, substituting
$\Lambda_\delta = \Theta(h^{-2})$ into the second-order error allowance
$\Lambda_\delta\,r_s^2\varepsilon^2$
gives an allowance of order $(r_s/r_p)^2$: the two factors of $\varepsilon$ cancel, so
decaying the temperature does not shrink this bound, although the realized
remainder may be smaller. Even before accounting for the change of surrogate between iterations,
summing this allowance gives an order-$T$ term at fixed $r_s/r_p$. This estimate
alone does not provide a vanishing stationarity bound for a decaying-temperature
schedule. It is not a nonconvergence result. \qed

\paragraph{Sharpness.}
The constant $1/2$ in Eq.~\ref{eq:grad-lb} cannot be improved without a
hypothesis on the kernel's shape: for the uniform kernel on $[-h,h]$ in
one dimension, $\Ls_h$ is exactly affine across the jump with slope
$J/(2h)$, so Eq.~\ref{eq:two-point} holds with equality and the mean value theorem introduces no slack. For a planar unit jump smoothed by a uniform ball in
dimension $d>1$, the coefficient is
$\Gamma(d/2+1)/(\sqrt\pi\,\Gamma((d+1)/2))>1/2$; it is $8/(3\pi)$
in $\R^4$, as computed below. This comparison concerns uniform-ball
kernels. The obstruction is
independent of the particular probe distribution: every smoothing kernel
covered by the theorem must traverse the trace gap, up to the $2hG$ branch
variation, across a window of width $2h$.

\paragraph{Closed-form constants.}
For the uniform ball in $\R^4$ applied to a unit step, $\Ls_h(x)$ is the fraction
of $B(x,h)$ with $y_1 \geq 0$, so
$\partial_{x_1}\Ls_h = \mathrm{vol}_3\big(B_3(\sqrt{h^2-x_1^2})\big)/\mathrm{vol}_4\big(B_4(h)\big)
= \tfrac43\pi(h^2-x_1^2)^{3/2}\big/\tfrac12\pi^2h^4$, maximal at $x_1 = 0$ with
value $\sup_x\|\nabla\Ls_h\| = 8/(3\pi h)$. Differentiating once more,
$|\partial^2_{x_1}\Ls_h| \propto |x_1|(h^2-x_1^2)^{1/2}$ is maximal at
$|x_1| = h/\sqrt2$, giving $\Lambda_h = 4/(\pi h^2)$. Both exceed the
lower bounds of Eq.~\ref{eq:blowup} and Eq.~\ref{eq:blowup-hess} by an absolute
factor.

\subsection{Sections of the Discontinuity Set}\label{app:proof-generic-section}

A discontinuity set can have zero ambient volume yet contain an entire
probe plane. The next two lemmas give conditions that exclude this case.
These geometric conditions are not needed for the tail-kernel identity.

\begin{lemma}[Sections of a definable discontinuity set]\label{lem:generic-section}
Let $\mathcal D\subset\R^d$ be definable with $\dim\mathcal D\leq d-1$.
Fix a linear map $\Pi$ and a coordinate block $E_i$, and write
$F_i=\Pi E_i$ and $r_i=\dim F_i$. For Lebesgue-almost every base point
$\theta\in\R^d$,
\[
 \dim\big(\mathcal D\cap(\theta+F_i)\big)\leq r_i-1.
\]
For $r_i=0$, the conclusion means that the section is empty.
\end{lemma}

\begin{proof}
For $r_i>0$, decompose $\R^d=F_i^\perp\oplus F_i$.
The definable set $\mathcal D$ has ambient Lebesgue measure zero.
Fubini's theorem implies that for almost every $y\in F_i^\perp$,
the section $\mathcal D\cap(y+F_i)$ has zero $r_i$-dimensional measure.
Since that section is definable, its dimension is at most $r_i-1$.
The exceptional translates form a null set in $F_i^\perp$; their union
is null in $\R^d$, by Fubini on bounded boxes. For $r_i=0$, the
exceptional base points are exactly $\mathcal D$.
\end{proof}

The relevant dimension is the rank of the restricted map, which may be
less than the nominal block dimension $d_p$. No generic projection
assumption is needed for this almost-every-base-point result. It does
not cover adaptive iterates merely because they lie outside $\mathcal D$:
for $\mathcal D=[1,2]\times\{0\}$, the horizontal section through the
origin has positive length even though the origin is outside $\mathcal D$.

For a finite union of affine walls, a stronger conclusion is available
under a random projection whose columns assign zero probability to each
fixed hyperplane through the origin. A wall with normal $a\ne0$ can
contain a translated probe plane only if every projected column is
orthogonal to $a$. That event has probability zero, so almost surely
every translate has a proper section. Layerwise projections need not
satisfy this column condition, and the argument does not extend to
arbitrary definable sets uniformly over all base points.

\subsection{Sections Along the Iterates}\label{app:proof-blockwise-section}

An almost-everywhere statement does not by itself cover adaptive iterates.
The next lemma uses random initialization and independent step jitter to
obtain a section bound along the whole trajectory.

\begin{lemma}[Sections at the realized iterates]\label{lem:blockwise-section}
Let $\mathcal{D}\subset\R^d$ be semialgebraic with $\dim\mathcal{D}\leq d-1$, and
write $\mathcal{D} \subseteq \bigcup_m \{h_m = 0\}$ for the cell polynomials the proof
constructs. Run \sysname{} with any fixed projection $\Pi$ and initialization
$\theta_0$ whose law is absolutely continuous. Let the step jitters have
independent densities, conditionally independent of the history and the
current probe directions, costs, and batch (Eq.~\ref{eq:step-jitter}). Then almost surely $h_m(\theta_t) \neq 0$ for
every $m$ and every $t$. Hence $\theta_t \notin \mathcal{D}$, and, each $h_m$
restricted to the plane being a polynomial that does not vanish at the base point,
$\dim(\mathcal{D}\cap(\theta_t + \Pi E_i))\leq r_i-1$, where
$r_i=\dim(\Pi E_i)$, for every particle $i$ and
every $t$, without assuming transversality.
\end{lemma}

The section bound follows from nonvanishing of every $h_m$. The weaker
statement $\theta_t \notin \mathcal{D}$ would not suffice: for $\mathcal{D} = [1,2]\times\{0\}$ and the line
$\R \times \{0\}$ the base point at the origin lies off $\mathcal{D}$ while the section
is one-dimensional.

HybridSubspace is block-diagonal by layer, with an identity block for every one-dimensional parameter and for
any weight matrix the rank budget does not compress. For a
wall $\{n^\top x = c\}$ there are three cases: transversal, and the section is an affine
$(r_i-1)$-subspace; parallel with $\theta$ off the wall, and the section is empty; parallel
with $\theta$ on the wall, and the section is the whole plane. Only the third violates the
conclusion, and it needs the base point to sit exactly on $\mathcal{D}$. The lemma is
stated for semialgebraic $\mathcal{D}$ because a depth-two pre-activation
$w_2^\top W_1 x$ is bilinear in the joint parameter vector, so its threshold set is a
quadric and not a hyperplane.

Write $\mathcal{S} = \mathrm{range}(\Pi)$, of dimension $\operatorname{rank}(\Pi)$; the
iterates never leave the affine subspace $\theta_0 + \mathcal{S}$:
every update is $\Pi$ applied to a subspace displacement. All statements below are
within that subspace, which is also where Theorem~\ref{thm:piecewise-smooth}'s
conclusion applies (Remark~\ref{rem:plane-smoothing}).

\paragraph{Step 0: one polynomial per cell.}
Under the lemma's hypotheses, a semialgebraic set of dimension at
most $d-1$ has Zariski closure of the same dimension, hence a \emph{proper}
algebraic subset of $\R^d$, so it is contained in $\{x : h(x) = 0\}$ for some
nonzero polynomial $h$ \citep[Prop.~2.8.2]{bochnak1998real}. Writing $\mathcal{D}$
as a finite union of cells and taking one such $h_m$ per cell, it
suffices to show that almost surely $h_m(\theta_t) \neq 0$ for every $m$ and $t$,
and then to take a countable union.

Polynomial zero sets include the quadratic threshold surfaces of multilayer
networks as well as affine walls.

\paragraph{Step 1: affinity in the jitter.}
Write $\theta_{t+1} = \theta_t + \sum_i (1+\eta_i)\, s_i$ with
$s_i := \Pi\iota_i u_i$, where $u_i \in E_i$ is the barycentric displacement of
particle $i$ and $\iota_i$ is the inclusion of block $i$. The displacement is
measurable with respect to $\mathcal{F}_t$ together with the step-$t$ rotations,
probe jitter and minibatch, because its cost row is evaluated at the probe points; the step jitters $\eta_i$ are drawn independently of all of those, per
particle, with a density supported in $(-\eta_{\max}, \eta_{\max})$. Condition on
everything the displacement depends on, leaving only the $\eta_i$ random, so that the
$s_i$ are fixed. The map
$\eta \mapsto \theta_{t+1}$ is affine, hence
\[
Q(\eta) \;:=\; h_m\Big(\theta_t + \sum_i (1+\eta_i)\, s_i\Big)
\]
is a polynomial in $\eta = (\eta_1,\dots,\eta_P) \in \R^P$ of degree at most
$\deg h_m$.

\paragraph{Step 2: $Q \not\equiv 0$.}
Suppose $Q \equiv 0$ on $\R^P$. Polynomials are determined by their values on any
open set, so this holds at every $\eta$, in particular at
$\eta = (-1,\dots,-1)$, where the displacement vanishes identically and
$Q(-1,\dots,-1) = h_m(\theta_t)$. Then $h_m(\theta_t) = 0$, contrary to the induction hypothesis in Step~3. Hence $Q \not\equiv 0$, its zero set is Lebesgue-null
in $\R^P$, and since the $\eta_i$ are independent with densities their joint law is
absolutely continuous on $\R^P$. Therefore
$\Pr[h_m(\theta_{t+1}) = 0 \mid \mathcal{F}_t] = 0$, and averaging over the
conditioning gives the same conclusion unconditionally.

Step~2 does not assume the joint one-step law has a density on $\mathcal{S}$, which
would be false: on a plateau every $u_i$ vanishes
(Theorem~\ref{prop:plateau-freeze}) and the law is a point mass. In that case, $Q \equiv h_m(\theta_t) \neq 0$, a nonzero
constant whose zero set is empty.

\paragraph{Step 3: induction.}
$\Pr[h_m(\theta_0) = 0] = 0$ by absolute continuity of the initialization, since
$\{h_m = 0\}$ is Lebesgue-null. On the event $\{h_m(\theta_t) \neq 0\}$, Step~2
gives $\Pr[h_m(\theta_{t+1}) = 0 \mid \mathcal{F}_t] = 0$ almost surely. On the
complement the conditional probability can be $1$, which is why the induction
hypothesis is needed. A union over the finitely many
cells, the finitely many particles, and $t \in \mathbb{N}$ preserves the null set,
so almost surely $\theta_t \notin \mathcal{D}$ for every $t$.

\paragraph{Step 4: conclusion.}
Step~3 establishes more than $\theta_t \notin \mathcal{D}$: it gives
$h_m(\theta_t) \neq 0$ for every cell polynomial $h_m$, and that is what the
section-dimension claim needs. Fix a particle $i$ and write
$\mathcal{E}_{i,t} := \theta_t + \Pi E_i$ for its probe plane, an affine space of dimension $r_i$.
The restriction $h_m|_{\mathcal{E}_{i,t}}$ is a polynomial on $\mathcal{E}_{i,t}$ with
$h_m|_{\mathcal{E}_{i,t}}(\theta_t) \neq 0$, so $h_m|_{\mathcal{E}_{i,t}} \not\equiv 0$ and its zero set is a proper
algebraic subset of $\mathcal{E}_{i,t}$, of dimension at most $r_i-1$ (empty when $r_i=0$). Since
$\mathcal{D} \cap \mathcal{E}_{i,t} \subseteq \{h_m = 0\} \cap \mathcal{E}_{i,t}$ cell by cell, and the cells are
finite,
\[
\dim\big(\mathcal{D} \cap (\theta_t + \Pi E_i)\big) \;\leq\; r_i - 1
\]
for every $i$ and $t$. For an affine wall, the transversal case has
$h_m|_{\mathcal{E}_{i,t}}$ of degree one and not constant,
and the parallel-with-base-point-off case is $h_m|_{\mathcal{E}_{i,t}}$ a nonzero constant. A plane lying inside a wall corresponds to $h_m|_{\mathcal{E}_{i,t}} \equiv 0$, which
$h_m(\theta_t) \neq 0$ excludes.
This conclusion holds along the realized trajectory, under the initialization
and jitter hypotheses. It need not hold for every base point. \qed

\paragraph{Hypotheses used.}
The initialization must be absolutely continuous; deterministic zero biases,
for example, do not satisfy that ambient assumption. The lemma also needs
independent step jitters with densities; a
$W^{2,1}$ density is required for Lemma~\ref{lem:smooth-kernel}, but not
for this argument. It does \emph{not} need transversality, a generic projection, $\Pi$ to have
i.i.d.\ entries, $\mathcal{D}$ to be a union of hyperplanes, or the one-step law to
be absolutely continuous on $\mathcal{S}$.
Appendix~\ref{app:theory-numerics} reports finite-sample intersection diagnostics
on HybridSubspace planes along training trajectories.

\paragraph{The remaining role of genericity.}
The lemma covers any fixed projection, including the identity, along the
realized iterates. It does not hold uniformly over every base point. For
example, for a spike threshold
$\mathcal{D} = \{\theta : \theta_j = c\}$ and a coordinate plane
$E = \mathrm{span}(e_{k}, e_{k+1})$ with $j \notin \{k, k+1\}$, every point of
$\theta + E$ with $\theta_j = c$ lies in $\mathcal{D}$, so the section has full
measure in the plane. In the language of the argument above, $a_m = e_j$ is exactly
orthogonal to the plane. Such a base point lies on $\mathcal{D}$, the
probability-zero event the induction excludes, so the example rules out a uniform statement without contradicting the lemma. The same construction works for argmax Voronoi walls
$\{\theta_{j_1} = \theta_{j_2}\}$ and for $\mathrm{floor}$ level sets.
Appendix~\ref{app:theory-numerics} exhibits all three.

\subsection{Path-Weighted Occupancy of a Tube}\label{app:proof-occupancy}

The lemma is purely geometric and allows any tube radius
$r_{\mathrm{tube}}$. Taking $r_{\mathrm{tube}}=\delta_{\mathrm{out}}$, the
outer probe reach, bounds where a probe can cross a discontinuity.
It does not identify where a cost row can saturate.

\begin{lemma}[Path-weighted occupancy of a tube]\label{lem:occupancy}
Let $\mathcal{D}$ be, cell by cell, a locally finite union of hypersurfaces within
$\mathcal{K}$, and let $L_{\mathrm{exc}} > 0$ lower-bound the path length between the end
of one visit to the tube $\{\mathrm{dist}(\cdot,\mathcal{D}) \leq r_{\mathrm{tube}}\}$, for a
radius $r_{\mathrm{tube}} > 0$, and the start
of the next, and write $\mathcal{O}_t := \{\mathrm{dist}(\theta_t,\mathcal{D}) > r_{\mathrm{tube}}\}$ for the
steps taken outside the tube, and $\rho_T := \max_{t<T}\|\Delta\theta_t\|$ for the
largest joint step. Define a visit as a maximal consecutive set of indices with
$\theta_t$ inside the tube. Assume that the sum of $\|\Delta\theta_t\|$
over each visit, including a visit truncated by the horizon, is at most
$W_T:=2r_{\mathrm{tube}}+\rho_T$. This is a discrete path-length
assumption; transversality alone does not establish it. Then along any trajectory with
$\sum_{t<T}\|\Delta\theta_t\| > 0$ (a motionless trajectory has occupancy zero by
convention),
\[
\frac{\sum_{t < T} \|\Delta\theta_t\|\, \mathbf{1}[\mathcal{O}_t^c]}{\sum_{t < T}\|\Delta\theta_t\|}
\;\leq\; \frac{2r_{\mathrm{tube}} + \rho_T}{L_{\mathrm{exc}}} \;+\; \frac{2\,(2r_{\mathrm{tube}} + \rho_T)}{\sum_{t<T}\|\Delta\theta_t\|} ,
\]
where the second term retains its dependence on $r_{\mathrm{tube}}$.
\end{lemma}

Each particle's step is a convex combination of unit vertices scaled by
$(1+\eta_i)\,r_{s,t}\varepsilon$. With orthonormal projection columns
(Assumption~\ref{asm:main}(iii)) and a nonincreasing schedule,
$\rho_T \leq \sqrt{P}\,(1+\eta_{\max})\,r_{s,0}\varepsilon$; for large $P$, this joint-step bound can dominate the tube width. The simpler bound
with numerator $3r_{\mathrm{tube}}$ is recovered when $\rho_T \leq r_{\mathrm{tube}}$.
On a suffix beginning at $t_0$, this condition holds if
$r_{s,t_0}\leq r_{\mathrm{tube}}/(\sqrt{P}(1+\eta_{\max})\varepsilon)$.
The maximum must then be taken over that suffix, rather than the full run.
The bound needs the excursion length, not the sheet separation: a trajectory oscillating
across a single sheet satisfies the crossing hypothesis at every visit yet has occupancy
tending to one (Appendix~\ref{app:theory-remarks}). Being
path-weighted, the bound does not control the fraction of iterations spent
inside the tube after the trajectory stops moving.
Both hypotheses are testable on a realized trajectory;
Appendix~\ref{app:theory-numerics} reports the bound holding when they are met and
failing when either is dropped.

\paragraph{Proof by counting visits.}
Fix $T$, write $\rho_T := \max_{t<T}\|\Delta\theta_t\|$, and let $N$
be the number of visits to the tube
$\{\mathrm{dist}(\cdot,\mathcal{D}) \leq r_{\mathrm{tube}}\}$ completed before $T$. By the visit-length hypothesis, each completed visit contributes at most
$W_T=2r_{\mathrm{tube}}+\rho_T$ to the numerator. Hence
\[
\sum_{t<T}\|\Delta\theta_t\|\,\mathbf{1}[\mathcal{O}_t^c] \;\leq\; (2r_{\mathrm{tube}} + \rho_T)\,N + (2r_{\mathrm{tube}} + \rho_T) ,
\]
the extra term covering a visit in progress at $T$. Between the end of one visit and
the start of the next the trajectory covers at least $L_{\mathrm{exc}}$ of path
length, so with $\Sigma_T := \sum_{t<T}\|\Delta\theta_t\|$ we have
$\Sigma_T \geq (N-1)L_{\mathrm{exc}}$, hence
$(N-1)/\Sigma_T \leq 1/L_{\mathrm{exc}}$. Splitting $N+1 = (N-1) + 2$,
writing $W := 2r_{\mathrm{tube}} + \rho_T$ for the path length one visit can cover,
\begin{align*}
\frac{\sum_{t<T}\|\Delta\theta_t\|\,\mathbf{1}[\mathcal{O}_t^c]}{\Sigma_T}
&\;\leq\; \frac{W\,(N+1)}{\Sigma_T}
\;=\; W\,\frac{N-1}{\Sigma_T} + \frac{2W}{\Sigma_T} \\
&\;\leq\; \frac{W}{L_{\mathrm{exc}}} + \frac{2W}{\Sigma_T} ,
\end{align*}
which is the displayed bound with an explicit constant. The error term is in path
length and not in the visit count, which matters because $N$ need not grow: a
trajectory that stops visiting the tube has a bounded numerator against a growing
$\Sigma_T$, and the ratio tends to zero rather than to
$(2r_{\mathrm{tube}}+\rho_T)/L_{\mathrm{exc}}$. For $N \leq 1$ the numerator is at most
$2(2r_{\mathrm{tube}}+\rho_T)$ and the same bound holds. \qed

\subsection{Discontinuity Sets of the Trained Architectures}\label{app:discontinuity-sets}

The hard operations determine candidate threshold sets; whether a crossing
changes the loss also depends on the data and downstream computation.
The measurements in Appendix~\ref{app:measured-jump} establish nonzero
examples, not a uniform jump bound over all threshold surfaces.

\paragraph{Local isolation.}
A finite stratification of the threshold set permits a neighborhood containing
one smooth sheet around a regular point away from intersections and singularities.
Shrinking that neighborhood can give local isolation and bounded curvature.
The particular reach, gradient, and positive trace-gap bounds in
Definition~\ref{def:pwc-loss} remain local hypotheses. They do not follow
from definability alone.

For fixed-grid INT8 weights, one coordinate's parallel walls have spacing
$\Delta_q$. Other coordinates supply intersecting walls, so the isolation
radius at a chosen point is also limited by its distance to those walls;
$\Delta_q/2$ is not a uniform isolation radius. The smoothing theorem applies
only when its support fits the stated local bounds. A small finite-difference
radius may miss every wall at one point and straddle a wall at another, as the
flat-pair diagnostic illustrates.

\paragraph{(a) Hard-LIF SNN.}
The threshold equation is $u(\theta)=u_{\mathrm{th}}$ for each example,
neuron, and timestep. Within a fixed spike and reset pattern and with fixed
incoming features, $u$ is affine in the incoming weights. Jointly training
upstream layers generally gives polynomial threshold equations within each
cell. Finite unrolling gives finitely many patterns. Some threshold crossings
preserve the final spike count and leave the loss unchanged; the measured
nonzero crossings occur on only part of the candidate threshold set.

\paragraph{Batch averaging.}
Let $n_B$ be the batch size and $j$ the one-example loss jump. If a crossing
changes only that example, the batch jump is $J=j/n_B$. A shared weight
quantization boundary can affect several examples, so this identity does not
apply to every architecture or crossing.

\paragraph{(b) INT8 quantization.}
For a fixed quantization grid, candidate walls are
$\theta_j=(m+1/2)\Delta_q$. Only finitely many meet a bounded parameter
region. A crossing changes one quantized weight by one level, but dead units
or downstream cancellation can leave the batch loss unchanged. The shared
weight can affect multiple examples, so batch averaging alone gives no
universal $1/n_B$ jump law.

\paragraph{(c) Hard-MoE routing.}
For a linear router on fixed input $x$, score ties satisfy
$\langle x,w_j-w_k\rangle=0$. A decision boundary also requires these scores
to be maximal. Crossing it swaps experts; the loss changes only if their
outputs affect the loss differently at that input. Trainable upstream
features can make the parameter-space surfaces nonlinear.

\paragraph{(d) Staircase activations.}
The implemented activation is $\lfloor5\sigma(a)\rfloor/5$, where
$\sigma$ is the logistic sigmoid. Its finite pre-activation thresholds are
\[
 a=\log\!\frac{k}{5-k},\qquad k=1,2,3,4.
\]
Thus the threshold equations are polynomial in the parameters within each
fixed activation cell, with these real constants as offsets. They are
not the integer pre-activation walls of a bare floor function. A transition
need not change the final loss if later computation suppresses it.

\paragraph{(e) Argmax attention.}
The candidate walls are score-tie sets. A score can be bilinear when both
query and key projections are trainable, so these sets need not be
hyperplanes. Switching the selected value can change the loss; a nonzero
jump is an additional property of the crossing.

\paragraph{(f) Continuous references.}
Clipped or piecewise-affine activations, hinge losses, and the evaluated
smooth surrogate networks are locally Lipschitz on bounded neighborhoods.
Under the remaining assumptions, Corollary~\ref{cor:goldstein} transfers the
finite-horizon expected bound to the projected Goldstein distance.
It does not certify a deterministic tolerance or an ambient stationary point.

For the hard architectures, the finite cell descriptions above give definable
threshold sets. Nondegenerate threshold equations define hypersurfaces;
intersections, singular points, and equations that vanish identically require
separate treatment. The auxiliary section lemmas state their own hypotheses.
None of these geometric descriptions verifies a uniform attained jump along
the entire training trajectory.

\subsection{Hitting Time on a Piecewise-Constant Loss}\label{sec:cor-pwc}\label{app:proof-pwc-corollary}

For a piecewise-constant loss, the smoothed gradient vanishes wherever the
kernel support lies in one constant region. Stationarity of the smoothing
can therefore hold away from any successful solution. A binary correctness
oracle, for example, returns zero on success and one otherwise without
exposing internal derivatives.

Corollary~\ref{cor:piecewise-constant} gives a conditional hitting-time bound.
If every step in a specified neighborhood has success probability at least
$p_0>0$, the probability of neither success nor departure within $T$ steps
is at most $(1-p_0)^T$. The lower bound $p_0$ is assumed. Probe geometry alone
does not establish it: even a target within step reach may be invisible to
the probes or give a zero barycenter. The discussion below
describes these obstructions and relates the result to isotropic and
pseudo-Boolean search
\citep{glasmachers2020global,vicente2012discontinuous,antipov2021plateau,doerr2024plateau}.

\begin{corollary}[Hitting-time bound for a piecewise-constant loss]\label{cor:piecewise-constant}
Let $\Ls : \R^d \to \{0,1\}$ be measurable, let
$\mathcal{A} := \Ls^{-1}(\{0\})$ be the success region, and fix a neighborhood
width $r_{\mathrm{hit}}>0$. Write
$\mathcal{R} := \{\theta : \mathrm{dist}(\theta, \mathcal{A}) \leq r_{\mathrm{hit}}\}$.
Suppose the \sysname{} iterates are adapted to $\mathcal F_t$ and satisfy
a uniform per-step hitting-probability bound $p_0\in(0,1]$:
\[
\Pr[\theta_{t+1} \in \mathcal{A} \mid \mathcal{F}_t] \;\geq\; p_0
\qquad \text{on } \{\theta_t \in \mathcal{R}\setminus\mathcal{A}\} .
\]
Define the first entry, subsequent exit, and success times by
\[
\tau_{\mathrm{ent}} := \inf\{t : \theta_t \in \mathcal{R}\},\quad
\tau_{\mathrm{ex}} := \inf\{t > \tau_{\mathrm{ent}} : \theta_t \notin \mathcal{R}\},\quad
\tau_{\mathcal{A}} := \inf\{t : \theta_t \in \mathcal{A}\}.
\]
For every integer $T\geq1$, we have
\[
\Pr\big[\tau_{\mathcal{A}} > \tau_{\mathrm{ent}} + T \ \text{ and }\ \tau_{\mathrm{ex}} > \tau_{\mathrm{ent}} + T \;\big|\; \mathcal{F}_{\tau_{\mathrm{ent}}}\big]
\;\leq\; (1 - p_0)^T
\qquad\text{on } \{\tau_{\mathrm{ent}} < \infty\} ,
\]
so within $T$ steps of entering $\mathcal{R}$ the iterate has hit $\mathcal{A}$ or
left $\mathcal{R}$, with probability at least $1 - (1-p_0)^T$. \end{corollary}

\begin{proof}
On $\{\tau_{\mathrm{ent}}<\infty\}$, define
$M_k=\mathbb 1[\tau_{\mathcal A}>\tau_{\mathrm{ent}}+k,
\tau_{\mathrm{ex}}>\tau_{\mathrm{ent}}+k]$.
When $M_k=1$, the iterate lies in $\mathcal R\setminus\mathcal A$, so
the hitting-probability assumption gives
$\E[M_{k+1}\mid\mathcal F_{\tau_{\mathrm{ent}}+k}]
\leq(1-p_0)M_k$. Iterating proves the bound. No independence across steps
is needed. The conclusion concerns hitting or exiting; conditioning afterward
on remaining in $\mathcal R$ changes the path distribution and does not give
a conditional hitting bound of the same form.
\end{proof}

We specify $r_{\mathrm{hit}}$ separately from the step multiplier: a
displacement bound does not establish a positive hitting probability.
Without jitter and with an isometric search map, one block moves by at
most $r_s\varepsilon$, whereas $P$ blocks can move jointly by as much as
$\sqrt P\,r_s\varepsilon$. Neither bound guarantees a successful step.
The probability $p_0$ is assumed. The following properties describe possible
obstructions to it, without providing sufficient conditions:
\begin{enumerate}[label=(\alph*),nosep,leftmargin=1.6em]
\item Fixed step and probe radii keep the geometric neighborhoods fixed.
A decaying schedule requires a new hitting-probability hypothesis
(Remark~\ref{rem:pwc-schedule}).
\item For row-wise softmax, the probes must produce weights with a nonzero
barycenter. If a row has costs zero on $\mathcal Z$ and $x>0$ elsewhere,
then its displacement is
\[
\Delta x_i=r_s\varepsilon\,
\frac{1-e^{-x/\varepsilon}}{|\mathcal Z|+(V-|\mathcal Z|)e^{-x/\varepsilon}}
R\sum_{v\in\mathcal Z}v_v.
\]
For $r_s>0$, it is nonzero exactly when
\begin{equation}\label{eq:pwc-detect}
\sum_{v\in\mathcal Z}v_v\ne0.
\end{equation}
On a regular simplex this holds for every proper nonempty subset. On an
orthoplex it can fail, as for costs $(0,0,1,1)$ on $(\pm e_1,\pm e_2)$.
For a general finite frame the smallest nonzero subset-sum norm is positive,
but need not be at least one.
\item Independent step-radius jitter randomizes the length of a nonzero
block displacement. For binary losses and finitely many probes, the cost
patterns form a finite set. Without step jitter their possible displacement
norms are finite, so a single moved block is supported on a finite union of
spheres, possibly with an atom at zero. Step jitter can give a density on the
reachable portions of annuli; it gives no density lower bound over an entire
annulus. For a general loss the weights vary continuously with the costs,
so fixed-radius probes do not imply a spherical step law.
\end{enumerate}

A nonzero barycenter alone need not hit $\mathcal A$: angular reachability
and overlap with the step distribution also matter. Positive ambient volume
of $\mathcal A$ is insufficient when search is restricted to an affine
subspace. The jitter density also vanishes at its support boundary.
These facts prevent deriving a uniform $p_0$ from the distance neighborhood
alone. On a plateau beyond probe reach, the core row-wise update without
momentum or direction reuse remains fixed by Theorem~\ref{prop:plateau-freeze}.
The geometric-tail calculation adds no
oracle-complexity improvement over prior isotropic or pseudo-Boolean search
results \citep{glasmachers2020global,vicente2012discontinuous,
antipov2021plateau,doerr2024plateau}.

\paragraph{Localization.}
For the finite networks and fixed data considered here, restricting parameters
to a compact neighborhood bounds the logits and loss and leaves finitely many
quantization levels. The neighborhood used by the joint smoothing must include
the enlargement $\mathcal K_{\delta_{\mathrm{joint}}}$, with
$\delta_{\mathrm{joint}}=\sqrt P\,\delta_{\mathrm{out}}$.
The mean probe radius alone is too small for this purpose. Local boundedness
for these networks does not follow from compactness for an arbitrary
measurable loss, so the latter remains an explicit hypothesis in the theorem.

\subsection{Conditional Drift on a Quadratic Oracle}\label{app:proof-stable-schedules}

A bound on step length does not show that iterates stay near a minimum.
We use a quadratic loss to study when the mean update has a restoring
effect. The result separates small probes from concentrated weights.

\begin{proposition}[Conditional drift on a quadratic oracle]\label{prop:stable-schedules}
Consider one block of dimension $d_p\geq2$ with a centered frame of unit
vertices whose probe costs evaluate
$f(z)=\tfrac12\langle Hz,z\rangle$, $H\succ0$, with Haar rotations,
row-wise softmax, fixed temperature, and no jitter. Set $s=r_s\varepsilon$,
$z=\theta_t-\theta^\star$, and $V_H(z)=\langle Hz,z\rangle$.

\emph{(a) Linearized regime.} As $r_p\to0$
at fixed $\varepsilon$, uniformly on a fixed bounded set of $z$,
\[
\E_R\langle\Delta z,z/\|z\|\rangle
\leq -cs\lambda_{\min}(H)\|z\|+s\,O(r_p^3),
\qquad c=r_p/(2d_p),\quad z\ne0.
\]
This upper bound is negative when the leading term exceeds the remainder.

\emph{(b) Concentrated regime.} Suppose the frame consists of $M=V/2$
antipodal pairs. There is a constant $\gamma_{\mathrm{fr}}>0$, depending only
on $M$ and $d_p$, with the following property. At any $z\ne0$, let
$w_{v^\star}$ be the weight of a cost-minimizing vertex. If
\begin{equation}\label{eq:overlap-condition}
\E_R[1-w_{v^\star}\mid z]\leq\varrho<\gamma_{\mathrm{fr}}/2,
\end{equation}
then
\[
\E_R[\Delta V_H\mid z]
\leq-2s(\gamma_{\mathrm{fr}}-2\varrho)\|Hz\|
+s^2\lambda_{\max}(H).
\]
Thus the conditional drift is negative for
$\|z\|>s\lambda_{\max}(H)/[2(\gamma_{\mathrm{fr}}-2\varrho)
\lambda_{\min}(H)]$, wherever Eq.~\ref{eq:overlap-condition} holds.
At $z=0$, the bound $\Delta V_H\leq s^2\lambda_{\max}(H)$ holds directly.
\end{proposition}

The winning antipodal
pair can depend on both $z$ and the rotation. Antipodality determines the
cost-minimizing sign within that pair, which suffices for the $H$-metric
drift bound. Euclidean radial drift need not be negative for every realized
step when $H$ is anisotropic.

This proposition requires the raw probe costs to be quadratic. Smoothness
of the surrogate alone does not supply that hypothesis for a jumping loss.
The expected concentration condition must also be checked on the region
under consideration; a small temperature alone does not ensure it uniformly
near the minimum. A local drift inequality therefore gives neither a global
confinement theorem nor an exact stationary-amplitude law. The finite-horizon
maximal bound in Proposition~\ref{prop:confinement-ville} states its drift
hypothesis separately.

The schedule comparisons in Appendix~\ref{app:schedule-fragility} report
accuracy changes on two network sizes and separate quadratic diagnostics.
They do not establish a dimension-independent accuracy effect.

\begin{proof}
\emph{Linearized regime.} A symmetric smoothing of the quadratic differs
from it by a constant, so its gradient is $Hz$. The Haar-averaged softmax
expansion of Theorem~\ref{thm:exact-step} gives
$\E_R\Delta z=-csHz+s\,O(r_p^3)$ on any fixed bounded set, at fixed
$\varepsilon$. Consequently,
\begin{equation}\label{eq:drift-lambda-min}
\E_R\langle\Delta z,z/\|z\|\rangle
\leq-cs\lambda_{\min}(H)\|z\|+s\,O(r_p^3).
\end{equation}
The negative term need not dominate arbitrarily close to zero.

\emph{Antipodal costs.} Write the frame as $\{\pm v_j\}_{j=1}^M$ and
$u_j=Rv_j$. With $\rho_k=r_p\varepsilon\xi_k$, its costs are
\[
C_{\pm j}=f(z)\pm\bar\rho\langle Hz,u_j\rangle
+\tfrac12\overline{\rho^2}\langle Hu_j,u_j\rangle,
\quad \bar\rho=K^{-1}\sum_k\rho_k,\quad
\overline{\rho^2}=K^{-1}\sum_k\rho_k^2.
\]
The winning pair minimizes
$\tfrac12\overline{\rho^2}\langle Hu_j,u_j\rangle
-\bar\rho|\langle Hz,u_j\rangle|$, so it can depend on $z$.
Within that pair the winning sign is
$\sigma=-\operatorname{sign}\langle Hz,u_j\rangle$; ties may be resolved
measurably. Set $u=u_{j^\star}$.

\emph{Uniform overlap.} For any unit vector $a$ and $b>0$, a union bound gives
\[
\Pr\{|\langle a,u\rangle|<b\}
\leq\sum_{j=1}^M\Pr\{|\langle a,Rv_j\rangle|<b\}
=M\Pr\{|U_1|<b\},
\]
where $U$ is uniform on $S^{d_p-1}$. Since $d_p\geq2$, this last probability
tends to zero with $b$. Choose $b>0$ such that $M\Pr(|U_1|<b)\leq1/2$
and put $\gamma_{\mathrm{fr}}=b/2$. Then
$\E|\langle a,u\rangle|\geq b/2$ uniformly in $a$ and in the selection
rule, including its dependence on $z$. No dimension-independent bound is claimed.

\emph{Softmax drift.} Let $m_R=1-w_{v^\star}$. The barycentric direction is
$\sigma u+e_R$, with $\|e_R\|\leq2m_R$, and its norm is at most one.
Using Eq.~\ref{eq:overlap-condition},
\begin{equation}\label{eq:drift-H-metric}
\begin{aligned}
\E_R[\Delta V_H\mid z]
&=2\langle Hz,\E_R\Delta z\rangle+\E_R\langle H\Delta z,\Delta z\rangle\\
&\leq-2s\E_R|\langle Hz,u\rangle|
+4s\|Hz\|\E_R m_R+s^2\lambda_{\max}(H)\\
&\leq-2s(\gamma_{\mathrm{fr}}-2\varrho)\|Hz\|
+s^2\lambda_{\max}(H).
\end{aligned}
\end{equation}
The stated radius follows from $\|Hz\|\geq\lambda_{\min}(H)\|z\|$.
This proves a conditional drift inequality at the points where the expected
concentration hypothesis holds. It does not extend that hypothesis to future
iterates or establish a stationary distribution.
\end{proof}

The next result converts a conditional drift bound into a bound on the
probability of leaving an ellipsoid. It requires the drift condition
throughout the stated horizon.

\begin{proposition}[Maximal bound under quadratic drift]\label{prop:confinement-ville}
Let $V_H(z)=\langle Hz,z\rangle$, $H\succ0$, and fix a deterministic
horizon $T_{\mathrm{ov}}\leq\infty$. Suppose the adapted iterates satisfy
\[
\E[V_H(z_{t+1})-V_H(z_t)\mid\mathcal F_t]
\leq\bar\lambda\varepsilon^2r_{s,t}^2
\quad (t<T_{\mathrm{ov}}),
\]
with deterministic $z_0$, $\bar\lambda\geq0$, $\varepsilon>0$, and
$\sum_{t<T_{\mathrm{ov}}}r_{s,t}^2<\infty$.
Eq.~\ref{eq:drift-H-metric} is one sufficient condition, with
$\bar\lambda=\lambda_{\max}(H)$, provided its hypotheses hold at all
nonzero iterates in this horizon. Then for every $a>0$,
\[
\Pr\!\left[\sup_{t<T_{\mathrm{ov}}}V_H(z_t)\geq a\right]
\leq\frac{V_H(z_0)+\bar\lambda\varepsilon^2
\sum_{t<T_{\mathrm{ov}}}r_{s,t}^2}{a}.
\]
In particular the supremum is finite almost surely. At confidence $1-\alpha$,
the iterates lie in the deterministic ellipsoid obtained by setting
$a=[V_H(z_0)+\bar\lambda\varepsilon^2\sum_t r_{s,t}^2]/\alpha$.
\end{proposition}

\begin{proof}
The process
$X_t=V_H(z_t)+\bar\lambda\varepsilon^2
\sum_{t\leq u<T_{\mathrm{ov}}}r_{s,u}^2$ is nonnegative and integrable.
The assumed drift gives $\E[X_{t+1}\mid\mathcal F_t]\leq X_t$.
Ville's maximal inequality yields $\Pr(\sup_tX_t\geq a)\leq X_0/a$,
and $V_H(z_t)\leq X_t$. Sending $a\to\infty$ gives almost-sure boundedness.
\end{proof}

The drift assumption must hold throughout the horizon; local validity near a
quadratic minimum does not suffice. Almost-sure boundedness gives a random
bound, whereas Assumption~\ref{asm:main}(vi) requires a fixed compact on
which the regularity constants are controlled. The displayed ellipsoid provides
such a compact only at a specified confidence level. At a constant nonzero
step radius, the compensator grows linearly with a finite horizon and is
infinite over an infinite horizon. These qualifications prevent using this
model calculation to verify compactness for the trained losses.

\subsection{Finite-Horizon RL as Policy Search}\label{sec:rl-policy-theory}\label{app:proof-rl-policy-search}

For finite-horizon reinforcement learning (RL), the loss is the negative
expected return. Each probe estimates that loss from $M$ rollouts, so its
cost is noisy. With bounded per-step rewards $R_{\max}$ and episode horizon
$H_{\mathrm{ep}}$, the maximum cost error over $N$ probes is
$O(H_{\mathrm{ep}}R_{\max}\sqrt{\log(N/\alpha)/M})$ with probability at least
$1-\alpha$, under the sampling assumptions in
Proposition~\ref{prop:rl-policy-search}. The optimizer therefore assigns
weights to a perturbed cost matrix without differentiating through actions
or the simulator.

The scalar-oracle algorithm applies to negative expected return, with
rollout noise in the measured cost matrix. Each theoretical guarantee
still requires its stated assumptions. Let
$\mathcal{M}$ be a finite-horizon MDP with horizon $H_{\mathrm{ep}}$, rewards with
$|r_t| \leq R_{\max}$, and
\begin{equation}\label{eq:rl-loss}
\Ls_{\mathrm{RL}}(\theta)
 := -\E_{\tau \sim p_\theta}\Big[\textstyle\sum_{t=0}^{H_{\mathrm{ep}}-1} r_t\Big].
\end{equation}

\begin{proposition}[Concentration of the empirical cost matrix]\label{prop:rl-policy-search}
Fix one iteration with $N$ candidate policies $\{\theta_i\}$ from the polytope
probes. Estimate Eq.~\ref{eq:rl-loss} for each with $M$ independent rollouts. With
$\widehat{C}_i$ the empirical and $C_i$ the true negative expected return, with
probability at least $1-\alpha$,
\begin{equation}\label{eq:rl-concentration}
\max_{i \leq N} |\widehat{C}_i - C_i|
\;\leq\;
2 H_{\mathrm{ep}} R_{\max}\sqrt{\frac{\log(2N/\alpha)}{2M}} .
\end{equation}
\end{proposition}

The softmax/OT weights computed from empirical rollouts are therefore the exact
\sysname{} weights for a uniformly perturbed cost matrix with perturbation radius
$O(H_{\mathrm{ep}}R_{\max}\sqrt{\log N/M})$.
Proposition~\ref{prop:softmax-cost-stability} converts this cost error into
a direction error for the row-wise softmax update. Deterministic
or hard-action policies remain valid because
\sysname{} never differentiates through action selection or the simulator.
A further union at level $\alpha/T$ makes Eq.~\ref{eq:rl-concentration} uniform
over $T$ iterations at the cost of a $\sqrt{\log T}$ factor in the radius;
confidence sequences \citep{howard2021time} can support adaptive rollout
allocation for a fixed policy, but do not remove the need to account for
multiple changing candidate policies.

\paragraph{Proof of Proposition~\ref{prop:rl-policy-search}.}
Eq.~\ref{eq:rl-concentration} requires independence only across the $M$
rollouts for each fixed candidate. Common random numbers across candidates
are optional; their effect on pairwise differences is not captured by this
sup-norm bound.

For a fixed candidate policy $\theta_i$, define the bounded random return
$G_i(\tau)=\sum_{t=0}^{H_{\mathrm{ep}}-1} r_t$ and empirical negative-return estimate
$\widehat{C}_i=-M^{-1}\sum_{m=1}^{M}G_i(\tau_{i,m})$.
Because $|r_t|\leq R_{\max}$, each $G_i(\tau)$ lies in
$[-H_{\mathrm{ep}}R_{\max}, H_{\mathrm{ep}}R_{\max}]$, an interval of width $2H_{\mathrm{ep}}R_{\max}$.
Hoeffding's inequality gives
\[
\Pr\!\left(|\widehat{C}_i-C_i| \geq u\right)
\leq
2\exp\!\left(-\frac{2M u^2}{(2H_{\mathrm{ep}}R_{\max})^2}\right).
\]
Setting the right-hand side to $\alpha/N$ and applying a union bound over
the $N$ candidate policies yields Eq.~\ref{eq:rl-concentration}.

An empirical cost error bounded by $\eta$ changes every pairwise logit
gap by at most $2\eta/\varepsilon$. For unit vertex directions,
Proposition~\ref{prop:softmax-cost-stability} bounds the change in the
softmax direction by $\min\{2,\eta/\varepsilon\}$.
The OT plan therefore matches the exact plan for a uniformly perturbed
cost matrix, and the perturbation vanishes as
$M^{-1/2}$ for fixed $H_{\mathrm{ep}},R_{\max},N$.
Common random numbers do not change the expectation $C_i$. They correlate
rollout noise across candidates and can reduce the variance of differences
$\widehat{C}_i-\widehat{C}_j$, which are the quantities that determine
the row-wise softmax/OT ordering. \qed

\section{Dataset Details}\label{app:datasets}

All datasets load through standard libraries, with no preprocessing beyond what
follows.

\paragraph{MNIST.}
60,000 training images and 10,000 test images of handwritten digits, $28 \times 28$
grayscale. Pixel values are divided by 255, then normalized by subtracting
0.1307 and dividing by 0.3081. No augmentation is used. Images are loaded via \texttt{torchvision.datasets.MNIST}
\citep{lecun1998gradient}.

\paragraph{Fashion-MNIST.}
Used by the argmax-attention benchmark (Table~\ref{tab:argmax-accuracy}): 60,000
training and 10,000 test images of clothing items, $28 \times 28$ grayscale, same
split sizes as MNIST, with pixels divided by 255 and normalized using
mean 0.2860 and standard deviation 0.3530. No augmentation. Loaded via
\texttt{torchvision.datasets.FashionMNIST} \citep{xiao2017fashion}.

\paragraph{MNIST and Fashion-MNIST (hard MoE).}
The hard-MoE benchmark (Table~\ref{tab:moe-accuracy}) concatenates the two
image datasets, retaining MNIST labels 0--9 and shifting Fashion-MNIST
labels to 10--19. Each domain uses its own normalization above. A random
10\% of the combined 120K training images forms validation, with split
indices determined by the run seed. The official test sets give 20K images.

\paragraph{SNN benchmark.}
The SNN accuracy benchmark (Table~\ref{tab:snn-accuracy}) uses SpikingMNISTNet on
standard MNIST data with $T_{\mathrm{seq}}{=}15$ LIF timesteps. The memory scaling benchmark
uses the larger SNNVGG11Small model (${\sim}$2.4M parameters) with $T_{\mathrm{seq}}$ varied from 25 to 400.

\paragraph{SST-2.}
Binary sentiment classification from the Stanford Sentiment Treebank
\citep{socher2013recursive}, GLUE version: 67,349 training
sentences and 872 validation sentences. The from-scratch transformer uses
the BERT tokenizer (vocabulary 30,522), truncated to 128 tokens. The GPT-2
experiments use that model's tokenizer and their stated subsets.

\paragraph{ETTh1.}
Electricity Transformer Temperature dataset~\citep{zhou2021informer}: 17,420 hourly
observations of oil temperature (OT column). We use the first 8,640/2,880/2,880
observations for train/validation/test, leaving the final 3,020 unused. Univariate prediction, z-score normalized on
train-split statistics only. Lookback window $L{=}96$, prediction
horizon $H{=}96$.

\section{Model Architectures}\label{app:architectures}

\paragraph{MNISTNet.}
Two-layer MLP: \texttt{Flatten} $\to$ \texttt{Linear(784, 128)} $\to$
\texttt{ReLU} $\to$ \texttt{Linear(128, 10)}.
Total parameters: 101,770 ($\approx$102K).
Input: flattened $28 \times 28$ grayscale image.

\paragraph{SNNVGG11Small.}
VGG-11-style architecture with leaky integrate-and-fire (LIF) neuron layers replacing
ReLU activations. Accepts a temporal sequence of input frames (parameterized by $T$).
Forward-only evaluation processes timesteps sequentially without retaining an
autograd graph. The input, current activations, and membrane states still
occupy memory (Figure~\ref{fig:scalability-plate}b). Total parameters: approximately 2.4M.

\paragraph{TransformerClassifier.}
Two-layer transformer encoder for binary text classification. Embedding dimension: 128.
Number of attention heads: 2. Feed-forward dimension: 256. Maximum sequence length: 128.
Vocabulary size: 30,522 (BERT tokenizer). A learned CLS token is used for classification.
Total parameters: approximately 4.2M.

\paragraph{TimeSeriesLSTM.}
A single-layer LSTM with input width 1 and hidden width 64
followed by \texttt{Linear(64, 96)}. Total parameters: 23,392. Input: 96-step lookback
window of shape $(96, 1)$. Output: direct 96-step prediction, the last LSTM hidden
state passed through the linear head.

\section{Hyperparameter Details}\label{app:hyperparameters}

Table~\ref{tab:hyperparameters} gives the selected PolyStep settings and
Table~\ref{tab:selected-baselines} gives the selected gradient-free baseline
settings for the six image benchmarks. Values come from the saved run configurations and are identical across the five final seeds.
The experiment runners set these values explicitly; the library defaults
do not reproduce the reported configurations.

\paragraph{Tuning and selection.}
The image sweeps use two rounds of nine configurations per method and
validation accuracy at seed $42$. PolyStep varies multipliers on temperature,
step radius, and probe radius, with ten training epochs per configuration.
The baseline sweeps use $0.5\%$ of the reference evaluation budget for
population methods, subject to a minimum of $1{,}000$ evaluations; sequential
methods receive their step-matched allowance. The second round recenters
its axes on the first-round selection and extends the range when that
selection is at an edge. Equal grid sizes therefore do not imply equal
numbers of tuning evaluations or equal tuning time. Seed $42$ is also one
of the five reported seeds; their spread measures variation under the
selected settings, not uncertainty from repeating configuration selection.
The frozen-head comparison, controlled weighting study, MAX-SAT, and ETTh1
use their separately stated protocols.

\paragraph{Reference perturbation scale.}
Let $b$ be the physical probe-radius parameter of the reference PolyStep
configuration, before the fractions $\xi_k$. For a scheduled radius the
conversion uses its target value; for a scalar multiplier it uses
$r_p\varepsilon_{\mathrm{target}}$. A baseline with $D$ search coordinates
starts its scale grid at $b/\sqrt D$. An isotropic Gaussian with coordinate
standard deviation $\sigma$ has root-mean-square norm $\sigma\sqrt D$.
At the unmultiplied reference scale this is $b$, whereas PolyStep's sole
probe at $K=1$ lies at $b/2$ in one block. The conversion accounts for
dimension; it does not equate the probe laws or their norms. The grid then
tunes each baseline's scale independently. For EGGROLL, $D$ is the dimension
of its factored representation. Baseline radii remain fixed, while PolyStep
may anneal its radii or temperature. Thus neither the perturbation schedules
nor the resulting search distributions are matched.

\paragraph{Physical radii.}
Table~\ref{tab:hyperparameters} reports physical distances:
$\ell_t$ is the step radius before weighting or momentum and
$\rho_t^{\mathrm{base}}$ is the probe radius before multiplying by $\xi_k$.
For scalar settings these are $r_s\varepsilon_t$ and $r_p\varepsilon_t$;
for scheduled settings they are the schedule values themselves
(Remark~\ref{rem:radius-parameterization}). All rows use $K=1$, so the
unjittered probe distance is $\rho_t^{\mathrm{base}}/2$.
The image and ETTh1 schedules span the configured training horizon.
MAX-SAT uses the separate schedule periods described below.

\begin{table}[ht]
\caption{Recorded PolyStep configurations. Arrows give initial and target
physical radii and temperatures, not necessarily attained endpoints.
Image runs use 30 epochs, batch 512, and $d_p=8$; ETTh1 uses 30 epochs,
batch 64, and $d_p=8$. MAX-SAT uses $d_p=2$; RL uses $d_p=8$.
All use the simplex and softmax. ``Reuse'' is the interval between fresh
probe sets; momentum coefficients and biased rotations are listed when active.}
\label{tab:hyperparameters}
\centering\small
\setlength{\tabcolsep}{3pt}
\begin{tabular}{@{}lcccl@{}}
\toprule
Benchmark (rank) & $\varepsilon_t$ & $\ell_t$ & $\rho_t^{\mathrm{base}}$ & Additional settings \\
\midrule
MNIST (8) & $1\to0.01$ & $2.887\to0.5774$ & $5.774\to1.155$ & reuse 3 \\
SNN (4) & $0.5$ & $1$ & $0.5$ & biased rotation \\
INT8 (8) & $5\to0.3$ & $32\to8$ & $2\to0.5$ & momentum $0.3\to0.5$ \\
Argmax (8) & $15\to0.9$ & $96\to24$ & $6\to1.5$ & \\
Staircase (4) & $15\to3$ & $192\to96$ & $6\to3$ & \\
Hard MoE (4) & $0.5$ & $36\to12$ & $0.5$ & biased rotation \\
ETTh1 (8) & $10\to0.1$ & $5\to1$ & $10\to2$ & reuse 3; momentum $0.5\to0.95$ \\
\midrule
MAX-SAT (size-scaled) & $5\to0.5$ & $(3000\to600)s_n$ & $(100\to20)s_n$ & reuse 3; momentum $0.5\to0.95$ \\
MAX-SAT ($n=10^6$) & $5\to0.5$ & $5000\to1000$ & $500\to100$ & momentum $0.5\to0.95$ \\
CartPole & $2\to0.3$ & $0.2\to0.03$ & $3\to0.45$ & probe jitter $0.05$ \\
Acrobot & $2\to0.3$ & $0.2\to0.03$ & $4\to0.6$ & probe jitter $0.05$ \\
\bottomrule
\end{tabular}
\smallskip

{\raggedright\scriptsize $s_n=\sqrt{n/10^5}$. MNIST's exact initial radii
are $5/\sqrt3$ and $10/\sqrt3$, with targets one fifth as large.
The size-scaled MAX-SAT configuration applies through $10^5$ variables;
the 100-/500-variable configurations multiply its endpoints as specified
in Appendix~\ref{app:maxsat-configurations}.
Reuse uses EMA weight $0.7$. Unlisted momentum, reuse, and jitter options
are disabled. SNN and hard MoE periodically absorb coordinates into the
reference parameters without changing the fixed basis.
The SST-2 limitation studies and quantized-head sweep state their
settings in Appendices~\ref{app:gpt2-finetune} and~\ref{app:gpt2-headquant}.\par}
\end{table}

\begin{table}[ht]
\centering\small
\setlength{\tabcolsep}{4pt}
\caption{Selected baseline settings for the primary image comparisons,
rounded to three significant digits. Each cell follows the parameter order
at left: perturbation scale, learning rate $\eta$, and population $N$, as
applicable. $N_0$ is CMA-ES's initial population; restarts may change it.
EGGROLL uses perturbation rank 1. SPSA uses exponents $0.602$ and $0.101$.
These are the per-architecture-table settings; the population-2048
comparisons are specified in Section~\ref{sec:other-axes}.}
\label{tab:selected-baselines}
\begin{tabular}{@{}lcccccc@{}}
\toprule
Method; parameter order & SNN & MNIST & INT8 & Argmax & Staircase & MoE \\
\midrule
\shortstack[l]{OpenAI-ES\\$\sigma,\eta,N$} & \shortstack{$0.0304$\\$0.003$\\$32$} & \shortstack{$0.0216$\\$0.003$\\$32$} & \shortstack{$0.0773$\\$0.003$\\$32$} & \shortstack{$0.0738$\\$0.003$\\$32$} & \shortstack{$0.054$\\$0.01$\\$32$} & \shortstack{$0.0673$\\$0.003$\\$32$} \\
\addlinespace[3pt]
\shortstack[l]{CMA-ES\\$\sigma_0,N_0$} & \shortstack{$0.000474$\\$128$} & \shortstack{$0.0216$\\$64$} & \shortstack{$0.00483$\\$128$} & \shortstack{$0.00115$\\$32$} & \shortstack{$0.027$\\$128$} & \shortstack{$0.00421$\\$22$} \\
\addlinespace[3pt]
\shortstack[l]{EGGROLL\\$\sigma,\eta,N$} & \shortstack{$0.0381$\\$0.0178$\\$32$} & \shortstack{$0.0568$\\$0.03$\\$46$} & \shortstack{$0.0409$\\$0.03$\\$64$} & \shortstack{$0.0202$\\$0.01$\\$16$} & \shortstack{$0.0548$\\$0.0949$\\$46$} & \shortstack{$0.104$\\$0.0949$\\$128$} \\
\addlinespace[3pt]
\shortstack[l]{MeZO\\$\mu,\eta$} & \shortstack{$0.00759$\\$0.000949$} & \shortstack{$0.0433$\\$0.003$} & \shortstack{$0.0193$\\$0.003$} & \shortstack{$0.0185$\\$0.01$} & \shortstack{$0.216$\\$0.0949$} & \shortstack{$0.0673$\\$0.01$} \\
\addlinespace[3pt]
\shortstack[l]{SPSA\\$a,c$} & \shortstack{$0.1$\\$0.0152$} & \shortstack{$0.949$\\$0.00135$} & \shortstack{$0.3$\\$0.0193$} & \shortstack{$0.3$\\$0.00462$} & \shortstack{$0.949$\\$0.0135$} & \shortstack{$0.3$\\$0.00421$} \\
\addlinespace[3pt]
\shortstack[l]{Random search\\$\sigma$} & \shortstack{$0.00759$} & \shortstack{$0.173$} & \shortstack{$0.0386$} & \shortstack{$0.0369$} & \shortstack{$0.108$} & \shortstack{$0.0673$} \\
\bottomrule
\end{tabular}
\end{table}

The image Adam references use learning rate $10^{-3}$, zero weight decay,
and 30 epochs. They use the differentiable or relaxed models identified in
Section~\ref{sec:setup}. ETTh1's baseline settings are given below.

\paragraph{Solver acceleration features.}\label{app:turbo-features}
Table~\ref{tab:turbo-features} lists the solver acceleration features enabled in each benchmark. These settings change PolyStep's executed update and cost; the compared baseline implementations do not use them.

\begin{table}[ht]
\caption{Acceleration features enabled in the primary runs. \checkmark = active; -- = disabled. Hard MoE (not a column) also runs biased rotation. The Sinkhorn-path options (Anderson acceleration, adaptive $\varpi$, data-dependent initialization, dual momentum) are disabled in every primary run: the softmax solver has no iterations for them to accelerate, and their effects are measured separately in this appendix.}
\label{tab:turbo-features}
\centering
\small
\begin{tabular}{lcccc}
\toprule
Feature & MNIST & SNN & ETTh1 & MAX-SAT \\
\midrule
EMA amortization & \checkmark & -- & \checkmark & \checkmark \\
Biased rotation & -- & \checkmark & -- & -- \\
\bottomrule
\end{tabular}
\end{table}

\paragraph{MNIST configuration.}
MNIST uses EMA-amortized OT (direction reuse over 3 steps at moving-average weight
0.7); biased rotation is off, having no measured effect
on this objective, and the Sinkhorn-path options are inert under the softmax
solver the primary run uses.
The primary runs use $K=1$; their temperature schedules are listed in
Table~\ref{tab:hyperparameters}. The probe-count ablation compares this
choice with additional probes under its own budget and selection rule.
The reported ablation favors $K=1$. Softmax regularization changes the
assignment weights; it does not make multiple probe evaluations redundant
in general.

\begin{table}[ht]
\caption{ETTh1 forecasting ($H{=}96$, z-scored on training statistics, 5 seeds, mean $\pm$ std); best gradient-free result per column in bold. Learned rows report test error at the validation-selected checkpoint. Gradient-free budgets derive from 3,990 updates; Adam uses 6,650. Persistence is an untrained reference.}
\label{tab:lstm-timeseries}
\centering
\small
\begin{tabular}{lccc}
\toprule
Method & MSE ($\downarrow$) & MAE ($\downarrow$) & val MSE ($\downarrow$) \\
\midrule
Persistence (naive) & $0.069$ & $0.202$ & $0.138$ \\
\midrule
\sysname{} & \timeseriesBestMse{} & $0.428 \pm 0.023$ & $\mathbf{0.109}$ \\
EGGROLL & $\mathbf{0.240 \pm 0.007}$ & $\mathbf{0.416 \pm 0.006}$ & $0.118$ \\
CMA-ES & $0.281 \pm 0.006$ & $0.451 \pm 0.006$ & $0.123$ \\
OpenAI-ES & $0.289 \pm 0.015$ & $0.442 \pm 0.014$ & $0.135$ \\
MeZO & $0.957 \pm 0.099$ & $0.911 \pm 0.052$ & $0.168$ \\
SPSA & $1.853 \pm 0.061$ & $1.311 \pm 0.022$ & $0.323$ \\
Random search & $1.979 \pm 0.049$ & $1.355 \pm 0.018$ & $0.349$ \\
\midrule
Adam & $0.247 \pm 0.023$ & $0.413 \pm 0.020$ & $0.109$ \\
\bottomrule
\end{tabular}
\end{table}

In the separate momentum ablation, disabling momentum increases validation
MSE by a factor of four. That comparison uses the ablation protocol and
does not explain the ordering of the validation-selected test errors.

\paragraph{ETTh1 time-series configuration.}\label{app:lstm-details}
PolyStep uses the settings in Table~\ref{tab:hyperparameters}, with 4,704
search coordinates and 3,990 updates. The gradient-free baselines receive
budgets based on the same 3,990-update allowance: population 32 gives
127,680 candidate evaluations; two-point methods receive 7,980.
Their reference perturbation scale is $2/\sqrt{4704}\simeq0.02916$.
EGGROLL uses 3,680 factored coordinates and scale
$2/\sqrt{3680}\simeq0.03297$, with perturbation rank 1.
OpenAI-ES and EGGROLL use learning rate $0.01$ and population 32;
CMA-ES starts with population 32. MeZO uses learning rate $0.01$.
SPSA uses $a=0.1$, $c=0.02916$, and exponents $0.602$ and $0.101$.
Random search uses scale $0.02916$. Adam uses learning rate $0.001$ for
50 epochs (6,650 updates). Every learned method selects by validation MSE.
Training windows are shuffled; validation and test windows remain within
their respective chronological splits. Persistence repeats the last observed
value for all 96 future steps.

\paragraph{Distribution shift on ETTh1.}
In training-normalized units, the test segment has mean $-1.338$ and
standard deviation $0.343$, compared with training mean $0$ and standard
deviation $1$. Persistence uses the last observed level and is invariant to
an additive shift shared by a window's inputs and targets. Its test MSE is
$0.069$, below every learned method's $0.240$--$1.979$.
The measured level and variance changes may contribute to this ordering;
the experiment does not isolate their effects from model fit or optimization.

\section{Direction Reuse and the Cost of Staleness}\label{app:turbo-mode}

Three optional features modify the fresh-step procedure.
(1)~\textbf{EMA amortization} stores an exponential moving average of
barycentric displacements. Between fresh solves, the optimizer reapplies
that direction with a linearly decreasing multiplier and no new probe
evaluations. It carries forward the last recorded cost; these steps do not
check the updated loss or perform loss-based rejection.
(2)~\textbf{Adaptive probe reuse} caches probe costs using displacement
and configuration checks. The supplied implementation also requires a
stationary objective token and checks the whole coordinate state, since
moving any block can change every candidate's loss. This is distinct from
the optional reduction of the number of probes per vertex.
(3)~\textbf{Transport-biased rotation} biases a rotation axis toward the
previous barycentric displacement and reorthogonalizes the frame.

The setting \emph{plan reuse} in the reported configurations refers
to direction amortization. Across a jump, an arbitrarily small change in
parameters can change the cost ordering, so a previously useful direction
need not remain useful. This is a possible failure of reuse, not a guarantee
that reuse fails on every discontinuous objective.

\paragraph{Effect of the reuse interval.}
A three-seed MNIST sweep compares reuse intervals at a fixed budget, using
maximum test accuracy over epochs as in Section~\ref{sec:ablation}. It uses
a separate base configuration, so its scores are not directly comparable
with the primary validation-selected result of \mnistBestAcc{}.
With a fresh solve every step, maximum accuracy over 30 epochs is
76.4\% $\pm$ 6.2\%. A fresh solve every 3 steps gives
74.3\% $\pm$ 3.7\% maximum accuracy and 54.8\% over the last three epochs.
With a fresh solve every 10 steps, accuracy reaches 58.8\% during training
but ends at 24.2\% $\pm$ 4.7\%, a $-34.6$~pp change.
On the non-differentiable gallery, disabling amortization improves accuracy by 2--3~pp on staircase and MoE and up to 6~pp on SNN.
Every non-differentiable architecture benchmark therefore uses
reuse over 1 step; MAX-SAT keeps $3$ below $10^6$ variables
(Section~\ref{sec:ablation}). These retained settings do not certify that
a reused direction agrees with a fresh update.

The reuse comparisons measure accuracy and candidate evaluations. They do
not establish wall-clock savings (Section~\ref{sec:limitations-disc}).

\section{Scalability Analysis}\label{app:scalability}

\paragraph{Size-scaled MAX-SAT configuration and incremental evaluation.}
The size-scaled configuration up to $10^5$ variables uses step-radius endpoints
$3000\to600$ at $10^5$, scaled by $\sqrt{n_{\mathrm{vars}}/10^5}$ at
smaller sizes, with momentum and direction reuse over three steps.
The realized schedules remain close to their initial values over the
recorded budgets (Table~\ref{tab:hypothesis-coverage}). The million-variable
configuration uses radius endpoints $5000\to1000$, $500$ steps, and an
incremental evaluator: a compressed sparse row index maps each variable to
the clauses containing it, so a block perturbation updates only affected
clauses. Its configuration and evaluation protocol differ from the smaller
instances. Figure~\ref{fig:maxsat-scaling-full} separates that point from
the $100$--$100{,}000$-variable curves. The million-variable mean runtime is
$1{,}216$\,s with $2{,}310.5$\,MiB peak GPU memory; the $100{,}000$-variable
mean runtime is $9{,}236$\,s with the full evaluator. The change in evaluator
and budget prevents interpreting this decrease as a scaling law.

\paragraph{Size-specific MAX-SAT configurations.}\label{app:maxsat-configurations}
The 100- and 500-variable rows of Table~\ref{tab:maxsat-scaling} use
PolyStep configurations selected on their respective fixed instances.
Each run uses 1,000 updates, the simplex with $d_p=2$ and $K=1$,
softmax, zero jitter, momentum $0.5\to0.95$, and direction reuse every
three updates with EMA weight $0.7$. These give 334 fresh probe sets,
or 50,100 and 250,500 candidate evaluations
per run, respectively. CMA-ES uses initial population 50 and coordinate
scale $\sigma_0=1$; OpenAI-ES uses population 50, $\sigma=1$, and
learning rate $0.1$. These baseline settings are fixed across sizes.
For MAX-SAT, OpenAI-ES uses antithetic Gaussian probes and rewards
centered and divided by their sample standard deviation; it applies no
update when that deviation is at most $10^{-8}$.

The coarse grid multiplies both endpoints of the size-scaled temperature,
step radius, and probe radius by
$\{0.01,0.1,1\}\times\{0.3,1,3\}\times\{0.3,1,3\}$.
Each of its 27 settings uses optimizer seeds 2027--2029. Around the
selected setting, a nine-setting refinement holds temperature fixed and
multiplies its step radius by $\{1/3,1,3\}$ and probe radius by
$\{1,3,9\}$, using the same tuning seeds. Selection maximizes mean final
clause satisfaction; ties favor the earlier grid entry, with the coarse
winner first in the refinement. Each size therefore receives 108 tuning
runs, including repeated settings, before five final runs on
$\{42,123,456,789,1337\}$. The final seeds do not select the configuration.
Configuration selection costs 5,410,800 and 27,054,000 candidate evaluations
at the two sizes. These selection budgets are specific to PolyStep and
are not matched to the baselines.
The selected multiplier triples (temperature, step radius, probe radius)
are $(0.1,3,27)$ at 100 variables and $(0.1,3,9)$ at 500 variables.

Temperature and both radii use a cosine schedule. For initial value $u_0$
and target $u_\star$, it is
$u(t)=u_\star+(u_0-u_\star)[1+\cos(\pi\min\{t,T_u\}/T_u)]/2$,
where $t=0,\ldots,999$ and
$T_u=\max\{1,\lceil(u_0-u_\star)/0.01\rceil\}$.
Thus rescaling endpoints also changes the schedule period; targets need
not be reached within 1,000 updates. On fresh steps, the heavy-ball
coefficient is $\beta_t=0.5+0.45\min\{t/49,1\}$.
Both selected temperatures decrease from $0.5$ to $0.05$ by $t=45$.
The physical step radii change from $284.60$ to $283.53$ and from
$636.40$ to $635.91$ over the run; the physical probe-radius parameters
change from $85.38$ to $81.84$ and from $63.64$ to $58.96$, respectively.
The actual probe distances are half these latter values.

These runs use the released optimizer on CPU in float32. An exact
clause lookup reuses unaffected clauses and evaluates the changed
two-variable block; it returns the same full-instance costs as direct
evaluation. Candidate counts include each full-instance score, including
those computed through this reuse. Final Boolean assignments are scored
against every clause. Configurations, assignments, and candidate counts
are provided with the results. Figure~\ref{fig:maxsat-scaling-full} distinguishes these size-specific
settings from the size-scaled GPU configuration and its timings.

\begin{figure*}[ht]
  \centering
  \includegraphics[width=\linewidth]{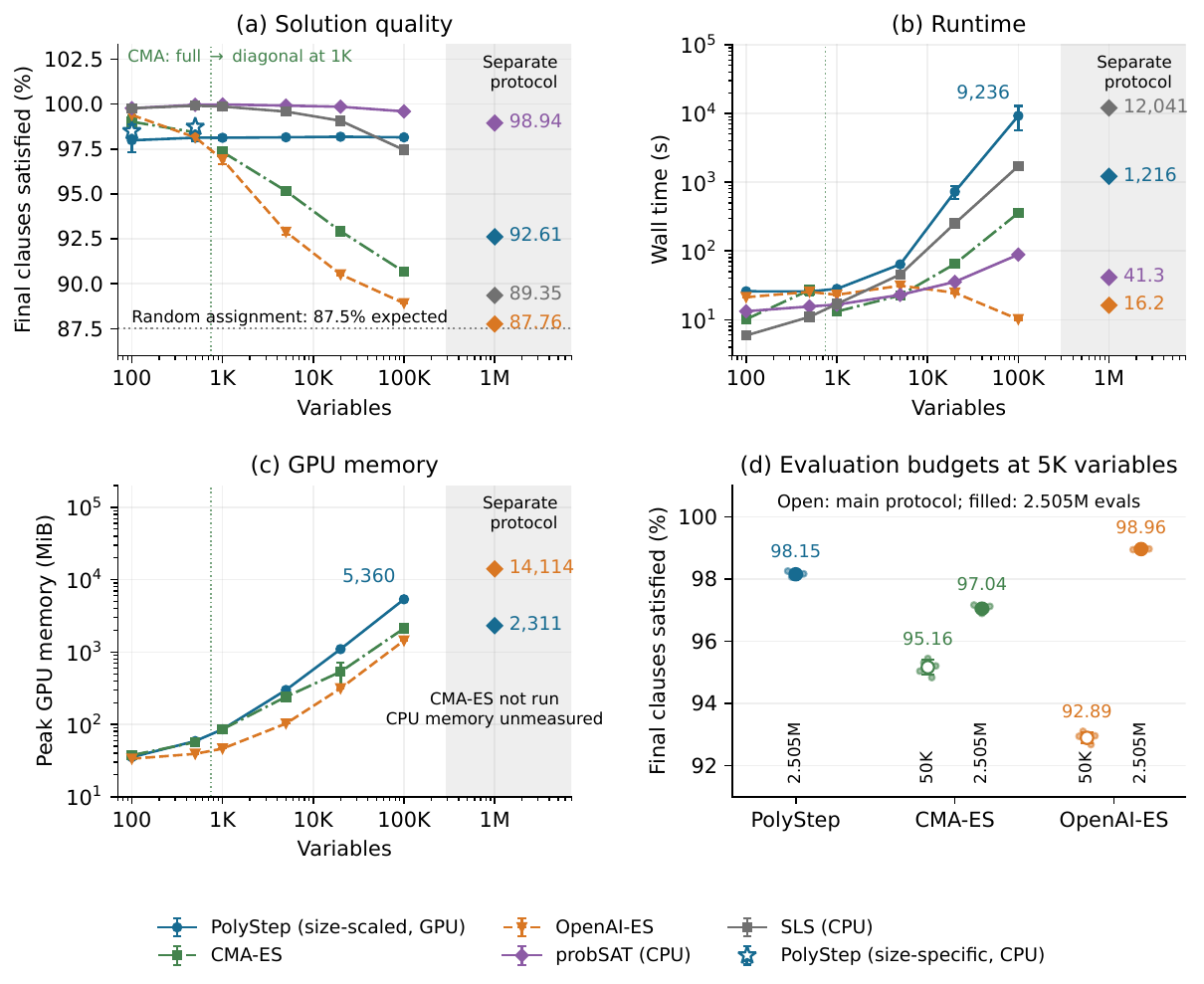}
  \caption{Random MAX-3-SAT at clause-to-variable ratio approximately $4.27$.
  (a--c)~Final satisfaction, runtime, and GPU memory. General-purpose methods
  use five optimizer seeds on one fixed instance per size (mean $\pm$ sample
  std); probSAT and SLS are single-run CPU references with separate budgets.
  Connected PolyStep curves use the size-scaled GPU configuration;
  open stars in (a) show the size-specific CPU configurations at 100 and
  500 variables used in Table~\ref{tab:maxsat-scaling}. Panels (b,c)
  give runtime and memory for the connected-curve configurations.
  Below $10^6$, \sysname{} and OpenAI-ES use $1{,}000$ updates; CMA-ES
  receives $50$K evaluations and switches to diagonal covariance at $1$K
  variables. Isolated diamonds at $10^6$ use a separate protocol with $500$
  updates and incremental evaluation for \sysname{}.
  (d)~At $5$K variables, raising the baselines' budget to $2.505$M evaluations
  reverses the OpenAI-ES/\sysname{} ordering. Dots show individual seeds;
  \sysname{}'s five runs provide a single reference for both budgets.}
  \label{fig:maxsat-scaling-full}
\end{figure*}

\paragraph{Parameter scaling.}
Forward evaluation of the $P(d_p+1)K$ simplex probes dominates the cost of
a fresh step. The number of probes grows linearly with the searched
dimension; \texttt{torch.vmap} batches their evaluation \citep{pytorch2024}.
The orthoplex used in theory mode costs $2Pd_pK$ evaluations.
At fixed rank, HybridSubspace's dimension grows with the layer widths and
number of layers. Compression reduces the searched dimension, but neither the
probe count nor the assignment cost is constant across model sizes.

\paragraph{Sparse projection.}
Figure~\ref{fig:scalability-plate}a compares projection-storage estimates.
Dense float32 storage is $4d r$ bytes; sparse COO storage is
$20\,\mathrm{nnz}$ bytes with two int64 indices and one float32 value per
entry. For the displayed configurations, dense-to-sparse storage ratios
range from $200$ to $2{,}000$. The estimates exclude working buffers and
cached representations.
Rank and density vary across the three sizes, so the panel describes these
configurations rather than a scaling law at fixed rank. It does not measure
training accuracy or runtime at $100$M parameters.

\begin{figure*}[ht]
  \centering
  \includegraphics[width=\linewidth]{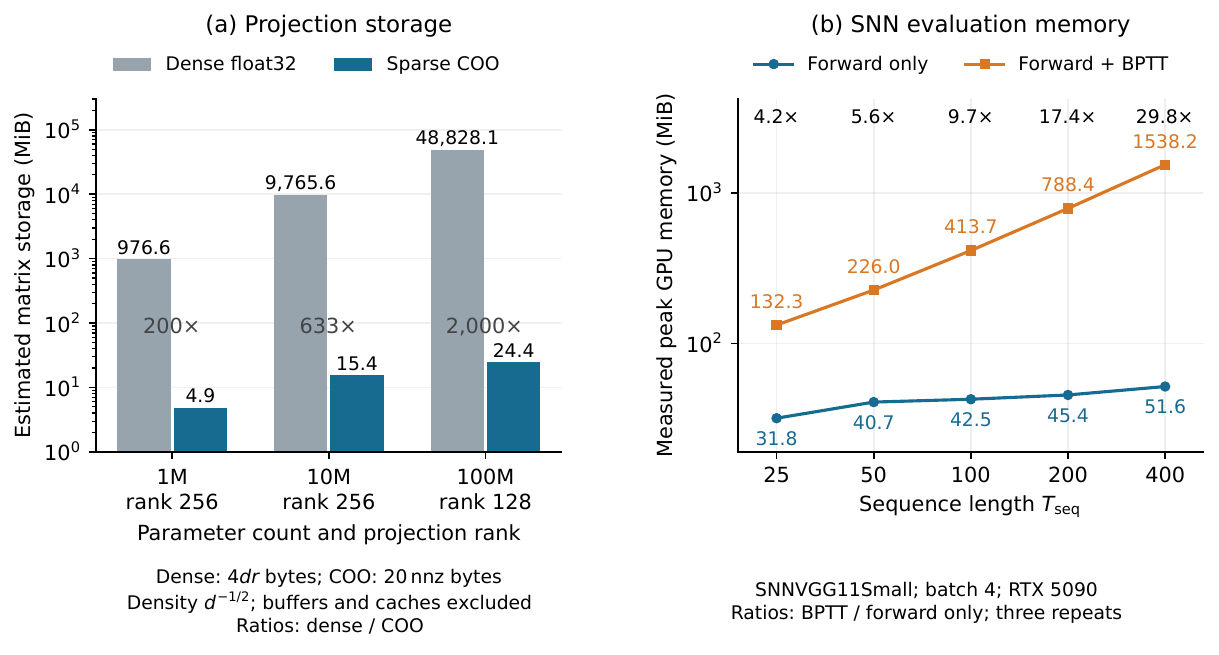}
  \caption{Projection storage and recurrent evaluation memory.
  (a)~Dense float32 and sparse COO estimates, with storage ratios.
  COO includes two int64 indices per nonzero; ranks vary as shown and density
  is $d^{-1/2}$. Working buffers and caches are excluded.
  (b)~Peak GPU memory for forward-only evaluation and forward plus BPTT on
  SNNVGG11Small, batch $4$, RTX 5090. Labels give every measurement and the
  BPTT-to-forward ratio. Three repetitions agree at each point.
  Both axes in (b) are logarithmic. These measurements cover evaluation,
  not a complete optimizer step or an asymptotic memory law.}
  \label{fig:scalability-plate}
\end{figure*}

\paragraph{\texttt{torch.compile} acceleration.}
The ten-epoch MNIST comparison at HybridSubspace rank $4$ takes
$2{,}094$\,s in eager mode and $2{,}120$\,s with compilation (three seeds;
Figure~\ref{fig:ablation-plate}f). It shows no runtime gain in this
configuration. The total times do not separate compilation overhead from
kernel execution.

\section{Extended Ablation Results}\label{app:ablation}

\begin{figure*}[t]
  \centering
  \includegraphics[width=\linewidth]{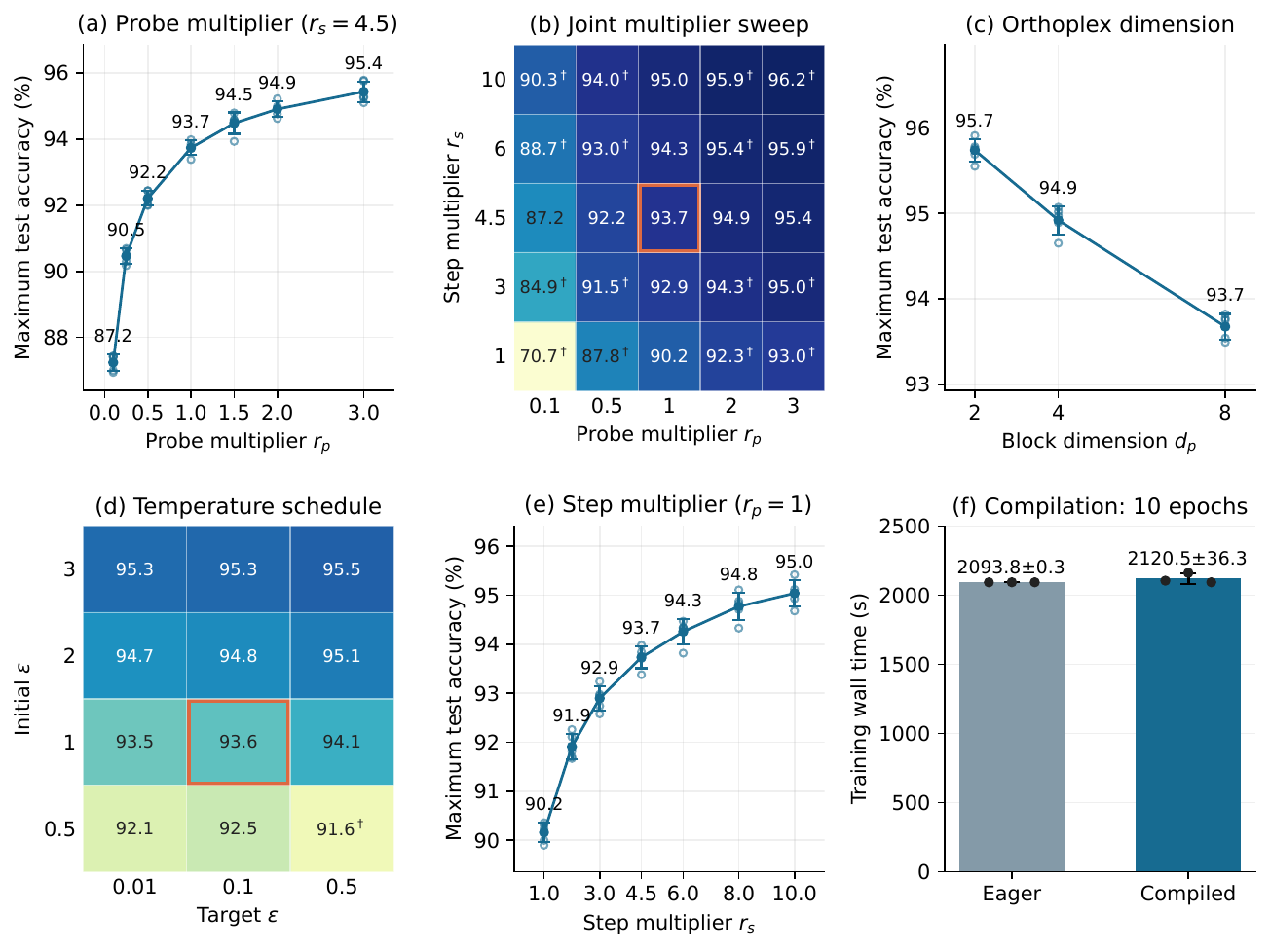}
  \caption{Exploratory MNIST sweeps (HybridSubspace rank $4$, ten epochs).
  Accuracy is the maximum test accuracy over epochs.
  (a,b,e)~Probe and step multipliers. (c)~Orthoplex block dimension at fixed
  updates. (d)~Initial and target temperature. Heatmap labels give mean accuracy (\%);
  $\dagger$ marks seed $42$ alone. Other accuracy settings use five seeds.
  Line plots show sample-standard-deviation error bars and individual seeds as open dots.
  Heatmap uncertainty is retained in the supplementary result files.
  Outlined cells mark the reference settings
  $(r_p,r_s)=(1,4.5)$ and $(\varepsilon_{\rm init},\varepsilon_{\rm target})=(1,0.1)$.
  (f)~Training time with and without compilation (three seeds).
  These sweeps use a separate configuration and selection protocol from the
  primary results.}
  \label{fig:ablation-plate}
\end{figure*}

\paragraph{Simplex and orthoplex evaluation counts.}
The simplex uses $P(d_p+1)K$ evaluations per fresh step; the orthoplex uses
$2Pd_pK$. Their expected linearized step constants agree under Haar rotations.
The ratio of evaluation counts alone establishes neither an accuracy ordering
nor a wall-clock saving. Tuned configurations use the simplex and theory mode
uses the orthoplex.

\paragraph{Blockwise OT.}
A full-space MNIST diagnostic compares one global Sinkhorn solve with
separate per-layer solves.
Over three seeds, per-layer blocking reduces the sample standard deviation
of maximum test accuracy over 30 epochs from $6.22$ to $1.49$~pp, a factor
of $4.2$. It raises mean accuracy over the last three epochs by $2.3$~pp
and takes $50\%$ longer.
This comparison measures a change in column coupling and computation on
one configuration; it does not establish a preferred decomposition for
other architectures.

\subsection{Schedule Fragility}\label{app:schedule-fragility}

We compare flat and cosine temperature schedules on two model sizes,
with three seeds per setting. This is a sensitivity experiment:
Proposition~\ref{prop:fragility} bounds displacement and does not predict
training accuracy.

\begin{table}[ht]
\caption{Flat and cosine temperature schedules on MNIST (3 seeds, mean
$\pm$ sample std). Last-3 accuracy averages the last three epochs; best
accuracy is the maximum test accuracy over 30 epochs. The std ratio uses
last-3 accuracy. Bold marks the higher mean last-3 accuracy for each model.
These two model sizes do not establish dimension independence.}
\label{tab:fragility}
\centering\small
\begin{tabular}{llccc}
\toprule
Model & Schedule & Last-3 acc (\%) & Best acc (\%) & Std ratio \\
\midrule
102K MNISTNet & flat   & $\mathbf{95.87 \pm 0.29}$ & $95.89 \pm 0.29$ & $1.0\times$ \\
              & cosine & $93.63 \pm 0.94$ & $95.87 \pm 0.25$ & $3.2\times$ \\
\midrule
1M MLP        & flat   & $\mathbf{96.92 \pm 0.24}$ & $96.98 \pm 0.21$ & $1.0\times$ \\
              & cosine & $94.46 \pm 1.07$ & $96.98 \pm 0.17$ & $4.5\times$ \\
\bottomrule
\end{tabular}
\end{table}

The cosine schedule lowers mean last-three-epoch accuracy by $2.24$~pp
at $102$K parameters and $2.46$~pp at $1$M parameters. Its standard
deviation is $3.2$ and $4.5$ times larger, respectively. Their mean maximum
test accuracies differ by $0.02$~pp at $102$K parameters and agree to two
decimal places at $1$M. The larger gaps occur late in training.
The displacement envelope alone implies
neither oscillations of a fixed amplitude nor confinement. These results
support checking schedule sensitivity; they do not identify a universal
schedule for discontinuous objectives.

Figure~\ref{fig:ablation-plate} gives the radius, dimension, and temperature
sweeps. Larger configured radii improve maximum test accuracy over the
ranges shown. The orthoplex favors smaller blocks at fixed steps, while
the temperature grid favors larger initial values. These observations
concern finite-budget sensitivity under the maximum-test protocol, not
validation-selected generalization or verification of a convergence rate.

\section{Synthetic Examples and Measured Jump Geometry}\label{app:theory-numerics}

We use synthetic losses to examine how probe geometry, smoothing, and
weighting affect the update. Measurements on neural networks show how
crossing a threshold can change the loss. Appendix~\ref{app:numerics-limits} describes the
sampling limitations of these experiments.

\paragraph{Method.} Monte Carlo exponents are fitted log-log slopes over five values of
$\varepsilon$ spanning a factor of $16$; section fractions are over $2\times10^{4}$
probe points per plane and $20$ independent Gaussian planes; the annulus identity is
evaluated at $3\times10^{6}$ samples with common random numbers across the
finite-difference stencil; schedule-stability exponents are fitted over the last three
quarters of $2{,}500$-step trajectories averaged over $12$ runs.

\subsection{Nonlinear Jump Protocol}\label{app:nonlinear-protocol}

Figure~\ref{fig:nonlinear-jumps} tests the finite-temperature predictions of
Section~\ref{sec:nonlinear-descent} on explicit losses. For planar jumps, we draw
4,096 Haar orthoplexes in each dimension $d_p\in\{2,3,8,16\}$ and reuse each
frame across 17 logarithmically spaced values of $J/\varepsilon$ from
$10^{-2}$ to $10^2$. Every realized drift agrees with
Eq.~\ref{eq:jump-softmax-drift} to within $4.5\times10^{-16}$ in float64.
The drift remains negative after the cost spread exceeds the temperature,
where the conditional saturation assumption no longer covers this example.

For curved jumps, let $\psi(x,y)=x+\kappa y^2/2$,
$g_-(x,y)=0.02y$, and $g_+(x,y)=2+0.03x+0.06y$, assigning the boundary
to $g_+$. Both branches are $G$-Lipschitz with $G=0.08$.
Starting at the origin, we use $\varepsilon=1$, $\ell=0.1$,
$\kappa\in\{0,0.5,2,8,20\}$, $\rho\in\{0.01,0.05,0.1,0.2\}$,
and independent uniform cost errors in $[-b,b]$ for
$b\in\{0,0.01,0.1,0.5\}$. Each of the 80 settings receives 4,096 fresh
rotations and noise draws from NumPy's generator with seed 20260913,
following the planar draws. All 173,274 trials satisfying the stated
sufficient conditions cross to the lower branch and decrease the loss.
Overall, 327,343 of 327,680 trials decrease it. The sufficient conditions
therefore exclude many successful steps; their failure does not predict an
increase. These simulations check the local mechanism and do not establish
the assumptions or performance of a trained network.

\subsection{Blow-Up at an Attained Jump (Theorem~\ref{thm:smoothing-blowup})}

For a degree-zero homogeneous loss and a scaled kernel
$\mu_h(\cdot)=h^{-d}\mu(\cdot/h)$, the identity
$\Ls_h(x)=\Psi(x/h)$ already fixes the scaling exponent. Checking that
exponent alone would provide little information. We therefore also compare
constants, maximizers, and losses without this scaling identity.
Against the closed forms derived in
Appendix~\ref{app:proof-smoothing-blowup} ($\sup_x\|\nabla\Ls_h\| = 8/(3\pi h)$
and $\Lambda_h = 4/(\pi h^2)$ for the uniform ball in $\R^4$ applied to a unit
step, with the Hessian maximizer at $|x_1| = h/\sqrt2$), a grid evaluation with
independent randomness per radius gives a gradient constant of $0.836$--$0.841$ against
$8/(3\pi) = 0.8488$, a smoothness constant of $1.253$--$1.263$ against
$4/\pi = 1.2732$, and a maximizer at $0.70$. The fitted exponents agree to two decimal places with $-1$, $-2$, and $0$
from Eq.~\ref{eq:blowup} and Eq.~\ref{eq:blowup-hess}. On a
non-homogeneous loss, where the scaling is not forced, the exponents are the same; on a
locally Lipschitz loss ($\Ls = |\theta_1|$, $J = 0$) the gradient exponent is $0$ and
$\Lambda_h \sim h^{-1}$, consistent with the different regularity of a continuous loss.

The curved-sheet diagnostics also measure the remainder and crossing geometry. The planar decomposition
$\Ls = J\,\mathbf{1}[y\cdot n > 0] + \Ls_{\mathrm{sm}}$ leaves a remainder that jumps by
the full $J$ at every $h$, on a lens of relative volume $\Theta(\kappa h)$;
this is why Appendix~\ref{app:proof-smoothing-blowup} avoids the decomposition. The
straddle offset is $h/h = 1$ to six decimals throughout the regime
Eq.~\ref{eq:straddle-offset} covers, and on the quadratic model sheet
$g = \kappa s^2/2$ it matches that model's own formula
$h + (\kappa h-1)_+^2/(2\kappa)$ exactly above the model's threshold
($1.0833$ at $\kappa h = 1.5$, $1.6667$ at $3$); the crossing window stays inside
$[-0.9997,\,0.9975]\,h$; and $h\|\nabla\Ls_h\|/J$ on a curved sheet lies
between $0.803$ and $0.858$, above the $1/2 - Gh/J$ that Eq.~\ref{eq:grad-lb} guarantees,
as expected for a lower bound.

\subsection{Sections of \texorpdfstring{$\mathcal{D}$}{D} Under the Deployed Projection (Lemmas~\ref{lem:generic-section} and~\ref{lem:blockwise-section})}

In the projection diagnostic, a global i.i.d.\ dense projection is
transversal to $100\%$ of tested coordinate-aligned walls. The projected
weight blocks reach $62.3\%$ and $31.2\%$, and identity blocks reach
$1.3\%$. This $38$--$99\%$ shortfall in transversality need not prevent
proper sections: a parallel plane off the wall has an empty section.
All $74{,}240$ of $74{,}240$ tested sections along HybridSubspace trajectories
are proper, and no sampled next iterate lies on a wall. These observations
illustrate the distinction in Lemma~\ref{lem:blockwise-section}; they do
not verify its initialization or jitter assumptions for the trained runs.

Coordinate planes at exceptional base points give a different result on
all three geometries: spike thresholds, argmax Voronoi walls and
$\mathrm{floor}$ level sets. Sampling $2\times10^4$ points per plane, axis-aligned
planes place $100.0\%$ of probes on $\mathcal{D}$ while Gaussian-spanned planes through
the same base point place $0.0\%$, below the $10^{-3}$ resolution of the sample.

\subsection{The Annulus Kernel and Its Score (Lemma~\ref{lem:smooth-kernel})}

On a discontinuous test loss in $\R^3$ at $3\times10^6$ samples, the score identity
Eq.~\ref{eq:annulus-stein} matches a common-random-number finite difference to $0.9\%$
relative error for the single annulus and to $0.2$--$0.6\%$ for the mixture of
Eq.~\ref{eq:mixture-radius} at $K \in \{1,3,5\}$.

Applying the uniform-ball estimator to the same probes gives a $41\%$
difference, reflecting the different smoothing targets. At $K = 3$,
$\eta_{\max} = 0.5$, the single-annulus support contains only $48.7\%$ of
the mixture's probes. On their shared support, the radial scores differ
by a median factor of about $5$. The probe-law and tail
smoothings separate the same way next to a discontinuity: at a base point one probe
radius from a unit jump their gradients differ by $31\%$.

The numerical surrogate gradients also agree with the closed-form Bessel
example of Appendix~\ref{app:proof-piecewise-smooth}: $+5.118$ at the
realized radius $\delta$ against $-2.596$ at $2\delta$, the opposite signs that
$J_0(1.5) = +0.512$ and $J_0(3.0) = -0.260$ require.

\subsection{The Step Constant and Its Remainder (Theorem~\ref{thm:exact-step})}

Fitting the residual of the barycentric step against $r_p$, by Monte Carlo at
$d_p \in \{2,4,8\}$, gives an exponent of $3.00$
for the orthoplex and $2.96$ for the simplex. Both give a coefficient ratio
of $0.5000$, equal to $\bar\xi$.

The additional order comes from averaging rotations: on a $C^{2,1}$ loss under Haar rotations the exponent is $2.13$ pointwise in
the rotation and $3.00$ in expectation, it also appears with a jumping Hessian ($2.08$ and
$2.99$) and odd $d_p$, and it disappears when the rotation is held fixed ($2.00$ both
ways). On a loss that \emph{jumps} it fails, $0.29$ pointwise and $0.15$ in
expectation: a straddling probe makes the cost-row spread $\Theta(1)$ rather than
$\Theta(r_p)$, which is why the smooth-region reduction of the first floor term requires
the full probe neighborhood to lie in one smooth region. Saturated draws
are controlled separately by the terms involving $\bar\pi$ and $\bar\sigma$.
This separates the two halves of the parity argument in
Appendix~\ref{app:proof-piecewise-smooth}: the second-order term
$\sum_v \langle g, Rv_v\rangle^2 Rv_v$ vanishes \emph{pointwise} in $R$ on the orthoplex
($1.1\times10^{-16}$ at the worst rotation) and only in expectation on the simplex
($4.4\times10^{-1}$ at the worst rotation, $1.6\times10^{-4}$ on average). Both reach
the same $O(r_p^3)$ order and the same constant $c = r_p/(2d_p)$; their variances can differ despite this common remainder order.

\subsection{Bias Floor, Occupancy and the Optimal Radius (Theorem~\ref{thm:piecewise-smooth}, Corollary~\ref{cor:optimal-rp})}

The diagnostics distinguish two smoothness constants. Over a sweep in $\delta$ the supremum of the
smooth part's second derivative where the kernel clears $\mathcal{D}$ is
$\delta$-independent, while $\Lambda_\delta$ grows from $9$ to $2{,}002$; this is why the
coupling constant in condition~(vii-L) is defined through the partially
smoothed field, separately from the joint smoothing's curvature.

The saturation term of Eq.~\ref{eq:bias-floor} is linear in $\varepsilon$ and carries
$r_p^{-1}$ under the geometric model, the straddle frequency scaling as
$\varepsilon^{1.00}$ and the floor term following at $\varepsilon^{1.00}$, with the ratio
to $d_p\bar\xi r_s\varepsilon/L_{\mathcal{D}}$ constant to $1.05\times$ across a
$4\times$ range of $r_p$. The interior optimum $r_p^\star = \Theta(\varepsilon^{1/5})$ of
Eq.~\ref{eq:optimal-rp} is recovered in closed form from that model. Measured on the runs
rather than modeled, the saturation frequency is $0.00$
(Appendix~\ref{sec:hypothesis-coverage}), so the model is an upper bound that is not tight
at our operating point. A separate synthetic loss activates the saturation term using walls at
spacing $L$ with jump $10\varepsilon$ and base points uniform over one
wall period, as in the model for the corollary's $\bar\sigma$ hypothesis.
Stratified sampling of the wall band ensures that rare saturated events
are represented. In this experiment, the
measured saturated fraction is positive and increasing in $r_p$, and the measured floor
grows as $r_p^{-0.92}$ as the probe radius falls (mechanism exponent $-1$: saturated
mass over the squared descent constant), consistent with the corollary's middle term.
Beyond the saturation spike the realized floor is monotone decreasing, so the interior
optimum remains a property of the bound $\hat B$ rather than of the floor, as stated in the corollary.

The occupancy diagnostic follows the optimizer recursion on a synthetic loss. On an adapted trajectory over a staircase, where the smooth
tilt makes every visit a crossing and prevents re-crossing, the path-weighted occupancy
is within $12\%$ of $\delta/L_{\mathcal{D}}$ across an $8\times$ range of $\delta$, with
a fitted exponent of $1.05$ against the predicted $1$; $\delta/L_{\mathcal{D}}$ is the
lemma's $(2\delta+\rho_T)/L_{\mathrm{exc}}$ at leading order, since the simulated path between distinct sheets has
$L_{\mathrm{exc}}\approx2L_{\mathcal D}-2\delta$ and the discrete step
correction $\rho_T$ is negligible at these settings. The lemma still
requires its explicit discrete visit-length and excursion assumptions. Dropping either condition
illustrates the oscillating sequence in Appendix~\ref{app:theory-remarks}:
across a single sheet at amplitude $\delta + a$, occupancy rises to $0.50$, $0.91$,
$0.99$ as $a/\delta$ falls through $1$, $0.1$, $0.01$, while the sheet-separation bound
stays flat, showing why sheet separation cannot replace excursion length.

The straddle statistic also depends on the conditioning convention. For
an $\mathcal{F}_t$-measurable indicator, the ratio of its weighted square to
its weighted mean is $1.0000$. For the fractional straddle frequency in a
uniform-base-point model, the measured ratio is $0.92$, and the gap
widens as $\varepsilon$ falls. The saturation statistic of condition~(viii) is a
conditional probability rather than an indicator, so the first reading is unavailable to
it and Eq.~\ref{eq:bias-floor} carries the linear form as an upper bound only
(Appendix~\ref{app:theory-remarks}).

\subsection{Goldstein Inclusion and Its Radius (Corollary~\ref{cor:goldstein})}

On the realized kernel with $J = 0$, Eq.~\ref{eq:goldstein-inclusion} holds with margin
at least $0.07$ and fails at a smaller radius in the tested example. With
$J > 0$, $\|\nabla\Ls_\delta\|\,\delta_{\mathrm{out}}$ is constant to $1.3\%$ over a
$4\times$ range of radii, which is Theorem~\ref{thm:smoothing-blowup}'s $\Theta(1/\delta)$.

The joint-radius diagnostics examine both the support and its probability mass. The support of
the joint displacement never exceeds $\sqrt{P}\,\delta_{\mathrm{out}}$ for $P \leq 64$,
in the diagnostic; the near-corner region has positive probability, computed by quadrature rather than
sampled, at $\log_{10}\Pr[\|u\| > 0.95\sqrt{P}\delta_{\mathrm{out}}] = -12.9$ for
$P = 4$; and an admissible measure makes the inclusion fail at every radius below
$\sqrt{P}\delta_{\mathrm{out}}$. The realized law is more concentrated than the support bound suggests: its $99.9$th percentile sits at $0.568$ of the joint radius at
$P = 64$.

\subsection{Radial Drift, Stationary Excursion and Confinement (Propositions~\ref{prop:stable-schedules} and~\ref{prop:confinement-ville})}

A random-walk diagnostic illustrates why the envelope does not imply
confinement. Keeping
Eq.~\ref{eq:fragility-bound} while replacing the update by an independent
isotropic direction at each step gives a fitted excursion exponent of $0.507$ against the free-walk value $1/2$. With the drift restored
the exponent is $0.00$ under constant $r_s$ and $-0.61$ under the
decaying schedule of Theorem~\ref{thm:piecewise-smooth}.

The diagnostic uses a quadratic oracle and an independent softmax temperature
$10^{-5}$ while varying the geometric probe and step scales. It therefore
illustrates a concentrated regime, rather than testing the paper's coupled
temperature schedule directly. At $d_p=4$ and $H=2I$, the measured normalized
radial drift is $-0.777$ and the late-run root-mean-square excursion is
$0.650$ times the step scale, varying by $0.8\%$ over the tested range.

With eigenvalues evenly spaced, the normalized radial drift is $-0.779$ at
$H=I$, $-0.350$ at condition number $10$, $-0.272$ at $100$, and $-0.278$
at $10^3$. Different spectra at the same condition number give different
drifts. These finite-grid measurements do not prove a uniform Euclidean
drift bound or an exact stationary-amplitude law. In the small-probe regime,
the measured mean direction agrees with $-Hz$ to six decimals; the analytic
identity retains a remainder.

At $H=\mathrm{diag}(100,1)$ and $d_p=2$, $49\%$ of $2\times10^4$ sampled
steps have an outward Euclidean radial component, while the mean normalized
component is $-0.274$. This distinguishes a realized radial component from
the expected $H$-metric drift analyzed above.

The confinement bound of Proposition~\ref{prop:confinement-ville} is evaluated on the
same model: over $400$ seeds the realized tail of $\sup_t V_H$ is dominated by the
supermartingale bound at every grid point, for isotropic and condition-$9$ anisotropic
$H$, and replacing the inward drift by an outward one of the same magnitude violates
the bound, showing the diagnostic's sensitivity to the drift direction.

\subsection{The KL-Penalized Program (Theorem~\ref{thm:kl-softmax-endpoints})}

At $\lambda = 0$, the computed plan matches
$\mathrm{diag}(\va)\,\mathrm{softmax}(-\mC/\varepsilon)$ to $1.4\times10^{-17}$, and the balanced-endpoint column-marginal error is $2.8\times10^{-17}$. The scaling exponent
is $\alpha = \lambda/(\lambda+\varepsilon)$ exactly, and the column-marginal violation
is non-increasing at every step of a ten-point sweep, falling by a factor of
$1.7\times10^{5}$ with a log-log slope of $-1.36$ against the $O(1/\lambda)$ the proof
gives.

Theorem~\ref{prop:plateau-freeze} and Proposition~\ref{prop:ghost-drift} are checked
using the released solver. At $\lambda = 0$ the step on
a constant row is exactly zero for every rotation. At $\lambda > 0$ the closed form
$T_{iv} = a_i\,\mathrm{softmax}(g_v/\varepsilon)$ holds to $1.7\times10^{-16}$, the
realized step grows monotonically in $\lambda$ from $0.10$ to $0.46$, and the
Haar-averaged step decays as $n^{-1/2}$, reaching $9.2\times10^{-4}$ at $n = 10^5$:
the step can be nonzero while its expectation remains zero, as the proposition states.

\subsection{Measured Jump Magnitudes (Definition~\ref{def:pwc-loss})}\label{app:measured-jump}

Appendix~\ref{app:discontinuity-sets} derives the geometry of $\mathcal{D}$ for each
architecture but not the magnitude $J$, which Definition~\ref{def:pwc-loss} needs
positive on a relatively open piece.
Table~\ref{tab:measured-jump} reports the loss changes observed at threshold crossings. Throughout, $n_B$ is
the minibatch size and losses are averaged over the batch.

\begin{table}[ht]
\centering
\caption{Measured jump magnitudes. A \emph{wall} is a solved crossing of one threshold,
confirmed by forward hooks to change the intended threshold state. An SNN crossing may change later spikes through reset dynamics. ``zero frac.'' is the
fraction of walls across which the batch loss is \emph{bit-identical}; ``$J$'' is the
median over walls with nonzero loss change; ``example-local'' is the largest $|\Delta\mathrm{CE}|$
over examples whose decisions did \emph{not} change, used to check whether the batch change satisfies $J = j/n_B$. These models are trained through smooth relaxations and then evaluated with
hard operations; they differ from the reported \sysname{} checkpoints, and magnitudes are
checkpoint-dependent. The SNN row retains 256 resolved crossings and discards 79 that remain unresolved at relative half-width $10^{-4}$.}
\label{tab:measured-jump}
\small
\begin{tabular}{@{}lcccccc@{}}
\toprule
Architecture & walls & zero frac. & median $J$ & $\max J$ & example-local & $J = j/n_B$ \\
\midrule
(a) Hard-LIF SNN     & 256 & $0.69$   & $4.1{\times}10^{-7}$ & $1.5{\times}10^{-4}$ & $0$ & yes \\
(b) INT8             & 512 & $0.11$   & $1.6{\times}10^{-5}$ & $1.4{\times}10^{-4}$ & n/a & \textbf{no} \\
(c) Hard-MoE argmax  & 256 & $0.00$   & $1.2{\times}10^{-3}$ & $1.9{\times}10^{-2}$ & $0$ & yes \\
(d) Staircase        & 256 & $0.00$   & $2.9{\times}10^{-6}$ & $1.2{\times}10^{-4}$ & $0$ & yes \\
(e) Argmax attention & 256 & $0.00$   & $1.5{\times}10^{-5}$ & $1.7{\times}10^{-3}$ & $2.4{\times}10^{-13}$ & approx. \\
\bottomrule
\end{tabular}
\end{table}

\paragraph{The SNN.}
Across $69\%$ of the 256 resolved crossings, the measured loss is unchanged.
A reset can suppress a later spike, preserving the final count although
spike times change. In this sample, all 177 zero-loss changes coincide with
unchanged summed counts. Only 39 crossings change exactly one spike
indicator; the remaining crossings change two or more. These finite-width
measurements give nonzero examples but do not verify a uniform trace gap
or an isolated jump sheet at every sampled point.

\paragraph{INT8.} A crossing changes one quantized weight by one level.
Of 512 sampled crossings, 54 leave the loss unchanged. Units that are
inactive on the whole batch provide one source of such zero changes.
The wall $\{\theta_j=(m+1/2)\Delta_q\}$ depends on a weight, so the crossing is shared across examples. One crossing changes the losses of 232 of the
512 examples on average, with changes of both signs. On nested sub-batches,
the median paired ratio $|J(512)|/|J(8)|$ is 0.202 over the 438 walls
with nonzero denominator. The corresponding $1/n_B$ ratio is 0.0156.
This comparison rejects that particular single-example scaling explanation
for these measurements; it does not establish a universal batch-size law.

\paragraph{The remaining three architectures.} Every one of the $256$ resolved crossings per architecture changes the
loss for hard-MoE routing, staircase, and argmax attention. Hard MoE has the largest median jump, $J=1.2\times10^{-3}$, compared
with $1.6\times10^{-5}$ for INT8, $1.5\times10^{-5}$ for argmax attention,
$2.9\times10^{-6}$ for staircase, and $4.1\times10^{-7}$ for SNN. Argmax
attention is the one case where $J = j/n_B$ is not exact, since the routed key enters
the output through a continuous path as well, so the losses of examples that did not reroute still
change, by $2.4\times10^{-13}$, almost eight orders below the row's median jump.

\subsection{Limits of the Jump Measurements}\label{app:numerics-limits}

We measure straddling frequency at base points sampled uniformly in a
window, but measure occupancy and saturation along optimization trajectories.
These procedures sample different distributions. Substituting parameters
into the bias bound checks how its terms scale; it does not measure the
actual optimization error. We measure saturation over $150$ steps at
seed $42$ on four architectures. Observing no saturated rows in this sample
does not bound their conditional probability at every possible iterate.

A limited statement follows from the tower property. If each of $N$ rows
has saturation probability at least $p$ conditional on every preceding
observation, including earlier rows of the same step, then observing no
saturation has probability at most $(1-p)^N$. For $N=1{,}048{,}500$, the
observations are inconsistent at the $99\%$ level with a uniform conditional
lower bound above $4.4\times10^{-6}$. This does not upper-bound the probability
at a particular state or certify Assumption~\ref{asm:main}(viii).

The drift and excursion diagnostics use quadratic oracles and, in the
concentrated-regime tests, a softmax temperature independent of geometric
scale. They do not verify the drift premise or compactness for trained
networks. Histogram estimates of smoothed-gradient peaks also have binning
and sampling error. The proofs establish the identities; the diagnostics illustrate them under
finite sampling.

\section{Extended Experiment Results}\label{app:extended-results}

\subsection{Controlled Weighting Protocol}\label{app:controlled-protocol}

\paragraph{Data and randomization.}
We used the architecture, spike-count loss, and MNIST normalization from
Section~\ref{sec:setup}. A CPU PyTorch generator with seed 0 partitioned the
60K official training images into a 54K pool and 6K validation images.
A second permutation from that generator reserved a separate 1,024-image
holdout, leaving 52,976 update examples. The holdout was excluded from
updates and configuration selection. The official
10K test images were used only after validation selection. The split was
fixed across all seeds and methods. Minibatches contained 128 examples,
shuffled without replacement each epoch, with the final short batch retained.
Tuning used seeds 2027--2029; final runs used seeds 1000--1009. Initializations,
projection draws, minibatch order, and probe draws were paired across rules
within each seed, using separate random streams for data, projections, and probes.

\paragraph{Candidates and updates.}
We used a fixed per-layer rank-4 HybridSubspace with dense QR bases,
4,336 active coordinates, and two frozen coordinates in the incomplete
final block. Each of the 542 complete blocks had $d_p=8$ coordinates and
$V=16$ unit directions from a fresh Haar-rotated orthoplex, with $K=1$.
All rules used the same directions and minibatch at a given update index.
For costs $C_v$, directions $u_v$, and positive scale $\tau$, the rules were
\[
D_{\mathrm{softmax}}=\sum_v\mathrm{softmax}(-C/\tau)_v u_v,
\qquad
D_{\mathrm{linear}}=-\frac{1}{V\tau}\sum_v(C_v-\bar C)u_v,
\qquad
D_{\mathrm{rank}}=\sum_v\mathrm{softmax}(-r/\tau)_v u_v,
\]
where $r_v\in[0,1]$ is the average tied rank divided by $V-1$.
Greedy weighting averaged all minimum-cost vertices. Every rule returned
zero for constant costs. Linear weighting used signed coefficients rather
than a convex barycenter. Each block update was $\ell D$ with probes at
radius $\rho$. For softmax, these scales correspond to
$\varepsilon=\tau$, $r_p=2\rho/\tau$, and $r_s=\ell/\tau$.
Plan reuse, momentum, radius adaptation, projection refresh, and jitter were disabled.

\paragraph{Tuning, budgets, and selection.}
Softmax and linear weighting each used 27 configurations combining
$\rho\in\{10^{-3},10^{-2},10^{-1}\}$,
$\ell/\rho\in\{0.1,1,10\}$, and
$\tau\in\{10^{-3},10^{-2},10^{-1}\}$.
Rank weighting used the same radii and step ratios with
$\tau\in\{0.05,0.2,1\}$. Greedy used the same three radii and nine
step ratios $10^{k/4}$, $k=-4,\ldots,4$.
Each configuration received $2\times10^6$ candidate evaluations per tuning
seed; each final run received $10^7$. One update required 8,672 evaluations,
so tuning and final runs completed 230 and 1,153 updates, respectively,
leaving 5,440 and 1,184 evaluations unused. An update was admitted only if
its full candidate set fit the remaining budget.
We therefore used the same tuning and final evaluation allowances for all four rules.

Validation used all 6K images at initialization, after an update crossed
each additional $10^5$ candidate evaluations, and at the final update.
The selected checkpoint maximized validation accuracy, breaking ties by
lower validation loss and then earlier checkpoint. Configuration selection
maximized mean best validation accuracy across the three tuning seeds,
with ties broken by mean validation loss and then the configuration identifier.
For each rule and final seed, we used that rule's selected configuration
and evaluated its selected checkpoint on test once. Validation selected
$(\rho,\ell,\tau)=(0.1,0.01,0.01)$ for softmax and linear weighting,
$(0.1,0.1,1)$ for rank weighting, and $(\rho,\ell)=(0.1,0.0177827941)$
for greedy. Final-seed test scores did not affect configuration selection.

\paragraph{Execution and uncertainty.}
Each run used a fresh process on the RTX 5090 with PyTorch 2.11.0 and CUDA 12.8,
float32, disabled TF32,
and the compiled forward evaluator with candidate chunks of 512.
No completed final run fell back to eager evaluation. Strict deterministic
algorithms were not enabled. Records include source hashes, software and
device information, split indices, configuration, seed, and validation
trajectories against updates, candidate evaluations, processed examples,
and elapsed time. These records do not establish a time-matched advantage.

We retained all ten final seeds. The numerical-failure rule specified the
last finite eligible validation checkpoint, or initialization if none existed;
no final run required this rule. Paired percentile-bootstrap
intervals used 10,000 resamples with NumPy seed 0. Tests enumerated all
$2^{10}$ sign assignments and assumed symmetry of paired differences under
the null. These three comparisons belong to a broader prespecified family
of 16 contrasts across five weighting rules, two candidate laws, and two
step normalizations. We report unadjusted tests and pointwise intervals
for the observed comparisons; no family-wide significance claim is made.

\subsection{The Four Remaining Hard Architectures}\label{app:gallery-tables}

Tables~\ref{tab:int8-accuracy}, \ref{tab:staircase-accuracy},
\ref{tab:argmax-accuracy} and~\ref{tab:moe-accuracy} report the gallery
architectures under the protocol of Section~\ref{sec:setup}, in the same format as
Table~\ref{tab:snn-accuracy}: both cost axes, the axis each row was budgeted on, and
test accuracy at the validation-selected checkpoint over 5 seeds. \sysname{} leads
every gradient-free baseline on every seed of all four, so the exact two-sided
sign-flip test returns $p = 0.0625$ on each (Appendix~\ref{app:stats}).
All models use MNIST except argmax attention (Fashion-MNIST) and hard MoE
(MNIST and Fashion-MNIST, 20 classes).

\begin{table}[ht]
\caption{INT8 weight quantization on MNIST: test accuracy at the validation-selected checkpoint (mean $\pm$ sample std, 5 seeds). Best gradient-free result in bold.}
\label{tab:int8-accuracy}
\centering
\small
\begin{threeparttable}
\begin{tabular}{lrrS[table-format=2.1(2)]}
\toprule
{Method} & {Steps} & {Evaluations} & {Test accuracy (\%)} \\
\midrule
\sysname{} & 3{,}180 & 38.35M & \bfseries 97.0 \pm 0.1 \\
OpenAI-ES & 1{,}198{,}462 & 38.35M$^{e}$ & 94.4 \pm 0.2 \\
CMA-ES & 299{,}615 & 38.35M$^{e}$ & 91.6 \pm 0.4 \\
EGGROLL & 599{,}231 & 38.35M$^{e}$ & 83.4 \pm 1.2 \\
MeZO & 3{,}180 & 6.36K$^{s}$ & 80.7 \pm 0.9 \\
SPSA & 3{,}180 & 6.36K$^{s}$ & 82.5 \pm 0.5 \\
Random search & 3{,}179 & 3.18K$^{s}$ & 29.2 \pm 4.9 \\
Adam$^\dagger$ & 3{,}180 & 3.18K & 97.8 \pm 0.1 \\
\bottomrule
\end{tabular}
\begin{tablenotes}[para,flushleft]
\scriptsize\raggedright
$^{e}$Evaluation-matched. $^{s}$Step-based allowance: $2N$ evaluations for MeZO/SPSA and $N$ for random search, with $N=3{,}180$; random search counts its initial candidate. $^\dagger$Adam trains and tests a relaxed forward model; its count is forward--backward passes.
\end{tablenotes}
\end{threeparttable}
\end{table}

\begin{table}[ht]
\caption{Staircase activations on MNIST: test accuracy at the validation-selected checkpoint (mean $\pm$ sample std, 5 seeds). Best gradient-free result in bold.}
\label{tab:staircase-accuracy}
\centering
\small
\begin{threeparttable}
\begin{tabular}{lrrS[table-format=2.1(2)]}
\toprule
{Method} & {Steps} & {Evaluations} & {Test accuracy (\%)} \\
\midrule
\sysname{} & 3{,}180 & 19.66M & \bfseries 94.3 \pm 0.1 \\
OpenAI-ES & 614{,}435 & 19.66M$^{e}$ & 88.9 \pm 0.5 \\
CMA-ES & 86{,}981 & 19.66M$^{e}$ & 89.2 \pm 0.4 \\
EGGROLL & 427{,}433 & 19.66M$^{e}$ & 52.0 \pm 1.5 \\
MeZO & 3{,}180 & 6.36K$^{s}$ & 61.1 \pm 2.8 \\
SPSA & 3{,}180 & 6.36K$^{s}$ & 45.2 \pm 4.2 \\
Random search & 3{,}179 & 3.18K$^{s}$ & 18.1 \pm 4.8 \\
Adam$^\dagger$ & 3{,}180 & 3.18K & 97.5 \pm 0.1 \\
\bottomrule
\end{tabular}
\begin{tablenotes}[para,flushleft]
\scriptsize\raggedright
$^{e}$Evaluation-matched. $^{s}$Step-based allowance: $2N$ evaluations for MeZO/SPSA and $N$ for random search, with $N=3{,}180$; random search counts its initial candidate. $^\dagger$Adam trains and tests a relaxed forward model; its count is forward--backward passes.
\end{tablenotes}
\end{threeparttable}
\end{table}

\begin{table}[ht]
\caption{Argmax attention over eight slots on Fashion-MNIST: test accuracy at the validation-selected checkpoint (mean $\pm$ sample std, 5 seeds). Best gradient-free result in bold. Population baselines use step-based budgets; evaluation matching reverses the OpenAI-ES/PolyStep ordering (Section~\ref{sec:other-axes}).}
\label{tab:argmax-accuracy}
\centering
\small
\begin{threeparttable}
\begin{tabular}{lrrS[table-format=2.1(2)]}
\toprule
{Method} & {Steps} & {Evaluations} & {Test accuracy (\%)} \\
\midrule
\sysname{} & 3{,}180 & 42.01M & \bfseries 86.6 \pm 0.4 \\
OpenAI-ES & 3{,}180 & 101.76K$^{s}$ & 80.0 \pm 0.4 \\
CMA-ES & 3{,}180 & 101.76K$^{s}$ & 79.5 \pm 0.5 \\
EGGROLL & 6{,}360 & 101.76K$^{s}$ & 71.0 \pm 0.4 \\
MeZO & 3{,}180 & 6.36K$^{s}$ & 71.6 \pm 1.2 \\
SPSA & 3{,}180 & 6.36K$^{s}$ & 67.4 \pm 1.3 \\
Random search & 3{,}179 & 3.18K$^{s}$ & 42.8 \pm 1.9 \\
Adam$^\dagger$ & 3{,}180 & 3.18K & 88.6 \pm 0.1 \\
\bottomrule
\end{tabular}
\begin{tablenotes}[para,flushleft]
\scriptsize\raggedright
$^{s}$Step-based allowance, with $N=3{,}180$: population methods receive $32N$ evaluations at their selected population sizes, MeZO/SPSA receive $2N$, and random search receives $N$, including initialization. Generation counts therefore vary. $^\dagger$Adam trains and tests a relaxed forward model; its count is forward--backward passes.
\end{tablenotes}
\end{threeparttable}
\end{table}

\begin{table}[ht]
\caption{Hard MoE on MNIST and Fashion-MNIST (20 classes): test accuracy at the validation-selected checkpoint (mean $\pm$ sample std, 5 seeds). Best gradient-free result in bold. Population baselines use step-based budgets; the evaluation-matched OpenAI-ES/PolyStep difference is $0.03$ percentage points (Section~\ref{sec:other-axes}).}
\label{tab:moe-accuracy}
\centering
\small
\begin{threeparttable}
\begin{tabular}{lrrS[table-format=2.1(2)]}
\toprule
{Method} & {Steps} & {Evaluations} & {Test accuracy (\%)} \\
\midrule
\sysname{} & 6{,}330 & 100.55M & \bfseries 90.6 \pm 0.2 \\
OpenAI-ES & 6{,}330 & 202.56K$^{s}$ & 82.5 \pm 0.8 \\
CMA-ES & 9{,}207 & 202.55K$^{s}$ & 77.8 \pm 1.5 \\
EGGROLL & 1{,}582 & 202.50K$^{s}$ & 52.5 \pm 1.5 \\
MeZO & 6{,}330 & 12.66K$^{s}$ & 55.3 \pm 3.0 \\
SPSA & 6{,}330 & 12.66K$^{s}$ & 25.4 \pm 3.5 \\
Random search & 6{,}329 & 6.33K$^{s}$ & 13.7 \pm 1.9 \\
\bottomrule
\end{tabular}
\begin{tablenotes}[para,flushleft]
\scriptsize\raggedright
$^{s}$Step-based allowance, with $N=6{,}330$: population methods receive $32N$ evaluations at their selected population sizes, MeZO/SPSA receive $2N$, and random search receives $N$, including initialization. Generation counts therefore vary; a final incomplete generation is omitted.
\end{tablenotes}
\end{threeparttable}
\end{table}

\subsection{Population Size, Throughput, and the Other Two Axes}\label{app:wallclock-matched}

Table~\ref{tab:step-sweep} compares validation accuracy under step-based
evaluation allowances. The separate throughput study varies
population size to measure how batching changes the cost of an ES generation.

Result files record optimizer steps, evaluations, budget axis, and subspace
dimension. EGGROLL uses
a factored subspace preserving matrix shapes, whereas the other
methods use HybridSubspace. This representation difference limits attribution
to the update rule alone.

\paragraph{Generation cost.}
On hard MoE, OpenAI-ES takes $68.4$ ms per generation at population $32$
and $48.3$ ms at population $2{,}048$. Throughput rises from $468$ to
$42{,}364$ candidate evaluations per second. These measurements show that
small-population overhead can dominate evaluation cost. They were collected
with three concurrent runs on one GPU, so isolated-device timings remain
unmeasured. On argmax, population tuning reduces the recorded time to reach
about $84.2\%$ from $3{,}688$ to $54$ seconds under the tested settings.
A full wall-clock-matched comparison with separately optimized batching is
not available.

\paragraph{Accuracy of a tuned population.}
For argmax and hard MoE, we swept the learning rate at population
$2{,}048$ using seed $42$, initially at $3.0$M evaluations and then,
on MoE, at PolyStep's $100.55$M-evaluation budget.

\begin{center}\small
\begin{tabular}{@{}llrrrr@{}}
\toprule
Benchmark & OpenAI-ES configuration & Steps & Evaluations & Time (s) & Test (\%) \\
\midrule
Argmax   & population $32$, matched steps        & $3{,}180$   & $101.8$K & $98$~s    & $80.0$ \\
         & population $32$, matched wall-clock   & $122{,}750$ & $3.93$M  & $3{,}688$~s & $84.2$ \\
         & population $2{,}048$, $\mathrm{lr}=0.03$ & $1{,}464$ & $3.00$M  & $54$~s    & $\mathbf{84.9}$ \\
         & \sysname{}                            & $3{,}180$   & $42.01$M & $3{,}691$~s & $86.6$ \\
\midrule
Hard MoE & population $32$, matched steps        & $6{,}330$   & $202.6$K & $433$~s   & $82.5$ \\
         & population $2{,}048$, $\mathrm{lr}=0.03$ & $1{,}464$ & $3.00$M  & $94$~s    & $87.0$ \\
         & population $2{,}048$, matched evaluations & $49{,}097$ & $100.55$M & $2{,}375$~s & $\mathbf{90.5}$ \\
         & \sysname{}                            & $6{,}330$   & $100.55$M & $12{,}541$~s & $90.6$ \\
\bottomrule
\multicolumn{6}{l}{\scriptsize Bold marks the best OpenAI-ES configuration per block; the \sysname{} row is the reference.}\\
\end{tabular}
\end{center}

On argmax, at exactly \sysname{}'s $42.01$M evaluations, OpenAI-ES reaches
$86.95 \pm 0.1\%$ over seeds $42/123/456$ against \sysname{}'s $86.41\%$ on the same
three (Section~\ref{sec:other-axes}).

We also varied PolyStep's subspace rank at the same evaluation budget.
A smaller rank reduces the candidates per update and permits more updates.
The following seed-$42$ comparison selects the rank on validation:

\begin{center}\small
\begin{tabular}{@{}rrrcc@{}}
\toprule
Rank & Steps & Evaluations & Val. acc. (\%) & Test acc. (\%) \\
\midrule
$8$ (selected) & $3{,}180$  & $42.01$M & $\mathbf{87.38}$ & $\mathbf{86.65}$ \\
$4$           & $6{,}148$  & $41.78$M & $86.50$ & $85.63$ \\
$2$           & $11{,}872$ & $42.10$M & $84.68$ & $84.24$ \\
\bottomrule
\end{tabular}
\end{center}

Accuracy increases with rank across these three settings, with rank $8$
selected by validation; Section~\ref{sec:other-axes} interprets the sweep. On hard
MoE the matched-evaluation result is $90.52\pm0.11\%$ against \sysname{}'s
$90.49\pm0.13\%$ on the same three seeds, 42, 123, and 456
(Section~\ref{sec:other-axes}). The seed-$42$ MoE learning-rate sweep at
population $2{,}048$ reaches $82.5$, $87.0$, $86.4$, and $79.1\%$ at
$\mathrm{lr}=0.01,0.03,0.1,0.3$, respectively; the step-axis learning
rate $0.003$ gives $60.2\%$. Thus the selected learning rate at one
population need not remain effective at another.

\paragraph{The other four benchmarks.}
On SNN, MNIST, INT8, and staircase, population baselines use the matched
evaluation budget and take more updates than PolyStep. Table~\ref{tab:axes}
reports their highest logged validation accuracy within $32N$ evaluations
and over the full run, where $N=3{,}180$ is PolyStep's update count.
PolyStep's entries also use validation
accuracy. Using the full run recovers each saved checkpoint selection.
Test accuracy is $0.0$--$0.9$~pp above full-budget validation across these
population methods and PolyStep. Truncation compares progress under the
selected configurations, whose selection used the full evaluation budget.

\begin{table}[ht]
\centering
\caption{Validation accuracy (\%, five-seed means) from the same runs under two budgets; highest mean per benchmark and budget in bold. Step-based: population baselines are read within $32\times3{,}180=101{,}760$ evaluations. Evaluation-matched: each baseline receives PolyStep's full evaluation budget. Population sizes vary, so the step-based budget need not match update counts. Per-benchmark tables give full-budget test accuracy and standard deviations. PolyStep supplies the reference for both columns, so its entries coincide.}
\label{tab:axes}
\small
\begin{tabular}{llcc}
\toprule
Benchmark & Method & Step-based budget & Evaluation-matched \\
\midrule
SNN (hard LIF) & \sysname{} & $\mathbf{92.9}$ & $\mathbf{92.9}$ \\
 & OpenAI-ES & $77.6$ & $78.9$ \\
 & CMA-ES & $72.8$ & $76.3$ \\
 & EGGROLL & $37.8$ & $55.2$ \\
\midrule
MNIST & \sysname{} & $\mathbf{96.7}$ & $\mathbf{96.7}$ \\
 & OpenAI-ES & $86.2$ & $87.8$ \\
 & CMA-ES & $89.7$ & $94.2$ \\
 & EGGROLL & $73.2$ & $78.0$ \\
\midrule
INT8 & \sysname{} & $\mathbf{97.0}$ & $\mathbf{97.0}$ \\
 & OpenAI-ES & $87.6$ & $94.3$ \\
 & CMA-ES & $88.5$ & $90.9$ \\
 & EGGROLL & $72.5$ & $83.1$ \\
\midrule
Staircase & \sysname{} & $\mathbf{94.3}$ & $\mathbf{94.3}$ \\
 & OpenAI-ES & $86.2$ & $88.8$ \\
 & CMA-ES & $74.5$ & $88.7$ \\
 & EGGROLL & $41.6$ & $51.7$ \\
\bottomrule
\end{tabular}
\end{table}

On those four benchmarks, the baselines peak at $0.2$--$25\%$ of their
budgets and then decline or flatten (Section~\ref{sec:other-axes}). A larger
population or a different schedule could change these trajectories. The narrowest margin in the
table, MNIST against CMA-ES at $96.7$ to $94.2$, is a comparison at matched evaluations, independent of runtime.

The step-matched argmax OpenAI-ES run reaches its highest logged validation
accuracy at the final point. With population 2,048 and PolyStep's full
evaluation budget, its test accuracy rises from $80.0\%$ to $86.95\%$;
the corresponding hard-MoE means rise from $82.5\%$ to $90.52\%$.
The latter comparisons use three seeds instead of five and separately
tuned learning rates (Section~\ref{sec:other-axes}).

\subsection{The Transport Constraint in the High-Particle Regime}\label{app:ot-highparticle}

The subspace MNIST comparison gives both solvers $96.4\pm0.2\%$
maximum test accuracy. The following full-space studies compare their
accuracy when the same loss probes are coupled through a column constraint.

On full-space MNIST with $P = 50{,}885$ particles and $V = 4$ vertices,
entropic OT leads softmax by $23$~pp (Table~\ref{tab:ot-fullspace}).
These runs report maximum test accuracy over epochs. The five-seed means
are $57.6\pm11.0\%$ and $34.6\pm13.9\%$, with entropic OT ahead on every
paired seed. Both are below the $96.4\%$ subspace means.

\begin{table}[ht]
\centering
\caption{Entropic OT and softmax on MNIST: maximum test accuracy over epochs (\%, mean $\pm$ sample std, 5 seeds). Full-space runs use $P{=}50{,}885$ particles, $V{=}4$ vertices, and 3 epochs. The gap is OT minus softmax in percentage points. Table~\ref{tab:ot-maxsat} gives the corresponding MAX-SAT comparison.}
\label{tab:ot-fullspace}
\begin{tabular}{lcc}
\toprule
Update rule & Subspace & Full-space ($P{=}50$K) \\
\midrule
Entropic OT & $96.4 \pm 0.2$ & $57.6 \pm 11.0$ \\
Softmax-weighted & $96.4 \pm 0.2$ & $34.6 \pm 13.9$ \\
\midrule
Gap & $0.0$ & $+23.0$ \\
\bottomrule
\end{tabular}
\end{table}

On full-space MAX-SAT, the Sinkhorn-minus-softmax differences range from
$-0.084$ to $+0.028$~pp across four instance sizes
(Table~\ref{tab:ot-maxsat}). The differences alternate in sign and are
smaller than the reported seed variation; they do not establish an accuracy
benefit from the column constraint.

\begin{table}[ht]
\centering
\caption{Full-space MAX-SAT, softmax versus Sinkhorn. Final clause satisfaction (\%, mean $\pm$ sample std), 5 seeds except the $5{,}000$-variable Sinkhorn run ($n{=}2$). $\Delta$ is Sinkhorn minus softmax in percentage points. The measured differences have no consistent sign and are smaller than the reported seed variation.}
\label{tab:ot-maxsat}
\begin{tabular}{lccc}
\toprule
Variables & Softmax & Sinkhorn & $\Delta$ \\
\midrule
$100$    & $98.310 \pm 0.509$ & $98.263 \pm 0.458$ & $-0.047$ \\
$500$    & $98.183 \pm 0.288$ & $98.211 \pm 0.247$ & $+0.028$ \\
$1{,}000$ & $98.084 \pm 0.090$ & $98.000 \pm 0.069$ & $-0.084$ \\
$5{,}000$ & $98.101 \pm 0.061$ & $98.117 \pm 0.066$ & $+0.016$ \\
\bottomrule
\end{tabular}
\end{table}

Particle count alone therefore does not predict a benefit from column
coupling. Theorem~\ref{prop:plateau-freeze} and
Proposition~\ref{prop:ghost-drift} establish a different property: coupling
can move a constant-cost block whose equivariant row-wise update is zero.
That motion does not by itself imply loss decrease.

\subsection{Reinforcement-Learning Policy Search}\label{app:rl-policy-search}

Hard quantization can block derivatives while leaving episodic returns
available to a forward-only optimizer. We compare \sysname{} with
OpenAI-ES~\citep{salimans2017evolution} on CartPole-v1 and
Acrobot-v1~\citep{gymnasium2024}. Proposition~\ref{prop:rl-policy-search}
controls how finite-rollout cost error affects the assignment weights.
It does not establish sample efficiency or robustness to quantization.

Both forward-only methods use a width-$16$, single-hidden-layer MLP and
$200$ optimizer updates, with five seeds per setting. Float32 uses $\tanh$;
INT8 instead maps hidden preactivations to $\mathrm{round}(x/s)s$, with
$s=\max\{\max|x|,10^{-8}\}/127$ computed separately for each candidate
across its rollout batch; binary maps them to $\mathrm{sign}(x)$.
The INT8 scale can therefore couple rollouts within a batch. The concentration hypothesis requiring independent rollouts in Proposition~\ref{prop:rl-policy-search} is an
additional assumption, not established by this implementation. The two methods use different candidate and rollout counts per update. No validation split or checkpoint selection is used; Table~\ref{tab:rl_nondiff}
reports evaluation of the final policy. Normalized scores are
$\operatorname{clip}_{[0,1]}((R-22)/478)$ for CartPole and
$\operatorname{clip}_{[0,1]}((R+500)/420)$ for Acrobot; the curves retain raw
returns to show differences that clipping would hide.

PolyStep uses a rank-4 HybridSubspace and $d_p=8$. CartPole retains all
114 parameters, giving 15 blocks and 135 candidates per step. Acrobot
projects its first 96 weights to one coordinate and retains the remaining
67 parameters, giving 68 coordinates, nine blocks, and 81 candidates.
Both runs request a 24-coordinate cap, which cannot be met while retaining
the unprojected parameters. Their final blocks are padded. OpenAI-ES
perturbs all parameters, with population 32. Thus the Acrobot comparison
also changes the search representation.
Candidate costs are negative mean episodic returns. Candidates within an
update share rollout initial states, with fresh seeds at the next update.
PolyStep uses 32 rollouts per candidate for Float32 CartPole and 16 for
the other five settings; OpenAI-ES uses 16 throughout. The horizon is 500.
Final evaluation averages three separately seeded batches at offsets
20,000, 30,000, and 40,000 from the run seed, after the 200th update.
These final evaluations do not select a checkpoint.

PPO and DQN use Stable-Baselines3~\citep{raffin2021sb3}, with hidden
width $16$ but additional algorithm-specific networks or outputs.
Their quantized controls apply the operator to every hidden activation
and to the policy/value output, without a straight-through estimator.
This differs from the forward-only policies and blocks the trainable
parameters' gradients. These controls diagnose gradient flow; they do
not compare quantization-aware training methods. PPO and DQN use $1$M
training transitions on CartPole and $0.5$M on Acrobot.

\begin{table}[t]
\centering
\caption{RL policy search: 2 environments $\times$ 3 precision regimes (5 seeds, normalized return of the final policy, mean $\pm$ std); best gradient-free result per column in bold. Costs in Table~\ref{tab:rl_cost}; PPO and DQN run the quantization in the forward pass with no straight-through estimator, a diagnostic of gradient flow rather than of quantization-aware training.}
\label{tab:rl_nondiff}
\small
\setlength{\tabcolsep}{4pt}
\begin{tabular}{lcccccc}
\toprule
Method & \multicolumn{3}{c}{CartPole-v1} & \multicolumn{3}{c}{Acrobot-v1} \\
\cmidrule(lr){2-4} \cmidrule(lr){5-7}
 & F32 & INT8 & Binary & F32 & INT8 & Binary \\
\midrule
\sysname{} & $\mathbf{1.000 \pm 0.000}$ & $\mathbf{1.000 \pm 0.000}$ & $\mathbf{1.000 \pm 0.000}$ & $\mathbf{0.996 \pm 0.006}$ & $\mathbf{0.995 \pm 0.009}$ & $\mathbf{1.000 \pm 0.000}$ \\
OpenAI-ES & $0.645 \pm 0.364$ & $0.942 \pm 0.130$ & $0.337 \pm 0.389$ & $0.980 \pm 0.019$ & $0.957 \pm 0.044$ & $0.755 \pm 0.424$ \\
PPO & $1.000 \pm 0.000$ & $0.033 \pm 0.074$ & $0.000 \pm 0.000$ & $0.995 \pm 0.007$ & $0.186 \pm 0.416$ & $0.144 \pm 0.251$ \\
DQN & $0.283 \pm 0.345$ & $0.015 \pm 0.034$ & $0.007 \pm 0.015$ & $0.355 \pm 0.352$ & $0.000 \pm 0.000$ & $0.000 \pm 0.000$ \\
\bottomrule
\end{tabular}
\end{table}

\begin{figure}[t]
\centering
\includegraphics[width=\linewidth]{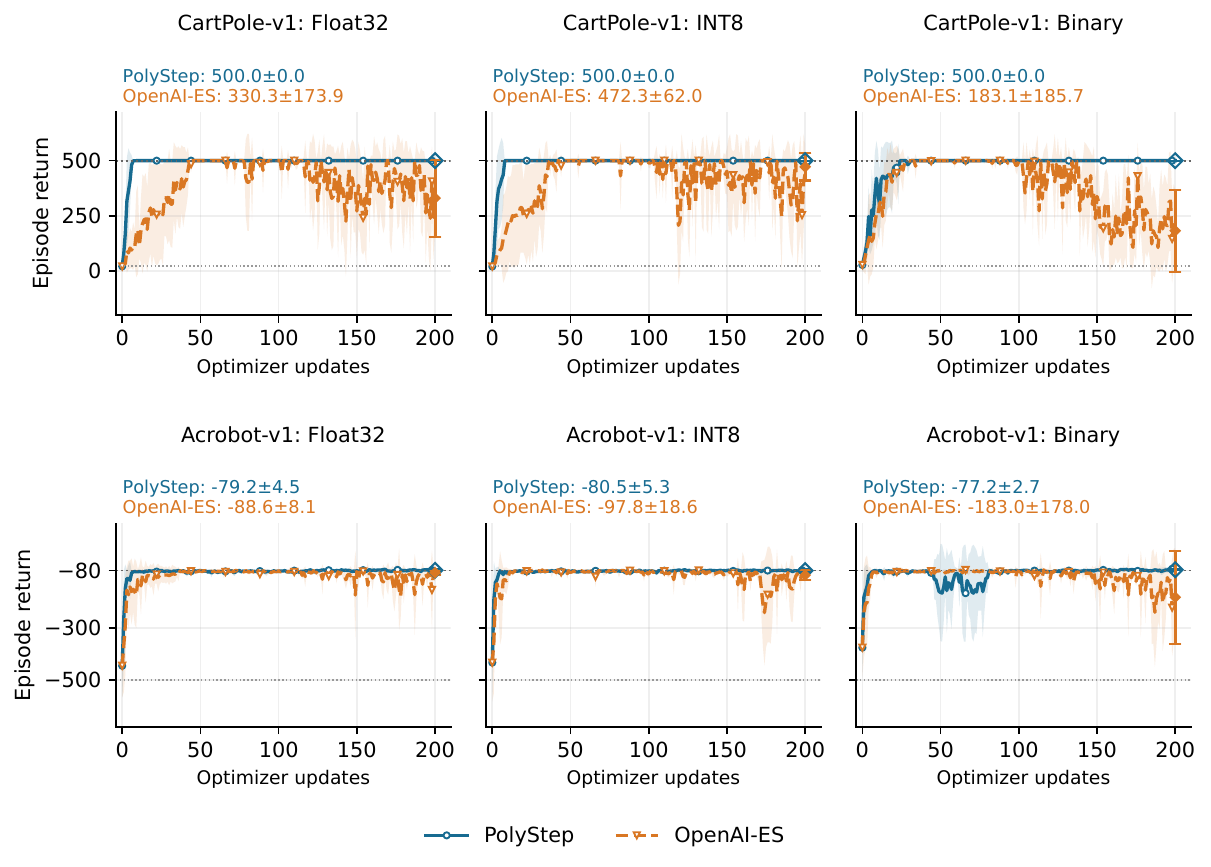}
\caption{Forward-only policy search (five seeds, mean $\pm$ sample std).
Curves show logged episode returns over $200$ optimizer updates; INT8 and
binary transform only the first hidden activation. Diamonds and printed
values give the separate final-policy evaluation, with more rollouts.
Dotted lines mark the score-zero and score-one references used in
Table~\ref{tab:rl_nondiff}: $22$ and $500$ for CartPole, $-500$ and $-80$
for Acrobot. Returns and bands are unclipped. Equal update counts do not
imply equal rollout costs (Table~\ref{tab:rl_cost}).}
\label{fig:rl-curves}
\end{figure}

\begin{figure}[t]
\centering
\includegraphics[width=\linewidth]{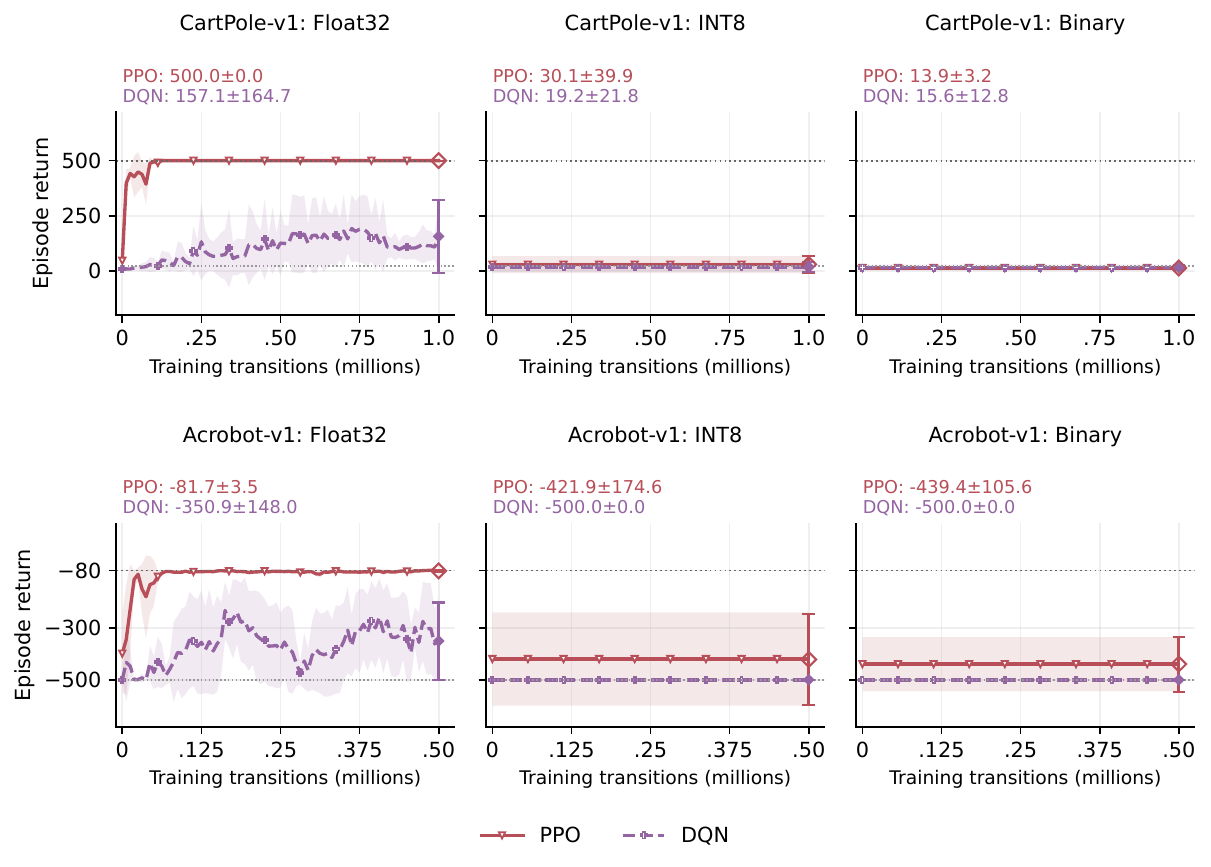}
\caption{PPO and DQN controls (five seeds, mean $\pm$ sample std).
Curves show logged episode returns against training transitions; diamonds
and printed values give the separate final-policy evaluation. Dotted lines
mark the normalization references ($22$, $500$ on CartPole; $-500$, $-80$
on Acrobot). Returns and bands are unclipped. INT8 and binary quantize every
hidden activation and the output without a straight-through estimator.
These controls use algorithm-specific networks and different quantization
placement from Figure~\ref{fig:rl-curves}; they test blocked gradients.}
\label{fig:rl-controls}
\end{figure}

\begin{table}[t]
\centering
\caption{Reported rollout counts, in millions (mean over five
seeds). \sysname{} counts realized transitions; OpenAI-ES counts
population $\times$ rollouts $\times$ horizon, including unused horizon
slots. PPO/DQN count training transitions plus the allocated final
evaluation horizon. These counters are not identical measures of
sample cost. PPO/DQN training budgets are $1$M on CartPole and $0.5$M
on Acrobot.}
\label{tab:rl_cost}
\small
\begin{tabular}{lcccccc}
\toprule
Method & \multicolumn{3}{c}{CartPole-v1} & \multicolumn{3}{c}{Acrobot-v1} \\
\cmidrule(lr){2-4} \cmidrule(lr){5-7}
 & F32 & INT8 & Binary & F32 & INT8 & Binary \\
\midrule
\sysname{} & $416.760$ & $208.440$ & $201.820$ & $29.352$ & $27.660$ & $31.564$ \\
OpenAI-ES & $51.200$ & $51.200$ & $51.200$ & $51.200$ & $51.200$ & $51.200$ \\
PPO & $1.025$ & $1.025$ & $1.025$ & $0.525$ & $0.525$ & $0.525$ \\
DQN & $1.025$ & $1.025$ & $1.025$ & $0.525$ & $0.525$ & $0.525$ \\
\bottomrule
\end{tabular}
\end{table}

The figures report trajectories, whereas the table evaluates each final
policy without choosing a checkpoint from that trajectory.
Table~\ref{tab:rl_nondiff} gives final normalized PolyStep returns
from $0.995$ to $1.000$ across the six settings. OpenAI-ES also
learns with quantized hidden activations, but its final mean is lower and
its seed variation is larger in these runs. Score clipping limits
comparisons near one. The results do not isolate whether the gap arises
from the update rule, candidate count, rollout count, or their interaction.
Table~\ref{tab:rl_cost} reports the available cost counters, whose different
accounting conventions prevent a precise sample-efficiency ratio.
Figure~\ref{fig:rl-controls} shows PPO reaching the reference return in
Float32, while DQN falls short. Their quantized trajectories remain nearly
constant, consistent with the blocked gradients. These experiments do not
assess modern quantization-aware policy-gradient methods.

The binary CartPole policy uses $\mathrm{sign}(\cdot)$ activations and is
trained without a surrogate gradient, straight-through estimator, or
critic. This demonstrates direct policy search through the stated hard
quantizer. The experiments do not establish competitive performance in
full-precision continuous control.

\section{Statistical Testing}\label{app:stats}

With 5 paired seeds, an exact two-sided sign-flip permutation test has $2^5 = 32$
assignments, so its smallest attainable $p$-value is $2/32 = 0.0625$.
The test assumes independent paired differences with distributions symmetric
about zero under the null. The per-architecture test tables report means and
sample standard deviations; unanimous differences still cannot reject at
the $5\%$ level with five seeds.

In the six architecture tables, under each table's stated budget convention,
\sysname{} leads the best tuned gradient-free baseline on all five seeds,
so each returns $p = 0.0625$, with mean paired margins: SNN
$+13.4$~pp against OpenAI-ES,
MNIST $+2.3$ against CMA-ES (the rounded table entries differ by $2.4$), INT8 $+2.6$
against OpenAI-ES, staircase $+5.1$ against
CMA-ES, argmax $+6.6$ against OpenAI-ES, hard MoE $+8.1$ against OpenAI-ES. These tests do not reject at level $0.05$. With five paired seeds, this
resolution limit applies even when every observed difference has the same
sign. No multiplicity correction is applied, so the comparisons are
descriptive and do not establish family-wise significance.

SST-2 from-scratch results are single-seed (seed 42); \sysname{} sits at chance
($\sim$50\% on a binary task) in that run; variation across seeds is unknown.
This experiment supports no conclusion about gradient-free methods in
general at that scale.
The GPT-2 head-only fine-tuning experiment (Appendix~\ref{app:gpt2-finetune})
uses 3 seeds.

\paragraph{Checkpoint selection.}
The primary supervised test tables report the checkpoint selected on a held-out
validation split for every method. Hyperparameters are selected on validation
under the corresponding tuning budget in Section~\ref{sec:setup}; the ten-seed
weighting comparison uses the grid in Section~\ref{sec:controlled-comparison}.
Each selected configuration is evaluated on test once per final seed.

The ablations in Section~\ref{sec:ablation} and
Appendix~\ref{app:ot-highparticle} instead report maximum test accuracy
over epochs. This selection uses the test set and can favor methods with
more variable trajectories or more evaluation opportunities. It is kept
separate from the validation-selected comparisons. The schedule ablation
reports maximum and last-three-epoch accuracies to distinguish peak
performance from late-run retention.

\section{Sinkhorn Solver Details}\label{app:sinkhorn-details}

The softmax solver used in all primary experiments is the $\lambda = 0$
endpoint of a solver continuum:
softmax $\to$ KL-softmax (Theorem~\ref{thm:kl-softmax-endpoints}, solved by
\texttt{KLSoftmaxSolver}) $\to$ full entropic OT
(Eq.~\ref{eq:ot-objective}, solved by \texttt{SinkhornSolver}).
The full Sinkhorn solver below is the $\lambda \to \infty$ endpoint used
in the solver comparison. Whether the target constraint changes an update
depends on the cost matrix and temperature, not on particle count alone.

\begin{remark}[Generalization to entropic optimal transport]\label{rem:ot-generalization}
Algorithm~\ref{alg:polystep-step} specifies the core update. The primary runs
use the schedules, rotations, and optional direction reuse reported for each benchmark. For the
generalization, replace the softmax line with the entropic-OT solve
$\mT^* \leftarrow \mathrm{Sinkhorn}(\mC, \varepsilon, \varpi, f_{\mathrm{prev}}, g_{\mathrm{prev}})$
of Eq.~\ref{eq:ot-objective} and set $p_{iv}=T^*_{iv}/a_i$, at overrelaxation $\varpi \in [0.5, 1.95]$ and with the
dual potentials warm-started from the previous iteration. That path adds the column
constraint $\mT^\top\bm{1} = \vb$, measured close to satisfied in the tested subspace configurations
(Section~\ref{sec:ablation}).
\end{remark}

\paragraph{Log-domain Sinkhorn iterations.}
The entropic OT problem (Eq.~\ref{eq:ot-objective}) is solved via
alternating dual potential updates with overrelaxation parameter
$\varpi \in [0.5, 1.95]$:
\begin{align}
f_i &\leftarrow (1 - \varpi) f_i + \varpi \cdot \varepsilon \!\left(\log a_i - \mathrm{logsumexp}_{v}\!\left(\frac{g_v - C_{iv}}{\varepsilon}\right)\right), \label{eq:sinkhorn-f} \\
g_v &\leftarrow (1 - \varpi) g_v + \varpi \cdot \varepsilon \!\left(\log b_v - \mathrm{logsumexp}_{i}\!\left(\frac{f_i^{\mathrm{new}} - C_{iv}}{\varepsilon}\right)\right). \label{eq:sinkhorn-g}
\end{align}

\paragraph{Warm-started dual potentials.}
The dual potentials $(f,g)$ from the previous solve initialize the next
solve. Across a discontinuity, an arbitrarily small parameter displacement
can change the cost matrix by the jump size, so warm starts provide no
fixed iteration-count reduction.
The stopping tolerance still determines when a solve is accepted.
Each full row-and-column sweep costs $O(PV)$ operations; the number of
sweeps depends on the tolerance and conditioning of the transport problem.

\section{System Design}\label{app:system}

The method needs only a scalar loss oracle, so it is not tied to any one array framework.
Our reference implementation uses PyTorch and wraps most standard \texttt{nn.Module}
architectures for gradient-free training; the description below is specific to it.

\paragraph{Blockwise OT.}
For large models in full-space mode, the blockwise variant solves a separate
entropic OT problem for each layer's particles, with its own column marginal.
A plan has $P\times V$ entries. Processing layer plans sequentially reduces
plan storage from $O(PV)$ to $O(\max_l P_lV)$, where $\sum_lP_l=P$;
for equally sized layers this is $O(PV/L)$. Retaining all layer plans still
costs $O(PV)$. Separate solves remove cross-layer column coupling and hence
change the globally coupled problem when $\lambda>0$.

\paragraph{Sparse projection.}
HybridSubspace uses the per-layer $10^9$-byte threshold of
Appendix~\ref{app:search-implementation}. Its sparse signed bases reduce storage but do not
satisfy the theory's orthonormality assumption. The separate global adaptive
projection supports an optional automatic mode that selects sparse storage
above 2M parameters on GPU or 1M on CPU. These thresholds describe distinct
implementations. Sparse projection storage follows the construction of
\citet{li2006sparse}, with the sampling convention stated in
Appendix~\ref{app:search-implementation}.
For $1$M parameters at rank $256$, dense float32 storage is
$976.6$\,MiB. At density $0.001$, sparse COO storage is $4.9$\,MiB,
including both int64 indices and float32 values; cached copies and
working buffers add to this estimate.
The primary neural-network runs use the per-layer implementation.

\paragraph{Forward-only evaluation.}
\sysname{} treats the model as a black-box function $\theta \mapsto \Ls$: every probe evaluation is wrapped in \texttt{@torch.no\_grad()}, so no backward graph or activation tape is stored. The current forward
activations, recurrent state, and candidate buffers still require memory.
Ordinary BPTT retains intermediate activations across timesteps;
forward-only evaluation avoids that graph. The SNNVGG11Small measurements
in Figure~\ref{fig:scalability-plate}b concern evaluation memory and do
not include all optimizer buffers or establish asymptotic scaling
(Section~\ref{sec:snn-memory}).

\paragraph{Layer compatibility.}
Standard layers (\texttt{nn.Linear}, \texttt{nn.Conv2d}, activations) work under \texttt{vmap} without modification.
Attention and recurrent layers use explicit matrix operations through
vectorization-compatible replacements for their specialized kernels.

\paragraph{Usage.}
\sysname{} is not a \texttt{torch.optim.Optimizer} subclass, so no \texttt{zero\_grad()} or parameter groups are needed.
The optimizer is initialized with \texttt{optimizer = PolyStepOptimizer(model, epsilon=1.0)}.
Mixed precision is optional and disabled in the controlled weighting comparison.
Projection storage follows the selected subspace implementation.

\paragraph{torch.compile modes.}
When enabled, hot-path compilation uses \texttt{torch.compile} in
\texttt{"default"} mode, with an eager fallback on compilation failure.
The fixed-iteration Sinkhorn path compiles individual dual updates.
The measured eager/compiled MNIST runtimes are reported in
Appendix~\ref{app:scalability}.

\paragraph{vmap batching strategy.}
Batched forward evaluation uses \texttt{torch.vmap} with \texttt{functional\_call}
to evaluate $n$ parameter candidates in a single vectorized call. Parameters are stacked
into a leading batch dimension; \texttt{functional\_call} reconstructs the
\texttt{state\_dict} per candidate. This avoids a Python-level loop over candidates,
providing GPU parallelism over the particle dimension.

\paragraph{Memory management.}
Three mechanisms bound peak memory: (1) \texttt{chunk\_size} batches vmap evaluation
into chunks, bounding the intermediate activation tensor size; (2) mixed precision
mode (\texttt{mixed\_precision=True}) stores model parameters in BF16 while keeping
the Sinkhorn solver in FP32; (3) blockwise OT (one Sinkhorn solve per layer)
processes per-layer cost matrices separately and removes cross-layer
column coupling. Its effect on accuracy depends on the loss and configuration.

\paragraph{SNN output scaling.}
The MNIST hard-LIF model returns integer spike counts over 15 time steps,
in $\{0,\ldots,15\}$. Its training loss applies cross-entropy directly to
these counts, without division by time or a separate factor of 10.
The count convention determines the cost scale seen by the temperature.

\paragraph{SST-2 multi-argument vmap.}
The TransformerClassifier requires \texttt{input\_ids} and \texttt{attention\_mask}
as separate positional arguments. The vmap wrapper passes these as a positional
argument tuple to \texttt{functional\_call}, matching the model's forward signature.

\section{Reproducibility Details}\label{app:reproducibility}

Hyperparameters are given in Appendix~\ref{app:hyperparameters}; the controlled
weighting grids and selected settings are in Section~\ref{sec:controlled-comparison}.

\paragraph{Random seeds.}
The primary five-seed comparisons use $\{42,123,456,789,1337\}$.
Theory-mode MAX-SAT uses $\{42,123\}$; evaluation-matched argmax and hard MoE
use $\{42,123,456\}$. The configuration ladder and from-scratch SST-2 run
use seed $42$. Ablation captions state their own counts.
GPT-2 head-only fine-tuning (Appendix~\ref{app:gpt2-finetune}) uses seeds $\{42, 123, 456\}$ (3 seeds); the quantized-head grid (Appendix~\ref{app:gpt2-headquant}) uses all five.
The controlled weighting comparison uses seeds 2027--2029 for tuning and
1000--1009 for final runs, with one fixed split and separately seeded streams.
Seeds are set via \texttt{torch.manual\_seed}, \texttt{numpy.random.seed}, and Python's
\texttt{random.seed} at the start of each experiment.

\paragraph{Determinism.}
The supplied reproducibility configuration sets
\texttt{CUBLAS\_WORKSPACE\_CONFIG=:4096:8}, enables deterministic PyTorch
algorithms, disables TF32, and seeds the data-loader generator and workers.
The controlled runners set the same cuBLAS workspace option, disable
TF32, and use separate seeded generators, but do not enable PyTorch's
strict deterministic-algorithm mode.
Result records specify the environment and seed. These controls reduce
variation; they do not imply bitwise identity across software versions,
kernels, or devices. Hard thresholds can amplify small floating-point
differences. Reproduction should compare both the selected metric and the
reported variation across seeds.

\paragraph{Hardware.}
GPU experiments use one NVIDIA GeForce RTX 5090. The specialized SAT
references and the size-specific 100-/500-variable PolyStep runs use CPU.
No experiment uses multiple GPUs.

\paragraph{Software.}
Software versions vary across experiment batches, including PyTorch 2.11.0
and 2.13.0. Primary records use an \texttt{environment} field; the controlled
records store \texttt{torch\_version}, \texttt{cuda}, and \texttt{device}
directly, together with runner and evaluator hashes.

\paragraph{Rank ties.}
The architecture-table OpenAI-ES runs used sorting-order ranks: for utility
inputs $x=-L$, the implementation was
\texttt{x.argsort().argsort()/(n-1)-0.5}. The released shared helper uses average ranks for ties. Reproducing the reported
OpenAI-ES trajectories requires the sorting-order convention. The rank-weighted polytope control in
Section~\ref{sec:controlled-comparison} uses average tied ranks throughout.

\paragraph{Code availability.}
\sysname{} is released as an open-source package at
\url{https://github.com/anindex/polystep}.

\paragraph{Ablation sweep.}
Beyond the primary experiments we provide an ablation sweep of $253$ runs over $66$ configurations
at $90.9$ GPU-hours on a single RTX~5090, covering all reported
ablation groups (solver comparison, subspace rank, low-rank and blockwise
OT, amortization, KL interpolation, marginal-violation rate, and schedule
fragility) with the recorded configurations and outcomes.
The supplement contains the configurations and result records under
\texttt{experiments/results/softmax/ablations/}.

\paragraph{Dataset access.}
MNIST: via \texttt{torchvision.datasets} (auto-downloaded).
SST-2: via HuggingFace \texttt{datasets} library, GLUE benchmark split.

\paragraph{Runtime records.}
Per-run records contain measured wall times. Appendix~\ref{app:wallclock-matched}
and Figure~\ref{fig:maxsat-scaling-full} report times with their evaluation
budgets. These depend on batching, compilation, and concurrent device use;
they are not hardware-independent estimates.

\paragraph{Statistical tests.}
Pairwise tests use exact two-sided sign flips, with minimum attainable
$p=0.0625$ for five seeds and $p=2/1024$ for ten seeds;
Appendix~\ref{app:stats} gives the per-benchmark results and the test assumptions.

\section{GPT-2 Fine-Tuning}\label{app:gpt2-finetune}

\subsection{The Flat-Pair Fraction Under a Quantizer}
\label{app:flat-pair}

A probe pair confined to one constant-loss cell gives a zero finite
difference. We measure how often this occurs at several finite radii.

On the quantized GPT-2 head we record the \emph{flat-pair fraction}: the proportion of
antithetic probe pairs $(\theta + \mu z,\, \theta - \mu z)$ whose two evaluations are
bit-identical, so that the estimated directional derivative is exactly $0$. The pairs
are MeZO's own, so $z$ is a random direction over all $1{,}538$ coordinates at once and
$\mu$ is its finite-difference radius, not this paper's temperature $\varepsilon$; a
pair can be flat when it crosses no effective boundary, or when crossings
leave the evaluated loss unchanged. Each entry is one seed-42 run at the
stated $\mu$, averaged over the run's generations.

\begin{center}
\small
\begin{tabular}{lccc}
\toprule
Probe radius $\mu$ & INT8 head & Binary head & Smooth control \\
\midrule
$10^{-5}$ & $0.995$ & $0.567$ & $0.016$ \\
$10^{-3}$ & $0.0004$ & $0.0027$ & $0.0004$ \\
$2 \times 10^{-2}$ & $0.000$ & $0.0002$ & $0.000$ \\
\bottomrule
\end{tabular}
\end{center}

At $\mu=10^{-5}$, the INT8 head has identical losses on $99.5\%$ of
pairs, compared with $1.6\%$ for the smooth control. At $\mu=10^{-3}$,
all three fractions are below $0.3\%$. Larger probes therefore recover
cost differences in this diagnostic. A nonzero difference need not give a
useful update, but it can estimate variation of a smoothed objective even
when the original loss is discontinuous. The selected binary-head MeZO
configuration performs poorly in Table~\ref{tab:gpt2-headquant}; the
flat-pair statistic alone does not explain that result.

The transition depends on the quantization grid, direction distribution,
and distance of the current parameters to nearby boundaries. It is not
fixed by a uniform isolation radius for the whole network.

\subsection{Quantized Head at a Matched Budget}\label{app:gpt2-headquant}

For Table~\ref{tab:gpt2-headquant}, the GPT-2 124M backbone is frozen and its
final-token features are cached once. A fixed permutation gives a
$4{,}500/500/872$ train/validation/test split. The implementation blocks
test access during training and tuning. Every method trains the same
1,538-parameter head on the same feature stream.

We compare three head variants. The INT8 head quantizes each weight and bias
using grid spacing $\Delta_q=0.01$:
\[
\mathrm{clamp}(\mathrm{round}(w/\Delta_q),-128,127)\,\Delta_q.
\] The binary head applies $\mathrm{sign}(\cdot)$ to the
weights and leaves the biases continuous. Its 2 continuous parameters
make Adam a partial control. The third variant is an unquantized smooth head.

Each gradient-free method receives the same nine-configuration validation
sweep at one fifth of the final budget. The selected configuration then
runs with a $500$K-candidate-evaluation budget over five seeds, and its
validation-selected head is scored once on test. The result files include
the selected settings and sweep records. At this budget, MeZO takes
$250{,}000$ antithetic-pair steps, and EGGROLL uses $16$--$32$ candidates per
generation. PolyStep's full-space step costs $1.5$ evaluations per parameter,
allowing $216$ complete steps and $498{,}312$ evaluations. It starts a new
step only if the remaining budget covers it. Validation selection considers
both its mean iterate and its best sampled candidate.

On INT8, MeZO's sweep selects a probe radius of $\mu=1.0$. On the binary
head, the selected MeZO configuration remains near chance. These are
outcomes of the tested sweep; they do not imply that no radius can produce
informative differences.

\subsection{Head-Only Fine-Tuning}

We load pretrained GPT-2 124M weights~\citep{radford2019language} into a custom
VmapSafe model (Conv1D weight transpose, fused QKV split; forward pass matches
HuggingFace within $10^{-4}$) and fine-tune on SST-2 binary sentiment
classification~\citep{socher2013recursive}, either through a 128-dimensional
sparse random projection of all 124M parameters (100 optimizer steps) or in
full-space mode on the 1,538-parameter classification head (50 epochs at batch
size 8, about $31{,}000$ steps). Three seeds, maximum sequence length 128,
5,000 training samples; the Adam baselines run 3 epochs on the same data.

This head has the same architecture as the smooth control in
Appendix~\ref{app:gpt2-headquant}, but the budgets and selection rules differ.
The $76.8\%$ result below is the maximum test accuracy during a run of about
$31{,}000$ steps. The $51.2\%$ result there uses the validation-selected
checkpoint after $216$ steps under a $500$K-evaluation cap. Comparing these
two values would mix the effects of training budget and test-set selection.

\begin{table}[p]
\centering
\caption{GPT-2 124M fine-tuning on SST-2 (3 seeds). Accuracy (\%), trainable parameters, peak VRAM, and wall time. Every method reports maximum test accuracy over its run, with unequal numbers of checkpoint evaluations (51 for head-only \sysname{} against 4 for head-only Adam). Unequal selection opportunities can bias the comparison. The size and direction of a validation-selected difference are not established by these runs.}
\label{tab:gpt2-finetune}
\begin{tabular}{lcccc}
\toprule
Method & Accuracy (\%) & Trainable & Peak VRAM & Wall Time \\
\midrule
\sysname{} (head only, full-space) & $76.8 \pm 0.7$ & 1,538 & 1,075 MB & 8 min \\
\sysname{} (all parameters, rank 128) & $49.7 \pm 1.1$ & 124M & 9,043 MB & 10 min \\
From scratch (Sec.~\ref{sec:limitations-disc}) & 49.4 & 4.2M & -- & -- \\
\midrule
Adam (lr=1e-3, head only) & $81.3 \pm 0.6$ & 1,538 & 1,075 MB & 38 s \\
Adam (lr=2e-5, all parameters) & $86.4 \pm 0.8$ & 124M & 2,419 MB & 44 s \\
\bottomrule
\end{tabular}
\end{table}

\paragraph{Results.}
The all-parameter configuration with a $128$-dimensional projection remains
near chance, while the head-only run reaches $76.8\%$ under its maximum-test
protocol. The experiment does not isolate compression as the cause of the
failure or establish a validation-selected comparison with Adam.
Table~\ref{tab:gpt2-memory} reports memory for the measured implementations.
PolyStep stores no parameter gradients or Adam moments, but it retains
subspace state and projection data; batched forward activations dominate
the measured peak.

\begin{table}[p]
\centering
\caption{GPT-2 124M training memory breakdown on RTX 5090 (32 GB). Measured peak and component-level estimates for Adam (FP32) and \sysname{} (FP32, sparse projection).}
\label{tab:gpt2-memory}
\begin{tabular}{lcc}
\toprule
Component & Adam (FP32) & \sysname{} (FP32, sparse) \\
\midrule
Model weights (124M) & 472 MB & 472 MB \\
Gradients & 472 MB & 0 \\
Optimizer states ($m$+$v$) & 944 MB & 0 \\
Sparse projection matrix & 0 & $\sim$11 MB \\
Subspace state & 0 & $\sim$10 MB \\
Forward activations & $\sim$500 MB & $\sim$8{,}550 MB\textsuperscript{$\dagger$} \\
Backward activations & included above & 0 \\
\midrule
Measured peak & 2{,}419 MB & 9{,}043 MB \\
\bottomrule
\end{tabular}
\vspace{4pt}

\raggedright\footnotesize
$\dagger$\,\texttt{torch.vmap} materializes activations for chunk\_size$=$4 candidates simultaneously across 12 transformer layers with 768-dim hidden states and 128-token sequences, dominating memory.
\end{table}

\citet[Table~1]{malladi2023mezo} report $91.4\%$ SST-2 accuracy for
full-parameter MeZO on OPT-13B with $1{,}000$ training examples. That
experiment uses a different task protocol and a differentiable backbone
about $105\times$ larger than our GPT-2 model, so it is not a controlled
comparison with this experiment. In our six-architecture step-based
comparison, tuned MeZO attains lower accuracy than PolyStep
(Table~\ref{tab:step-sweep}). Whether adaptive or learned subspaces improve
our GPT-2 result at the same budget remains open.

\end{document}